\documentclass[11pt]{now_arxiv}
\usepackage{subfig}
\usepackage{subfloat}

\usepackage{color}
\usepackage{epsf}
\usepackage{makeidx}
\usepackage{ifthen}
\usepackage{eucal}
\usepackage{url}
\usepackage{graphicx}
\usepackage{psfrag}

\usepackage[fleqn]{amsmath}
\usepackage{amssymb}
\usepackage{amsbsy}

\usepackage[chapter]{algorithm}
\usepackage{algorithmicx}

\usepackage[ruled,vlined]{algorithm2e} 
\usepackage{bm}

\usepackage{pst-plot}

\def\defin{:=}
\def\st{~\text{s.t.}~}
\def\mxJ{\mx_{\!J}^{}}

\def\pll{{\scriptscriptstyle /\!\!/}}
\def\tr{{\rm tr}}
\def\zbullet{$ $}
\def\zspace{\vspace{1mm}}

\newcommand{\bmt}[1]{\bm{#1}^\top} 

\def\hati{{\hat{\imath}}}
\newcommand{\va}{\bm{a}}

\newcommand{\vd}{\bm{d}}               
\newcommand{\ve}{\bm{e}}

\newcommand{\vn}{\bm{n}}

\newcommand{\vr}{\bm{r}}               
\newcommand{\vs}{\bm{s}}               
\newcommand{\vt}{\bm{t}}               
\newcommand{\vu}{\bm{u}}               
\newcommand{\vv}{\bm{v}}               
\newcommand{\vw}{\bm{w}}               
\newcommand{\vx}{\bm{x}}               
\newcommand{\vy}{\bm{y}}               \newcommand{\yh}{\hat{y}}
\newcommand{\vz}{\bm{z}}               

\newcommand{\valpha}  {\bm{\alpha}}         
\newcommand{\vbeta}   {\bm{\beta}}           
         
\newcommand{\vepsilon}{\bm{\varepsilon}}

\newcommand{\veta}    {\bm{\eta}}             
         
\newcommand{\vkappa}  {\bm{\kappa}}

\newcommand{\md}{\bm{D}}

\newcommand{\mH}{\bm{H}}          
\newcommand{\mi}{\bm{I}}          
          
\newcommand{\mk}{\bm{K}}          
          
\newcommand{\mm}{\bm{M}}          
\newcommand{\mn}{\bm{N}}

\newcommand{\mU}{\bm{U}}          
\newcommand{\mv}{\bm{V}}          
\newcommand{\mw}{\bm{W}}          
\newcommand{\mx}{\bm{X}}          
\newcommand{\my}{\bm{Y}}

    \newcommand{\Bc}{\mathcal{B}}

    \newcommand{\Gc}{\mathcal{G}}  
    \newcommand{\Hc}{\mathcal{H}}

    \newcommand{\Kc}{\mathcal{K}}

    \newcommand{\Nc}{\mathcal{N}}  
      
    \newcommand{\Pc}{\mathcal{P}}

    \newcommand{\Tc}{\mathcal{T}}

    \newcommand{\Xc}{\mathcal{X}}  
    \newcommand{\Yc}{\mathcal{Y}}

\newcommand{\mynorm}[2]{\| {#1} \|_{#2}}

\newcommand{\frob}[1]{\|{#1}\|_{\text{F}}}

\newcommand{\R}{\mathbb{R}}

\DeclareMathOperator*{\argmin}{argmin}

\DeclareMathOperator{\sgn}{sgn}
\DeclareMathOperator{\trace}{Tr}
\DeclareMathOperator{\Diag}{Diag}

\newcommand{\lem}[1]{Lemma~\protect\ref{#1}}

\newcommand{\eq}[1]{(\protect\ref{#1})}

\begin{document}

\isbn{xxxxxxxxxxx}

\DOI{xxxxxx}


\abstract{Sparse estimation methods are aimed at using or obtaining parsimonious representations of data or models. They were first dedicated to linear variable selection but numerous extensions have now emerged such as structured sparsity or kernel selection. It turns out that many of the related estimation problems can be cast as convex optimization problems by regularizing the empirical risk with appropriate non-smooth norms. The goal of this paper is to present from a general perspective optimization tools and techniques dedicated to such sparsity-inducing penalties. We cover proximal methods, block-coordinate descent, reweighted $\ell_2$-penalized techniques, working-set and homotopy methods, as well as non-convex formulations and extensions, and provide an extensive set of experiments to compare various algorithms from a computational point of view.
 
}

\articletitle{Optimization with Sparsity-Inducing Penalties}

\authorname1{Francis Bach}
\affiliation1{INRIA - SIERRA Project-Team}
\author1address2ndline{INRIA - SIERRA Project-Team, Laboratoire d'Informatique de l'Ecole Normale Sup\'erieure, 23, avenue d'Italie,~}
\author1city{Paris,~}
\author1zip{75013}
\author1country{France}
\author1email{francis.bach@inria.fr}

\authorname2{Rodolphe Jenatton}
\affiliation2{INRIA - SIERRA Project-Team}
\author2address2ndline{INRIA - SIERRA Project-Team}
\author2email{rodolphe.jenatton@inria.fr}

\authorname3{\\ Julien Mairal}
\affiliation3{Department of Statistics, University of California}
\author3address2ndline{Department of Statistics, University of California,~}
\author3city{Berkeley,~}
\author3zip{CA 94720-1776}
\author3country{USA}
\author3email{julien@stat.berkeley.edu}

\authorname4{Guillaume Obozinski}
\affiliation4{INRIA - SIERRA Project-Team}
\author4address2ndline{INRIA - SIERRA Project-Team}
\author4email{guillaume.obozinski@inria.fr}

\relax \input sample.jtl 
\volume{xx}
\issue{xx}
\copyrightowner{xxxxxxxxx}
\pubyear{xxxx}

\maketitle

\cleardoublepage \pagenumbering{roman}

\tableofcontents

\clearpage

\setcounter{page}{0}
\pagenumbering{arabic}

\newtheorem{theorem}{Theorem}[chapter]
\newtheorem{definition}{Definition}[chapter]
\newtheorem{proposition}{Proposition}[chapter]
\newtheorem{lemma}{Lemma}[chapter]
\newtheorem{remark}{Remark}[chapter]

\chapter{Introduction}
The principle of parsimony is central to many areas of science: the simplest explanation of a given phenomenon should be preferred over more complicated ones. In the context of machine learning, it takes the form of variable or feature selection, and it is
commonly used in two situations. First, to make the model or the prediction more interpretable
or computationally cheaper to use, i.e., even if the underlying problem is not sparse, one looks for the best sparse
approximation. Second, sparsity can also be used given prior knowledge that the model should be
sparse.

For variable selection in linear models, parsimony may be directly achieved by penalization of the empirical risk or the log-likelihood by the cardinality of the support\footnote{We call the set of non-zeros entries of a vector the support.} of the weight vector. However, this leads to hard combinatorial problems (see, e.g., \cite{natarajan,tropp}). A traditional convex approximation of the problem is to replace the cardinality of the support  by the $\ell_1$-norm.
Estimators may then be obtained as solutions of convex programs.

Casting sparse estimation as convex optimization problems has two main benefits: First, it leads to efficient estimation algorithms---and this paper focuses primarily on these. Second, it allows a fruitful theoretical analysis answering fundamental questions related to estimation consistency, prediction efficiency~\cite{bickel_lasso_dantzig,negahban2009unified} or model consistency~\cite{martin,Zhaoyu}. In particular, when the sparse model is assumed to be well-specified, regularization by the $\ell_1$-norm is adapted to high-dimensional
problems, where the number of variables to learn from may be exponential in the number of observations.

Reducing parsimony to finding the model of lowest cardinality turns
out to be limiting, and \emph{structured parsimony}~\cite{baraniuk,Jacob2009,huang,jenatton} has emerged as a natural extension, with
applications to computer vision~\cite{cehver,huang,jenatton2010sspca}, text processing~\cite{Jenatton2010a}, bioinformatics~\cite{Jacob2009,kim3} or audio processing~\cite{augustin}.  Structured sparsity may be achieved by penalizing other functions than the cardinality of the support or regularizing by other norms than the $\ell_1$-norm. In this paper, we focus primarily on norms which can be written as linear combinations of norms on subsets of variables, but we also consider traditional extensions such as multiple kernel learning and spectral norms on matrices (see Sections~\ref{sec:norms} and~\ref{sec:mkl}). One main objective of this paper is to present methods which are adapted to most sparsity-inducing norms with loss functions potentially beyond least-squares.

Finally, similar tools are used in other communities such as signal processing. While the objectives and the problem set-ups are different, the resulting convex optimization problems are often similar, and most of the techniques reviewed in this paper also apply to sparse estimation problems in signal processing. Moreover, we consider in Section~\ref{sec:nonconvex} non-convex formulations and extensions.

This paper aims at providing a general overview of the main optimization techniques that have emerged as most relevant and efficient for methods of variable selection based on sparsity-inducing norms. We survey and compare several algorithmic approaches as they apply to the $\ell_1$-norm, group norms, but also to norms inducing structured sparsity and to general multiple kernel learning problems. We complement these by a presentation of some greedy and non-convex methods. Our presentation is essentially based on existing literature, but the process of constructing a general framework leads naturally to new results, connections and points of view. 

This monograph is organized as follows:\zspace

\zbullet Sections \ref{sec:notation} and \ref{sec:losses} introduce respectively the notations used throughout the manuscript and the optimization problem \eq{eq:formulation} which is central to the learning framework that we will consider.\zspace

\zbullet Section \ref{sec:norms} gives an overview of common sparsity and structured sparsity-inducing norms, with some of their properties and examples of structures which they can encode.\zspace

\zbullet Section~\ref{sec:opt_tools} provides an essentially self-contained presentation of
concepts and tools from convex analysis that will be needed in the rest of the manuscript,
and which are relevant to understand algorithms for solving the main optimization problem \eq{eq:formulation}.
Specifically, since sparsity inducing norms are non-differentiable convex functions\footnote{Throughout this paper, we refer to sparsity-inducing norms such as the $\ell_1$-norm as nonsmooth norms; note that all norms are non-differentiable at zero, but some norms have more non-differentiability points (see more details in Section~\ref{sec:norms}).}, we introduce relevant elements of subgradient theory and Fenchel duality---which are particularly well suited to formulate the optimality conditions associated to learning problems regularized with these norms. We also introduce a general quadratic variational formulation for a certain class of norms in Section \ref{sec:subqua}; the part on subquadratic norms is essentially relevant in view of sections on structured multiple kernel learning and can safely be skipped in a first reading.\zspace

\zbullet Section~\ref{sec:mkl} introduces \emph{multiple kernel learning} (MKL) and shows that it can be interpreted as an extension of plain sparsity to reproducing kernel Hilbert spaces (RKHS), but formulated in the dual. This connection is further exploited in Section~\ref{sec:struct_mkl}, where its is shown how structured counterparts of MKL can be associated with structured sparsity-inducing norms. These sections rely on Section \ref{sec:subqua}. All sections on MKL can be skipped in a first reading.\zspace

\zbullet In Section~\ref{sec:opt_methods_classical}, we discuss classical approaches to solving the optimization problem arising from simple sparsity-inducing norms, such as interior point methods and subgradient descent, and point at their shortcomings in the context of machine learning.\zspace

\zbullet  Section~\ref{sec:proximal_methods} is devoted to a simple presentation of proximal methods.
After two short sections introducing the main concepts and algorithms, the longer Section~\ref{sec:proximal_operator} focusses on the \emph{proximal operator} and presents algorithms to compute it for a variety of norms. Section~\ref{sec:prox_mkl} shows how proximal methods for structured norms extend naturally to the RKHS/MKL setting.\zspace

\zbullet Section~\ref{sec:opt_methods_bcd} presents block coordinate descent algorithms, which provide an efficient alternative to proximal method for \emph{separable} norms like the $\ell_1$- and $\ell_1/\ell_2$-norms, and can be applied to MKL. This section uses the concept of proximal operator introduced in Section~\ref{sec:proximal_methods}.\zspace

\zbullet Section~\ref{sec:reweighted_l2} presents reweighted-$\ell_2$ algorithms that are based on the quadratic variational formulations introduced in Section~\ref{sec:subqua}. These algorithms are particularly relevant for the least-squares loss, for which they take the form of iterative reweighted least-squares algorithms (IRLS). Section~\ref{sec:gen_varia} presents a generally applicable quadratic variational formulation for general norms that extends the variational formulation of Section~\ref{sec:subqua}.\zspace

\zbullet Section~\ref{sec:workinghomo} covers algorithmic schemes that take advantage computationally of the sparsity of the solution by extending the support of the solution gradually. These schemes are particularly relevant to construct approximate or exact regularization paths of solutions for a range of values of the regularization parameter.
Specifically, Section~\ref{sec:active_sets} presents working-set techniques, which are meta-algorithms that can be used with the optimization schemes presented in all the previous chapters. Section~\ref{sec:lars} focuses on the homotopy algorithm, which can efficiently construct the entire regularization path of the Lasso.\zspace

\zbullet Section~\ref{sec:nonconvex} presents nonconvex as well as Bayesian approaches that provide alternatives to, or extensions of the convex methods that were presented in the previous sections.
More precisely, Section~\ref{sec:opt_greedy_algorithms} presents so-called greedy algorithms, that aim at solving the cardinality constrained problem and include matching pursuit, orthogonal matching pursuit and forward selection; Section~\ref{sec:DC} presents continuous optimization problems, in which the penalty is chosen to be closer to the so-called $\ell_0$-penalty (i.e., a penalization of the cardinality of the model regardless of the amplitude of the coefficients) at the expense of losing convexity, and corresponding optimization schemes.
Section~\ref{sec:DL} discusses the application of sparse norms regularization to the problem of matrix factorization, which is intrinsically nonconvex, but for which the algorithms presented in the rest of this monograph are relevant. Finally, we discuss briefly in Section~\ref{sec:opt_bayesian} Bayesian approaches to sparsity and the relations to sparsity-inducing norms.\zspace

\zbullet Section~\ref{sec:exp_intro} presents experiments comparing the performance of the algorithms
presented in Sections~\ref{sec:opt_methods_classical},~\ref{sec:proximal_methods},~\ref{sec:opt_methods_bcd},~\ref{sec:reweighted_l2}, in terms of speed of convergence of the algorithms. Precisely, Section~\ref{sec:exp_speed} is devoted to the $\ell_1$-regularization case, and Section~\ref{sec:exp_multitask} and~\ref{sec:exp_struct} are respectively covering the $\ell_1/\ell_p$-norms with disjoint groups and to more general structured cases.\zspace

\zbullet We discuss briefly methods and cases which were not covered in the rest of the monograph in Section~\ref{sec:extensions} and we conclude in Section~\ref{sec:conclusions}.\zspace

Some of the material from this paper is taken from an earlier book chapter~\cite{bookchapter} and the dissertations of Rodolphe Jenatton~\cite{jenatton_thesis} and Julien Mairal~\cite{mairal_thesis}.

\section{Notation}\label{sec:notation}
Vectors are denoted by bold lower case letters and matrices by upper case ones.
We define for $q \geq 1$ the \mbox{$\ell_q$-norm} of a vector~$\vx$ in~$\R^n$ as
$\|\vx\|_q \defin (\sum_{i=1}^n |\vx_i|^q)^{{1}/{q}}$, where~$\vx_i$ denotes the
$i$-th coordinate of~$\vx$, and $\|\vx\|_\infty \defin \max_{i=1,\ldots,n} |\vx_i|
= \lim_{q \to \infty} \|\vx\|_q$.  We also define the $\ell_0$-penalty as
the number of nonzero elements in a vector:\footnote{Note that it would
be more proper to write $\|\vx\|_0^0$ instead of $\|\vx\|_0$ to be consistent with the traditional notation $\|\vx\|_q$.
However, for the sake of simplicity, we will keep this
notation unchanged in the rest of the paper.}
$\|\vx\|_0 \defin \#\{i \st \vx_i
\neq 0  \} = \lim_{q \to 0^+}  (\sum_{i=1}^n |\vx_i|^q)$.  We consider the
Frobenius norm of a matrix~$\mx$ in~$\R^{m \times n}$: $\frob{\mx} \defin
(\sum_{i=1}^m \sum_{j=1}^n \mx_{ij}^2)^{{1}/{2}}$, where~$\mx_{ij}$ denotes the entry of~$\mx$ at row $i$ and column $j$. 
For an integer $n > 0$, and for any subset $J \subseteq \{1,\ldots,n\}$, we denote by~$\vx_J$ the vector of size~$|J|$ containing the entries of a vector~$\vx$ in~$\R^n$ indexed by~$J$, and by~$\mx_J$ the matrix in~$\R^{m \times |J|}$ containing the~$|J|$ columns of a matrix~$\mx$ in~$\R^{m \times n}$ indexed by~$J$.

\section{Loss Functions}\label{sec:losses}
\label{sec:intro_b}
We consider in this paper convex optimization problems of the form
\begin{equation}
   \min_{\vw \in \R^p} f(\vw) + \lambda \Omega(\vw), \label{eq:formulation}
\end{equation}
where $f: \R^p \to \R$ is a convex differentiable function and
$\Omega: \R^p \to \R$ is a sparsity-inducing---typically nonsmooth and non-Euclidean---norm.

In supervised learning, we predict outputs $y$ in $\Yc$ from observations~$\vx$ in $\Xc$;
these observations are usually represented by $p$-dimensional vectors with $\Xc=\R^p$.
In this supervised setting, $f$ generally corresponds to the empirical risk of a loss function $\ell:\Yc \times \R \to \R_+$.
More precisely, given $n$ pairs of data points
$
\{ (\vx^{(i)},y^{(i)}) \in \R^p\! \times\! \Yc;\ i\!=\!1,\dots,n\}
$,
we have for linear models\footnote{In Section~\ref{sec:mkl}, we consider extensions to non-linear predictors through multiple kernel learning.}
$
f(\vw)\defin\frac{1}{n}\sum_{i=1}^n\ell(y^{(i)},\bmt{\vw}\vx^{(i)})
$.
Typical examples of differentiable loss functions are the square loss for least squares regression,
i.e., $\ell(y,\yh) = \frac{1}{2}(y-\yh)^2$ with $y$ in $\R$,
and the logistic loss $\ell(y,\yh) = \log(1+e^{-y\yh})$ for logistic regression, with $y$ in $\{-1,1\}$. Clearly, several loss functions of interest are non-differentiable, such as the hinge loss $\ell(y,\yh) = (1-y\yh)_+$ or the absolute deviation loss $\ell(y,\yh) = |y-\yh|$, for which most of the approaches we present in this monograph would not be applicable or require appropriate modifications. Given the tutorial character of this monograph, 
we restrict ourselves to smooth functions $f$, which we consider is a reasonably broad setting, 
and we refer the interested reader to appropriate references in Section~\ref{sec:extensions}.
We refer the readers to \cite{Shawe-Taylor2004} for a more complete description of loss functions.

\paragraph{Penalty or constraint?}
Given our convex data-fitting term $f(\vw)$, we consider in this paper adding a convex penalty $\lambda \Omega(\vw)$. Within such a convex optimization framework, this is essentially equivalent to adding a constraint of the form $\Omega(\vw) \leq \mu$. More precisely, under weak assumptions on $f$ and $\Omega$ (on top of convexity), from Lagrange multiplier theory (see~\cite{borwein}, Section 4.3) $\vw$ is a solution of the constrained problem for a certain $\mu>0$ if and only if it is a solution of the penalized problem for a certain $\lambda \geq 0$. Thus, the two regularization paths, i.e., the set of solutions when $\lambda$ and $\mu$ vary, are equivalent. However, there is no direct mapping between corresponding values of $\lambda$ and~$\mu$. Moreover, in a machine learning context, where the parameters $\lambda$ and $\mu$ have to be selected, for example through cross-validation, the penalized formulation tends to be empirically easier to tune, 
as the performance is usually quite robust to small changes in $\lambda$, 
while it is not robust to small changes in~$\mu$.
Finally, we could also replace the penalization with a norm by a penalization with the squared norm. Indeed, following the same reasoning as for the non-squared norm, a penalty of the form $ \lambda \Omega(\vw)^2$ is ``equivalent'' to a constraint of the form $\Omega(\vw)^2 \leqslant \mu$, 
which itself is equivalent to  $\Omega(\vw) \leqslant \mu^{1/2}$, and thus to a penalty of the form $ \lambda' \Omega(\vw)^2$, for $\lambda' \neq \lambda$. Thus, using a squared norm,  as is often done in the context of multiple kernel learning (see Section~\ref{sec:mkl}), does not change the regularization properties of the formulation.

\section{Sparsity-Inducing Norms}\label{sec:norms}
\label{sec:intro_bb}

In this section, we present various norms as well as their main sparsity-inducing effects. These effects may be illustrated geometrically through the singularities of the corresponding unit balls (see Figure~\ref{intro:fig:balls}).

\paragraph{Sparsity through the $\ell_1$-norm.}

When one knows \emph{a priori} that the solutions $\vw^\star$ of problem~(\ref{eq:formulation})
should have a few non-zero coefficients, $\Omega$ is often chosen to be the
$\ell_1$-norm, i.e., $\Omega(\vw)=\sum_{j=1}^p|\vw_j|$.
This leads for instance to the Lasso~\cite{tibshirani} or basis pursuit~\cite{chen} with the square loss
and to $\ell_1$-regularized logistic regression~(see, for instance, \cite{koh2007interior,shevade2003simple}) with the logistic loss.
Regularizing by the $\ell_1$-norm is known to induce sparsity in the sense
that,
a number of coefficients of $\vw^\star$, depending on the strength of the regularization,
will be \textit{exactly} equal to zero.

\paragraph{$\ell_1/\ell_q$-norms.}

In some situations, the coefficients of $\vw^\star$ are naturally partitioned in subsets, or \textit{groups}, of variables. This is typically the case, when working with ordinal variables\footnote{Ordinal variables are integer-valued variables encoding levels of a certain feature, such as levels of severity of a certain symptom in a biomedical application, where the values do not correspond to an intrinsic linear scale: in that case it is common to introduce a vector of binary variables, each encoding a specific level of the symptom, that encodes collectively this single feature.}.
It is then natural to select or remove \textit{simultaneously} all the variables forming a group.
A regularization norm exploiting explicitly this group structure, or $\ell_1$-\emph{group norm}, can be shown to improve the prediction performance
and/or interpretability of the learned models~\cite{huang2,Lounici2009,obozinski-joint,roth,turlach,yuan}.
The arguably simplest group norm is the so-called-$\ell_1/\ell_2$ norm:
\begin{equation}
    \Omega(\vw) \defin \sum_{g \in \Gc} d_g \mynorm{\vw_g}{2}, \label{eq:def_omega_group}
\end{equation}
where $\Gc$ is a partition of $\{1,\dots,p\}$,
$(d_g)_{g\in\Gc}$ are some strictly positive weights,
and~$\vw_g$ denotes the vector in $\R^{|g|}$ recording the coefficients of $\vw$ indexed by $g$ in $\Gc$.
Without loss of generality we may assume all weights $(d_g)_{g\in\Gc}$ to be equal to one (when $\Gc$ is a partition, we can rescale the values of $\vw$ appropriately).
As defined in~Eq.~(\ref{eq:def_omega_group}), $\Omega$ is known as a mixed 
$\ell_1/\ell_2$-norm.
It behaves like an $\ell_1$-norm on the vector $(\mynorm{\vw_g}{2})_{g\in\Gc}$ in $\R^{|\Gc|}$,
and therefore, $\Omega$ induces group sparsity.
In other words, each $\mynorm{\vw_g}{2}$, and equivalently each $\vw_g$, is encouraged to be set to zero.
On the other hand, within the groups $g$ in~$\Gc$, the $\ell_2$-norm does not promote sparsity.
Combined with the square loss, it leads to the \emph{group Lasso} formulation~\cite{turlach,yuan}.
Note that when~$\Gc$ is the set of singletons, we retrieve the $\ell_1$-norm.
More general mixed $\ell_1/\ell_q$-norms for $q > 1$ are also used in the literature~\cite{zhao} (using $q=1$ leads to a weighted $\ell_1$-norm with no group-sparsity effects):
\begin{equation*}
\label{eq:def_group_lasso_norm}
    \Omega(\vw) = \sum_{g \in \Gc} \mynorm{\vw_g}{q} \defin
    \sum_{g \in \Gc} d_g \bigg\{\sum_{j\in g} |\vw_j|^q \bigg\}^{1/q}.
\end{equation*}
In practice though, the $\ell_1/\ell_2$- and $\ell_1/\ell_\infty$-settings remain the most popular ones. Note that using $\ell_\infty$-norms may have the undesired effect to favor solutions $\vw$ with many components of equal magnitude (due to the extra non-differentiabilities away from zero).
Grouped $\ell_1$-norms are typically used when extra-knowledge is available regarding an appropriate partition, in particular in the presence of categorical variables with orthogonal encoding~\cite{roth}, for multi-task learning where joint variable selection is desired~\cite{obozinski-joint}, and for multiple kernel learning (see Section~\ref{sec:mkl}).

\paragraph{Norms for overlapping groups: a direct formulation.}

In an attempt to better encode structural links between variables at play (e.g., spatial or hierarchical links related to the physics of the problem at hand),
recent research has explored the setting where $\Gc$ in Eq.~\eq{eq:def_omega_group} can contain groups of variables that
\textit{overlap}~\cite{hkl,Jacob2009,jenatton,kim3,schmidt2010convex,zhao}.
In this case, if the groups span the entire set of variables, $\Omega$ is still a norm, and it yields sparsity in the form of specific patterns of variables.
More precisely, the solutions $\vw^\star$ of problem~(\ref{eq:formulation}) can be shown to have a set of zero coefficients, or simply \textit{zero pattern},
that corresponds to a union of some groups $g$ in $\Gc$~\cite{jenatton}.
This property makes it possible to control the sparsity patterns of $\vw^\star$ by appropriately defining the groups in~$\Gc$. Note that here the weights $d_g$ should not be taken equal to one (see, \cite{jenatton} for more details).
This form of \textit{structured sparsity} has notably proven to be useful in various contexts, 
which we now illustrate through concrete examples:

\begin{figure}

\begin{center}
\includegraphics[width=0.45\linewidth]{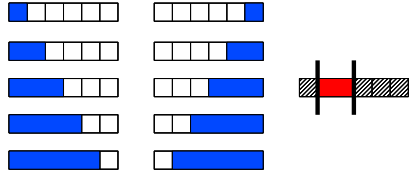}

\end{center}
\caption{ (Left) The set of blue groups to penalize in order to select contiguous patterns in a sequence.
(Right) In red, an example of such a nonzero pattern with its corresponding zero pattern (hatched area). } 
\label{intro:fig:sequence}
 \end{figure}

\begin{itemize}
\item[-]\textbf{One-dimensional sequence:}
Given $p$ variables organized in a sequence, if we want to select only contiguous nonzero patterns, we represent in Figure~\ref{intro:fig:sequence} the set of groups $\Gc$
to consider. In this case, we have $|\Gc|=O(p)$.
Imposing the contiguity of the nonzero patterns is for instance relevant in the context of time series, or for the diagnosis of tumors, based on the profiles of arrayCGH \cite{Rapaport2008}.
Indeed, because of the specific spatial organization of bacterial artificial chromosomes along the genome, the set of discriminative features is expected to have specific contiguous patterns.\\

\begin{figure} 
\begin{center}
\includegraphics[width=\linewidth]{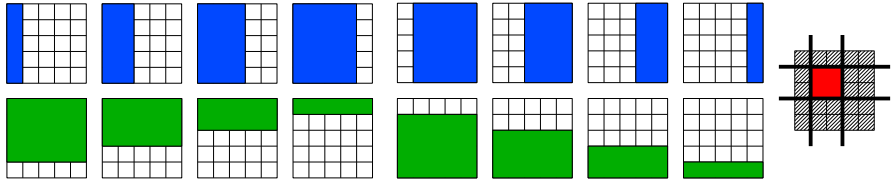}
\end{center}

\caption{ Vertical and horizontal groups: (Left) the set of blue and green groups to penalize in order to select
rectangles. (Right) In red, an example of nonzero pattern recovered in this setting, with its corresponding zero pattern (hatched area). } 
\label{intro:fig:axis-aligned}
\end{figure}

\item[-]\textbf{Two-dimensional grid:}
In the same way, assume now that the~$p$ variables are organized on a two-dimensional grid.
If we want the possible nonzero patterns $\Pc$ to be the set of all rectangles on this grid, 
the appropriate groups $\Gc$ to consider can be shown~(see~\cite{jenatton}) to be those represented in Figure~\ref{intro:fig:axis-aligned}.
In this setting, we have $|\Gc|=O(\sqrt{p})$.
Sparsity-inducing regularizations built upon such group structures have resulted in 
 good performances for background subtraction~\cite{huang,mairal2010,Mairal2011},
 topographic dictionary learning~\cite{Kavukcuoglu2009, Mairal2011}, wavelet-based denoising~\cite{Rao2011}, 
 and for face recognition with corruption by occlusions~\cite{jenatton2010sspca}.\\
 
\item[-]\textbf{Hierarchical structure:}
A third interesting example assumes that the variables have a hierarchical structure. Specifically, 
we consider that the $p$ variables correspond to the nodes of tree $\Tc$ (or a forest of trees).
Moreover, we assume that we want to select the variables according to a certain order:
a feature can be selected only if all its ancestors in $\Tc$ are already selected.
This hierarchical rule can be shown to lead to the family of groups displayed on Figure~\ref{intro:fig:treegroups}.
\begin{figure}[hbtp!]
   \centering
   \includegraphics[width=\textwidth]{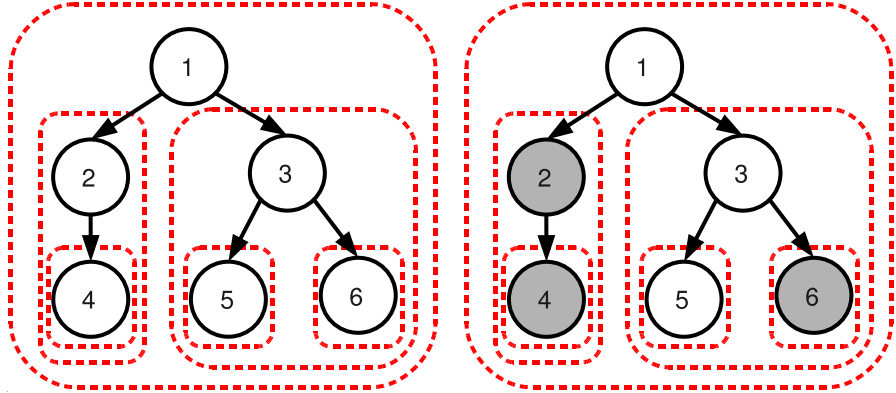}
   \caption{Left: example of a tree-structured set of groups $\Gc$ (dashed contours in red), corresponding to a tree $\mathcal{T}$ with $p=6$ nodes represented by black circles.
   Right: example of a sparsity pattern induced by the tree-structured norm corresponding to $\Gc$; the groups $\{2,4\},\{4\}$ and $\{6\}$ are set to zero, so that the corresponding nodes (in gray) that form subtrees of $\mathcal{T}$ are removed.
   The remaining nonzero variables $\{1,3,5\}$ form a rooted and connected subtree of $\mathcal{T}$.
   This sparsity pattern obeys the following equivalent rules: (i) if a node is selected, the same goes for all its ancestors; (ii) if a node is not selected, then its descendant are not selected.
}\label{intro:fig:treegroups}
\end{figure}

This resulting penalty was first used in~\cite{zhao}; since then, 
this group structure has led to numerous applications, for instance, 
 wavelet-based denoising~\cite{baraniuk,huang,Jenatton2010b,zhao},
 hierarchical dictionary learning for both topic modeling and image restoration~\cite{Jenatton2010a,Jenatton2010b},
 log-linear models for the selection of potential orders of interaction in a probabilistic graphical model~\cite{schmidt2010convex}, 
 bioinformatics, to exploit the tree structure of gene networks for multi-task regression~\cite{kim3},
 and multi-scale mining of fMRI data for the prediction of some cognitive task~\cite{Jenatton2011}.
 More recently, this hierarchical penalty was proved to be efficient for template selection in natural language processing~\cite{Martins2011}.\\

\item[-]\textbf{Extensions:}
The possible choices for the sets of groups $\Gc$ are not limited to the aforementioned examples.
More complicated topologies can be considered, for instance, three-dimensional spaces discretized in cubes or spherical volumes discretized in slices;
for instance, see~\cite{Varoquaux2010a} for an application to neuroimaging that pursues this idea. Moreover, directed acyclic graphs that extends the trees presented in Figure~\ref{intro:fig:treegroups}
have notably proven to be useful in the context of hierarchical variable selection~\cite{hkl,schmidt2010convex,zhao},

\end{itemize}

\paragraph{Norms for overlapping groups: a latent variable formulation.}
The family of norms defined in Eq.~(\ref{eq:def_omega_group}) is adapted to \emph{intersection-closed}
sets of nonzero patterns. However, some applications exhibit structures that can be more naturally modelled by \emph{union-closed} families of supports.
This idea was developed in~\cite{Jacob2009,Obozinski2011Group} where, given a set of groups $\Gc$, the following \emph{latent group Lasso} norm was proposed:
\begin{equation*}
\Omega_{\text{union}}(\vw) \defin \min_{ \vv \in \R^{p\times |\Gc|}  } \sum_{ g \in \Gc } d_g \| \vv^g \|_q\quad\text{s.t.}\ 
\begin{cases}
&\!\!\!\sum_{ g\in\Gc }\vv^g = \vw,\\
&\!\!\!\forall g\in\Gc,\ \vv^g_j = 0\ \text{if}\ j \notin g.
\end{cases}
\end{equation*}
The idea is to introduce latent parameter vectors $\vv^g$ constrained each to be supported on the corresponding group $g$, which should explain $\vw$ linearly and which are themselves regularized by a usual $\ell_1/\ell_q$-norm.
$\Omega_{\text{union}}$ reduces to the usual $\ell_1/\ell_q$ norm when groups are disjoint and provides therefore a different generalization of the latter to the case of overlapping groups than the norm considered in the previous paragraphs.
In fact, it is easy to see that solving Eq.~(\ref{eq:formulation}) with the norm $\Omega_{\text{union}}$ is equivalent to solving
\begin{equation}
\label{eq:expanded}
\min_{(\vv^g \in \R^{|g|})_g \in \Gc} \frac{1}{n} \sum_{i=1}^n \ell \big ( y^{(i)}, \sum_{g \in \Gc} {\vv^g_g}^\top \vx_g^{(i)}\big ) + \lambda \, \sum_{g \in \Gc} d_g \|\vv^g\|_q 
\end{equation}
and setting $\vw=\sum_{g \in \Gc}\vv^g$.
This last equation shows that using the norm $\Omega_{\text{union}}$ can be interpreted as implicitly duplicating the variables belonging to several groups and regularizing with a weighted $\ell_1/\ell_q$ norm for disjoint groups in the expanded space.
It should be noted that a careful choice of the weights is much more important in the situation of overlapping groups than in the case of disjoint groups, as it influences possible sparsity patterns~\cite{Obozinski2011Group}.

This latent variable formulation pushes some of the vectors $\vv^g$ to zero while keeping others with no zero components, hence leading to a vector $\vw$ with a support which is in general the union of the selected groups. 
Interestingly, it can be seen as a convex relaxation of a non-convex penalty encouraging similar sparsity patterns which was introduced by~\cite{huang}. Moreover, this norm can also be interpreted as a particular case of the family of \emph{atomic norms}, which were recently introduced by \cite{Chandrasekaran2010Convex}. 

\textit{Graph Lasso.} One type of a priori knowledge commonly encountered takes the form of graph defined on the set of input variables, which is such that connected variables are more likely to be simultaneously relevant or irrelevant; this type of prior is common in genomics where regulation, co-expression or interaction networks between genes (or their expression level) used as predictors are often available. To favor the selection of neighbors of a selected variable, it is possible to consider the edges of the graph
as groups in the previous formulation (see \cite{Jacob2009,Rao2011}).

\textit{Patterns consisting of a small number of intervals.} A quite similar situation occurs, when one knows a priori---typically for variables forming sequences (times series, strings, polymers)---that the support should consist of a small number of connected subsequences. In that case, one can consider the sets of variables forming connected subsequences (or connected subsequences of length at most $k$) as the overlapping groups.

\paragraph{Multiple kernel learning.}  For most of the sparsity-inducing terms described in this paper, 
we may replace real variables and their absolute values by pre-defined groups of variables with their Euclidean norms (we have already seen such examples with $\ell_1/\ell_2$-norms), or more generally, by members of reproducing kernel Hilbert spaces. As shown in Section~\ref{sec:mkl}, 
most of the tools that we present in this paper are applicable to this case as well, 
through appropriate modifications and borrowing of tools from kernel methods.
These tools have applications in particular in multiple kernel learning. 
Note that this extension requires tools from convex analysis presented in Section~\ref{sec:opt_tools}.

\paragraph{Trace norm.}
In learning problems on matrices, such as matrix completion, the rank plays a similar role to the cardinality of the support for vectors. Indeed, the rank of a matrix $\mathbf{M}$ may be seen as the number of non-zero singular values of $\mathbf{M}$. The rank of $\mathbf{M}$ however is not a continuous function of $\mathbf{M}$, and, following the convex relaxation of the $\ell_0$-pseudo-norm into the $\ell_1$-norm, we may relax the rank of $\mathbf{M}$ into the sum of its singular values, which happens to be a norm, and is often referred to as the trace norm or nuclear norm of $\mathbf{M}$, and which we denote by $\| \mathbf{M} \|_\ast$. As shown in this paper, many of the tools designed for the $\ell_1$-norm may be extended to the trace norm.
Using the trace norm as a convex surrogate for rank has many applications in control theory~\cite{fazel}, matrix completion~\cite{jake,Srebro2005Maximum}, multi-task learning~\cite{pontil}, or multi-label classification~\cite{srebro-mc}, where low-rank priors are adapted.

\paragraph{Sparsity-inducing properties: a geometrical intuition.}
Although we consider in Eq.~(\ref{eq:formulation}) a regularized formulation, as already described in Section~\ref{sec:losses}, we could equivalently focus on a \emph{constrained} problem, that is,
\begin{equation}\label{intro:eq:minf_c}
\min_{\vw \in \R^p} f(\vw)\quad \text{such that}\quad  \Omega(\vw) \leq \mu, 
\end{equation}
for some $\mu \in \R_+$. 
The set of solutions of Eq.~(\ref{intro:eq:minf_c}) parameterized by~$\mu$ is the same as that of Eq.~(\ref{eq:formulation}), 
as described by some value of $\lambda_\mu$ depending on $\mu$~(e.g., see Section 3.2 in \cite{borwein}).
At optimality, the gradient of $f$ evaluated at any solution $\hat{\vw}$ of~(\ref{intro:eq:minf_c}) is known to belong to the normal cone of $\Bc=\{\vw \in \R^p;\ \Omega(\vw) \leq \mu\}$
at $\hat{\vw}$~\cite{borwein}. In other words, 
for sufficiently small values of $\mu$, i.e., so that the constraint is active, 
the level set of $f$ for the value $f(\hat{\vw})$ is tangent to $\Bc$.

As a consequence, the geometry of the ball $\Bc$ is directly related to the properties of the solutions $\hat{\vw}$.
If $\Omega$ is taken to be the $\ell_2$-norm, then the resulting ball $\Bc$ is the standard, isotropic, ``round'' ball that does not favor any specific direction of the space.  
On the other hand, when~$\Omega$ is the $\ell_1$-norm, $\Bc$ corresponds to a diamond-shaped pattern in two dimensions, and to a pyramid in three dimensions.
In particular, $\Bc$ is anisotropic and exhibits some singular points due to the extra non-smoothness of $\Omega$. 
Moreover, these singular points are located along the axis of $\R^p$, so that if the level set of $f$ happens to be tangent at one of those points, sparse solutions are obtained.
We display in Figure~\ref{intro:fig:balls} the balls $\Bc$ for the $\ell_1$-, $\ell_2$-, and two different grouped $\ell_1/\ell_2$-norms.
\begin{figure}[hbtp!]
 \centering
   \subfloat[$\ell_2$-norm ball.]{\includegraphics[width=0.37\linewidth]{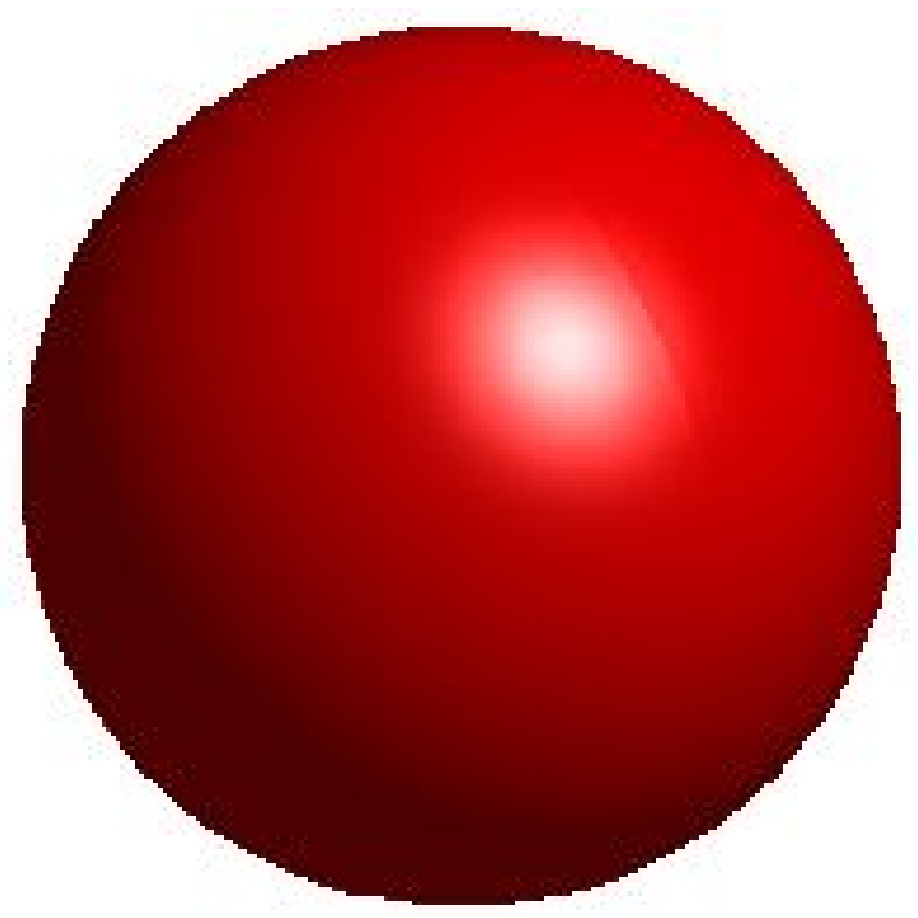}} \hfill
   \subfloat[$\ell_1$-norm ball.]{\includegraphics[width=0.37\linewidth]{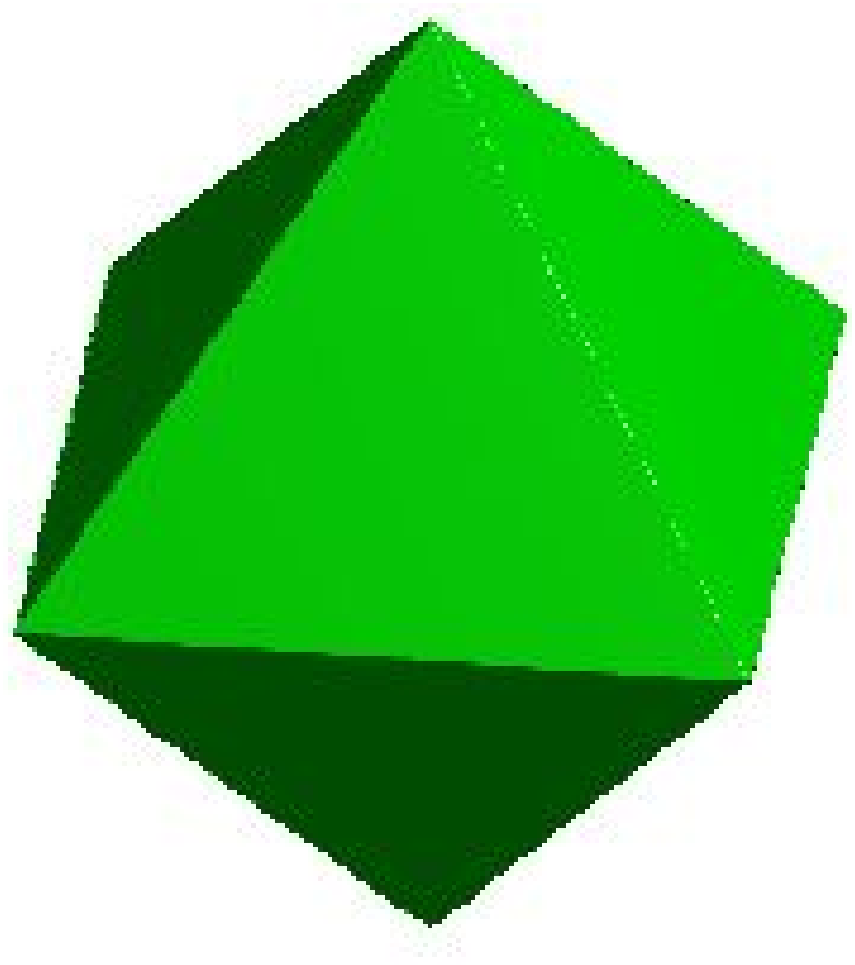}} \\
   \subfloat[$\ell_1/\ell_2$-norm ball: $\Omega(\vw)=\|\vw_{\{1,2\}}\|_2+|\vw_3|$.]
   {\includegraphics[width=0.39\linewidth]{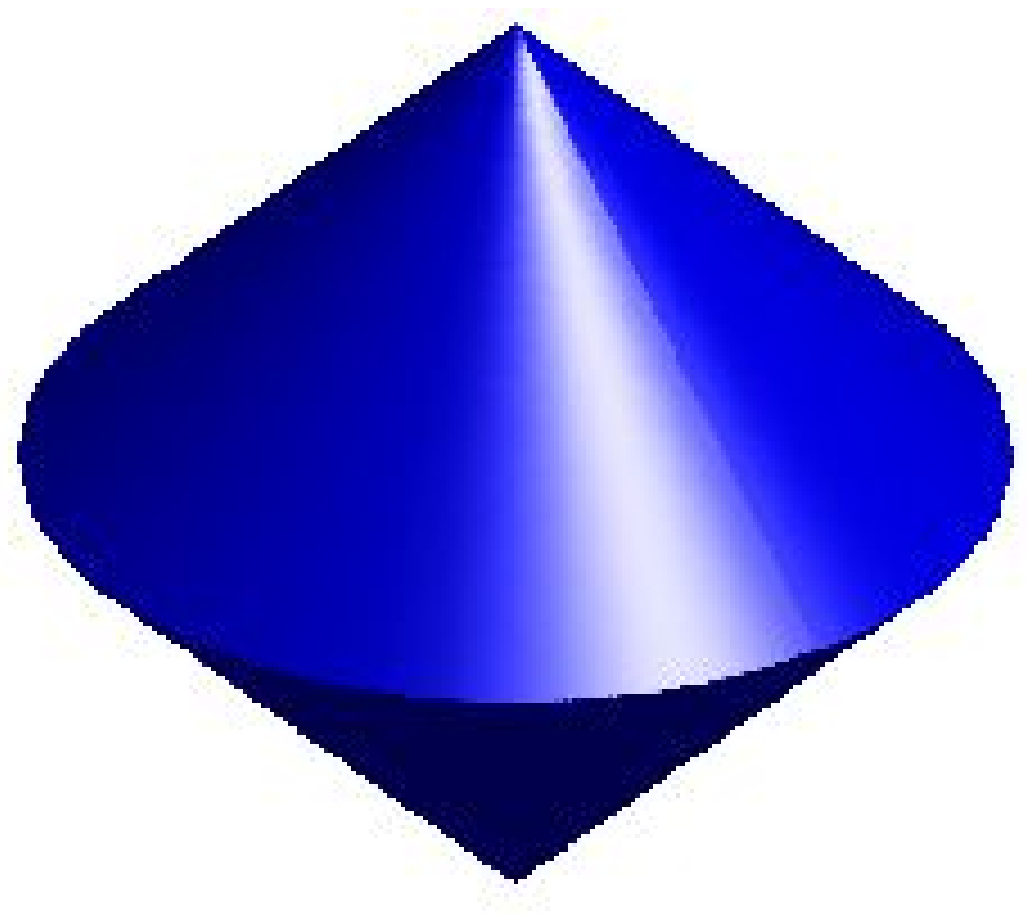}\label{intro:subfig:groupl1l2ball}} \hfill
   \subfloat[$\ell_1/\ell_2$-norm ball:\newline$\Omega(\vw)=\|\vw\|_2+|\vw_1|+|\vw_2|$.]
{\includegraphics[width=0.35\linewidth]{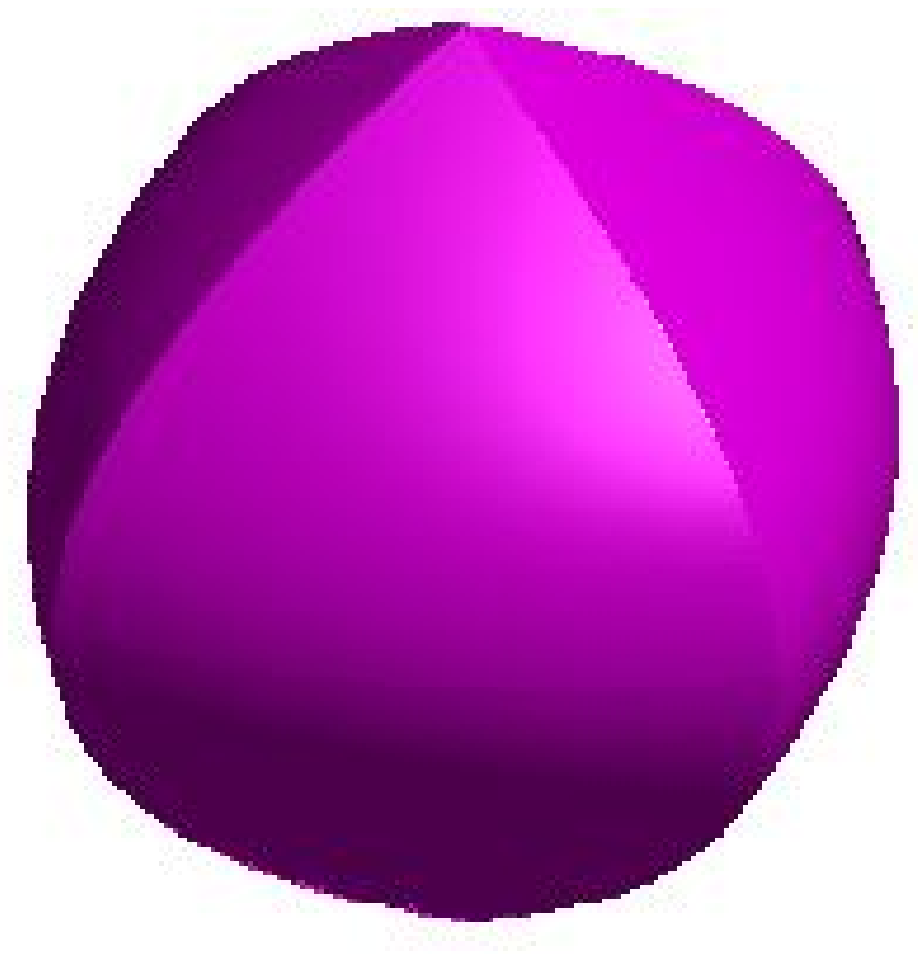}\label{intro:subfig:overl1l2ball}}\\
  \subfloat[$\Omega_{\text{union} }$ ball for $\Gc=\big \{\{1,3\},\{2,3\} \big \}$.]
   {\hspace*{.5cm}\includegraphics[width=0.32\linewidth]{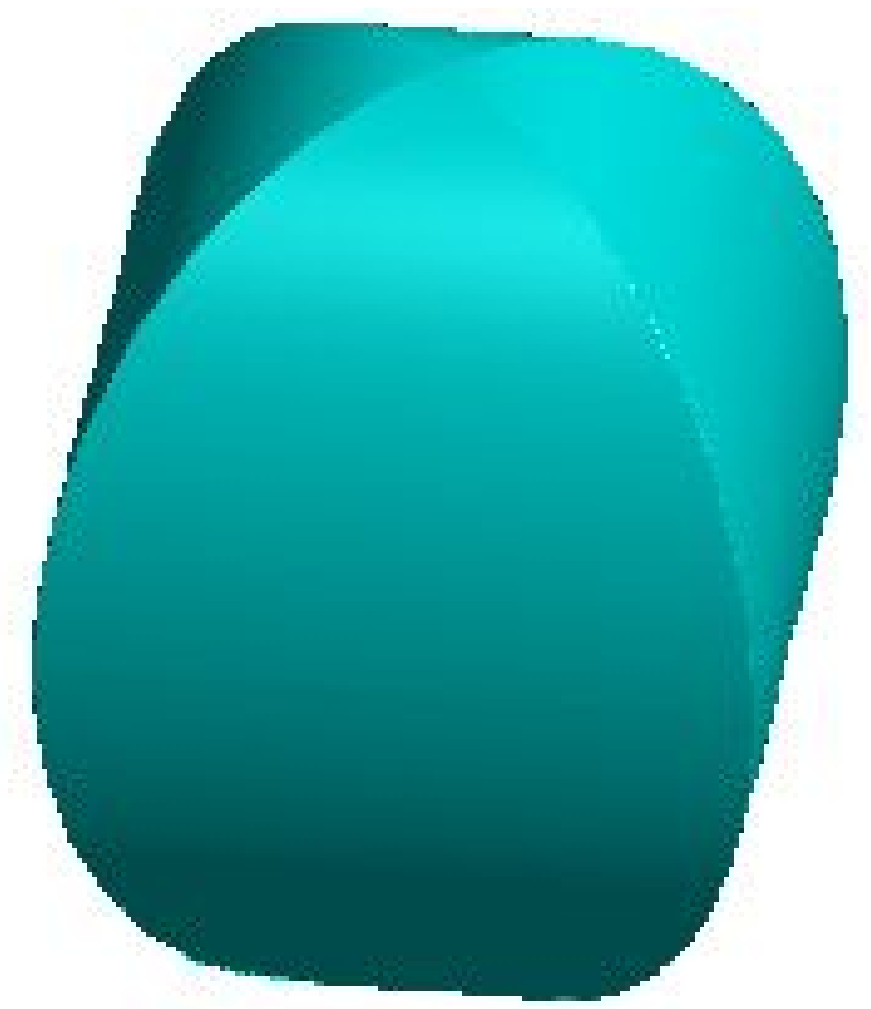}
   \label{subfig:lglball2}\hspace*{.5cm}} \hfill
   \subfloat[$\Omega_{\text{union} }$ ball for $\Gc=\big \{\{1,3\},\{2,3\},\{1,2\} \big \}$. ]{\hspace*{.1cm}\includegraphics[width=0.39\linewidth]{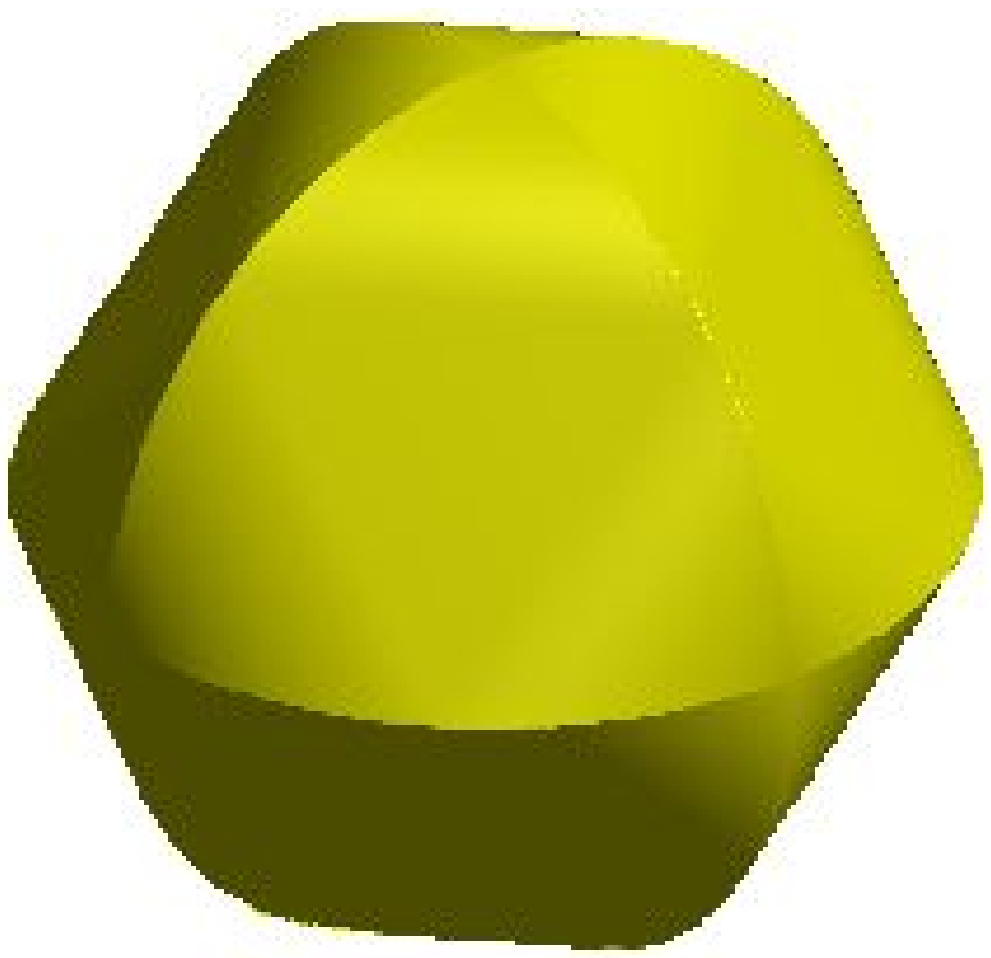}\label{subfig:lglball3}\hspace*{.1cm}}
\caption{Comparison between different balls of sparsity-inducing norms in three dimensions. 
The singular points appearing on these balls describe the sparsity-inducing behavior of the underlying norms $\Omega$.}\label{intro:fig:balls}
\end{figure}

\paragraph{Extensions.}
The design of sparsity-inducing norms is an active field of research and similar tools to the ones we present here can be derived for other norms. As shown in Section~\ref{sec:proximal_methods}, computing the proximal operator readily leads to efficient algorithms, and for the extensions we present below, these operators can be efficiently computed. 

In order to impose prior knowledge on the support of predictor, the norms based on overlapping $\ell_1/\ell_\infty$-norms can be shown to be convex relaxations of submodular functions of the support, and further ties can be made between convex optimization and combinatorial optimization (see~\cite{bach2010structured} for more details).
Moreover, similar developments may be carried through for norms which try to enforce that the predictors have many equal components and that the resulting clusters have specific shapes, e.g., contiguous in a pre-defined order, see some examples in Section~\ref{sec:proximal_methods}, and, e.g.,~\cite{shapinglevelsets,chambolle2005total,mairal2010,tibshirani2005sparsity,vert2010fast} and references therein.

\section{Optimization Tools}\label{sec:opt_tools}
\label{sec:intro_c}
The tools used in this paper are relatively basic and should be
accessible to a broad audience. Most of them can be found in classical books
on convex optimization~\cite{bertsekas,borwein,boyd.convex,nocedal}, but
for self-containedness, we present here a few of them related to non-smooth unconstrained optimization. 
In particular, these tools allow the derivation of rigorous approximate optimality conditions based on duality gaps (instead of relying on weak stopping criteria based on small changes or low-norm gradients).

\paragraph{Subgradients.}
Given a convex function $g: \R^p \to \R$ and a vector~$\vw$ in~$\R^p$, let us define the \emph{subdifferential} of $g$ at $\vw$ as
\begin{displaymath}
\partial g(\vw) \defin \{ \vz \in \R^p ~|~  g(\vw) + \bmt{\vz}(\vw'-\vw) \leq
g(\vw') ~\text{for all vectors}~ \vw' \in \R^p\}.
\end{displaymath}
The elements of $\partial g(\vw)$ are called the \emph{subgradients} of $g$ at
$\vw$. Note that all convex functions defined on $\mathbb{R}^p$ have non-empty subdifferentials at every point. This definition admits a clear geometric interpretation: any subgradient
$\vz$ in $\partial g(\vw)$ defines an affine function $\vw' \mapsto g(\vw) +
\bmt{\vz}(\vw'-\vw)$ which is tangent to the graph of the function~$g$ (because of the convexity of $g$, it is a lower-bounding tangent). Moreover, there is a
bijection (one-to-one correspondence) between such ``tangent affine functions''
and the subgradients, as illustrated in Figure~\ref{fig:subgrad}.

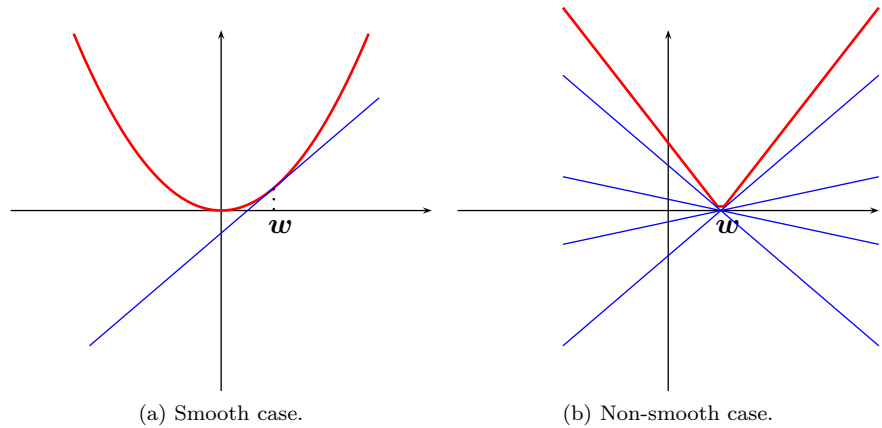
\begin{figure}[hbtp]
\centering
\subfloat[Smooth case.]{
   \psset{yunit=1.2,xunit=1.4}
   \begin{pspicture}(-2,-2)(2,2)
      \psline[linewidth=0.5pt]{->}(-2,0)(2,0)
      \psline[linewidth=0.5pt]{->}(0,-2)(0,2)
       \psplot[linecolor=red, linewidth=1pt]{-1.4}{1.4}{
          x x mul
       }
    \put(0.5,-0.3){ $\vw$ }
    \psline[linecolor=black,linewidth=1pt,linestyle=dotted]{-}(0.5,0)(0.5,0.25)
     \psplot[linecolor=blue, linewidth=0.5pt]{-1.25}{1.5}{
        -0.25 x 1 mul add
     }
   \end{pspicture}
} \hfill
\subfloat[Non-smooth case.]{
   \psset{yunit=1.2,xunit=1.4}
   \begin{pspicture}(-2,-2)(2,2)
      \psline[linewidth=0.5pt]{->}(-2,0)(2,0)
      \psline[linewidth=0.5pt]{->}(0,-2)(0,2)
       \psplot[linecolor=red, linewidth=1pt]{-1}{2}{
          1.5 -0.5 x add abs mul
       }
    \put(0.5,-0.3){ $\vw$ }
   \psplot[linecolor=blue, linewidth=0.5pt]{-1}{2}{
        -0.5 x add
     }
   \psplot[linecolor=blue, linewidth=0.5pt]{-1}{2}{
        -1.0 -0.5 x add mul
     }
   \psplot[linecolor=blue, linewidth=0.5pt]{-1}{2}{
        -0.25 -0.5 x add mul
     }
   \psplot[linecolor=blue, linewidth=0.5pt]{-1}{2}{
        0.25 -0.5 x add mul
     }
   \end{pspicture}
}

\caption[Gradients and subgradients for smooth and non-smooth functions.]{Red curves represent the graph of a smooth (left) and a non-smooth (right) function~$f$. Blue affine functions represent subgradients of the function~$f$ at a point~$\vw$.}
\label{fig:subgrad}
\end{figure}
Subdifferentials are useful
for studying nonsmooth optimization problems because of  the following proposition (whose proof is straightforward from the definition):

\begin{proposition}[Subgradients at Optimality]~\label{prop:opt}\\
For any convex function $g: \R^p \to \R$, a point $\vw$ in $\R^p$ is a global
minimum of $g$ if and only if the condition $0 \in \partial g(\vw)$ holds.
\end{proposition}
Note that the concept of subdifferential is mainly useful for nonsmooth functions.
If~$g$ is differentiable at $\vw$, the set~$\partial g(\vw)$ is indeed the singleton
$\{ \nabla g(\vw) \}$, where $\nabla g(\vw)$ is the gradient of $g$ at $\vw$, and the condition $0 \in \partial g(\vw)$ reduces to the
classical first-order optimality condition $\nabla g(\vw)=0$.
As a simple example, let us consider the
following optimization problem
\begin{displaymath}
   \min_{w \in \R} \frac{1}{2}(x-w)^2+\lambda|w|.
\end{displaymath}
Applying the previous proposition and
noting that the subdifferential $\partial|\cdot|$ is $\{+1\}$ for $w>0$, $\{-1\}$
for $w<0$ and $[-1,1]$ for $w=0$, it is easy to show that the unique solution admits
a closed form called the \emph{soft-thresholding} operator, following a terminology introduced
in~\cite{donoho}; it can be written
\begin{equation}
\label{eq:st_op}
   w^\star = \begin{cases}
      0 & ~\text{if}~|x| \leq \lambda \\
      (1-\frac{\lambda}{|x|})x &~\text{otherwise},
   \end{cases}
\end{equation}
or equivalently~$w^\star = \text{sign}(x)(|x|-\lambda)_+$, where $\text{sign}(x) $ is equal to  $1$ if $x>0$, $-1$ if $x<0$ and $0$ if $x=0$.   This operator
is a core component of many optimization techniques for sparse estimation,
as we shall see later. Its counterpart for non-convex optimization problems is the hard-thresholding operator. Both of them are presented in Figure~\ref{fig:soft}.
Note that similar developments could be carried through using directional derivatives instead of subgradients~(see, e.g., \cite{borwein}).

\begin{figure}[hbtp]
\centering
\psset{yunit=1.1,xunit=1.1}

\hspace{0.1cm}

\subfloat[soft-thresholding operator, \newline $w^\star=\text{sign}(x)(|x|-\lambda)_+$, \newline $\min_w \frac{1}{2}(x-w)^2+\lambda|w|$.]{ 
   \psset{yunit=1.2,xunit=1.2}
   \begin{pspicture}(-2,-2)(2,2)
   \psline[linewidth=0.5pt]{->}(-2,0)(2,0)
\put(2,-0.25){$x$}
\put(0.2,2){$w^\star$}
\psline[linewidth=0.5pt]{->}(0,-2)(0,2)
\psplot[linecolor=black, linestyle=dotted, linewidth=1pt]{-2}{2}{
      x 
}
\psplot[linecolor=red, linewidth=1pt]{-2}{-0.5}{
   x 0.5 add
}
\psline[linecolor=red, linewidth=1pt]{-}(-0.5,0)(0.5,0)
\psplot[linecolor=red, linewidth=1pt]{0.5}{2.0}{
   x -0.5 add
}
\put(0.5,-0.35){$\lambda$}
\put(-0.95,0.15){$-\lambda$}
\end{pspicture} 
} \hfill
\subfloat[hard-thresholding operator \newline $w^\star={\mathbf 1}_{|x|\geq \sqrt{2\lambda}}x$ \newline $\min_w \frac{1}{2}(x-w)^2+\lambda {\mathbf 1}_{|w|>0}$.]{ 
   \psset{yunit=1.2,xunit=1.2}
   \begin{pspicture}(-2,-2)(2,2)
   \psline[linewidth=0.5pt]{->}(-2,0)(2,0)
\put(2,-0.25){$x$}
\put(0.2,2){$w^\star$}
\psline[linewidth=0.5pt]{->}(0,-2)(0,2)
\psplot[linecolor=black, linestyle=dotted, linewidth=1pt]{-2}{2}{
      x 
}
\psplot[linecolor=red, linewidth=1pt]{-2}{-0.5}{
   x 
}
\put(0.4,-0.35){\small{$\sqrt{2\lambda}$}}
\put(-1.15,0.15){\small{$-\sqrt{2\lambda}$}}
\psline[linecolor=red, linewidth=1pt]{-}(-0.5,-0.5)(-0.5,0)
\psline[linecolor=red, linewidth=1pt]{-}(-0.5,0)(0.5,0)
\psline[linecolor=red, linewidth=1pt]{-}(0.5,0)(0.5,0.5)
\psplot[linecolor=red, linewidth=1pt]{0.5}{2.0}{
   x
}
\end{pspicture} 
} 
\caption{Soft- and hard-thresholding operators.} \label{fig:soft}
\end{figure}
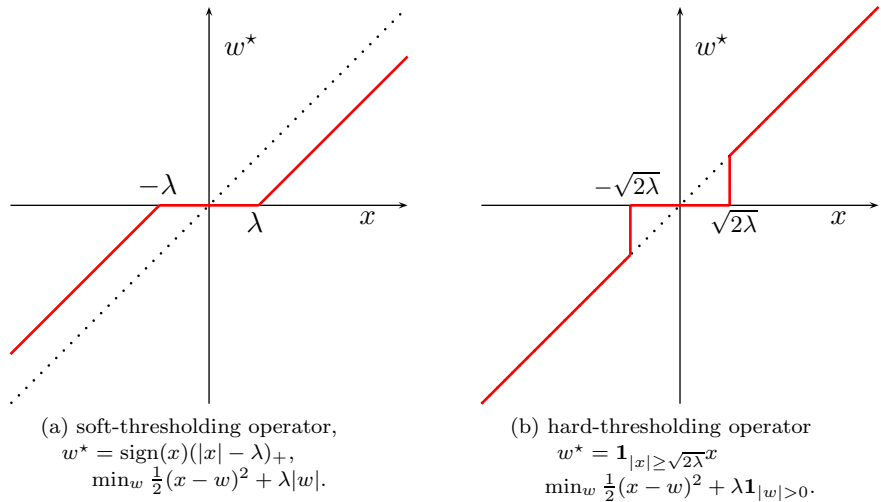

\paragraph{Dual norm and optimality conditions.} 
The next concept we introduce is the dual norm, which is important to study
sparsity-inducing regularizations~\cite{hkl,jenatton,negahban2009unified}. It
notably arises in the analysis of estimation bounds
\cite{negahban2009unified}, and in the design of working-set strategies as
will be shown in Section~\ref{sec:active_sets}.  The dual norm $\Omega^*$ of
the norm~$\Omega$ is defined for any vector $\vz$ in $\R^p$ by
\begin{equation}\label{eq:def_dualnorm}
   \Omega^*(\vz)\defin\max_{\vw \in \R^p}\: \bmt{\vz}\vw \: ~\text{such that}~ \: \Omega(\vw)\leq 1.
\end{equation}
Moreover, the dual norm of $\Omega^*$ is $\Omega$ itself, and as a consequence, the formula above holds also if the roles of $\Omega$ and $\Omega^*$ are exchanged. It is easy to show that in the case of an $\ell_q$-norm, $q \in [1; + \infty]$, the dual norm is the $\ell_{q'}$-norm, with $q'$ in $[1;
+\infty]$ such that $\frac{1}{q}+\frac{1}{q'}=1$. In
particular, the $\ell_1$- and $\ell_\infty$-norms are dual to each other, and
the $\ell_2$-norm is self-dual (dual to itself).

The dual norm plays a direct role in computing optimality conditions of sparse
regularized problems. By applying Proposition~\ref{prop:opt} to
Eq.~(\ref{eq:formulation}), we obtain the following proposition: 
\begin{proposition}[Optimality Conditions for Eq.~(\ref{eq:formulation})]\label{prop:optimality}
Let us consider problem~(\ref{eq:formulation}) where~$\Omega$ is a norm on~$\R^p$.
A vector $\vw$ in~$\R^p$ is optimal if and only if
$-\frac{1}{\lambda}\nabla f(\vw) \in \partial\Omega(\vw)$ with
\begin{equation}
\partial\Omega(\vw)=
\begin{cases}
\{\vz \in \R^p;\  \Omega^*(\vz)\leq 1 \}\ \text{if}\ \vw=0,\\
\{\vz \in \R^p;\  \Omega^*(\vz)= 1\ \text{and}\ \bmt{\vz}\vw=\Omega(\vw)\}\ \text{otherwise}.
\end{cases}\label{eq:opt}
\end{equation}
\end{proposition}
Computing the subdifferential of a norm is a classical course
exercise~\cite{borwein} and its proof will be presented in the next section, in
Remark~\ref{rem:proof_subnorm}. As a consequence, the vector~$\mathbf{0}$
is solution if and only if
$\Omega^*\big(\nabla f(\mathbf{0})\big) \leq \lambda$. 
Note that this shows that for all $\lambda$ larger than $\Omega^*\big(\nabla f(\mathbf{0})\big)$, 
$\vw=\mathbf{0}$ is a solution of the regularized optimization problem (hence this value is the start of the non-trivial regularization path).

These general optimality conditions can be specialized to the Lasso
problem~\cite{tibshirani}, also known as basis pursuit~\cite{chen}:
\begin{equation}
\min_{\vw \in \R^p} \frac{1}{2n}\|\vy-\mx \vw \|_2^2 + \lambda \|\vw\|_1,  \label{eq:lasso}
\end{equation}
where $\vy$ is in $\R^n$, and $\mx$ is a design matrix in $\R^{n \times p}$.
With Eq.~(\ref{eq:opt}) in hand, we can now derive necessary and sufficient optimality conditions:
\begin{proposition}[Optimality Conditions for the Lasso]\label{prop:optimality_lasso}~\\
A vector~$\vw$ is a solution of the Lasso problem~(\ref{eq:lasso}) if and only if
\begin{equation}
\forall j =1,\ldots,p,~
\left\{
\begin{array}{ccll}
|\bmt{\mx}_j(\vy-\mx\vw)| &\leq& n\lambda & \text{if}\ \vw_j= 0 \\
\bmt{\mx}_j(\vy-\mx\vw) &=& n\lambda \sgn(\vw_j) & \text{if}\ \vw_j \neq 0, \\
\end{array}
\right.
 \label{eq:optlasso}
\end{equation}
where $\mx_j$ denotes the $j$-th column of $\mx$, and $\vw_j$ the $j$-th entry
of $\vw$.
\end{proposition}
\begin{proof}
We apply Proposition~\ref{prop:optimality}. The
condition~$-\frac{1}{\lambda}\nabla f(\vw) \!\in\!  \partial \|\vw\|_1$ can be
rewritten: $ \bmt{\mx}(\vy-\mx\vw) \! \in  n \lambda \partial \|\vw\|_1$, which is equivalent to:
(i)~if~$\vw=\!0$, $\|\bmt{\mx}(\vy-\mx\vw)\|_\infty \!\leq\! n \lambda$ (using the fact that the~$\ell_\infty$-norm
is dual to the~$\ell_1$-norm); (ii)~if~$\vw \!\neq \!0$, $\|\bmt{\mx}(\vy-\mx\vw)\|_\infty \!=\!  n \lambda$ 
and~$\bmt{\vw}\bmt{\mx}(\vy-\mx\vw) \!=\! n \lambda\|\vw\|_1$. It is then easy to check that these conditions are equivalent to Eq.~(\ref{eq:optlasso}).
\end{proof}
 As we will see in Section~\ref{sec:lars}, it is possible to derive
from these conditions interesting properties of the Lasso, as well as efficient
algorithms for solving~it.  We have presented a useful duality tool for norms.
More generally, there exists a related concept for convex functions, which we
now introduce.

\subsection{Fenchel Conjugate and Duality Gaps}
Let us denote by $f^*$ the Fenchel
conjugate of $f$~\cite{rockafellar97}, defined~by
\begin{displaymath}
f^*(\vz)\defin\sup_{\vw \in \R^p} [\bmt{\vz} \vw  - f(\vw)].
\end{displaymath}
Fenchel conjugates are particularly useful to derive dual problems and duality gaps\footnote{For many of our norms, \emph{conic} duality tools would suffice (see, e.g.,~\cite{boyd.convex}).}. Under mild conditions, the conjugate of the conjugate of a convex function is itself, leading to the following representation of~$f$ as a maximum of affine functions:
\begin{displaymath}
f(\vw) = \sup_{\vz \in \R^p} [\bmt{\vz} \vw  - f^*(\vz)].
\end{displaymath}
In the context of this tutorial, it is notably useful to specify the expression of the conjugate of a norm.
Perhaps surprisingly and misleadingly, the conjugate of a norm is not equal to its dual norm, but
corresponds instead to the indicator function of the unit ball of its dual norm.
More formally, let us introduce  
the indicator function ${\iota}_{\Omega^*}$ such that ${\iota}_{\Omega^*}(\vz)$ is equal to $0$ if $\Omega^*(\vz) \leq 1$ and $+\infty$
otherwise. 
Then, we have the following well-known results, which appears in several text books (e.g., see Example 3.26 in \cite{boyd.convex}):
\begin{proposition}[Fenchel Conjugate of a Norm]\label{prop:fenchel_dualnorm}
Let $\Omega$ be a norm on $\R^p$. The following equality holds for any $\vz \in \R^p$
$$
\sup_{\vw \in \R^p} [\bmt{\vz} \vw  - \Omega(\vw)] = \iota_{\Omega^*}(\vw) = \begin{cases}  0 & \text{if}\ \Omega^*(\vz) \leq 1 \\ +\infty & \text{otherwise}.   \end{cases}
$$
\end{proposition}
\begin{proof}
On the one hand, assume that the dual norm of $\vz$ is greater than one, that is, $\Omega^*(\vz) > 1$.
According to the definition of the dual norm (see Eq.~(\ref{eq:def_dualnorm})), and since the supremum is taken over the compact set $\{ \vw\in\R^p;\ \Omega(\vw)\leq 1\}$,
there exists a vector $\vw$ in this ball such that $\Omega^*(\vz) = \vz^\top \vw > 1$.
For any scalar $t \geq 0$, consider $\vv=t\vw$ and notice that 
$$
\vz^\top \vv - \Omega(\vv) = t[ \vz^\top\vw - \Omega(\vw) ] \geq t,
$$
which shows that when $\Omega^*(\vz) > 1$, the Fenchel conjugate is unbounded.
Now, assume that $\Omega^*(\vz) \leq 1$.
By applying the generalized Cauchy-Schwartz's inequality, we obtain for any $\vw$
$$
\vz^\top \vw - \Omega(\vw) \leq \Omega^*(\vz)\, \Omega(\vw) - \Omega(\vw) \leq 0.
$$
Equality holds for $\vw=\mathbf{0}$, and the conclusion follows.
\end{proof}
An important and useful duality result is the so-called Fenchel-Young
inequality~(see \cite{borwein}), which we will shortly illustrate
geometrically:
\begin{proposition}[Fenchel-Young Inequality]\label{prop:fenchel_young}
Let~$\vw$ be a vector in~$\R^p$, $f$ be a function on~$\R^p$, and~$\vz$ be a vector in the domain of~$f^\ast$ (which we assume  non-empty). We have then the following inequality
$$
   f(\vw) + f^*(\vz) \geq \bmt{\vw}\vz,
$$
with equality if and only if~$\vz$ is in~$\partial f(\vw)$.
\end{proposition}
We can now illustrate geometrically the duality principle between a
function and its Fenchel conjugate in Figure~\ref{fig:geom}.
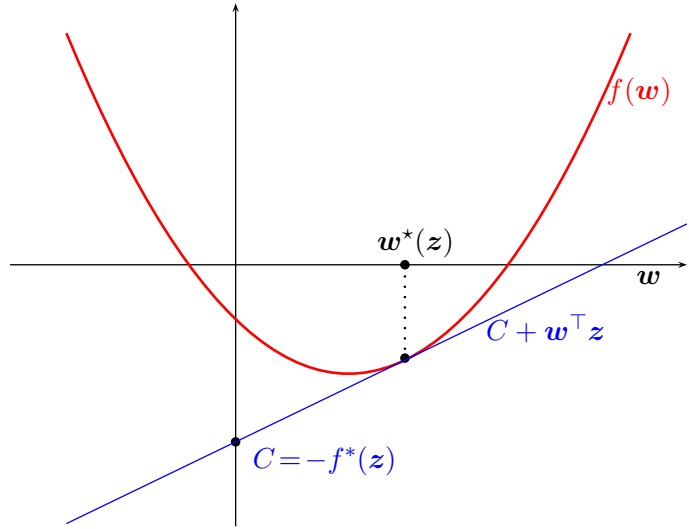
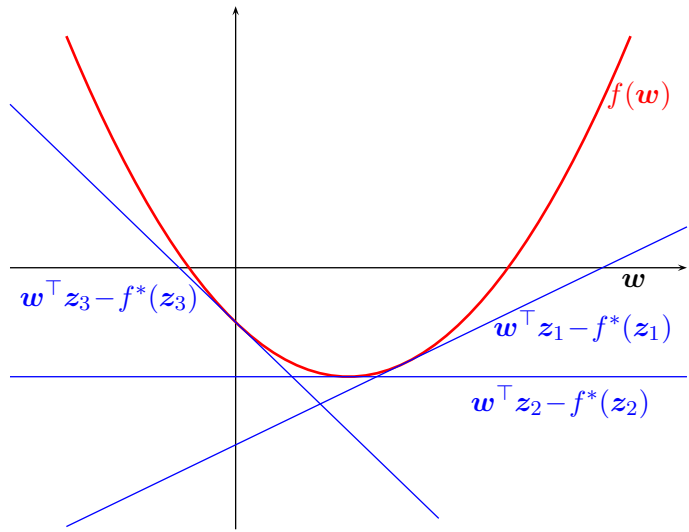
\begin{figure}[!hbtp]
\centering
\subfloat[Fenchel conjugate, tangent hyperplanes and subgradients.]{  \label{fig:geom:a}
   \psset{yunit=2.9,xunit=3.0}
   \begin{pspicture}(-1,-1.2)(2,1.2)
      \psline[linewidth=0.5pt]{->}(-1,0)(2,0)
      \psline[linewidth=0.5pt]{->}(0,-1.2)(0,1.2)
       \psplot[linecolor=red, linewidth=1pt]{-0.75}{1.75}{
          -0.5 x -0.5 add x -0.5 add mul add
       }
    \put(1.75,0.2){ $\vw^\star(\vz)$ }
    \put(5.2,-0.25){ $\vw$ }
    \put(4.8,2.2){ {\color{red} $f(\vw)$  } }
    \put(0.1,-2.7){ {\color{blue} $C\!=\!-f^*(\vz)$  } }
    \put(3.2,-1.0){ {\color{blue} $C+\bmt{\vw}\vz$  } }
    \psdots(0.75,0)(0.75,-0.4275)(0,-0.8125)
    \psline[linecolor=black,linewidth=1pt,linestyle=dotted]{-}(0.75,0)(0.75,-0.4375)
     \psplot[linecolor=blue, linewidth=0.5pt]{-0.75}{2}{
        -13 16 div x 0.5 mul add
     }
   \end{pspicture}
}  \\
\subfloat[The graph of~$f$ is the envelope of the tangent hyperplanes ${\mathcal P}(\vz)$.]{ \label{fig:geom:b}
   \psset{yunit=2.9,xunit=3.0}
   \begin{pspicture}(-1,-1.2)(2,1.2)
      \psline[linewidth=0.5pt]{->}(-1,0)(2,0)
      \psline[linewidth=0.5pt]{->}(0,-1.2)(0,1.2)
       \psplot[linecolor=red, linewidth=1pt]{-0.75}{1.75}{
          -0.5 x -0.5 add x -0.5 add mul add
       }
    \psplot[linecolor=blue, linewidth=0.5pt]{-1}{2}{
       -0.5 
    }
   \psplot[linecolor=blue, linewidth=0.5pt]{-1}{0.9}{
       -0.25 -1 x mul add 
    }
    \put(5,-0.25){ $\vw$ }
    \put(4.8,2.2){ {\color{red} $f(\vw)$  } }
    \put(3.3,-0.9){ {\color{blue} $\bmt{\vw}\vz_1 \! - \! f^*(\vz_1)$  } }
    \put(3.0,-1.9){ {\color{blue} $\bmt{\vw}\vz_2 \! - \! f^*(\vz_2)$  } }
    \put(-3.0,-0.5){ {\color{blue} $\bmt{\vw}\vz_3 \! - \! f^*(\vz_3)$  } }
     \psplot[linecolor=blue, linewidth=0.5pt]{-0.75}{2}{
        -13 16 div x 0.5 mul add
     }
   \end{pspicture}
} 

\caption[Geometric Interpretation of the Fenchel conjugate]{ 
   For all~$\vz$ in~$\R^p$, we denote by $\mathcal{P}(\vz)$ the hyperplane with normal~$\vz$ and tangent
   to the graph of the convex function~$f$.
Fig.~\subref{fig:geom:a}: For any contact point between the graph of~$f$ and an hyperplane $\mathcal{P}(\vz)$, we have that~$f(\vw)+f^*(\vz) = \bmt{\vw}\vz$ and~$\vz$ is in~$\partial f(\vw)$
   (the Fenchel-Young inequality is an equality).
Fig.~\subref{fig:geom:b}: the graph of~$f$ is the convex envelope of the collection of hyperplanes $(\mathcal{P}(\vz))_{\vz \in \R^p}$.}
\end{figure}\label{fig:geom}

\begin{remark}\label{rem:proof_subnorm}
With Proposition~\ref{prop:fenchel_dualnorm} in place, we can formally (and easily) prove the relationship in Eq.~(\ref{eq:opt}) that make explicit the subdifferential of a norm.
Based on Proposition~\ref{prop:fenchel_dualnorm}, we indeed know that the conjugate of $\Omega$ is $\iota_{\Omega^\ast}$. 
Applying the Fenchel-Young inequality (Proposition~\ref{prop:fenchel_young}), we have
$$
\vz\in \partial\Omega(\vw) \Leftrightarrow \Big[ \vz^\top\vw = \Omega(\vw)+\iota_{\Omega^\ast}(\vz)\Big],
$$
which leads to the desired conclusion.
\end{remark}

For many objective functions, the Fenchel conjugate admits closed forms, and can
therefore be computed efficiently~\cite{borwein}.  Then, it is possible to
derive a duality gap for problem~(\ref{eq:formulation}) from standard Fenchel
duality arguments (see \cite{borwein}),
as shown in the following proposition:
\begin{proposition}[Duality for Problem~(\ref{eq:formulation})]~\label{prop:duagap}\\
   If $f^*$ and $\Omega^*$ are respectively the Fenchel conjugate of a convex and differentiable function $f$ and the dual norm of $\Omega$, then we have
   \begin{equation}
   \label{eq:pd}
      \max_{\vz \in \R^p: \, \Omega^*(\vz) \leq \lambda} \: -f^*(\vz) \quad \leq \quad \min_{\vw \in \R^p} f(\vw) + \lambda\Omega(\vw).
   \end{equation}
   Moreover, equality holds as soon as the domain of $f$ has non-empty interior.
\end{proposition}
\begin{proof}
This result is a specific instance of Theorem 3.3.5 in \cite{borwein}. In particular, we use the fact that
the conjugate of a norm~$\Omega$ is the indicator function $\iota_{\Omega^\ast}$ of the unit ball of the dual norm $\Omega^*$ (see Proposition~\ref{prop:fenchel_dualnorm}).
\end{proof}
If $\vw^\star$ is a solution of Eq.~(\ref{eq:formulation}), and
$\vw,\vz$ in $\R^p$ are such that $\Omega^*(\vz) \leq \lambda$, this proposition implies that we have
\begin{equation}
   f(\vw)+ \lambda \Omega(\vw) \geq f(\vw^\star)+\lambda\Omega(\vw^\star) \geq  -f^*(\vz). \label{eq:dualitygap}
\end{equation}
The difference between the left and right term of Eq.~(\ref{eq:dualitygap}) is
called a duality gap.  It represents the difference between the value of the
primal objective function $f(\vw)+\lambda \Omega(\vw)$ and a dual objective
function $-f^*(\vz)$, where $\vz$ is a dual variable. The proposition says that the duality gap for a pair of optima~$\vw^\star$ and~$\vz^\star$ of the primal and dual problem is equal to $0$. When the optimal duality gap is zero one says that \emph{strong duality} holds. In our situation, the duality gap for the pair of primal/dual problems in Eq.~(\ref{eq:pd}), may be decomposed as the sum of two non-negative terms (as the consequence of Fenchel-Young inequality):
$$ \big( f(\vw) + f^*(\vz)  - \vw^\top \vz \big) + \lambda \big(  \Omega(\vw)  + \vw^\top (\vz/\lambda) + \iota_{\Omega^\ast}
( \vz / \lambda )
\big).
$$
It is equal to zero if and only if the two terms are simultaneously equal to zero.

Duality gaps are important in convex optimization because they provide an upper bound on the difference between the
current value of an objective function and the optimal value, which makes it possible to set proper
stopping criteria for iterative optimization algorithms. Given a current iterate $\vw$, computing a duality gap requires choosing a ``good" value for $\vz$ (and in particular a feasible one).
Given that at optimality, $\vz(\vw^\star)=\nabla f(\vw^\star)$ is the unique solution to the dual problem, a natural choice of dual variable is $\vz = \min \big  (1,\frac{\lambda}{\Omega^*(\nabla f(\vw))} \big )\nabla
f(\vw)$, which reduces to $\vz(\vw^\star)$ at the optimum and therefore yields a zero duality gap
at optimality.

Note that in most formulations that we will consider, the function~$f$
is of the form $f(\vw) = \psi(\mx\vw)$ with $\psi :\R^n \rightarrow \R$ and $\mx$ a design matrix.
Indeed, this corresponds to linear prediction on $\R^p$, given $n$ observations $\vx_i$, $i=1,\dots,n$, and the predictions $\mx \vw = ( \vw^\top \vx_i)_{i=1,\dots,n}$.
Typically, the Fenchel conjugate of $\psi$ is easy to compute\footnote{For the least-squares loss with output vector $\vy \in \mathbb{R}^n$, we have $\psi(\vu)=\frac{1}{2}\|\vy - \vu\|_2^2$ and $\psi^\ast(\vbeta) = \frac{1}{2} \| \vbeta\|_2^2 + \vbeta^\top \vy$.
For the logistic loss, we have $\psi(\vu) = \sum_{i=1}^n \log( 1 + \exp( - \vy_i \vu_i))$ and
$\psi^\ast(\vbeta) = \sum_{i=1}^n (1+\vbeta_i \vy_i) \log  (1+\vbeta_i \vy_i) - \vbeta_i \vy_i  \log(-\vbeta_i \vy_i )  $ if $\forall i, \ - \vbeta_i \vy_i \in [0,1]$ and $+\infty$ otherwise.
}  while the design matrix $\mx$ makes it hard\footnote{It would require to compute the pseudo-inverse of $\mx$.} to compute $f^*$.
In that case, Eq.~(\ref{eq:formulation}) can be rewritten as 
\begin{equation}
\label{eq:constraint}
\min_{\vw \in \R^p, \vu \in \R^n} \quad  \psi(\vu) + \lambda \: \Omega(\vw) \qquad \st \: \vu=\mx \vw,
\end{equation}
and equivalently as the optimization of the Lagrangian
\begin{equation*}
\min_{\vw \in \R^p, \vu \in \R^n} \quad \max_{\valpha \in \R^n} \quad   \psi(\vu)
+ \lambda \Omega(\vw)   + \lambda \valpha^\top \big(\mx \vw - \vu
\big),
\end{equation*}
\begin{equation}
\min_{\vw \in \R^p, \vu \in \R^n} \quad \max_{\valpha \in \R^n} \quad \big ( \psi(\vu) - \lambda \bmt{\valpha} \vu \big ) + \lambda \: \big (\Omega(\vw) + \valpha^\top \mx\vw \big ),
\end{equation}
which is obtained by introducing the Lagrange multiplier $\valpha$ for the constraint $\vu=\mx \vw$. The corresponding Fenchel dual\footnote{Fenchel conjugacy naturally extends to this case; see Theorem $3.3.5$ in \cite{borwein} for more details.} is then
\begin{equation}
\max_{\valpha \in \R^n} \quad -\psi^*(\lambda \valpha) \quad \text{such that} \quad \Omega^*(\mx^\top\valpha)\leq 1,
\label{eq:dualX}
\end{equation}
which does not require any inversion of $\mx^\top \mx$ (which would be required for computing the Fenchel conjugate of $f$). Thus, given a candidate $\vw$, we consider $\valpha = \min \big( 1, \frac{\lambda}{ \Omega^\ast(\mx^\top\nabla\psi(\mx \vw))  } \big) \nabla\psi(\mx \vw) $, and can get an upper bound on optimality using primal (\ref{eq:constraint}) and dual (\ref{eq:dualX})  problems.
Concrete examples of such duality gaps for various sparse regularized problems are presented in appendix D of~\cite{mairal_thesis}, and are implemented in the open-source software SPAMS\footnote{\url{http://www.di.ens.fr/willow/SPAMS/}}, which we have used in the experimental section of this paper.

\subsection{Quadratic Variational Formulation of Norms}\label{subsec:variational}
Several variational formulations are associated with norms, the most natural one being the one that results directly from \eqref{eq:def_dualnorm} applied to the dual norm:
$$\Omega(\vw)=\max_{ \vz \in \R^p} \vw^\top \vz \quad \st \quad \Omega^*(\vz) \leq 1.$$
However, another type of variational form is quite useful, especially for sparsity-inducing norms; among other purposes, as it is obtained by a variational upper-bound (as opposed to a lower-bound in the equation above), it leads to a general algorithmic scheme for learning problems regularized with this norm, in which the difficulties associated with optimizing the loss and that of optimizing the norm are partially decoupled. We present it in Section~\ref{sec:reweighted_l2}.
We introduce this variational form first for the $\ell_1$- and $\ell_1/\ell_2$-norms and
subsequently generalize it to norms that we call \emph{subquadratic norms}. 

\paragraph{The case of the $\ell_1$- and $\ell_1/\ell_2$-norms.}
The two basic variational identities we use are, for $a,b>0$, 
\begin{equation}
\label{eq:norm_varia}
2 ab=\inf_{\eta \in \R_{+}^*} \eta^{-1} a^2 + \eta\, b^2,
\end{equation}
where the infimum is attained at $\eta = a/b$, 
and, for $\va \in \R^p_+$,
\begin{equation}
\label{eq:sqnorm_varia}
\Big (\sum_{i=1}^p \va_i \Big )^2=\inf_{\veta \in (\R_{+}^*)^p} \sum_{i=1}^p \frac{\va_i^2}{\veta_i} \: \st \: \sum_{i=1}^p \veta_i=1.
\end{equation}
The last identity is a direct consequence of the Cauchy-Schwartz inequality:
\begin{equation}
\sum_{i=1}^p \va_i=\sum_{i=1}^p \frac{\va_i}{\sqrt{\veta_i}} \cdot \sqrt{\veta_i} \leq \Big (\sum_{i=1}^p \frac{\va_i^2}{\veta_i} \Big )^{1/2} \big ( \sum_{i=1}^p \veta_i\big)^{1/2}.
\end{equation}
The infima in the previous expressions can be replaced by a minimization if the function $q:\R \times \R_+ \rightarrow \R_+$ with $q(x,y)= \frac{x^2}{y}$ is extended in $(0,0)$ using the convention ``0/0=0", since the resulting function\footnote{This extension is in fact the function $\tilde{q}: (x,y) \mapsto \min \Big \{t \in \R_+ \mid 
\begin{bmatrix}
t & x \\ x & y
\end{bmatrix}
\succeq 0
\Big \}$.} is a proper closed convex function. We will use this convention implicitly from now on. The minimum is then attained when equality holds in the Cauchy-Schwartz inequality, that is for $\sqrt{\veta_i} \propto \va_i/\sqrt{\veta_i}$, which leads to $\veta_i=\frac{\va_i}{\|\va\|_1}$ if $\va\neq 0$ and $0$ else.

Introducing the simplex $\triangle_p=\{\veta \in \R^p_+ \mid \sum_{i=1}^p \veta_i=1\}$, 
we apply these variational forms to the $\ell_1$- and $\ell_1/\ell_2$-norms (with non overlapping groups) with $\|\vw\|_{\ell_1/\ell_2}=\sum_{g \in \Gc} \|\vw_g\|_2$ and $|\Gc|=m$, 
so that we obtain directly:
\begin{align*}
\!\!\!\!\!\!\!\!\!\!\!\!\!\!\!\!\|\vw\|_1&=\min_{\veta \in \R^p_+} \frac{1}{2}\sum_{i=1}^p \Big [ \frac{\vw_i^2}{\veta_i} +\veta_i \Big ], \quad &
\|\vw\|^2_1&=\min_{\veta \in \triangle_p} \sum_{i=1}^p \frac{\vw_i^2}{\veta_i},\\
\!\!\!\!\!\!\!\! \!\!\!\!\!\!\!\!\|\vw\|_{\ell_1/\ell_2}&=\min_{\veta \in \R^{m}_+} \frac{1}{2}\sum_{g \in \Gc} \Big [ \frac{\|\vw_g\|_2^2}{\veta^g} +\veta^g \Big ], \quad
& \|\vw\|_{\ell_1/\ell_2}^2&=\min_{\veta \in \triangle_{m}} \sum_{g \in \Gc} \frac{\|\vw_g\|_2^2}{\veta^g}.
\end{align*}

\paragraph{Quadratic variational forms for subquadratic norms.}
\label{sec:subqua}
The variational form of the $\ell_1$-norm admits a natural generalization for certain norms that we call \emph{subquadratic} norms. 
Before we introduce them, we review a few useful properties of norms. 
In this section, we will denote $|\vw|$ the vector $(|\vw_1|,\ldots,|\vw_p|)$.
\begin{definition}[Absolute and monotonic norm]
\label{def:monotonic}
We say that:\vspace{-3mm} 
\begin{itemize}
\item A norm $\Omega$ is \textbf{absolute} if for all $v \in \R^p$,  $\Omega(\vv)=\Omega(|\vv|)$.
\item A norm $\Omega$ is \textbf{monotonic} if for all $\vv,\vw \in \R^p$ s.t.\ $|\vv_i| \leq |\vw_i|, \: i=1, \ldots,p$, it holds that $\Omega(\vv) \leq \Omega(\vw)$.
\vspace{-3mm} 
\end{itemize}
\end{definition}
These definitions are in fact equivalent~(see, e.g.,~\cite{bauer1961absolute}):
\begin{proposition}
A norm is \emph{monotonic} if and only if it is \emph{absolute}.
\end{proposition}
\begin{proof}
If $\Omega$ is monotonic, the fact that $\big |\vv\big| = \big | |\vv| \big|$ implies $\Omega(\vv)=\Omega(|\vv|)$ so that $\Omega$ is absolute.

If $\Omega$ is absolute, we first show that $\Omega^*$ is absolute.
Indeed, $$\Omega^*(\vkappa)=\max_{\vw \in \R^p,\: \Omega(|\vw|) \leq 1} \vw^\top \vkappa=\max_{\vw \in \R^p,\: \Omega(|\vw|) \leq 1} |\vw|^\top |\vkappa| =\Omega^*(|\vkappa|).$$
 Then if $|\vv|\leq |\vw|$, since $\Omega^*(\vkappa)=\Omega^*(|\vkappa|)$,
 $$\Omega(\vv)=\max_{\vkappa \in \R^p,\: \Omega^*(|\vkappa|) \leq 1} |\vv|^\top |\vkappa| \leq \max_{\vkappa \in \R^p,\: \Omega^*(|\vkappa|) \leq 1} |\vw|^\top |\vkappa| =\Omega(\vw).$$
 which shows that $\Omega$ is monotonic.
\end{proof}
 
 We now introduce a family of norms, which have recently been studied in \cite{micchelli2011regularizers}.
\begin{definition}[$H$-norm] 
\label{def:Hnorm} Let $H$ be a compact convex subset of $\R_+^p$, such that $H \cap (\R_+^*)^p\neq \varnothing$, we say that $\Omega_H$ is an $H$-norm if $\Omega_H(\vw)=\min_{\veta \in H} \sum_{i=1}^p \frac{\vw_i^2}{\eta_i}$.
\end{definition}
The next proposition shows that $\Omega_H$ is indeed a norm and characterizes its dual norm.
\begin{proposition}
\label{prop:dual_Hnorm}
$\Omega_H$ is a norm and $\Omega_H^*(\vkappa)^2=\max_{\veta \in H} \sum_{i=1}^p \veta_i\,\vkappa_i^2$.
\label{lem:convex_var}
\end{proposition}
\begin{proof}
First, since $H$ contains at least one element whose components are all strictly positive, $\Omega$ is finite on $\R^p$.
Symmetry, nonnegativity and homogeneity of 
$\Omega_H$ are straightforward from the definitions.
Definiteness results from the fact that $H$ is bounded. 
$\Omega_H$ is convex, since it is obtained by minimization of $\veta$ in a jointly convex formulation.
Thus $\Omega_H$ is a norm.
Finally,
\begin{eqnarray*}
\frac{1}{2}\Omega_H^*(\vkappa)^2&=&\max_{\vw \in \R^p} \vw^\top
 \vkappa -\frac{1}{2}\Omega_H(\vw)^2\\
&=&\max_{\vw \in \R^p} \max_{\veta \in H} \vw^\top
 \vkappa -\frac{1}{2} \vw^\top \text{Diag}(\veta)^{-1}\vw.
 \end{eqnarray*}
The form of the dual norm follows by maximizing w.r.t. $\vw$.
\end{proof}

We finally introduce the family of norms that we call \emph{subquadratic}.

\begin{definition}[Subquadratic Norm]
Let $\Omega$ and $\Omega^*$ a pair of \emph{absolute} dual norms.
Let $\bar{\Omega}^*$ be the function defined as $\bar{\Omega}^*: \vkappa \mapsto [\Omega^*(|\vkappa|^{1/2})]^2$ where we use the notation $|\vkappa|^{1/2}=(|\vkappa_1|^{1/2}, \ldots, |\vkappa_p|^{1/2})^\top$. We say that $\Omega$ is \emph{subquadratic} if $\bar{\Omega}^*$ is convex.
\end{definition}
With this definition, we have:
\begin{lemma}
\label{lem:subqua_varia}
If 
$\Omega$ is \emph{subquadratic}, then $\bar{\Omega}^*$ is a norm, and denoting $\bar{\Omega}$ the dual norm of the latter, we have:
\begin{eqnarray*}
\Omega(\vw)&=&\frac{1}{2}\min_{\veta \in \R^p_+} \sum_{i} \frac{\vw_i^2}{\veta_i} + \bar{\Omega}(\veta)\\ 
\Omega(\vw)^2&=&\min_{\veta \in H} \sum_{i} \frac{\vw_i^2}{\veta_i} \mbox{ where } H = \{ \veta \in \R_+^p \mid \bar{\Omega}(\veta) \leq 1\}.
\end{eqnarray*}
\vspace*{-2mm}
\end{lemma}
\begin{proof}
Note that by construction, $\bar{\Omega}^*$ is homogeneous, symmetric and definite ($\bar{\Omega}^*(\vkappa)=0 \Rightarrow \vkappa=0$).
If $\bar{\Omega}^*$ is convex then $\bar{\Omega}^*(\frac{1}{2}(\vv+\vu)) \leq \frac{1}{2} \big ( \bar{\Omega}^*(\vv)+\bar{\Omega}^*(\vu)\big )$, which by homogeneity shows that $\bar{\Omega}^*$ also satisfies the triangle inequality. Together, these properties show that $\bar{\Omega}^*$ is a norm.
To prove the first identity we have, applying \eq{eq:norm_varia}, and since $\Omega$ is absolute,
\begin{eqnarray*}
\Omega(\vw)&=&\max_{\vkappa \in \R^p_+} \vkappa^\top |\vw| \quad \st \quad \Omega^*(\vkappa)\leq 1\\
&=&\max_{\vkappa \in \R^p_+} \sum_{i=1}^p \vkappa_i^{1/2} |\vw_i| \quad \st \quad \Omega^*(\vkappa^{1/2})^2\leq1\\
&=&\max_{\vkappa \in \R^p_+} \min_{\veta \in \R^p_+} \frac{1}{2} \sum_{i=1}^p \frac{\vw_i^2}{\veta_i} + \vkappa^\top \veta \quad \st \quad \bar{\Omega}^*(\vkappa)\leq1\\
&=& \min_{\veta \in \R^p_+} \max_{\vkappa \in \R^p_+}\frac{1}{2} \sum_{i=1}^p \frac{\vw_i^2}{\veta_i} + \vkappa^\top \veta \quad \st \quad \bar{\Omega}^*(\vkappa)\leq1, 
\end{eqnarray*}
which proves the first variational formulation (note that we can switch the order of the $\max$ and $\min$ operations because strong duality holds, which is due to the non-emptiness of the unit ball of the dual norm).
The second one follows similarly by applying \eq{eq:sqnorm_varia} instead of \eq{eq:norm_varia}.
\begin{eqnarray*}
\Omega(\vw)^2
&=&\max_{\vkappa \in \R_+^p} \big(\sum_{i=1}^p \vkappa_i^{1/2} |\vw_i|\big)^2 \quad \st \quad \Omega^*(\vkappa^{1/2})^2\leq1\\
&=&\max_{\vkappa \in \R^p_+} \min_{\tilde{\veta} \in \R^p_+} \sum_{i=1}^p \frac{\vkappa_i \vw_i^2}{\tilde{\veta}_i} \quad \st \quad \sum_{i=1}^p \tilde{\veta}_i=1, \: \bar{\Omega}^*(\vkappa)\leq 1\\
&=&\max_{\vkappa \in \R^p_+} \min_{\veta \in \R^p_+}   \sum_{i=1}^p \frac{\vw_i^2}{\veta_i} \quad \st \quad \veta^\top\vkappa=1, \: \bar{\Omega}^*(\vkappa)\leq1.\\
\end{eqnarray*}
\end{proof}
 Thus, given a subquadratic norm, we may define a convex set $H$, namely the intersection of the unit ball of $\bar{\Omega}$ with the positive orthant $\R_+^p$, such that $\Omega(\vw)^2=\min_{\veta \in H} \sum_{i=1}^p \frac{\vw_i^2}{\veta_i} $, i.e., a subquadratic norm is an $H$-norm. We now show that these two properties are in fact equivalent.
\begin{proposition}
$\Omega$ is \emph{subquadratic} if and only if it is an $H$-norm.
\end{proposition}
\begin{proof}
The previous lemma show that subquadratic norms are $H$-norms. Conversely, let $\Omega_H$ be an $H$-norm. By construction, $\Omega_H$ is absolute, and as a result of Prop.~\ref{prop:dual_Hnorm}, $\bar{\Omega}^*_H(\vw)=\big ( \Omega^*_H(|\vw|^{1/2}) \big )^2=\max_{\veta \in H} \sum_i \veta_i|\vw_i|$, which shows that $\bar{\Omega}^*_H$ is a convex function, as a maximum of convex functions.
\end{proof}

It should be noted that the set $H$ leading to a given $H$-norm $\Omega_H$ is not unique; in particular $H$ is not necessarily the intersection of the unit ball of a norm with the positive orthant. Indeed, for the $\ell_1$-norm, we can take $H$ to be the unit simplex.

\begin{proposition}
Given a convex compact set $H$, let $\Omega_H$ be the associated $H$-norm and $\bar{\Omega}_H$ as defined in Lemma~\ref{lem:subqua_varia}. Define the mirror image of $H$ as the set $\text{Mirr}(H)=\{\vv \in \R^p \mid |\vv| \in H\}$ and denote the convex hull of a set $S$ by $\text{Conv}(S)$. Then the unit ball of $\bar{\Omega}_H$ is $\text{Conv}(\text{Mirr}(H))$.
\end{proposition}
\begin{proof}
By construction:
\begin{eqnarray*}
\bar{\Omega}^*_H(\vkappa) & = & \Omega^*_H(|\vkappa|^{1/2})^2
=\max_{\veta \in H} \veta^\top |\vkappa| \\
&= &  \max_{|\vw| \in H} \vw^\top\vkappa=\max_{\vw \in \text{Conv}(\text{Mirr}(H))} \vw^\top\vkappa,
\end{eqnarray*}
since the maximum of a convex function over a convex set is attained at its extreme points.
But $C=\text{Conv}(\text{Mirr}(H))$ is by construction a centrally symmetric convex set, which is bounded and closed like $H$, and whose interior contains $0$ since $H$ contains at least one point whose components are strictly positive. This implies by Theorem 15.2 in \cite{rockafellar97} that $C$ is the unit ball of a norm (namely $\vx \mapsto \inf \{\lambda \in \R_+ \mid \vx \in \lambda C\}$), which by duality has to be the unit ball of $\bar{\Omega}_H$.
\end{proof}
This proposition combined with the result of Lemma~\ref{lem:subqua_varia} therefore shows that if $\text{Conv}(\text{Mirr}(H))=\text{Conv}(\text{Mirr}(H'))$ then $H$ and $H'$ define the same norm.

Several instances of the general variational form we considered in this section have appeared in the literature~\cite{jenatton2010sspca,pontil,simpleMKL}. 
For norms that are not subquadratic, it is often the case that their dual norm is itself subquadratic, in which case symmetric variational forms can be obtained~\cite{aflalo2011variable}.
Finally, we show in Section \ref{sec:reweighted_l2} that all norms admit a quadratic variational form provided the bilinear form considered is allowed to be non-diagonal.

\section{Multiple Kernel Learning}
\label{sec:mkl}

A seemingly unrelated problem in machine learning, the problem of \emph{multiple kernel learning} is in fact intimately connected with sparsity-inducing norms by duality. It actually corresponds to the most natural extension of sparsity to reproducing kernel Hilbert spaces.
We will show that for a large class of norms and, among them, many sparsity-inducing norms, there exists for each of them a corresponding multiple kernel learning scheme, and, vice-versa, each multiple kernel learning scheme defines a new norm.

The problem of kernel learning is a priori quite unrelated with parsimony.
It emerges as a consequence of a convexity property of the so-called ``kernel trick", which we now describe.
Consider a learning problem with $f(\vw)=\psi(\mx\vw)$. As seen before, this corresponds to linear predictions of the form $\mx \vw = ( \vw^\top \vx_i)_{i=1,\dots,n}$. Assume that this learning problem is this time regularized this time by the square of the norm $\Omega$ (as shown in Section~\ref{sec:losses}, this does not change the regularization properties), so the we have the following optimization problem:
\begin{equation}
\label{eq:primal_learn}
\min_{\vw\in \R^p} f(\vw)+\frac{\lambda}{2} \Omega(\vw)^2.
\end{equation}
As in Eq.~\eqref{eq:constraint} we can introduce the linear constraint
\begin{equation}
\label{eq:constraint2}
\min_{\vu \in \R^n,\vw\in \R^p} \psi(\vu)+\frac{\lambda}{2} \Omega(\vw)^2 \quad \st \quad \vu=\mx \vw,
\end{equation}
and reformulate the problem as the saddle point problem
\begin{equation}
\min_{\vu \in \R^n,\vw\in \R^p} \max_{\alpha \in \R^n} \psi(\vu)+\frac{\lambda}{2} \Omega(\vw)^2 - \lambda \valpha^\top (\vu-\mx \vw).
\end{equation}
Since the primal problem (\ref{eq:constraint2}) is a convex problem with feasible linear constraints, it satisfies Slater's qualification conditions and the order of maximization and minimization can be exchanged:
\begin{equation}
\max_{\valpha \in \R^n} \min_{\vu \in \R^n,\vw\in \R^p}  (\psi(\vu)- \lambda \valpha^\top \vu)+ \lambda \big ( \frac{1}{2} \Omega(\vw)^2 +\valpha^\top\mx \vw ).
\end{equation}
Now, the minimization in $\vu$ and $\vw$ can be performed independently.
One property of norms is that the Fenchel conjugate of $\vw \mapsto \frac{1}{2} \Omega(\vw)^2$ is $\vkappa \mapsto \frac{1}{2} \Omega^*(\vkappa)^2$; 
this can be easily verified by finding the vector $\vw$ achieving equality in the sequence of inequalities
 $\vkappa^\top \vw \leq \Omega(\vw) \, \Omega^*(\vkappa) \leq \frac{1}{2} \big [ \Omega(\vw)^2+\Omega^*(\vkappa)^2 \big ]$. 
As a consequence, the dual optimization problem is
\begin{equation}
\max_{\valpha \in \R^n} - \psi^*(\lambda \valpha) -\frac{\lambda}{2} \Omega^*(\mx^\top \valpha)^2.
\end{equation}
If $\Omega$ is the Euclidean norm (i.e., the $\ell_2$-norm) then the previous problem is simply
\begin{equation}
\label{eq:dual_learn}
G(\mk) \defin \max_{\valpha \in \R^n} - \psi^*(\lambda \valpha) -\frac{\lambda}{2} \valpha^\top \mk \valpha \quad \text{with} \quad \mk=\mx \mx^\top.
\end{equation}
Focussing on this last case, a few remarks are crucial: 
\begin{enumerate}
\item The dual problem depends on the design $\mx$ only through the kernel matrix $\mk =\mx \mx^\top \in \mathbb{R}^{n \times n}$.
\item $G$ is a \emph{convex} function of $\mk$ (as a maximum of linear functions).
\item The solutions $\vw^{\star}$ and $\valpha^{\star}$ to the primal and dual problems satisfy $\vw^{\star}=\mx^\top \valpha^{\star}=\sum_{i=1}^n \valpha_i^{\star} \vx_i$.
\item The exact same duality result applies for the generalization to $\vw, \vx_i \in \Hc$ for $\Hc$ a Hilbert space.\\
\end{enumerate}

The first remark suggests a way to solve learning problems that are non-linear in the inputs $\vx_i$: in particular consider a non-linear mapping~$\phi$ which maps $\vx_i$ to a high-dimensional $\phi(\vx_i) \in \Hc$ with $ \Hc=\R^d$ for $d \gg p$ or possibly an infinite dimensional Hilbert space. Then consider the problem (\ref{eq:primal_learn}) with now $f(\vw)=\psi \big( ( \langle\vw, \phi(\vx_i)\rangle)_{i=1,\dots,n} \big)$, which is typically of the form of an empirical risk
$f(\vw)=
\frac{1}{n} \sum_{i=1}^n \ell( y^{(i)}, \langle\vw, \phi(\vx_i)\rangle)$. It  becomes high-dimensional to solve in the primal, while it is simply solved in the dual by choosing a kernel matrix with entries $\mk_{i,j}=\langle \phi(\vx_i),\phi(\vx_j) \rangle$, which is advantageous as soon as $n^2\leq d$; this is the so-called ``kernel trick'' (see more details in~\cite{scholkopf-smola-book,Shawe-Taylor2004}).

In particular if we consider functions $h \in \Hc$ where $\Hc$ is a reproducing kernel Hilbert space (RKHS) with reproducing kernel $K$ then 
\begin{equation}
\label{eq:rkhs_learn}
\min_{h \in \Hc} \psi \big( (h(\vx_i))_{i=1,\dots,n} \big) +\frac{\lambda}{2}\|h\|^2_{\Hc}
\end{equation} 
is solved by solving Eq.~\eq{eq:dual_learn} with $\mk_{i,j}=K(\vx_i,\vx_j)$.  
When applied to the mapping $\phi: \vx \mapsto K(\vx,\cdot)$, the third remark above yields a specific version of the  representer theorem of Kimmeldorf and Wahba~\cite{representer}\footnote{Note that this provides a proof of the representer theorem for \emph{convex} losses only and that the parameters $\valpha$ are obtained through a dual \emph{maximization} problem.} stating that
$h^{\star}(\cdot)=\sum_{i=1}^n {\valpha^{\star}_i} K(\vx_i,\cdot)$. In this case, the predictions may be written equivalently as $
h(\vx_i)$ or $\langle\vw, \phi(\vx_i)\rangle$, $i=1,\dots,n$. 

As shown in~\cite{gert}, the fact that $G$ is a convex function of $\mk$ suggests the possibility of optimizing the objective with respect to the choice of the kernel itself by solving a problem of the form $\min_{\mk \in \Kc} G(\mk)$ where $\Kc$ is a convex set of kernel matrices.
 
In particular, given a finite set of kernel functions $(K_i)_{1\leq i\leq p}$ it is natural to consider to find the best \emph{linear} combination of kernels, which requires to add a positive definiteness constraint on the kernel, leading to a semi-definite program~\cite{gert}:
\begin{equation}
\label{eq:sdp}
\!\!\!\!\!\!\!\! \min_{\veta \in \R^p} G({\textstyle \sum_{i=1}^p \veta_i \mk_i}) \: \st \: {\textstyle \sum_{i=1}^p \eta_i \mk_i} \succeq 0, \: {\rm tr}( {\textstyle \sum_{i=1}^p \veta_i \mk_i}) \leq 1.
\end{equation}
Assuming that the kernels have equal trace, the two constraints of the previous program are avoided by considering convex combinations of kernels, which leads to a quadratically-constrained quadratic program (QCQP)~\cite{genomic_fusion}:
\begin{equation}
\displaystyle \min_{\veta \in \R^p_+} G({\textstyle \sum_{i=1}^p \veta_i \mk_i}) \quad \st \quad {\textstyle \sum_{i=1}^p \veta_i=1}.
\label{eq:mkl1}
\end{equation}
We now present a reformulation of Eq.~\eq{eq:mkl1} using sparsity-inducing norms (see~\cite{Bach2008a,skm, simpleMKL} for more details).

\subsection{From $\ell_1/\ell_2$-Regularization to MKL}
As we presented it above, MKL arises from optimizing the objective of a learning problem w.r.t. to a convex combination of kernels, in the context of plain $\ell_2$- or Hilbert norm regularization, which seems a priori unrelated to sparsity. We will show in this section that, in fact, the primal problem corresponding exactly to MKL (i.e., Eq.~\ref{eq:mkl1}) is an $\ell_1/\ell_2$-regularized problem (with the $\ell_1/\ell_2$- norm defined in Eq.~\eq{eq:def_omega_group}), in the sense that its dual is the MKL problem for the set of kernels associated with each of the groups of variables.
The proof to establish the relation between the two relies on the variational formulation presented in Section~\ref{sec:subqua}.

We indeed have, assuming that $\Gc$ is a partition of $\{1, \ldots,p\}$, with $|\Gc|=m$,
and $\triangle_{m}$ denoting the simplex in $\R^m$,
\begin{eqnarray*}
\min_{\vw \in \R^p}  & & \psi(\mx \vw)  +  \frac{\lambda}{2} \big ( {\textstyle \sum_{g \in \Gc}} \|\vw_g\|_2 \big )^2 \\
=\min_{\vw \in \R^p, \veta \in \triangle_{m}} & &  \psi(\mx \vw)  +  \frac{\lambda}{2}\sum_{g \in \Gc} \frac{\|\vw_g\|_2^2}{\veta_g}\\
=\min_{\widetilde{\vw} \in \R^p, \veta \in \triangle_{m}} & & \psi({\textstyle \sum_{g \in \Gc} \veta_g^{1/2} \mx_g \widetilde{\vw}_g})  +  \frac{\lambda}{2}  \sum_{g \in \Gc} \|\widetilde{\vw}_g\|_2^2\\
=\min_{\widetilde{\vw} \in \R^p, \veta \in \triangle_{m}} & & \psi(\widetilde{\mx} \widetilde{\vw})  +  \frac{\lambda}{2} \|\widetilde{\vw}\|_2^2 \quad \st \widetilde{\mx}=[\eta_{g_1}^{1/2} \mx_{g_1}, \ldots, \eta_{g_m}^{1/2} \mx_{g_m}]\\
=\min_{\veta \in \triangle_{m}} \max_{\valpha \in \R^n} & & - \psi^*(\lambda \valpha )  -  \frac{\lambda}{2}\valpha^\top \big ({\textstyle \sum_{g \in \Gc} \veta_g \mk_g }\big ) \valpha\\
=\min_{\veta \in \triangle_{m}} & & G({\textstyle \sum_{g \in \Gc} \veta_g \mk_g}),
\end{eqnarray*}
where the third line results from the change of variable $\widetilde{\vw}_g \veta_g^{1/2}=\vw_g$, and the last step from the definition of $G$ in Eq.~\eq{eq:dual_learn}.

Note that $\ell_1$-regularization corresponds to the special case where groups are singletons and where
$\mk_i=\vx_i \vx_i^\top$ is a rank-one kernel matrix. In other words, MKL with rank-one kernel matrices (i.e., feature spaces of dimension one) is equivalent to $\ell_1$-regularization (and thus simpler algorithms can be brought to bear in this situation).

We have shown that learning convex combinations of kernels through Eq.~\eq{eq:mkl1} turns out to be equivalent to an $\ell_1/\ell_2$-norm penalized problems. In other words, learning a linear combination $\textstyle \sum_{i=1}^m \veta_i \mk_i$ of kernel matrices, subject to $\veta$ belonging to the simplex $\triangle_{m}$ is equivalent to penalizing the empirical risk with an $\ell_1$-norm applied to norms of predictors $\| \vw_g \|_2$. This link between the $\ell_1$-norm and the simplex may be extended to other norms, among others to the subquadratic norms introduced in Section~\ref{sec:subqua}.

\subsection{Structured Multiple Kernel Learning}
\label{sec:struct_mkl}
In the relation established between $\ell_1/\ell_2$-regularization and MKL in the previous section, the vector of weights $\veta$ for the different kernels corresponded with the vector of optimal variational parameters defining the norm.
A natural way to extend MKL is, instead of considering a convex combination of kernels, to consider a linear combination of the same kernels, but with positive weights satisfying a different set of constraints than the simplex constraints. Given the relation between kernel weights and the variational form of a norm, we will be able to show that, for norms that   
 have a variational form as in \lem{lem:convex_var}, we can generalize the correspondence between the $\ell_1/\ell_2$-norm and MKL to  a correspondence between other structured norms and structured MKL schemes. 
 
Using the same line of proof as in the previous section, and given an $H$-norm (or equivalently a subquadratic norm) $\Omega_H$  as defined in Definition~\ref{def:Hnorm}, we have:
\begin{eqnarray}
\label{eq:struct_mkl}
\min_{\vw \in \R^p}  & & \psi(\mx \vw)  +  \frac{\lambda}{2} \Omega_H(\vw)^2 \\
=\min_{\vw \in \R^p, \veta \in H} & &  \psi(\mx \vw)  +  \frac{\lambda}{2}\sum_{i=1}^p \frac{\vw_i^2}{\veta_i}\notag\\
=\min_{\widetilde{\vw} \in \R^p, \veta \in H} & & \psi({\textstyle \sum_{i=1}^p \veta_i^{1/2} \mx_i \widetilde{\vw}_i})  +  \frac{\lambda}{2}  \sum_{i=1}^p \widetilde{\vw}_i^2\notag\\
=\min_{\widetilde{\vw} \in \R^p, \veta \in H} & & \psi(\widetilde{\mx} \widetilde{\vw})  +  \frac{\lambda}{2} \|\widetilde{\vw}\|_2^2 \quad \st \widetilde{\mx}=[\veta_{1}^{1/2} \mx_{1}, \ldots, \veta_{p}^{1/2} \mx_{p}] \notag\\
=\min_{\veta \in H} \max_{\valpha \in \R^n} & & - \psi^*(\lambda \valpha )  -  \frac{\lambda}{2}\valpha^\top \big ({\textstyle \sum_{i=1}^p \veta_i \mk_i }\big ) \valpha \notag\\
=\min_{\veta \in H} & & G({\textstyle \sum_{i=1}^p \veta_i \mk_i}). \notag
\end{eqnarray}

This results shows that the regularization with the norm $\Omega_H$ in the primal is equivalent to a multiple kernel learning formulation in which the kernel weights are constrained to belong to the convex set $H$, which defines $\Omega_H$ variationally. Note that we have assumed that $H \subset \R_+^p$, so that formulations such as \eq{eq:sdp}, where positive semidefiniteness of $\sum_{i=1}^p \veta_i \mk_i$ has to be added as a constraint, are not included.

Given the relationship of MKL to the problem of learning a function in a reproducing kernel Hilbert space, 
the previous result suggests a natural extension of structured sparsity to the RKHS settings.
Indeed let, $h=(h_1, \ldots, h_p) \in \Bc\defin \Hc_1\times \ldots\times\Hc_p$, where $\Hc_i$ are RKHSs. It is easy to verify that $\Lambda:h \mapsto \Omega_H \big ( (\|h_1\|_{\Hc_1}, \ldots,\|h_p\|_{\Hc_p}))$ is a convex function, using the variational formulation of $\Omega_H$, and since it is also non-negative definite and homogeneous, it is a norm\footnote{As we show in Section \ref{sec:prox_mkl}, it actually sufficient to assume that $\Omega$ is monotonic for $\Lambda$ to be a norm.}.  Moreover, the learning problem obtained by summing the predictions from the different RKHSs, i.e., 
\begin{equation}
\label{eq:structmkl_learn}
\min_{h \in \Bc} \psi \big( (h_1(\vx_i)+\ldots+h_p(\vx_i))_{i=1,\dots,n} \big)+\frac{\lambda}{2} \Omega_H \big ( (
\|h_1\|_{\Hc_1}, \ldots,\|h_p\|_{\Hc_p})
\big)^2
\end{equation}
 is equivalent, by the above derivation, to the MKL problem $\min_{\veta \in H} G({\textstyle \sum_{i=1}^p \veta_i \mk_i})$ with $[\mk_i]_{j,j'}=K_i(\vx_j,\vx_{j'})$ for $K_i$ the reproducing kernel of $\Hc_i$. See Section~\ref{sec:prox_mkl} for more details.

This means that, for most of the structured sparsity-inducing norms that we have considered in Section~\ref{sec:norms}, we may replace individual variables by whole Hilbert spaces. For example, tree-structured sparsity (and its extension to directed acyclic graphs) was explored in~\cite{hkl} where each node of the graph was a RKHS, with an application to non-linear variable selection.

\chapter{Generic Methods}
\label{sec:opt_methods_classical}

The problem defined in Eq.~(\ref{eq:formulation}) is convex, as soon as both the loss $f$ and the regularizer $\Omega$ are convex functions. In this section, we consider optimization strategies which are essentially blind to problem structure. The first of these techniques is subgradient descent (see, e.g., \cite{bertsekas}), which is widely applicable, has low running time complexity per iterations, but has a slow convergence rate. As opposed to proximal methods presented in Section~\ref{sec:proximal}, it does not use problem structure. At the other end of the spectrum, the second strategy
is to consider reformulations such as linear programs (LP), quadratic programs (QP) or more generally, second-order cone programming (SOCP) or semidefinite programming (SDP) problems~(see, e.g., \cite{boyd.convex}).
The latter strategy is usually only possible with the square loss and makes use of general-purpose optimization toolboxes. Moreover, these toolboxes are only adapted to small-scale problems and usually lead to solution with very high precision (low duality gap), while simpler iterative methods can be applied to large-scale problems but only leads to solution with low precision, which is sufficient in most applications to machine learning~(see~\cite{bottou2008tradeoffs} for a detailed discussion).

\paragraph{Subgradient descent.}
For all convex unconstrained problems, subgradient descent can be used as soon as one subgradient can be computed efficiently. In our setting, this is possible when a subgradient of the loss $f$, and a subgradient of the regularizer $\Omega$ can be computed. This is true for all the norms that we have considered. The corresponding algorithm consists of the following iterations:
$$
\vw_{t+1} = \vw_{t} - \frac{\alpha}{t^\beta} ( \vs + \lambda \vs' ), \mbox{ where } \vs \in \partial f(\vw_t), \ \vs' \in  \partial \Omega(\vw_t),
$$
with $\alpha$ a well-chosen positive parameter and $\beta$ typically $1$ or $1/2$. Under certain conditions, these updates are globally convergent. More precisely, we have, from~\cite{nesterov2004introductory},
$F(\vw_t) - \min_{\vw \in \R^p} F(\vw) = O(\frac{\log t}{\sqrt{t}})$ for Lipschitz-continuous function and $\beta=1/2$.
However, the convergence is in practice slow (i.e., many iterations are needed), and the solutions obtained are usually not sparse. This is to be contrasted with the proximal methods presented in the next section which are less generic but more adapted to sparse problems, with in particular convergence rates in $O(1/t)$ and $O(1/t^2)$.

\paragraph{Reformulation as LP, QP, SOCP, SDP.}
 For all the sparsity-inducing norms we consider in this paper the corresponding regularized least-square problem can be represented by standard mathematical programming problems, all of them being SDPs, and often simpler (e.g., QP). For example, for the $\ell_1$-norm regularized least-square regression, we can reformulate
$\min_{\vw \in \R^p} \frac{1}{2n} \| \vy - \mx \vw \|_2^2 + \lambda \Omega(\vw)$ as
$$
\min_{ \vw_+, \vw_- \in \R_+^p} \frac{1}{2n} \| \vy-\mx\vw_+ + \mx\vw_- \|_2^2 + \lambda (  1^\top \vw_+ + 1^\top \vw_-),
$$
which is a quadratic program. Grouped norms with combinations of $\ell_2$-norms leads to an SOCP, i.e., 
$\min_{w \in \R^p} \frac{1}{2n} \| \vy - \mx \vw \|_2^2 + \lambda \sum_{ g \in \mathcal{G}} d_g \| \vw_g \|_2$ may be formulated as
$$
\min_{ \vw  \in \R^p, \ (\vt_g)_{g \in \mathcal{G}} \in \R_+^{|\mathcal{G}|}} \frac{1}{2n} \| \vy-\mx \vw \|_2^2 + \lambda \sum_{g \in \mathcal{G}} d_g \vt_g 
\mbox{ s.t } \forall g \in \mathcal{G}, \ \|\vw_g\|_2 \leq \vt_g.$$
 Other problems can be similarly cast (for the trace norm, see~\cite{tracenorm,fazel}).
General-purpose toolboxes can then be used, to get solutions with high precision (low duality gap). However, in the context of machine learning, this is inefficient for two reasons: (1) these toolboxes are generic and blind to problem structure and tend to be too slow, or cannot even run because of memory problems, (2) as outlined in~\cite{bottou2008tradeoffs}, high precision is not necessary for machine learning problems, and a duality gap of the order of machine precision (which would be a typical result from such toolboxes) is not necessary.

We present in the following sections methods that are adapted to problems regularized by sparsity-inducing norms.

\chapter{Proximal Methods}
\label{sec:proximal_methods}
\label{sec:opt_methods_prox}

This chapter reviews a class of techniques referred to as \textit{proximal methods},
where, the non-smooth component of the objective \eqref{eq:formulation} will 
only be involved in the computations through an associated \textit{proximal operator}, which we formally define subsequently. 

The presentation that we make of proximal methods in this chapter is deliberately
simplified, and to be rigorous the methods that we will refer to as proximal methods in this section
are known as \textit{forward-backward splitting} methods. We refer the interested reader to Section~\ref{sec:extensions}
for a broader view and references.

\section{Principle of Proximal Methods}
\label{sec:proximal}

Proximal methods (i.e., forward-backward splitting methods) are specifically tailored to optimize an objective of the form \eqref{eq:formulation}, i.e., 
which can be written as the sum of a generic smooth differentiable function $f$ with Lipschitz-continuous gradient, 
and a non-differentiable function $\lambda \Omega$.

They have drawn increasing attention in
the machine learning community, especially because of their convergence rates and their ability to deal with large nonsmooth convex problems
(e.g., \cite{beck2009fast,combettes2010proximal,nesterov2007gradient,wright2009sparse}).
\\
   Proximal methods can be described as follows: at each iteration the function $f$ is linearized around the current point and a problem of the form
\begin{equation}
\label{eq:prox_linearization}
\min_{\vw \in \R^p} \: f(\vw^t)\! +\! \nabla\! f(\vw^t)^\top (\vw - \vw^t)  + \lambda \Omega(\vw) + { \frac{L}{2}}\|\vw - \vw^t\|_2^2
\end{equation}
is solved. The quadratic term, called proximal term, keeps the update in a neighborhood of the current iterate $\vw^t$ where $f$ is close to its linear approximation; $L\!>\!0$ is a parameter, which should essentially be an upper bound on the Lipschitz constant of $\nabla f$ and is typically set with a line-search. This problem can be rewritten as
\begin{equation}
\label{eq:prox_least_square}
\min_{\vw \in \R^p} \: \frac{1}{2}\big\|\vw - \big(\, \vw^t - {\frac{1}{L}}\nabla\! f(\vw^t) \, \big) \big\|_2^2 \, + \, {\frac{\lambda}{L}} \Omega(\vw).
\end{equation}
It should be noted that when the nonsmooth term $\Omega$ is not present, the solution of the previous
proximal problem, also known as the backward or implicit step, 
just yields the standard gradient update rule $\vw^{t+1} \leftarrow \vw^t - {\frac{1}{L}}\nabla\! f(\vw^t)$. Furthermore, if $\Omega$ is the indicator function of a set $\iota_C$, i.e., defined by $\iota_C(x)=0$ for $x \in C$ and $\iota_C(x)=+\infty$ otherwise, then solving (\ref{eq:prox_least_square}) yields the projected gradient update with projection on the set $C$. This suggests that the solution of the proximal problem provides an interesting generalization of gradient updates, and motivates the introduction of the notion of a \textit{proximal operator} associated with the regularization term $\lambda\Omega$.

The proximal operator, which we will denote $\text{Prox}_{\mu \Omega}$, was defined in \cite{moreau1962fonctions} as the function that maps a vector $\vu \in \R^{p}$ to the unique\footnote{Since the objective is strongly convex.} solution of
\begin{equation}
\label{eq:prox_problem}
\min_{\vw \in \R^{p}} \: \frac{1}{2} \|\vu-\vw\|_2^2 + \mu \, \Omega(\vw).
\end{equation}
This operator is clearly central to proximal methods since their main step consists in computing
$\displaystyle
\text{Prox}_{\frac{\lambda}{L} \Omega} \big (\vw^t - \textstyle {\frac{1}{L}} \nabla f(\vw^t) \big ).$

In Section \ref{sec:proximal_operator}, we present analytical forms of proximal operators associated with simple norms and algorithms to compute them in some more elaborate cases. Note that the proximal term in Eq.~(\ref{eq:prox_linearization}) could be replaced by any Bregman divergences~(see, e.g.,~\cite{tseng2008accelerated}), which may be useful in settings where extra constraints (such as non-negativity) are added to the problem.

\section{Algorithms}
The basic proximal algorithm
uses the solution of problem (\ref{eq:prox_least_square}) as the next update $\vw^{t+1}$; 
however fast variants such as the accelerated algorithm presented in \cite{nesterov2007gradient} or FISTA \cite{beck2009fast}
maintain two variables and use them to combine at marginal extra computational cost the solution of (\ref{eq:prox_least_square}) with information about previous steps. Often, an upper bound on the Lipschitz constant of $\nabla f$ is not known, and even if it is\footnote{For problems common in machine learning where $f(\vw) = \psi( \mx \vw)$ and $\psi$ is twice differentiable, then $L$ may be chosen to be the largest eigenvalue of $\frac{1}{n} \mx^\top \mx$ times the supremum over $\vu \in \R^n$ of the largest eigenvalue of the Hessian of $\psi$ at $\vu$.}, it is often better to obtain a local estimate. 
A suitable value for $L$ can be obtained by iteratively increasing $L$ by a constant factor until the condition
\begin{equation}
\label{eq:bound_L}
f(\vw^\star_L) \leq 
 f(\vw^t) + \nabla\! f(\vw^t)^\top (\vw^{\star}_L - \vw^t) + {\frac{L}{2} } \|\vw^{\star}_L - \vw^t\|_2^2
\end{equation}
is met, where $\vw^\star_L$ denotes the solution of (\ref{eq:prox_problem}).

For functions $f$ whose gradients are Lipschitz-continuous, the basic proximal algorithm has a global convergence rate in $O(\frac{1}{t})$ where $t$ is the number of iterations of the algorithm. 
Accelerated algorithms like FISTA can be shown to have global convergence rate---\textit{on the objective function}---in $O(\frac{1}{t^2})$, which has been proved to be optimal for the class of first-order techniques~\cite{nesterov2004introductory}.

Note that, unlike for the simple proximal scheme, we cannot guarantee that the sequence of iterates generated by the accelerated version is itself convergent~\cite{combettes2010proximal}.

Perhaps more importantly, both basic (ISTA) and accelerated \cite{nesterov2007gradient} proximal methods are adaptive in the sense that if $f$ is strongly convex---and the problem is therefore better conditioned---the convergence is actually linear (i.e., with rates in $O(C^t)$ for some constant $C < 1$; see \cite{nesterov2007gradient}).
Finally, it should be noted that accelerated schemes
are not necessarily descent algorithms, in the sense that the objective does not necessarily decrease at each iteration in spite of the global convergence properties.

\section{Computing the Proximal Operator}
\label{sec:proximal_operator}
Computing the \textit{proximal operator} efficiently and exactly allows to attain the 
fast convergence rates of proximal methods.\footnote{Note, however, that fast convergence rates can also be achieved while solving approximately the 
proximal problem, as long as the precision of the approximation iteratively
increases with an appropriate rate (see \cite{schmidt2011} for more details).}
We therefore focus here on properties of this operator and on its computation
for several sparsity-inducing norms. 
For a complete study of the properties of the proximal operator, we refer the interested reader to~\cite{combettes2006signal}.

\paragraph{Dual proximal operator.}
 In the case where~$\Omega$ is a norm, by Fenchel duality, the following problem is dual (see Proposition \ref{prop:duagap}) to problem (\ref{eq:prox_least_square}):
\begin{equation}
\label{eq:dual_prox}
\max_{\vv \in \R^p} \: -\frac{1}{2} \left [ \|\vv-\vu \|_2^2 - \|\vu\|^2 \right ] \quad \text{such that} \quad \Omega^*(\vv) \leq \mu.
\end{equation}

\begin{lemma}[Relation to Dual Proximal Operator]
Let $\text{Prox}_{\mu \Omega}$ be the proximal operator associated with the regularization $\mu \Omega$, where~$\Omega$ is a norm, and $\text{Proj}_{ \{ \Omega^*(\cdot) \leq \mu \}}$ be the projector on the ball of radius $\mu$ of the dual norm~$\Omega^*$. Then $\text{Proj}_{ \{ \Omega^*(\cdot) \leq \mu \}}$ is the proximal operator for the dual problem (\ref{eq:dual_prox}) and, denoting the identity $I_d$, these two operators satisfy the relation
\begin{equation}
\label{eq:prox_proj}
\text{Prox}_{\mu \Omega}=I_d - \text{Proj}_{ \{ \Omega^*(\cdot) \leq \mu \}}.
\end{equation}
\proof
By Proposition \ref{prop:duagap}, if $\vw^\star$ is optimal for (\ref{eq:prox_problem}) and $\vv^\star$ is optimal for (\ref{eq:dual_prox}), we have\footnote{The dual variable from Fenchel duality is $-\vv$ in this case.} $-\vv^\star=\nabla f(\vw^\star)=\vw^\star-\vu$. Since $\vv^\star$ is the projection of $\vu$ on the ball of radius $\mu$ of the norm $\Omega^*$, the result follows.
\end{lemma}
This lemma shows that the proximal operator can always be computed as the residual of a Euclidean projection onto a convex set. More general results appear in~\cite{combettes2006signal}. 

\paragraph{$\ell_1$-norm regularization.} Using optimality conditions for (\ref{eq:dual_prox}) and then \eqref{eq:prox_proj} or subgradient condition \eqref{eq:opt} applied to (\ref{eq:prox_problem}), it is easy to check that $\text{Proj}_{\{ \|\cdot\|_\infty \leq \mu\}}$ and  $\text{Prox}_{\mu \|\cdot\|_1}$  respectively satisfy:

$$\textstyle \big [ \text{Proj}_{\{ \|\cdot\|_\infty \leq \mu\}}(\vu) \big ]_j= \min \big (1,\frac{\mu}{|\vu_j|} \big ) \, \vu_j,$$
and
$$\big [ \text{Prox}_{\mu \|\cdot\|_1} (\vu) \big ]_j = \Big (1-\frac{\mu}{|\vu_j|} \Big )_+ \vu_j = \sgn(\vu_j)(|\vu_j|-\mu)_+,$$  
for $\ j \in \{1,\dots,p\}$, with $(x)_+ \defin \max(x,0)$. 
Note that $\text{Prox}_{\mu \|\cdot\|_1}$ is componentwise the \emph{soft-thresholding operator} of \cite{donoho} presented in Section~\ref{sec:intro_c}.

\paragraph{$\ell_1$-norm constraint.} Sometimes, the $\ell_1$-norm is used as a hard constraint and, in that case, the optimization problem is
$$\min_{\vw} f(\vw) \quad \text{such that} \quad \|\vw\|_1 \leq C. $$
This problem can still be viewed as an instance of \eqref{eq:formulation}, with $\Omega$ defined by $\Omega(\vu)=0$ if $ \|\vu\|_1 \leq C$ and $\Omega(\vu)=+\infty$ otherwise. Proximal methods thus apply and the corresponding proximal operator is the projection on the $\ell_1$-ball, itself an instance of a \emph{quadratic continuous knapsack problem} for which efficient pivot algorithms with linear complexity have been proposed \cite{brucker1984,maculan1989linear}. Note that when penalizing by the dual norm of the $\ell_1$-norm, i.e., the $\ell_\infty$-norm, the proximal operator is also equivalent to the projection onto an $\ell_1$-ball.

\paragraph{$\ell_2^2$-regularization (ridge regression).} This regularization function does not induce sparsity and is therefore slightly off topic here. It is nonetheless widely used and it is worth mentioning its proximal operator, which is a scaling operator: $$\text{Prox}_{\frac{\mu}{2} \|.\|_2^2}[\vu] = \frac{1}{1+\mu}\vu.$$

\paragraph{$\ell_1+\ell_2^2$-regularization (Elastic-net~\cite{zou}).} This regularization function combines the~$\ell_1$-norm and the classical squared~$\ell_2$-penalty. For a vector~$\vw$ in~$\R^p$, it can be written~$\Omega(\vw) = \|\vw\|_1 + \frac{\gamma}{2}\|\vw\|_2^2$, where~$\gamma>0$ is an additional parameter. It is not a norm, but the proximal operator can be obtained in closed form: 
$$\text{Prox}_{\mu(\|.\|_1+\frac{\gamma}{2}\|.\|_2^2)}=\text{Prox}_{\frac{\mu\gamma}{2}\|.\|_2^2} \circ \text{Prox}_{\mu \|.\|_1} = \frac{1}{\mu\gamma+1}\text{Prox}_{\mu\|.\|_1}.$$
 Similarly to the~$\ell_1$-norm, when~$\Omega$ is used as a constraint instead of a penalty, the proximal operator can be obtained in linear time using pivot algorithms (see \cite{mairal2010}, Appendix B.2).

\paragraph{1D-total variation.} Originally introduced in the image
processing community~\cite{rudin}, the total-variation penalty encourages
piecewise constant signals. It can be found
in the statistics literature under the name of ``fused lasso''
\cite{tibshirani2005sparsity}. For one-dimensional signals, 
it can be seen as the~$\ell_1$-norm of finite
differences for a vector~$\vw$ in~$\R^p$:~$\Omega_{\text{TV-1D}}(\vw)\defin\sum_{i=1}^{p-1} |\vw_{i+1}-\vw_i|$.
Even though no closed form is available
for~$\text{Prox}_{\mu\Omega_{\text{TV-1D}}}$, it can be easily obtained
using a modification of the homotopy algorithm presented later in this paper in
Section~\ref{sec:lars}~(see \cite{harchaoui2,harchaoui}). 
Similarly, it is possible to combine this penalty with the~$\ell_1$- and squared~$\ell_2$-penalties and efficiently compute~$\text{Prox}_{\Omega_{\text{TV-1D}} + \gamma_1\|.\|_1 + \frac{\gamma_2}{2}\|.\|_2^2}$ or use such a regularization function in a constrained formulation (see~\cite{mairal2010}, Appendix B.2).

\paragraph{Anisotropic 2D-total variation.} The regularization function above
can be extended to more than one dimension. For a two dimensional-signal
$\mw$ in $\R^{p \times l}$ this penalty is defined as
$\Omega_{\text{TV-2D}}(\mw)\defin\sum_{i=1}^{p-1}\sum_{j=1}^{l-1}
|\mw_{i+1,j}-\mw_{i,j}| + |\mw_{i,j+1}-\mw_{i,j}|$. Interestingly, it has been shown in~\cite{chambolle2009} that the corresponding proximal operator can be obtained by solving a parametric maximum flow problem.

\paragraph{$\ell_1/\ell_q$-norm (``group Lasso'').} If $\Gc$ is a partition of $\{1,\ldots,p\}$, the dual norm of the $\ell_1/\ell_q$ norm is the $\ell_\infty/\ell_{q'}$ norm, with $\frac{1}{q}+\frac{1}{q'}\!=\!1$. It is easy to show that the orthogonal projection on a unit $\ell_\infty/\ell_{q'}$ ball is obtained by projecting separately each subvector $\vu_g$ on a unit $\ell_{q'}$-ball in $\R^{|g|}$.
For the $\ell_1/\ell_2$-norm $\Omega: \vw \mapsto \sum_{g \in \Gc} \|\vw_g\|_2$ we have
\begin{equation}
[ \text{Prox}_{\mu \Omega}(\vu)]_g= \Big (1-\frac{\lambda}{\|\vu_g\|_2} \Big )_+ \vu_g, \quad g \in \Gc. \label{eq:prox_grouplasso}
\end{equation}
This is shown easily by considering that the subgradient of the $\ell_2$-norm is $\partial\|\vw\|_2=\big \{ \frac{\vw}{\|\vw\|_2} \big \}$ if $\vw\neq \mathbf{0}$ or $\partial\|\vw\|_2=\{\vz \mid \|\vz\|_2 \leq 1\}$ if $\vw=0$ and by applying the result of Eq.~\eqref{eq:opt}.

For the $\ell_1/\ell_\infty$-norm, whose dual norm is the $\ell_\infty/\ell_1$-norm, an efficient algorithm to compute the proximal operator is based on Eq.~\eqref{eq:prox_proj}. Indeed this equation indicates that the proximal operator can be computed on each group $g$ as the residual of a projection on an $\ell_1$-norm ball in $\R^{|g|}$; the latter is done efficiently with the previously mentioned linear-time algorithms.

\paragraph{$\ell_1/\ell_2$-norm constraint.} When the~$\ell_1/\ell_2$-norm is used as a constraint of the form~$\Omega(\vw) \leq C$, computing the proximal operator amounts to perform an orthogonal projection onto the~$\ell_1/\ell_2$-ball of radius~$C$. It is easy to show that such a problem can be recast as an orthogonal projection onto the simplex~\cite{schmidt2008}.
We know for instance that there exists a parameter~$\mu\geq 0$ such that the solution~$\vw^\star$ of the projection is~$\text{Prox}_{\mu\Omega}[\vu]$ whose form is given in Equation~(\ref{eq:prox_grouplasso}). As a consequence, there exists scalars~$z^g\geq 0$ such that~$\vw^{\star}_g=\frac{z^g}{\|\vu_g\|_2} \vu_g$ (where to simplify but without loss of generality we assume all the~$\vu_g$ to be non-zero). By optimizing over the scalars~$z^g$, one can equivalently rewrite the projection as
$$\min_{(z^g)_{g\in\Gc} \in \R_+^{|\Gc|}} \frac{1}{2} \sum_{g\in \Gc} (\|\vu_g\|_2 - z^g)^2  \quad \st \quad \sum_{g \in \Gc} z^g \leq C,$$
which is a Euclidean projection of the vector~$[\|\vu_g\|_2]_{g\in \Gc}$ in~$\R^{|\Gc|}$ onto a simplex\footnote{This result also follows from Lemma~\ref{lem:prox_rkhs} applied to the compute the proximal operator of the $\ell_\infty/\ell_2$-norm which is dually related to the projection on the $\ell_1/\ell_2$-norm}. The optimization problem above is then solved in linear time using the previously mentioned pivot algorithms~\cite{brucker1984,maculan1989linear}.

We have shown how to compute the proximal operator of group-norms when the
groups form a partition. In general, the case where groups overlap is more
complicated because the regularization is no longer separable. Nonetheless, in
some cases it is still possible to compute efficiently the proximal operator.

\paragraph{Hierarchical $\ell_1/\ell_q$-norms.} Hierarchical norms were proposed in \cite{zhao}. Following \cite{Jenatton2010a}, we focus on the case of a norm $\Omega: \vw \mapsto \sum_{g \in \Gc} \|\vw_g\|_q$, with $q \in \{2, \infty\}$, where the set of groups~$\Gc$ is \emph{tree-structured}, meaning that two groups are either disjoint or one is included in the other. Let $\preceq$ be a total order such that $g_1 \preceq g_2$ if and only if either $g_1 \subset g_2$ or $g_1 \cap g_2=\emptyset$.\footnote{For a tree-structured $\Gc$ such an order exists.} Then, if $g_1 \preceq \ldots \preceq g_{m}$ with $m=|\Gc|$, and if we define $\Pi_g$ as (a) the proximal operator $\vw_g \mapsto \text{Prox}_{\mu \|\cdot\|_q}(\vw_g)$ on the subspace corresponding to group $g$ and (b) the identity on the orthogonal, it can be shown \cite{Jenatton2010a} that:
\begin{equation}
\text{Prox}_{\mu\Omega}=\Pi_{g_m} \circ \ldots \circ \Pi_{g_1}.
\end{equation}

In other words, the proximal operator associated with the norm can be obtained as the composition of the proximal operators associated with individual groups provided that the ordering of the groups is well chosen. Note that this result does not hold for $q \notin \{1,2,\infty\}$ (see \cite{Jenatton2010a} for more details). In terms of complexity, $\text{Prox}_{\mu\Omega}$ can be computed in $O(p)$ operations when~$q=2$ and~$O(pd)$ when~$q=\infty$, where~$d$ is the depth of the tree.

\paragraph{Combined $\ell_1\! +\! \ell_1/\ell_q$-norm (``sparse group Lasso''), with~$q\!\in\! \{2,\infty\}$.} The possibility of combining an $\ell_1/\ell_q$-norm that takes advantage of sparsity at the group level with an $\ell_1$-norm that induces sparsity within the groups is quite natural \cite{friedman2010note,sprechmann2010collaborative}. Such regularizations are in fact a special case of the hierarchical $\ell_1/\ell_q$-norms presented above and the corresponding proximal operator is therefore readily computed by applying first the proximal operator associated with the~$\ell_1$-norm (soft-thresholding) and the one associated with the~$\ell_1/\ell_q$-norm (group soft-thresholding).

\paragraph{Overlapping $\ell_1/\ell_\infty$-norms.} When the groups overlap but do not have a tree structure, computing the proximal operator has proven to be more difficult, but it can still be done efficiently when $q=\infty$.
Indeed, as shown in~\cite{mairal10}, there exists a dual relation between such an operator and a quadratic min-cost flow problem on a particular graph, which can be tackled using network flow optimization techniques. Moreover, it may be extended to more general situations where structured sparsity is expressed through submodular functions~\cite{bach2010structured}.

\paragraph{Trace norm and spectral functions.} The proximal operator for the trace norm, i.e., the unique minimizer of
$\frac{1}{2} \|\mm - \mn\|_F^2 + \mu \| \mm \|_\ast$ for a fixed matrix~$\mm$, may be obtained by computing a singular value decomposition of $\mn$ and then replacing the singular values by their soft-thresholded versions~\cite{cai2010singular}.
This result can be easily extended to the case of spectral functions.
Assume that the penalty~$\Omega$ is of the form~$\Omega(\mm) = \psi(\vs)$ where $\vs$ is a vector carrying the singular values of~$\mm$ and $\psi$ a convex function which is invariant by permutation of the variables~(see, e.g.,~\cite{borwein}). Then, it can be shown that
$\text{Prox}_{\mu\Omega}[\mn] = \mU \Diag(\text{Prox}_{\mu\psi}[\vs])
\bmt{\mv}$, where $\mU \Diag(\vs)\bmt{\mv}$ is a singular value
decomposition of~$\mn$.

\section{Proximal Methods for Structured MKL}
\label{sec:prox_mkl}
 
In this section we show how proximal methods can be applied to solve multiple kernel learning problems. More precisely, we follow \cite{mosci2010solving} who showed, in the context of plain MKL that proximal algorithms are applicable in a RKHS. We extend and present here this idea to the general case of structured MKL, showing that the proximal operator for the structured RKHS norm may be obtained from the proximal operator of the corresponding subquadratic norms.

Given a collection of reproducing kernel Hilbert spaces  $\Hc_1,\ldots,\Hc_p$,
we consider the Cartesian product $\Bc \defin \Hc_1 \times \ldots \times \Hc_p$,
equipped with the norm $\|h\|_{\Bc}\defin (\|h_1\|_{\Hc_1}^2+\ldots+\|h_p\|_{\Hc_p}^2)^{1/2}$, where $h=(h_1,\ldots,h_p)$ with $h_i \in \Hc_i$.

The set $\Bc$ is a Hilbert space, in which gradients and subgradients are well defined and in which we can extend some algorithms that we considered in the Euclidean case easily. 

In the following, we will consider a \emph{monotonic} norm as defined in Definition~\ref{def:monotonic}. This is motivated by the fact that a monotonic norm may be composed with norms of elements of Hilbert spaces to defines a norm on $\Bc$:
\begin{lemma}
Let $\Omega$ be a \emph{monotonic} norm on $\R^p$ with dual norm $\Omega^*$, then $\Lambda: h \mapsto \Omega \big ( (\|h_1\|_{\Hc_1}, \ldots,\|h_p\|_{\Hc_p}) \big )$ is a norm on $\Bc$ whose dual norm is $\Lambda^*: g \mapsto \Omega^*\big ( (\|g_1\|_{\Hc_1}, \ldots,\|g_p\|_{\Hc_p}) \big )$. 
Moreover the subgradient of $\Lambda$ is $\partial \Lambda(h)=A(h)$ with $A(h)\defin\{(u_1 s_1, \ldots,u_p s_p) \mid \vu \in B(h), \: s_i \in \partial \|\cdot\|_{\Hc_i}(h_i)\}$
 with  $B(h)\defin\partial \Omega\big ( (\|h_1\|_{\Hc_1}, \ldots,\|h_p\|_{\Hc_p}) \big )$.
\end{lemma}
\begin{proof}
It is clear that $\Lambda$ is symmetric, nonnegative definite and homogeneous. The triangle inequality results from the fact that $\Omega$ is monotonic. Indeed the latter property implies that $\Lambda(h+g)=\Omega \big ((\|h_i+g_i\|_{\Hc_i})_{1 \leq i \leq p} \big) \leq \Omega \big ((\|h_i\|_{\Hc_i}+\|g_i \|_{\Hc_i})_{1 \leq i \leq p} \big)$ and the result follows from applying the triangle inequality for $\Omega$.

Moreover, we have the generalized Cauchy-Schwarz inequality:
$$\langle h,g\rangle_{\Bc}=\sum_i \langle h_i,g_i\rangle_{\Hc_i} \leq \sum_i \|h_i\|_{\Hc_i} \|g_i\|_{\Hc_i} \leq \Lambda(h) \: \Lambda^*(g),$$ 
and it is easy to check that equality is attained if and only if $g \in A(h)$. 
This simultaneously shows that $\Lambda(h)=\max_{g \in \Bc, \: \Lambda^*(g) \leq 1} \langle h,g\rangle_{\Bc}$ (as a consequence of Proposition~\ref{prop:fenchel_dualnorm}) and that $\partial \Lambda(h)=A(h)$ (by Proposition~\ref{prop:optimality}).
\end{proof}

We consider now a learning problem of the form:
\begin{equation}
\min_{h = (h_1,\ldots,h_p)\in \Bc} f(h_1,\ldots,h_p)+\lambda \,\Lambda(h),
\end{equation}
with, typically, following Section~\ref{sec:mkl}, $f(h)=\frac{1}{n} \sum_{i=1}^n \ell(y^{(i)},h(\vx_i))$. The structured MKL case corresponds more specifically to the case where $f(h)=\frac{1}{n} \sum_{i=1}^n \ell \big (y^{(i)},h_1(\vx_i)+ \ldots+h_p(\vx_i)\big )$. Note that the problem we consider here is regularized with $\Lambda$ and not $\Lambda^2$ as opposed to the formulations (\ref{eq:rkhs_learn}) and (\ref{eq:structmkl_learn}) considered in Section~\ref{sec:mkl}. 

To apply the proximal methods introduced in this chapter using $\|\cdot\|_{\Bc}$ as the proximal term requires one to be able to solve the proximal problem:
\begin{equation}
\label{eq:prox_rkhs}
\min_{h \in \Bc} \frac{1}{2}\|h-g\|_{\Bc}^2+\mu \,\Lambda(h).
\end{equation}

The following lemma shows that if we know how to compute the proximal operator of $\Omega$ for an $\ell_2$ proximity term in $\R^p$, we can readily compute the proximal operator of $\Lambda$ for the proximity defined by the Hilbert norm on $\Bc$. Indeed, to obtain the image of $h$ by the proximal operator associated with $\Lambda$, it suffices to apply the proximal operator of $\Omega$ to the vector of norms $(\|h_1\|_{\Hc_1},\ldots,\|h_p\|_{\Hc_p})$ which yields a vector $(\vy_1,\ldots,\vy_p)$, and to scale in each space $\Hc_i$, the function $h_i$ by $\vy_i/\|h_i\|_{\Hc_{i}}$. Precisely:
\begin{lemma}
\label{lem:prox_rkhs}
$\text{Prox}_{\mu \Lambda}(g)=(\vy_1 s_1,\ldots,\vy_p s_p)$ where $s_i=0$ if $g_i=0$,  $$s_i=\frac{g_i}{\|g_i\|_{\Hc_i}}\: \text{if}  \:g_i \neq 0 \quad \text{and} \quad \vy=\text{Prox}_{\mu\Omega}\big ( (\|g_i\|_{\Hc_i})_{1\leq i \leq p} \big ).$$
\end{lemma}
\begin{proof}
To lighten the notations, we write $\|h_i\|$ for $\|h_i\|_{\Hc_i}$ if $h_i \in \Hc_i$.
The optimality condition for problem (\ref{eq:prox_rkhs}) is $h-g \in -\mu\partial\Lambda$ so that we have $h_i=g_i- \mu \,s_i \vu_i$, with $\vu \in B(h), \: s_i \in \partial \|\cdot\|_{\Hc_i}(h_i)$. The last equation implies that $h_i,g_i$ and $s_i$ are colinear. If $g_i=0$ then the fact that $\Omega$ is monotonic implies that $h_i=s_i=0$. 
If on the other hand, $g_i \neq 0$ we have $h_i=g_i \, (1-\frac{\mu \, \vu_i}{\|g_i\|})_+$ and thus $\|h_i\|=(\|g_i\| -\mu \, \vu_i)_+$ and $h_i=s_i\,\|h_i\|$, but by definition of $\vy_i$ we have $\vy_i=(\|g_i\| -\mu \, \vu_i)_+$, which shows the result.
\end{proof}

This results shows how to compute the proximal operator at an abstract level. For the algorithm to be practical, we need to show that the corresponding computation can be performed by manipulating a finite number of parameters.

Fortunately, we can appeal to a representer theorem to that end, which leads to the following lemma:

\begin{lemma}
\label{lem:prox_mkl}
Assume that for all $i$, $g_i=\sum_{j=1}^n \valpha_{ij} K_i(\vx_j, \cdot)$. Then 
the solution of problem (\ref{eq:prox_rkhs}) is of the form $h_i=\sum_{j=1}^n \vbeta_{ij} K_i(\vx_j, \cdot)$. Let $\vy=\text{Prox}_{\mu \Omega} \big ((\sqrt{\valpha_k^\top \mk_k \valpha_k })_{1\leq k\leq p} \big )$. Then if $\valpha_i \neq 0$, $\vbeta_i=\frac{\vy_i}{\sqrt{\valpha_i^\top\mk_i \valpha_i}} \, \valpha_i$ and otherwise $\vbeta_i=0$.  
\end{lemma}
\begin{proof}
We first show that a representer theorem holds. For each $i$ let $h_i^{\pll}$ be the component of $h_i$ in the span of $K_i(\vx_j,\cdot)_{1\leq j\leq n}$ and $h_i^{\bot}=h_i-h_i^{\pll}$.
We can rewrite the objective of problem (\ref{eq:prox_rkhs}) as\footnote{We denote again $\|h_i\|$ for $\|h_i\|_{\Hc_i}$, when the RHKS norm used is implied by the argument.}
\begin{equation*}
\frac{1}{2} \sum_{i=1}^p \big [ \|h_i^{\pll}\|^2+\|h_i^\bot\|^2-2\langle h_i^{\pll},g_i\rangle_{\Hc_i}+\|g_i\|^2 \big ]+\mu \Omega \big ( (\|h_i^{\pll}\|^2+\|h_i^{\bot}\|^2)_{1\leq i \leq p}\big),
\end{equation*}
from which, given that $\Omega$ is assumed monotonic, it is clear that setting $h_i^{\bot}=0$ for all $i$ can only decrease the objective.
To conclude, the form of the solution in $\vbeta$ results from the fact that $\|g_i\|_{\Hc_i}^2=\sum_{1\leq j,j' \leq n}{\valpha_{ij}} \valpha_{ij'} \langle K_i(\vx_j,\cdot),K_i(\vx_{j'},\cdot)\rangle_{\Hc_i}$ and $\langle K_i(\vx_j,\cdot),K_i(\vx_{j'},\cdot)\rangle_{\Hc_i}=K_i(\vx_j,\vx_{j'})$ by the reproducing property, and by identification (note that if the kernel matrix $\mk_i$ is not invertible the solution might not be unique in $\vbeta_i$).
\end{proof}
 
Finally, in the last lemma we assumed that the function $g_i$ in the proximal problem could be represented as a linear combination of the $K_i(\vx_j,\cdot)$. Since $g_i$ is typically of the form $h^t_i-\frac{1}{L}\frac{\partial}{\partial h_i}f(h^t_1,\ldots,h^t_p)$, then the result follows by linearity if the gradient is in the span of the $\mk_i(\vx_j,\cdot)$.  The following lemma shows that this is indeed the case:
\begin{lemma}
\label{lem:grad_representer}
For $f(h)=\frac{1}{n} \sum_{j=1}^n\ell(y^{(j)},h_1(\vx_j),\ldots,h_p(\vx_j))$ then 
$$\frac{\partial}{\partial h_i}f(h)= \sum_{j=1}^n  \valpha_{ij} K_i(\vx_j,\cdot) \quad  \text{for} \quad\valpha_{ij}=\frac{1}{n} \partial_i \ell(y^{(j)},h_1(\vx_j),\ldots,h_p(\vx_j)),$$ where $\partial_i \ell$ denote the partial derivative of $\ell$ w.r.t. to its $i+1$-th scalar component.
\end{lemma}
\begin{proof}
This result follows from the rules of composition of differentiation applied to the functions
$$(h_1,\ldots,h_p) \mapsto \ell \big ( y^{(j)},\, \langle h_1,K_1(x_j,\cdot)\rangle_{\Hc_1},\, \ldots,\, \langle h_p,K_p(x_j,\cdot)\rangle_{\Hc_p}\big ),$$ and the fact that, since $h_i\mapsto \langle h_i,K_i(\vx_j,\cdot)\rangle_{\Hc_i}$ is linear, its gradient in the RKHS $\Hc_i$ is just $K_i(\vx_j,\cdot)$.

\end{proof}

\chapter{(Block) Coordinate Descent Algorithms}
\label{sec:opt_methods_bcd}
Coordinate descent algorithms solving $\ell_1$-regularized learning problems go back to \cite{fu_penalized_1998}.
They optimize  (exactly or approximately) the objective function with respect to one variable at a time while all others are kept fixed. Note that, in general, coordinate descent algorithms are not necessarily convergent for non-smooth optimization (cf.\ \cite{bertsekas}, p.636); they are however applicable in this setting because of a \emph{separability} property of the nonsmooth regularizer we consider (see end of Section~\ref{sec:cd}).

\section{Coordinate Descent for $\ell_1$-Regularization}
\label{sec:cd}
We consider first the following special case of the one-dimensional $\ell_1$-regularized problem:
\begin{equation}
\label{eq:absvalproxprob}
\min_{w \in \R} \, \frac{1}{2}(w-w_0)^2 + \lambda |w|.
\end{equation}
As shown in \eqref{eq:st_op}, the minimizer $w^\star$ can be obtained by \emph{soft-thresholding}:
\begin{equation}
\label{eq:st}
w^\star=\text{Prox}_{\lambda\,|\cdot|}(w_0) \defin \Big (1-\frac{\lambda}{|w_0|} \Big )_+ \, w_0.
\end{equation}

\paragraph{Lasso case.} In the case of the square loss, the minimization with respect to a single coordinate can be written as
$$\min_{\vw_j \in \R}\: \nabla_j f(\vw^t)\, (\vw_j-\vw^t_j)+\frac{1}{2} \nabla^2_{jj}\, f(\vw^t)(\vw_j-\vw^t_j)^2 + \lambda |\vw_j|,$$
with $\nabla_j f(\vw)=\frac{1}{n}\mx_j^\top(\mx \vw-\vy)$ and $\nabla^2_{jj}f(\vw)=\frac{1}{n} \mx_j^\top \mx_j$ independent of~$\vw$.
Since the above equation is of the form \eqref{eq:absvalproxprob}, it can be solved in closed form:
\begin{equation}
\label{eq:shooting}
\vw_j^\star=\text{Prox}_{\frac{\lambda}{\nabla^2_{jj} f}|\cdot|} \big (\, \vw_j^t-\nabla_jf(\vw_j^t)/\nabla^2_{jj}f \, \big ).
\end{equation}
In other words, $\vw_j^\star$ is obtained by solving the unregularized problem with respect to coordinate $j$ and soft-thresholding the solution.

This is the update proposed in the shooting algorithm of Fu \cite{fu_penalized_1998}, which cycles through all variables in a fixed order.\footnote{Coordinate descent with a cyclic order is sometimes called a Gauss-Seidel procedure.} Other cycling schemes are possible (see, e.g.,~\cite{nesterov2010efficiency}).

An efficient implementation is obtained if the quantity $\mx \vw^t-\vy$ or even better $\nabla f(\vw^{t})= \frac{1}{n} \mx^\top\mx \vw^t-\frac{1}{n} \mx^\top\vy$ is kept updated.\footnote{In the former case, at each iteration, $\mx \vw-\vy$ can be updated in $\Theta(n)$ operations if $\vw_j$ changes and $\nabla_{\!j^{t+1}} f(\vw)$ can always be updated in $\Theta(n)$ operations. The complexity of one cycle is therefore $O(pn)$.
However a better complexity is obtained in the latter case, provided the matrix $\mx^\top\mx$ is precomputed (with complexity $O(p^2n)$). Indeed $\nabla f(\vw^{t})$ is updated in $\Theta(p)$ iterations only if $w_j$ does not stay at $0$. Otherwise, if $w_j$ stays at $0$ the step costs $O(1)$; the complexity of one cycle is therefore $\Theta(ps)$ where $s$ is the number of non-zero variables at the end of the cycle.}

\paragraph{Smooth loss.} For more general smooth losses, like the logistic loss, the optimization with respect to\ a single variable cannot be solved in closed form. It is possible to
solve it numerically using a sequence of modified Newton steps as proposed in \cite{shevade2003simple}. We present here a fast algorithm of Tseng and Yun \cite{tseng2009coordinate} based on solving just a quadratic approximation of $f$ with an inexact line search at each iteration.

Let $L^t>0$ be a parameter and let $\vw_j^\star$ be the solution of 
$$\min_{\vw_j \in \R}\: \nabla_j f(\vw^t)\, (\vw_j-\vw^t_j)+\frac{1}{2} L^t\, (\vw_j-\vw^t_j)^2 + \lambda |\vw_j|,$$
Given $d=\vw_j^\star-\vw_j^t$ where $\vw_j^\star$ is the solution of $\eqref{eq:shooting}$,
the algorithm of Tseng and Yun performs a line search to choose the largest step of the form $\alpha d$ with $\alpha=\alpha_0 \beta^k$ and $\alpha_0>0,\beta \in (0,1), \, k \in \mathbb{N}$, such that the following modified Armijo condition is satisfied:
$$F(\vw^t+\alpha d\ve_j)-F(\vw^t) \leq \sigma \alpha \big ( \nabla_j f(\vw) d + \gamma L^t d^2 +|\vw^t_j+d|-|\vw^t_j| \big ),$$
where $F(\vw)\defin f(\vw)+\lambda\,\Omega(\vw)$, and $0\leq \gamma\leq 1$ and $\sigma<1$ are parameters of the algorithm.

Tseng and Yun \cite{tseng2009coordinate} show that under mild conditions on $f$ the algorithm is convergent and, under further assumptions, asymptotically linear. In particular, if $f$ is of the form $\frac{1}{n}\sum_{i=1}^n\ell(y^{(i)},\vw^\top \vx^{(i)})$ with 
$\ell(y^{(i)},\cdot)$ a twice continuously differentiable convex function with strictly positive curvature, the algorithm is asymptotically linear for $L^t=\nabla^2_{jj}f(\vw^t_j)$. 
We refer the reader to Section \ref{sec:bcd} and to \cite{tseng2009coordinate,wright2010accelerated} for results under much milder conditions.
It should be noted that the algorithm generalizes to other separable regularizations than the $\ell_1$-norm.

Variants of coordinate descent algorithms have also been considered in \cite{genkin_large-scale_2007,krishnapuram_sparse_2005,wu2008coordinate}.
Generalizations based on the Gauss-Southwell rule have been considered in \cite{tseng2009coordinate}.

\paragraph{Convergence of coordinate descent algorithms.}
In general, coordinate descent algorithms are not convergent for non-smooth objectives. Therefore, using such schemes always requires a convergence analysis. In the context of the $\ell_1$-norm regularized smooth objective, the non-differentiability is \emph{separable} (i.e., is a   sum of non-differentiable terms that depend on single variables), and this is sufficient for convergence~\cite{bertsekas,tseng2009coordinate}. In terms of convergence rates, coordinate descent behaves in a similar way to first-order methods such as proximal methods, i.e., if the objective function is strongly convex~\cite{nesterov2010efficiency,tseng2009coordinate}, then the convergence is linear, while it is slower if the problem is not strongly convex, e.g., in the learning context, if there are strong correlations between input variables~\cite{shalev2009stochastic}.

\section{Block-Coordinate Descent for $\ell_1/\ell_q$-Regularization}
\label{sec:bcd}
When $\Omega(\vw)$ is the $\ell_1/\ell_q$-norm with groups $g \in \Gc$ forming a partition of $\{1, \ldots, p\}$, the previous methods are generalized by block-coordinate descent (BCD) algorithms, which have been the focus of recent work by Tseng and Yun~\cite{tseng2009coordinate} and Wright~\cite{wright2010accelerated}. These algorithms do not attempt to solve exactly a reduced problem on a block of coordinates but rather optimize a surrogate of $F$ in which the function $f$ is substituted by a quadratic approximation.

Specifically, the BCD scheme of \cite{tseng2009coordinate} solves at each iteration a problem of the form:
\begin{equation}
\label{eq:bcd_blocknorm}
\min_{\vw_g \in \R^{|g|}} \nabla_g f(\vw^t)^\top(\vw_g-\vw^t_g)+\frac{1}{2} (\vw_g-\vw^t_g)^\top  \mH^t (\vw_g-\vw^t_g) + \lambda \|\vw_g\|_q,
\end{equation}
where the positive semi-definite matrix $\mH^t \in \R^{|g|\times |g|}$ is a parameter.  Note that this may correspond to a richer quadratic approximation around $\vw^t_g$ than the proximal terms used in Section~\ref{sec:proximal_methods}. However, in practice,
the above problem is solved in closed form if $\mH^t=L^t \mi_{|g|}$ for some scalar $L^t$ and $q \in\{2,\infty\}$\footnote{More generally for $q\geq 1$ and $\mH^t=L^t \mi_{|g|}$ , it can be solved efficiently coordinate-wise using bisection algorithms.}. In particular for $q=2$, the solution $\vw_g^{\star}$ is obtained by \emph{group-soft-thresholding}:
$$
\vw_g^{\star}=\text{Prox}_{\frac{\lambda}{L^t}\,\|\cdot\|_2}\big (\vw_g^t-\frac{1}{L^t} \nabla_gf(\vw_g^t)\big),
$$
with
$$\text{Prox}_{\mu\,\|\cdot\|_2}(\vw)=\Big (1-\frac{\mu}{\|\vw\|_2} \Big )_+ \vw.$$

In the case of general smooth losses, the descent direction is given by $\vd=\vw_g^{\star}-\vw^t_g$ with $\vw_g^{\star}$
as above.  The next point is of the form $\vw^t+\alpha \vd$, where $\alpha$ is a stepsize of the form $\alpha=\alpha_0\, \beta^k$, with $\alpha_0>0$, $0<\beta<1,\, k \in \mathbb{N}$. 
The parameter $k$ is chosen large enough to satisfy the following modified Armijo condition
$$F(\vw^t+\alpha \vd)-F(\vw^t) \leq \sigma \alpha \big ( \nabla_{\!g} f(\vw)^\top \vd +\gamma \, \vd^\top \mH^t\vd+\|\vw^t_g+\vd\|_q-\|\vw^t_g\|_q \big ),$$ 
for parameters $0\leq \gamma\leq 1$ and $\sigma<1$.

If $f$ is convex continuously differentiable, lower bounded on $\R^p$ and $F$ has a unique minimizer, provided that there exists $\tau, \bar{\tau}$ fixed constants such that for all $t$, $\tau \preceq H^t \preceq \bar{\tau}$, the results of
Tseng and Yun show that the algorithm converges (see Theorem 4.1 in \cite{tseng2009coordinate} for broader conditions).
Wright~\cite{wright2010accelerated} proposes a variant of the algorithm, in which the line-search on $\alpha$ is replaced by a line search on the parameter $L^t$, similar to the line-search strategies used in proximal methods.

\section{Block-coordinate descent for MKL}

Finally, block-coordinate descent algorithms are also applicable to classical multiple kernel learning . We consider the same setting and notation as in Section~\ref{sec:prox_mkl} and we consider specifically the optimization problem:
$$\min_{h \in \Bc} f(h_1,\ldots,h_p)+ \lambda \sum_{i=1}^p \|h_i\|_{\Hc_i}.$$
A block-coordinate algorithm can be applied by considering each RKHS $\Hc_i$ as one ``block"; this type of algorithm was considered in \cite{ravikumar2009sparse}. Applying the lemmas \ref{lem:prox_mkl} and \ref{lem:grad_representer} of Section \ref{sec:prox_mkl}, we know that $h_i$ can be represented as $h_i=\sum_{j=1}^n \valpha_{ij} K_i(\vx_j,\cdot)$.

The algorithm then consists in performing successively group soft-thresholding in each RKHS $\Hc_i$. This can be done by working directly with the dual parameters $\valpha_i$, with a corresponding proximal operator in the dual simply formulated as:
$$\text{Prox}_{\mu\,\|\cdot\|_{\mk_i}}(\valpha_i)=\Big (1-\frac{\mu}{\|\valpha_i\|_{\mk_i}} \Big )_+ \valpha_i.$$ 
with $\|\valpha\|_{\mk}^2=\valpha^\top \mk \valpha$. The precise equations would be obtained by kernelizing Eq.~\ref{eq:bcd_blocknorm} (which requires kernelizing the computation of the gradient and the Hessian as in Lemma~\ref{lem:grad_representer}). We leave the details to the reader.

\chapter{Reweighted-$\ell_2$ Algorithms}
\label{sec:reweighted_l2}
\label{sec:opt_methods_etatrick}

Approximating a nonsmooth or constrained optimization problem by a series of smooth unconstrained problems is common in optimization~(see, e.g., \cite{boyd.convex,nesterov2005smooth,nocedal}). In the context of objective functions regularized by sparsity-inducing norms, it is natural to consider variational formulations of these norms in terms of squared $\ell_2$-norms, since many efficient methods are available to solve  $\ell_2$-regularized problems (e.g., linear system solvers for least-squares regression).

\section{Variational formulations for grouped $\ell_1$-norms}

In this section, we show on our motivating example of sums of $\ell_2$-norms of subsets how such formulations arise~(see, e.g., \cite{grandvalet1999,ingrid2010iteratively,jenatton2010sspca,pontil,simpleMKL}).
The variational formulation we have presented in Section~\ref{subsec:variational} allows us to consider the following function $H(\vw,\veta)$ defined as
$$
H(\vw,\veta) = f(\vw) +   \frac{\lambda}{2}
\sum_{j = 1}^p \bigg\{
\sum_{g \in \Gc, j \in g} \veta_g^{-1}
\bigg\} \vw_j^2
+  \frac{\lambda}{2} \sum_{g \in \Gc} \veta_g  .
$$
It is jointly convex in $(\vw,\veta)$; the minimization with respect to $\veta$ can be done in closed form, and the optimum is equal to $F(\vw) = f(\vw) + \lambda \Omega(\vw)$; as for the minimization with respect to $\vw$, it is an $\ell_2$-regularized  problem.

Unfortunately, the alternating minimization algorithm that is immediately suggested is not convergent in general, because the function~$H$  is not continuous (in particular around $\veta$ which has zero coordinates). In order to make the algorithm convergent, two strategies are commonly used:

$-$ \textbf{Smoothing}: we can add a term of the form $\frac{\varepsilon}{2}  \sum_{g \in \Gc} \veta_g^{-1}$, which yields a joint cost function with compact level sets on the set of positive numbers. Alternating minimization algorithms are then convergent (as a consequence of general results on block coordinate descent), and have two different iterations: (1) minimization with respect to $\veta$ in closed form, through
$\veta_g =    \sqrt{\mynorm{\vw_g}{2}^2  + \varepsilon }$,  and (2) minimization with respect to $\vw$, which is an $\ell_2$-regularized problem, which can be for example solved in closed form for the square loss. Note however, that the second problem does not need to be exactly optimized at all iterations.

$-$ \textbf{First order method in $\veta$}: while the joint cost function $H(\veta,\vw)$ is not continuous, the function
$I(\veta) =\min_{\vw \in \R^p}  H(\vw,\veta) $ is continuous, and under general assumptions, continuously differentiable, and is thus amenable to first-order methods (e.g., proximal methods, gradient descent). When the groups in $\Gc$ do not overlap, one sufficient condition is that the function $f(\vw)$ is of the form $f(\vw) = \psi(\mx\vw)$ for $\mx \in \R^{n \times p}$ any matrix (typically the design matrix) and $\psi$ a strongly convex function on $\R^n$. This strategy is particularly interesting when evaluating $I(\veta)$ is computationally cheap.\\


In theory, the alternating scheme consisting of optimizing alternately over $\veta$ and $\vw$ can be used to solve learning problems regularized with \emph{any} norms: on top of the subquadratic norms defined in Section~\ref{sec:subqua}, we indeed show in the next section that any norm admits a quadratic variational formulation (potentially defined from a non-diagonal symmetric positive matrix). 
To illustrate the principle of $\ell_2$-reweighted algorithms, we first consider the special case of multiple kernel learning; in Section~\ref{sec:gen_varia}, we consider the case of the trace norm.

\paragraph {\bf Structured MKL.} 
Reweighted-$\ell_2$ algorithms are fairly natural for norms which admit a diagonal variational formulation  (see Lemma \ref{lem:convex_var} and \cite{micchelli2011regularizers}) and for the corresponding multiple kernel learning problem.
We consider the structured multiple learning problem presented in Section \ref{sec:struct_mkl}.

The alternating scheme applied to Eq.~(\ref{eq:struct_mkl}) then takes the following form: for $\veta$ fixed, one has to solve a single kernel learning problem with the kernel $K=\sum_{i} \veta_i K_i$;
 the corresponding solution in the product of RKHSs $\Hc_1 \times \ldots \times \Hc_p$ (see Section~\ref{sec:prox_mkl}) is of the form $h(\vx)=h_1(\vx)+\ldots+h_p(\vx)$ with $h_i(\vx)=\veta_i \sum_{j=1}^n \valpha_j K_i(\vx_j,\cdot)$. 
 Since $\|h_i\|^2_{\Hc_i}=\veta_i^2 \valpha^\top \mk_i \valpha$, for fixed $\valpha$, the update in $\veta$ then takes the form: 
$$\veta^{t+1} \leftarrow \argmin_{\veta \in H} \sum_{i=1}^p \frac{(\veta^t_i)^2 {\valpha^t}^\top \mk_i \, \valpha^t+\varepsilon}{\veta_i}.$$   
Note that these updates produce a non-increasing sequence of values of the primal objective.
Moreover, this MKL optimization scheme uses a potentially much more compact parameterization than proximal methods since in addition to the variational parameter $\veta \in \R^p$ a single vector of parameters $\valpha \in \R^n$ is needed as opposed to up to one such vector for each kernel in the case of proximal methods.
MKL problems can also be tackled using first order methods in $\veta$ described above: we refer the reader to \cite{simpleMKL} for an example in the case of classical MKL.

\section{Quadratic Variational Formulation for General Norms}
\label{sec:gen_varia}
 We now investigate a general variational formulation of norms that naturally leads to a sequence of reweighted $\ell_2$-regularized problems. The formulation is based on approximating the unit ball of a norm~$\Omega$ with enclosing ellipsoids.
 See Figure~\ref{fig:ellipsoids}.
 The following proposition shows that all norms may be expressed as a minimum of Euclidean norms:

 \begin{figure}
 \begin{center}
 \includegraphics[scale=.35]{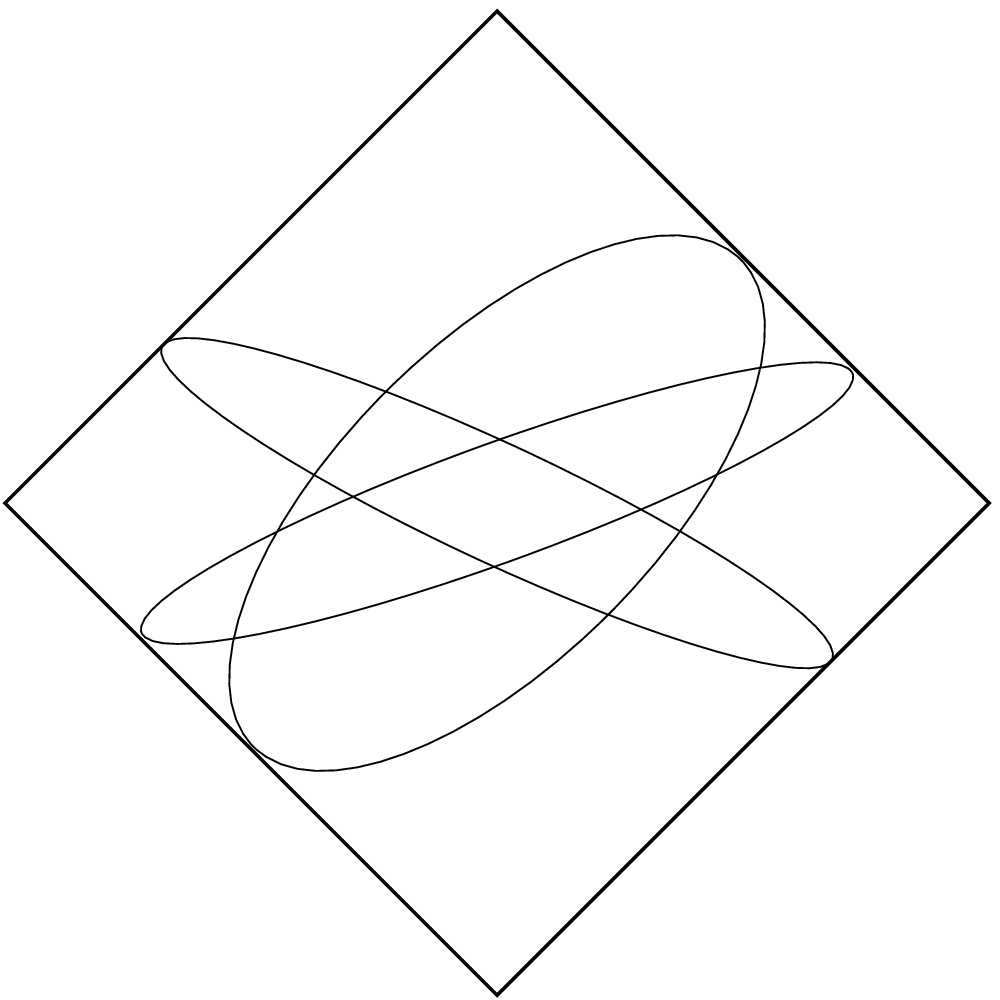} \hspace*{.5cm}
 \includegraphics[scale=.35]{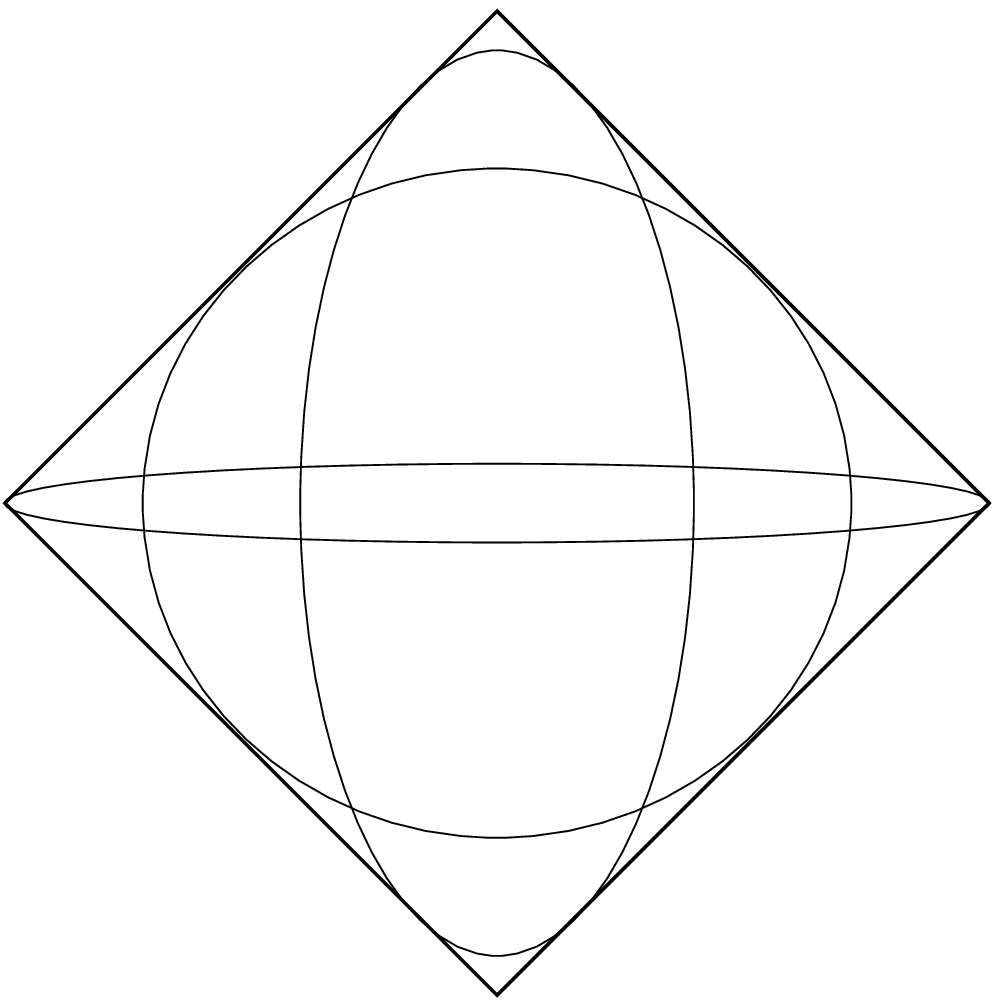} \hspace*{.5cm}
 \includegraphics[scale=.35]{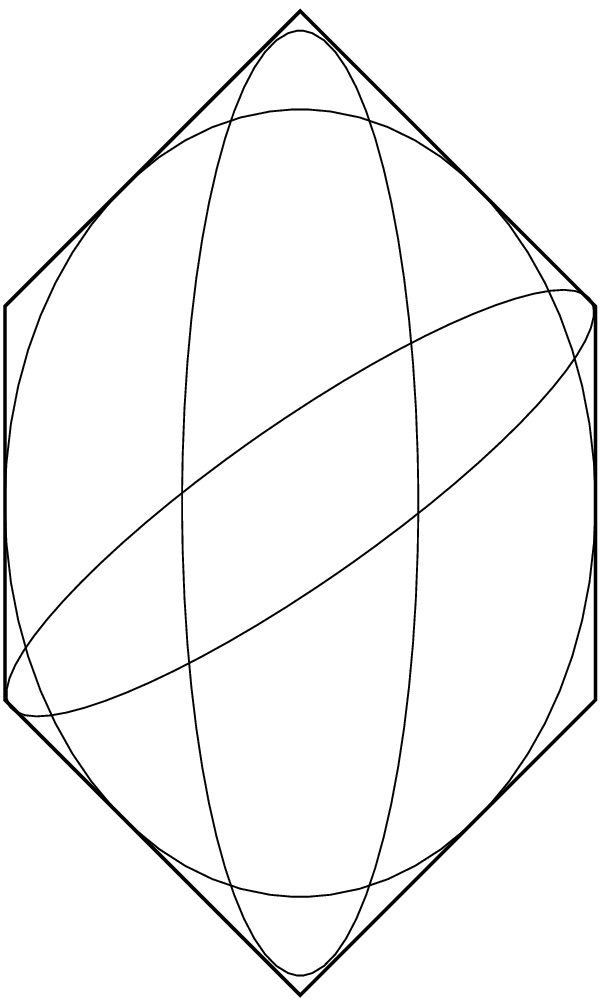}
 \end{center}
 \caption{Example of a sparsity-inducing ball in two dimensions, with enclosing ellipsoids. Left: ellipsoids with general axis for the $\ell_1$-norm; middle: ellipsoids with horizontal and vertical axis  for the $\ell_1$-norm; right: ellipsoids for another polyhedral norm.}
 \label{fig:ellipsoids}
 \end{figure}

 \begin{proposition}
 Let $\Omega: \R^p \to \R$ a norm on $\R^p$, then there exists a function $g$ defined on the cone of positive semi-definite matrices $\mathbf{S}^+_p$,  such that $g$ is convex, strictly positive except at zero, positively homogeneous and such that for all $ \vw \in \R^p$, 
 \begin{equation}
 \label{eq:eta}
 \Omega(\vw) = \! \min_{\bm{\Lambda} \in \mathbf{S}^+_p, \  g(\bm{\Lambda}) \leqslant 1 } \! \sqrt{ \vw^\top \bm{\Lambda}^{-1} \vw }
 = \frac{1}{2} \min_{\bm{\Lambda} \in \mathbf{S}^+_p} \big\{    \vw^\top \bm{\Lambda}^{-1} \vw  + g(\bm{\Lambda}) \big\}.
 \end{equation}
 \end{proposition}
 \begin{proof}
 Let $\Omega^\ast$ be the dual norm of $\Omega$, defined as $\Omega^\ast(\vs) = \max_{ \Omega(\vw) \leqslant 1 } \vw^\top \vs$ \cite{boyd.convex}. Let $g$ be the function defined through
 $g(\bm{\Lambda})=  \max_{ \Omega^\ast(\vs) \leqslant 1 } \vs^\top \bm{\Lambda} \vs$. This function is well-defined as the maximum of a continuous function over a compact set; moreover, as a maximum of linear functions, it is convex and positive homogeneous. Also, for nonzero~$\bm{\Lambda}$, the quadratic form $\vs \mapsto \vs^\top \bm{\Lambda} \vs$ is not identically zero around $\vs=0$, hence the strict positivity of $g$.

 Let $\vw  \in \R^p$ and $\bm{\Lambda} \in \mathbf{S}^+_p$; there exists $\vs $ such that $\Omega^\ast(\vs) = 1$ and $\vw^\top \vs = \Omega(\vw)$. We then have
 $$
 \Omega(\vw)^2 =  (\vw^\top \vs)^2 \leqslant (\vw^\top \bm{\Lambda}^{-1} \vw)  (\vs^\top \bm{\Lambda} \vs )
 \leqslant  g(\bm{\Lambda})  (\vw^\top \bm{\Lambda}^{-1} \vw) .
 $$
 This shows that
 $\Omega(\vw) \leqslant \min_{\bm{\Lambda} \in \mathbf{S}^+_p, \  g(\bm{\Lambda}) \leqslant 1 } \sqrt{ \vw^\top \bm{\Lambda}^{-1} \vw }$. Proving the other direction can be done using the following limiting argument. Given $\vw_0 \in \R^p$, consider
 $\bm{\Lambda}(\varepsilon) = ( 1 - \varepsilon) \vw_0 \vw_0^\top + \varepsilon (\vw_0^\top \vw_0) I $. We have $\vw_0^\top \bm{\Lambda}(\varepsilon) ^{-1} \vw_0 =1$ and $g(\bm{\Lambda}(\varepsilon)) \to g(\vw_0 \vw_0^\top) = \Omega(\vw_0)^2$.
 Thus, for $\tilde{\bm{\Lambda}}(\varepsilon)  = \bm{\Lambda}(\varepsilon) /  g(\bm{\Lambda}(\varepsilon)) $, we have
 that $\sqrt{ \vw_0^\top \tilde{\bm{\Lambda}}(\varepsilon) ^{-1} \vw_0 }$ tends to $\Omega(\vw_0)$, thus $\Omega(\vw_0)$ must be no smaller than the minimum over all $\bm{\Lambda}$. The right-hand side of Eq.~(\ref{eq:eta}) can be obtained by optimizing over the scale of $\bm{\Lambda}$.
 \end{proof}

 Note that while the proof provides a closed-form expression for a candidate function $g$, it is not unique, as can be seen in the following examples. The domain of $g$ (matrices so that $g$ is finite) may be reduced (in particular to diagonal matrices for the $\ell_1$-norm and more generally the sub-quadratic norms defined in Section~\ref{sec:subqua}):

 $-$ For the $\ell_1$-norm: using the candidate from the proof, we have $g(\bm{\Lambda})= 
 \max_{ \| \vs\|_\infty \leq 1} \vs^\top {\bm \Lambda} \vs $, but we could use $g(\bm{\Lambda}) = \trace{\bm{\Lambda}}$ if $\bm{\Lambda}$ is diagonal and $+\infty$ otherwise.
 
$-$  For subquadratic norms (Section~\ref{sec:subqua}), we can take $g(\bm{\Lambda}) $ to be $+\infty$ for non-diagonal $\bm{\Lambda}$, and either equal to the gauge function of the set $H$, i.e. the function $\vs \mapsto \inf \{\nu \in \R_+ \mid \vs \in \nu H\}$, or equal to the function  $\bar{\Omega}$ defined in Lemma~\ref{lem:subqua_varia}, both applied to the diagonal of $\bm{\Lambda}$. 
 
  $-$ For the $\ell_2$-norm:  $g(\bm{\Lambda}) = \lambda_{\max}(\bm{\Lambda})$ but we could of course use  $g(\bm{\Lambda}) = 1$ if $\bm{\Lambda}= I$  and $+\infty$ otherwise.

 $-$ For the trace norm: $\vw$ is assumed to be of the form $\vw = { \rm vect} (\bm{W})$ and the trace norm of $\bm{W}$ is regularized. The trace norm admits the variational form (see \cite{argyriou2008convex}) :
\begin{equation} 
\label{eq:tr_norm}
\|\mw\|_*= \frac{1}{2} \inf_{\md \succeq 0} \tr(\mw^\top \md^{-1} \mw+\md) \quad \st\quad \md \succ 0.
\end{equation}

But $\tr(\mw^\top \md^{-1} \mw)=\vw^\top (\mi \otimes \md)^{-1} \vw$, which shows that the regularization by the trace norm takes the form of Eq.~(\ref{eq:eta}) in which we can choose $g({\bm \Lambda})$ equal to $\tr(\md)$  if ${\bm \Lambda}=\mi \otimes \md$ for some $\md \succ 0$ and $+\infty$ otherwise.

The solution of the above optimization problem is given by
$\md^{\star}=(\mw\mw^\top)^{1/2}$ 
which can be computed via a singular value decomposition of $\mw$. The reweighted-$\ell_2$ algorithm to solve
$$\min_{\mw \in \R^{p\times k}} f(\mw)+ \lambda \|\mw\|_*$$
therefore consists of iterating between the two updates (see, e.g.,~\cite{argyriou2008convex} for more details):
$$\mw \leftarrow \argmin_{\mw} f(\mw)+\frac{\lambda}{2} \trace(\mw^\top \md^{-1} \mw) \quad \text{and} $$
$$\md \leftarrow (\mw\mw^\top+\varepsilon \, \mi_k)^{1/2},$$
where $\varepsilon$ is a smoothing parameter that arises from adding a term $\frac{\varepsilon\lambda}{2} \, \trace(\md^{-1})$ to Eq.~(\ref{eq:tr_norm}) and prevents the matrix from becoming singular.

\chapter{Working-Set and Homotopy Methods}
In this section, we consider methods that explicitly take into account the fact that the solutions are sparse, namely working set methods and homotopy methods.
\label{sec:workinghomo}

\section{Working-Set Techniques}
\label{sec:active_sets}

\label{sec:opt_methods_activeset}
Working-set algorithms address optimization problems
by solving an increasing sequence of small subproblems of (\ref{eq:formulation}), which we recall can be written as
\begin{displaymath}
   \min_{\vw \in \R^p} f(\vw) + \lambda \Omega(\vw),
\end{displaymath}
where~$f$ is a convex smooth function and~$\Omega$ a sparsity-inducing norm.
The working set, which we denote by~$J$,
refers to the subset of variables involved in the optimization of these subproblems.

Working-set algorithms proceed as follows:
after computing a solution to the problem restricted to the variables in $J$ (while setting the variables in~$J^c$ to zero),
global optimality is checked to determine whether the algorithm has to continue.
If this is the case, new variables enter the working set $J$ according to a strategy that has to be defined.
Note that we only consider \textit{forward} algorithms, i.e., where the working set grows monotonically.
In other words, there are no \textit{backward} steps where variables would be allowed to leave the set $J$.
Provided this assumption is met, it is easy to see that these procedures stop
in a finite number of iterations.

This class of algorithms is typically applied to linear programming and
quadratic programming problems (see, e.g.,~\cite{nocedal}), and here takes
specific advantage of sparsity from a computational point of
view~\cite{hkl,jenatton,ng-sparsecoding, obozinski-joint,  roth,
schmidt2010convex,szafranski2007hierarchical}, since the subproblems that need
to be solved are typically much smaller than the original one.

Working-set algorithms require three ingredients:\\
 
 $\bullet$  \textbf{Inner-loop solver}: At each iteration of the working-set algorithm,
 problem~(\ref{eq:formulation}) has to be solved on $J$,
 i.e., subject to the additional equality constraint that $\vw_j=0$ for all $j$ in $J^c$:
 \begin{equation}
   \min_{\vw \in \R^p} f(\vw) + \lambda \Omega(\vw),\ \text{such that}\ \vw_{J^c}=0.\label{eq:formulation_J}
\end{equation}
 The computation can be performed by any of the methods presented in this monograph.
 Working-set algorithms should therefore be viewed as ``meta-algorithms''.
 Since solutions for successive working sets are typically close to each other
the approach is efficient if the method chosen can use \textit{warm-restarts}.\\

 $\bullet$  \textbf{Computing the optimality conditions}: Given a solution~$\vw^\star$ of problem~(\ref{eq:formulation_J}),
 it is then necessary to check whether~$\vw^\star$ is also a solution for the original problem~(\ref{eq:formulation}).
 This test relies on the duality gaps of problems~(\ref{eq:formulation_J}) and (\ref{eq:formulation}).
 In particular, if $\vw^\star$ is a solution of problem~(\ref{eq:formulation_J}),
 it follows from Proposition \ref{prop:duagap} in Section~\ref{sec:opt_tools}
 that
  \begin{equation*}
  f(\vw^\star) + \lambda \Omega(\vw^\star) + f^\ast( \nabla f(\vw^\star) ) = 0.
  \end{equation*}
In fact, the Lagrangian parameter associated with the equality constraint
ensures the feasibility of the dual variable formed from the gradient of $f$ at~$\vw^\star$.
In turn, this guarantees that the duality gap of problem~(\ref{eq:formulation_J}) vanishes.
 The candidate $\vw^\star$ is now a solution of the full problem~(\ref{eq:formulation}),
i.e., without the equality constraint $\vw_{J^c}=0$, if and only if
 \begin{equation}
   \Omega^\ast( \nabla f(\vw^\star) )\leq \lambda.\label{eq:feasibility_grad}
  \end{equation}
 Condition~(\ref{eq:feasibility_grad}) points out that the dual norm $\Omega^\ast$ is a key quantity to monitor the progress of the working-set algorithm~\cite{jenatton}.
 In simple settings, for instance when $\Omega$ is the $\ell_1$-norm,
 checking condition~(\ref{eq:feasibility_grad}) can be easily computed since~$\Omega^\ast$ is just the $\ell_\infty$-norm.
 In this case, condition~(\ref{eq:feasibility_grad}) becomes
  \begin{equation*}
   | [\nabla f(\vw^\star)]_j |\leq \lambda,\ \text{for all}\ j\ \text{in}\ \{1,\dots,p\}.
  \end{equation*}
Note that by using the optimality of problem~(\ref{eq:formulation_J}),
the components of the gradient of $f$ indexed by $J$ are already guaranteed to be no greater than $\lambda$.\\

 $\bullet$  \textbf{Strategy for the growth of the working set}: If condition~(\ref{eq:feasibility_grad}) is not satisfied for the current working set $J$,
 some inactive variables in $J^c$ have to become active.
 This point raises the questions of \textit{how many} and \textit{how} these variables should be chosen.
 First, depending on the structure of~$\Omega$,
 a \textit{single} or a \textit{group} of inactive variables have to be considered to enter the working set.
 Furthermore, one natural way to proceed is to look at the variables that violate condition~(\ref{eq:feasibility_grad}) most.
 In the example of $\ell_1$-regularized least squares regression with normalized predictors,
 this strategy amounts to selecting the inactive variable that has the highest correlation with the current residual.\\

The working-set algorithms we have described so far aim at solving problem~(\ref{eq:formulation}) for a fixed value of the regularization parameter $\lambda$. However, for specific types of loss and regularization functions,
the set of solutions of problem~(\ref{eq:formulation}) can be obtained efficiently for all possible values of $\lambda$,
which is the topic of the next section.

\section{Homotopy Methods} \label{sec:lars}
We present in this section an active-set\footnote{Active-set and working-set
methods are very similar; they differ in that active-set methods allow (or sometimes require)
variables returning to zero to exit the set.} method for solving the Lasso
problem of Eq.~(\ref{eq:lasso}).  We recall the Lasso formulation:
\begin{equation}
   \min_{\vw \in \R^p} \frac{1}{2n}\|\vy-\mx \vw \|_2^2 + \lambda \|\vw\|_1, \label{eq:lasso2}
\end{equation}
where $\vy$ is in $\R^n$, and $\mx$ is a design matrix in $\R^{n \times p}$.
Even though generic working-set methods introduced above could be used to solve
this formulation (see, e.g.,~\cite{ng-sparsecoding}), a specific property of the $\ell_1$-norm associated
with a quadratic loss makes it possible to address it more efficiently.

Under mild assumptions (which we will detail later), the solution of
Eq.~(\ref{eq:lasso2}) is unique, and we denote it by~$\vw^\star(\lambda)$ for a given regularization parameter~$\lambda >0$. We use the name \emph{regularization path} to denote the function
$\lambda \mapsto \vw^\star(\lambda)$ that associates to a
regularization parameter~$\lambda$ the corresponding solution.
We will show that this function is piecewise linear, 
a behavior illustrated
in Figure~\ref{fig:lasso_piecewise}, where the entries of $\vw^\star(\lambda)$ for
a particular instance of the Lasso are represented as functions of $\lambda$.
\begin{figure}
\centering
\includegraphics[width=0.8\linewidth]{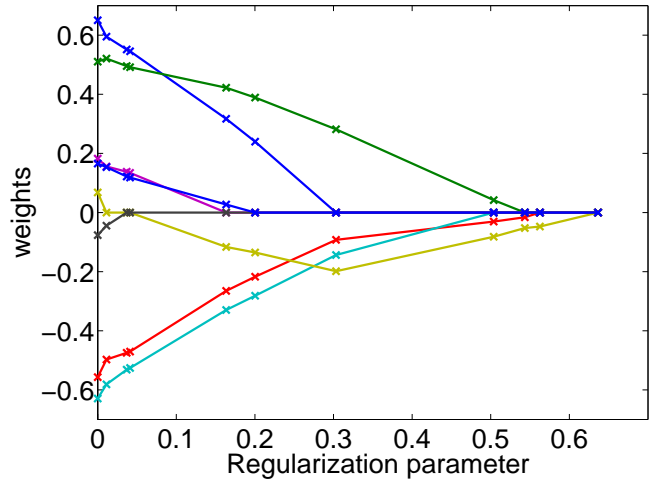}
\caption{The weights $\vw^\star(\lambda)$ are represented as functions of the regularization parameter $\lambda$.
When $\lambda$ increases, more and more coefficients are set to zero. These functions are all piecewise affine. Note that some variables (here one) may enter and leave the regularization path.}
\label{fig:lasso_piecewise}
\end{figure}

An efficient algorithm can thus be constructed 
by choosing a particular value of $\lambda$,
for which finding this solution is trivial, and by following the piecewise
affine path, computing the directions of the current affine parts, and the
points where the direction changes (also known as kinks).  This piecewise linearity
was first discovered and exploited in \cite{markowitz} in the context of
portfolio selection, revisited in~\cite{osborne} describing a \emph{homotopy}
algorithm, and studied in~\cite{efron} with the LARS algorithm.\footnote{Even
though the basic version of LARS is a bit different from the procedure we have
just described, it is closely related, and indeed a simple modification makes
it possible to obtain the full regularization path of Eq.~(\ref{eq:lasso}).}
Similar ideas also appear early in the optimization literature: Finding the
full regularization path of the Lasso is in fact a particular instance of a
\emph{parametric quadratic programming} problem, for which path following
algorithms have been developed~\cite{ritter}.

Let us show how to construct the path. From the optimality conditions we have
presented in Eq.~(\ref{eq:optlasso}), denoting by $J \defin \{j ;
|\bmt{\mx}_j(\vy-\mx\vw^\star)| = n\lambda \}$ the set of active variables, and
defining the vector $\vt$ in $\{-1; 0; 1\}^{p}$ as $\vt \defin
\text{sign}\big(\bmt{\mx}(\vy-\mx\vw^\star)\big)$,
we have the following closed-form expression
\begin{displaymath}
   \begin{cases}
 \vw_J^\star(\lambda) & = (\bmt{\mx}_J
\mxJ)^{-1}(\bmt{\mx}_J \vy-n\lambda \vt_J) \\
 \vw_{J^c}^\star(\lambda) &= 0,
 \end{cases}
\end{displaymath}
where we have assumed the matrix $\bmt{\mx}_J \mxJ$ to
be invertible (which is a sufficient condition to guarantee the
uniqueness of $\vw^\star$).
This is an important point: if one knows in advance the
set $J$ and the signs $\vt_J$, then $\vw^\star(\lambda)$ admits
a simple closed-form.
Moreover, when $J$ and $\vt_J$ are fixed, the function $\lambda \mapsto
(\bmt{\mx}_J \mxJ)^{-1}(\bmt{\mx}_J \vy-n\lambda \vt_J)$ is affine in
$\lambda$.
With this observation in hand, we can now present the main steps of the path-following
algorithm. It basically starts from a trivial
solution of the regularization path, follows the path by exploiting this formula,
updating $J$ and $\vt_J$ whenever needed
so that optimality conditions~(\ref{eq:optlasso}) remain satisfied.
This procedure requires some assumptions---namely that (a) the
matrix $\bmt{\mx}_J \mxJ$ is always invertible, and (b) that updating $J$ along the path 
consists of adding or removing from this set a single variable at the same time.
Concretely, we proceed as follows:
\begin{enumerate}
    \item Set $\lambda$ to $\frac{1}{n}\|\bmt{\mx}\vy\|_\infty$ for which it is easy to show from Eq.~(\ref{eq:optlasso}) that $\vw^\star(\lambda)=0$ (trivial solution on the regularization path).
    \item Set $J \defin \{j ; |\bmt{\mx}_j\vy| = n\lambda \}$.
    \item \label{item:lars} Follow the regularization path by decreasing the value of $\lambda$, with the formula
       $\vw^\star_J(\lambda) = (\bmt{\mx}_J \mxJ)^{-1}(\bmt{\mx}_J \vy-n\lambda\vt_J)$ keeping $\vw^\star_{J^c}=0$,
       until one of the following events (kink) occurs
       \begin{itemize}
          \item There exists $j$ in $J^c$ such that $|\bmt{\mx}_j(\vy-\mx\vw^\star)| =n \lambda$. Then, add $j$ to the set $J$.
          \item There exists $j$ in $J$ such that a non-zero coefficient $\vw^\star_j$ hits zero. Then, remove $j$ from $J$.
       \end{itemize}
       We assume that only one of such events can occur at the same time (b). It is also easy to show that the value of $\lambda$ corresponding to the next event can be obtained in closed form such that one can ``jump'' from a kink to another.
    \item  Go back to \ref{item:lars}.
\end{enumerate}
Let us now briefly discuss assumptions (a) and (b).
When the matrix~$\bmt{\mx}_J \mxJ$ is not invertible, the regularization path is
non-unique, and the algorithm fails. This can easily be fixed by addressing instead a slightly
modified formulation. It is possible to consider instead
the elastic-net formulation~\cite{zou} that uses
$\Omega(\vw)=\lambda\|\vw\|_1+ \frac{\gamma}{2}\|\vw\|_2^2$. Indeed, it amounts to replacing
the matrix $\bmt{\mx}_J \mxJ$ by $\bmt{\mx}_J \mxJ + n\gamma \mi$, which is positive definite and
therefore always invertible,  and to apply the same algorithm.
The second assumption (b) can be unsatisfied in practice
because of the machine precision. To the best of our knowledge, the algorithm will
fail in such cases, but we consider this scenario unlikely with real data, though possible when the Lasso/basis pursuit is used multiple times such as in dictionary learning, presented in Section~\ref{sec:DL}. In such situations, a proper use of optimality conditions can detect such problems and more stable algorithms such as proximal methods may then be used.

The complexity of the above procedure depends on the number of kinks of the
regularization path (which is also the number of iterations of the
algorithm). Even though it is possible to build examples where this number is
large, we often observe in practice that the event where one
variable gets out of the active set is rare.
The complexity also depends on the implementation.
By maintaining the computations of $\bmt{\mx}_j(\vy-\mx\vw^\star)$ and a
Cholesky decomposition of $(\bmt{\mx}_J \mxJ)^{-1}$,
it is possible to obtain an implementation in $O(psn+ps^2+s^3)$ operations, where $s$ is the
sparsity of the solution when the algorithm is stopped (which we approximately consider as equal to the number of iterations). 
The product $psn$ corresponds to the computation of the matrices $\bmt{\mx}_J \mxJ$, $ps^2$ to the updates of the
correlations $\bmt{\mx}_j(\vy-\mx\vw^\star)$ along the path, and $s^3$ to the
Cholesky decomposition.

\chapter{Sparsity and Nonconvex Optimization}
\label{sec:nonconvex}
In this section, we consider alternative approaches to sparse modelling, which are not based in convex optimization, but often use convex optimization problems in inner loops.

\section{Greedy Algorithms}
\label{sec:opt_greedy_algorithms}
We consider the $\ell_0$-constrained signal decomposition problem
\begin{equation}
   \min_{\vw \in \R^p} \frac{1}{2n}\|\vy-\mx\vw\|_2^2\quad  \st\quad  \|\vw\|_0 \leq s,\label{eq:greedy}
\end{equation}
where $s$ is the desired number of non-zero coefficients of the solution, and we assume for simplicity
that the columns of~$\mx$ have unit norm.
Even though this problem can be shown to be NP-hard~\cite{natarajan},
greedy procedures can provide an approximate solution. Under some 
assumptions on the matrix~$\mx$, they can also be shown to have some
optimality guarantees~\cite{tropp}. 

Several variants of these algorithms with different names have been developed
both by the statistics and signal processing communities. In a nutshell, they are known
as forward selection techniques in statistics~(see \cite{weisberg}), and matching
pursuit algorithms in signal processing~\cite{mallat}.
All of these approaches start with a null vector $\vw$, and iteratively
add non-zero variables to~$\vw$ until the threshold~$s$ is reached.
   
The algorithm dubbed \emph{matching pursuit}, was introduced in the 90's
in~\cite{mallat}, and can be seen as a non-cyclic coordinate descent procedure
for minimizing Eq.~(\ref{eq:greedy}).  It selects at each step the column
$\vx^\hati$ that is the most correlated with the residual according to the
formula
\begin{displaymath}
\hati \leftarrow \argmin_{i \in \{1,\dots,p\} } |\vr^\top \vx^{i}|,
\end{displaymath}
where $\vr$ denotes the residual $\vy-\mx\vw$.
Then, the residual is projected on $\vx^\hati$ and the entry~$\vw_\hati$ is
updated according to
\begin{displaymath}
\begin{split}
\vw_\hati & \leftarrow \vw_\hati + \vr^\top \vx^{\hati} \\
      \vr & \leftarrow \vr - (\vr^\top \vx^{\hati} )\vx^\hati.
\end{split}
\end{displaymath}
Such a simple procedure is guaranteed to decrease the objective function at
each iteration, but is not to converge in a finite number of steps (the same
variable can be selected several times during the process).  Note that 
such a scheme also appears in statistics in boosting procedures~\cite{friedman2001}.

\emph{Orthogonal matching pursuit}~\cite{mallat} was proposed as a major variant of
matching pursuit that ensures the residual of the decomposition to be always
\emph{orthogonal to all previously selected columns of~$\mx$}. Such technique existed
in fact in the statistics literature under the name \emph{forward
selection}~\cite{weisberg}, and a particular implementation exploiting a QR
matrix factorization also appears in~\cite{natarajan}.
More precisely, the algorithm is an active set procedure, which sequentially
adds one variable at a time to the active set, which we denote
by $J$. It provides an approximate solution
of Eq.~(\ref{eq:greedy}) for every value
$s' \leq s$, and stops when the desired level of sparsity is reached. Thus, it builds 
a regularization path, and shares many similarities with the homotopy algorithm for solving
the Lasso~\cite{efron}, even though the two algorithms address different optimization
problems. 
These similarities are also very strong in terms of implementation: Identical
tricks as those described in Section~\ref{sec:lars} for the homotopy algorithm
can be used, and in fact both algorithms have roughly the same complexity (if
most variables do not leave the path once they have entered it).  At each
iteration, one has to choose which new predictor should enter the active
set~$J$. A possible choice is to look for the column of~$\mx$ most correlated
with the residual as in the matching pursuit algorithm, 
but another criterion is to select the one that helps most reducing the
objective function
\begin{displaymath}
\hati \leftarrow \argmin_{i \in J^c} \min_{\vw' \in \R^{|J|+1}} \frac{1}{2n}\|\vy-\mx_{J \cup \{i \}}\vw' \|_2^2.
\end{displaymath}
Whereas this choice seem at first sight computationally expensive since it requires 
solving $|J^c|$ least-squares problems, the solution can in fact
be obtained efficiently using a Cholesky matrix
decomposition of the matrix~$\mx_{J}^\top \mx_J$ and basic linear algebra,
which we will not detail here for simplicity reasons (see~\cite{cotter} for
more details). 

After this step, the active set is updated $J \leftarrow J \cup \{\hati\}$, and
the corresponding residual $\vr$ and coefficients $\vw$ are
\begin{displaymath}
\begin{split}
\vw& \leftarrow (\mx_J^\top\mx_J)^{-1}\mx_J^\top \vy, \\
      \vr & \leftarrow (\mi-\mx_J(\mx_J^\top\mx_J)^{-1}\mx_J^\top) \vy, 
\end{split}
\end{displaymath}
where $\vr$ is the residual of the orthogonal projection of $\vy$ onto
the linear subspace spanned by the columns of $\mx_J$.
It is worth noticing that one does not need to compute these
two quantities in practice, but only updating the Cholesky decomposition
of $(\mx_J^\top\mx_J)^{-1}$ and computing directly $\mx^\top\vr$,
via simple linear algebra relations.

For simplicity, we have chosen to present matching pursuit algorithms for solving the
$\ell_0$-sparse approximation problem, but they admit several variants (see~\cite{needell} for example),
or extensions when the regularization is more complex than the~$\ell_0$-penalty or for other loss functions than the square loss. For instance,
they are used in the context of non-convex group-sparsity in~\cite{tropp2}, or
with structured sparsity formulations~\cite{baraniuk,huang}.

We also remark that other possibilities than greedy methods exists for
optimizing Eq.~(\ref{eq:greedy}). One can notably use the algorithm ISTA
(i.e., the non-accelerated proximal method) presented in
Section~\ref{sec:proximal_methods} when the function $f$ is convex and its
gradient Lipschitz continuous. Under this assumption, it is easy to see that
ISTA can iteratively decrease the value of the nonconvex objective function.
Such proximal gradient algorithms when~$\Omega$ is the~$\ell_0$-penalty
often appear under the name of iterative hard-thresholding
methods~\cite{herrity}.

\section{Reweighted-$\ell_1$ Algorithms with DC-Programming}
\label{sec:DC}
We focus in this section on optimization schemes for a certain type of
non-convex regularization functions.  More precisely, 
we consider problem~(\ref{eq:formulation}) when~$\Omega$ is a nonconvex separable penalty that can be written as
$\Omega(\vw)\defin\sum_{i=1}^p\zeta(|\vw_i|),$ where $\vw$ is in $\R^p$, and
$\zeta: \R_+ \to \R_+$ is a \emph{concave} non-decreasing differentiable
function. In other words, we address
\begin{equation}
  \min_{\vw \in \R^p} f(\vw) + \lambda\sum_{i=1}^p \zeta(|\vw_i|), \label{eq:dcprogformulation}
\end{equation}
where~$f$ is a convex smooth function.
Examples of such penalties include variants of the $\ell_q$-penalties
for $q < 1$ defined as $\zeta: t \mapsto (|t|+\varepsilon)^q$, log-penalties
$\zeta:t \mapsto \log(|t| + \varepsilon)$, where~$\varepsilon > 0$ makes the function~$\zeta$ differentiable at~$0$.  Other nonconvex
regularization functions have been proposed in the statistics community, such
as the SCAD penalty~\cite{fan2001}. 

The main motivation for using such penalties is that they induce more
sparsity than the $\ell_1$-norm, while they can be optimized with continuous optimization as opposed to greedy methods.  The unit balls corresponding to the $\ell_q$ pseudo-norms and
norms are illustrated in Figure~\ref{fig:lp},  for several values of $q$.
When~$q$ decreases, the $\ell_q$-ball approaches in a sense the $\ell_0$-ball, which allows to induce sparsity more aggressively.
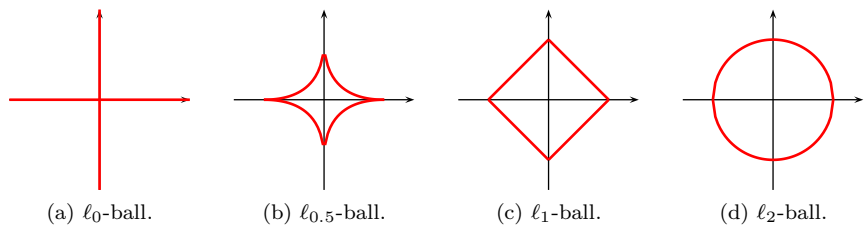
\begin{figure}[hbtp]
\centering
\hspace*{-.5cm}
\psset{yunit=0.8,xunit=0.8}
\subfloat[$\ell_0$-ball.]{ 
   \begin{pspicture}(-1.5,-1.5)(1.5,1.5)
   \psline[linewidth=0.5pt]{->}(-1.5,0)(1.5,0)
   \psline[linewidth=0.5pt]{->}(0,-1.5)(0,1.5)
   \psline[linecolor=red,linewidth=1pt]{-}(-1.5,0)(1.5,0)
   \psline[linecolor=red,linewidth=1pt]{-}(0,-1.5)(0,1.5)
\end{pspicture} 
} 
\hspace*{.2cm}
\subfloat[$\ell_{0.5}$-ball.]{ 
   \begin{pspicture}(-1.5,-1.5)(1.5,1.5)
   \psline[linewidth=0.5pt]{->}(-1.5,0)(1.5,0)
   \psline[linewidth=0.5pt]{->}(0,-1.5)(0,1.5)
   \psplot[linecolor=red, linewidth=1pt]{-1}{1}{
      1 -1 x abs sqrt mul add 1 -1 x abs sqrt mul add mul
   }
   \psplot[linecolor=red, linewidth=1pt]{-1}{1}{
      -1 1 -1 x abs sqrt mul add 1 -1 x abs sqrt mul add mul mul
   }
\end{pspicture} 
} 
\hspace*{.2cm}
\subfloat[$\ell_1$-ball.]{ 
   \begin{pspicture}(-1.5,-1.5)(1.5,1.5)
   \psline[linewidth=0.5pt]{->}(-1.5,0)(1.5,0)
   \psline[linewidth=0.5pt]{->}(0,-1.5)(0,1.5)
   \psplot[linecolor=red, linewidth=1pt]{0}{1}{
      1 -1 x mul add 
   }
   \psplot[linecolor=red, linewidth=1pt]{0}{1}{
      -1 x add 
   }
   \psplot[linecolor=red, linewidth=1pt]{-1}{0}{
      1 x add 
   }
   \psplot[linecolor=red, linewidth=1pt]{-1}{0}{
      -1 -1 x mul add 
   }
\end{pspicture} 
}  
\hspace*{.2cm}
\subfloat[$\ell_2$-ball.]{ 
   \begin{pspicture}(-1.5,-1.5)(1.5,1.5)
   \psline[linewidth=0.5pt]{->}(-1.5,0)(1.5,0)
   \psline[linewidth=0.5pt]{->}(0,-1.5)(0,1.5)
   \psplot[linecolor=red, linewidth=1pt]{-1}{1}{
      1 -1 x x mul mul add sqrt
   }
   \psplot[linecolor=red, linewidth=1pt]{-1}{1}{
      -1 1 -1 x x mul mul add sqrt mul
   }
\end{pspicture} 
} 
\hspace*{-.5cm}

\caption{Unit balls in 2D corresponding to $\ell_q$-penalties.}
\label{fig:lp}
\end{figure}

Even though the optimization problem~(\ref{eq:dcprogformulation}) is not convex and not
smooth, it is possible to iteratively decrease the value of the objective
function by solving a sequence of convex problems. Algorithmic schemes of this form appear
early in the optimization literature in a more general framework for minimizing the difference
of two convex functions (or equivalently the sum of a convex and a concave function),
which is called DC programming~(see \cite{gasso2009} and references therein). 
Even though the objective function of Equation~(\ref{eq:dcprogformulation}) is
not exactly a difference of convex functions (it is only the case on~$\R_+^p$),
the core idea of DC programming can be applied. 
We note that these algorithms were recently revisited under the particular form
of reweighted-$\ell_1$ algorithms~\cite{candes4}.
The idea is relatively simple. We denote by~$g: \R^p \to \R$ the objective
function which can be written as~$g(\vw) \defin f(\vw)+\lambda \sum_{i=1}^p
\zeta(|\vw_i|)$ for a vector~$\vw$ in~$\R^p$.  This optimization scheme consists
in minimizing, at iteration~$k$ of the algorithm, a convex upper bound of the
objective function~$g$, which is tangent to the graph of~$g$ at the current
estimate~$\vw^k$. 

A surrogate function with these properties is obtained easily by exploiting the
concavity of the functions $\zeta$ on $\R_+$, which always lie below their
tangents, as illustrated in Figure~\ref{fig:lp_tangent}. The iterative scheme
can then be written
\begin{displaymath}
   \vw^{k+1} \leftarrow \argmin_{\vw \in \R^p} f(\vw) + \lambda\sum_{i=1}^p \zeta'(|\vw^k_i|) |\vw_i|,
\end{displaymath}
which is a reweighted-$\ell_1$ sparse decomposition problem~\cite{candes4}. To
initialize the algorithm, the first step is usually a simple Lasso, with no
weights. In practice, the effect of the weights $\zeta'(|\vw^k_i|)$ is to push
to zero the smallest non-zero coefficients from iteration~$k-1$, and two or
three iterations are usually enough to obtain the desired sparsifying effect.
Linearizing iteratively concave functions to obtain convex surrogates is the
main idea of DC programming, which readily applies here to the functions~$w \mapsto \zeta(|w|)$.

\begin{figure}[hbtp]
\centering
\subfloat[red: $\zeta(|w|)=\log(|w|+\varepsilon)$. \protect\newline  blue: convex surrogate $\zeta'(|w^k|)|w|+C$.]{
\psset{yunit=1.8,xunit=1.8}
   \begin{pspicture}(-1.5,-1.2)(1.5,1.2)
   \psline[linewidth=0.5pt]{->}(-1.5,0)(1.5,0)
   \psline[linewidth=0.5pt]{->}(0,-1.2)(0,1.2)
\put(0.9,0.2){$w^k$}
   \psline[linewidth=0.5pt,linecolor=black,linestyle=dotted]{-}(0.5,0)(0.5,-0.22)
\psplot[plotpoints=200,linecolor=red, linewidth=1pt]{-1.5}{1.5}{
   x abs 0.1 add log
}
\psplot[plotpoints=200,linecolor=blue, linewidth=1pt]{-1.5}{1.5}{
   0.7238 x abs mul -0.5837 add
}
\end{pspicture}} \hfill
\subfloat[red: $f(w) + \zeta(|w|)$. \protect\newline  blue: $f(w)+\zeta'(|w^k|)|w|+C$.]{
\psset{yunit=1.8,xunit=1.8}
   \begin{pspicture}(-1.5,-1.2)(1.5,1.2)
   \psline[linewidth=0.5pt]{->}(-1.5,0)(1.5,0)
   \psline[linewidth=0.5pt]{->}(0,-1.2)(0,1.2)
\put(0.9,0.2){$w^k$}
   \psline[linewidth=0.5pt,linecolor=black,linestyle=dotted]{-}(0.5,0)(0.5,-0.22)
\psplot[plotpoints=200,linecolor=red, linewidth=1pt]{-0.8}{1.1}{
   -0.2 x add x -0.2 add mul -0.2 add x abs 0.1 add log add
}
\psplot[plotpoints=200,linecolor=blue, linewidth=1pt]{-0.8}{1.1}{
   -0.2 x add x -0.2 add mul -0.2 add 0.7238 x abs mul -0.5837 add add
}
\end{pspicture}}
\caption[Surrogate function used in the DC-programming approach]{Functions and their surrogates involved in the reweighted-$\ell_1$ optimization scheme in the one dimensional case ($p=1$). The function~$\zeta$ can be written here as~$\zeta(|w|)=\log(|w|+\varepsilon)$ for a scalar~$w$ in~$\R$, and the function~$f$ is quadratic. The graph of the non-convex functions are represented in red and their convex ``tangent'' surrogates in blue.}
\label{fig:lp_tangent}
\end{figure}
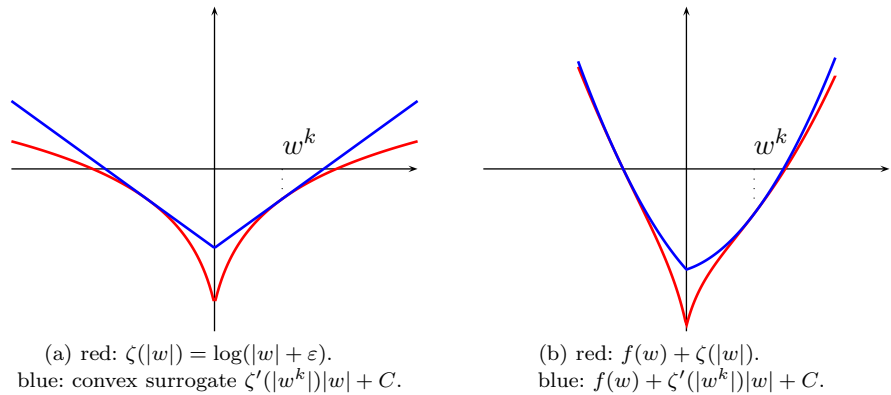

For simplicity we have presented these reweighted-$\ell_1$ algorithms 
when~$\Omega$ is separable. We note however that these optimization schemes
are far more general and readily apply to non-convex versions of most of the
penalties we have considered in this paper. For example, when the penalty~$\Omega$ has the form
\begin{displaymath}
   \Omega(\vw) = \sum_{g \in \Gc} \zeta(\|\vw_g\|),
\end{displaymath}
where~$\zeta$ is concave and differentiable on~$\R_+$, $\Gc$ is a set of
(possibly overlapping) groups of variables and~$\|.\|$ is any norm.
The idea is then similar, iteratively linearizing for each group~$g$ the
functions $\zeta$ around~$\|\vw_g\|$ and minimizing the resulting convex surrogate (see an application to structured sparse principal component analysis in~\cite{jenatton2010sspca}).

Finally, another possible approach to solve optimization problems regularized by nonconvex penalties of the type 
presented in this section is to use the reweighted-$\ell_2$ algorithms of Section~\ref{sec:reweighted_l2} 
based on quadratic variational forms of such penalties (see, e.g., \cite{jenatton2010sspca}).

\section{Sparse Matrix Factorization and Dictionary Learning}
\label{sec:DL}

Sparse linear models for regression in statistics and machine learning
assume a linear relation~$\vy \approx \mx\vw$, where $\vy$ in~$\R^n$ is a
vector of observations, $\mx$ in~$\R^{n \times p}$ is a design matrix
whose rows can be interpreted as features, and~$\vw$ is a weight vector
in~$\R^p$.  Similar models are used in the signal processing literature,
where $\vy$ is a signal approximated by a linear combination of columns
of~$\mx$, which are called dictionary elements, or basis element
when~$\mx$ is orthogonal.

Whereas a lot of attention has been devoted to cases where~$\mx$ is fixed
and pre-defined, another line of work considered the problem of learning~$\mx$
from training data. In the context of sparse linear models, this problem
was first introduced in the neuroscience community by Olshausen and
Field~\cite{field1996} to model the spatial receptive fields of simple
cells in the mammalian visual cortex.  Concretely, given a training set
of~$q$ signals $\my=[\vy^1,\ldots,\vy^q]$ in $\R^{n \times q}$, one would like to
find a dictionary matrix~$\mx$ in~$\R^{n \times p}$ and a coefficient
matrix~$\mw=[\vw^1,\ldots,\vw^q]$ in~$\R^{p \times q}$ such that each
signal $\vy^i$ admits a sparse approximation $\mx\vw^i$. In other words,
we want to learn a dictionary~$\mx$ and a sparse matrix~$\mw$ such that
$\my\approx\mx\mw$.

A natural formulation is the following
non-convex matrix factorization problem:
\begin{equation}
    \min_{\mx \in {\mathcal X}, \mw \in \R^{n \times q}}
\frac{1}{q}\sum_{i=1}^q\frac{1}{2}\|\vy^i-\mx\vw^i\|_2^2 +
\lambda\, \Omega(\vw^i), \label{eq:dict_learn}
\end{equation}
where~$\Omega$ is a sparsity-inducing penalty function, and~${\mathcal X}
\subseteq \R^{n \times p}$ is a convex set, which is typically the set of
matrices whose columns have $\ell_2$-norm less or equal to one. Without any sparse prior (i.e., for $\lambda=0$), the solution of this factorization problem is obtained through principal component analysis (PCA) (see, e.g.,~\cite{burges2009dimension} and references therein). However, when $\lambda >0$, the solution of Eq.~(\ref{eq:dict_learn}) has a different behavior, and may be used as an alternative to PCA for unsupervised learning.

A successful application of this approach is when the vectors~$\vy^i$ are small
natural image patches, for example of size~$n=10\times10$ pixels.  A typical
setting is to have an overcomplete dictionary---that is, the number of
dictionary elements can be greater than the signal dimension but small compared to the number of training signals, for example~$p=200$ and~$q=100\,000$.  For this sort of
data, dictionary learning finds linear subspaces of small dimension where the
patches live, leading to effective applications in image
processing~\cite{elad2006}. Examples of a dictionary for image patches is given in 
Figure~\ref{fig:patches}.
\begin{figure}[hbtp]
\centering
\includegraphics[width=0.48\linewidth]{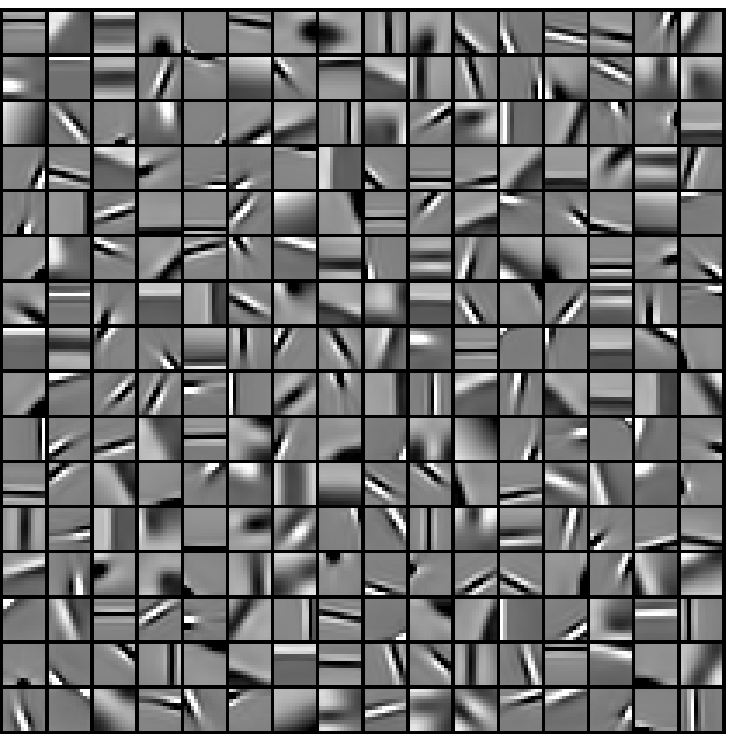}
\hfill
\includegraphics[width=0.48\linewidth]{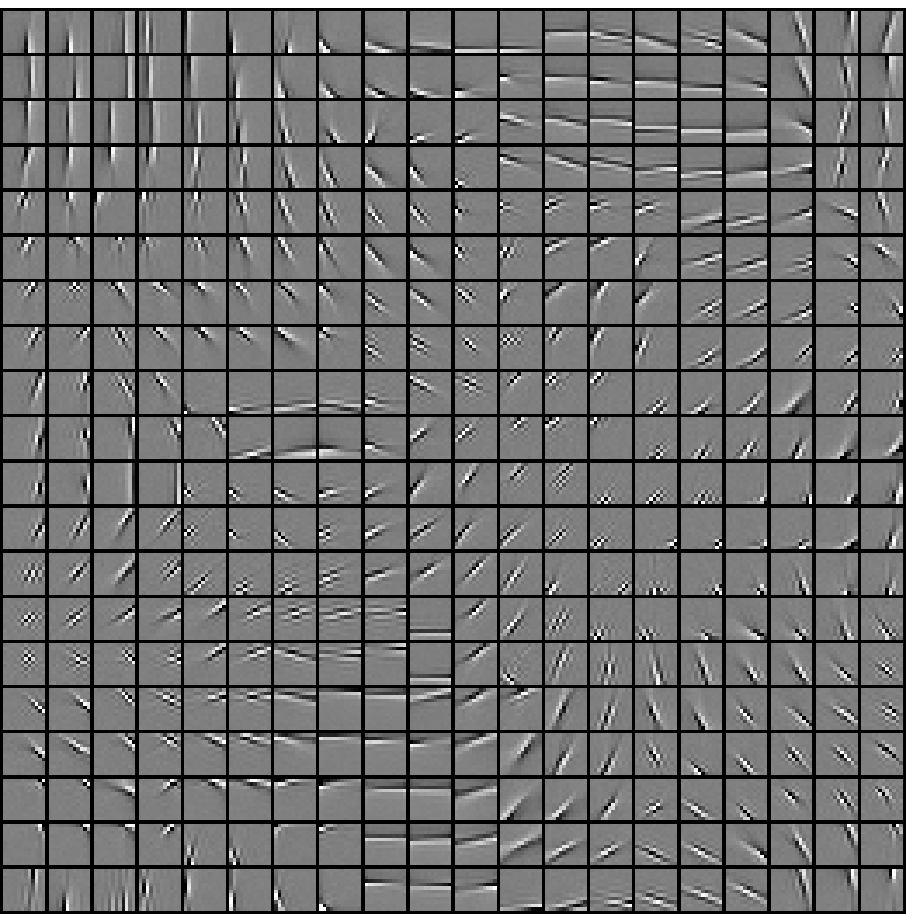}
\caption{Left: Example of dictionary with $p=256$ elements, learned on a database of
natural $12 \times 12$ image patches when~$\Omega$ is the~$\ell_1$-norm. Right: Dictionary with $p=400$ elements, learned with a structured sparsity-inducing penalty~$\Omega$ (see~\cite{Mairal2011}).}\label{fig:patches}
\end{figure}

In terms of optimization, Eq.~(\ref{eq:dict_learn}) is nonconvex and no known
algorithm has a guarantee of providing a global optimum in general, whatever
the choice of penalty~$\Omega$ is. A typical approach to find a local minimum is to use a block-coordinate
scheme, which optimizes $\mx$ and~$\mw$ in turn, while keeping the other one
fixed~\cite{engan1999}. Other alternatives include the K-SVD
algorithm~\cite{aharon2006} (when~$\Omega$ is the~$\ell_0$-penalty), and online learning
techniques~\cite{mairal2010,field1996} that have proven to be particularly
efficient when the number of signals~$q$ is large.\footnote{Such efficient algorithms are freely available in the open-source software package SPAMS \url{http://www.di.ens.fr/willow/SPAMS/}.}  Convex relaxations of
dictionary learning have also been proposed in~\cite{bach2008b,bradley2009}.

\section{Bayesian Methods}
\label{sec:opt_bayesian}
 
While the focus of this monograph is on frequentist approaches to sparsity, 
and particularly on approaches that minimize a regularized empirical risk, 
there naturally exist several Bayesian\footnote{Bayesian methods can of course not be reduced to nonconvex optimization, but given that they are often characterized by multimodality and that corresponding variational formulations are typically nonconvex, we conveniently discuss them here.} approaches to sparsity.

As a first remark, regularized optimization can be viewed as solving a maximum a posteriori (MAP) estimation problem if the loss $\ell$ (cf.  Sec.~\ref{sec:losses}) defining $f$ can be interpreted as a log-likelihood and the norm as certain log-prior. Typically, the $\ell_1$-norm can for instance be interpreted as the logarithm of a product of independent Laplace priors on the loading vectors $\mathbf{w}$ (see, e.g.,~\cite{seeger2008bayesian}). However, the Laplace distribution is actually not a sparse prior, in the sense that it is a continuous distribution whose samples are thus nonzero almost surely. Besides the fact that MAP estimation is generally not considered as a Bayesian method (the fundamental principle of Bayesian method is to integrate over the uncertainty and avoid point estimates), evidence in the literature suggests that MAP is not a good principle to yield estimators that are adapted to the corresponding prior. In particular, the Lasso does in fact not provide a good algorithm 
to estimate vectors whose coefficients follow a Laplace distribution~\cite{gribonval2011compressible}! 

To obtain exact zeros in a Bayesian setting, one must use so called ``spike and slab" priors~\cite{ishwaran2005spike}. Inference with such priors leads to nonconvex optimization problems, and sampling methods, while also simple to implement, do not come with any guarantees, in particular in high-dimensional settings.

In reality, while obtaining exact zeros can be valuable from a computational point of view, it is a priori not necessary to obtain theoretical guarantees associated with sparse methods. In fact,
sparsity should be understood as the requirement or expectation that a few coefficients are significantly larger than most of the rest, an idea which is somehow formalized as \emph{compressibility} in the compressed sensing literature, 
and which inspires \emph{automatic relevance determination} methods (ARD) and the use of \emph{heavy-tail priors} among Bayesian approaches \cite{neal1996bayesian,wipf2004sparse}. Using heavy-tailed prior distribution on $\mathbf{w}_i$ allows to obtain posterior estimates with many small values and a few large values, this effect being stronger when the tails are heavier, in particular with Student's $t$-distribution. Heavy-tailed distributions and ARD are very related 
since these prior distributions can be expressed as scaled mixture of Gaussians~\cite{Archambeau:NIPS08b,Caron:ICML08}.  This is of interest computationally, in particular for variational methods. 

Variable selection is also obtained in a Bayesian setting by optimizing the marginal likelihood of the data over the hyper-parameters, that is using \emph{empirical Bayes} estimation; in that context iterative methods based on DC programming may be efficiently used~\cite{wipf2008new}.

It should be noted that the heavy-tail prior formulation points to an interesting connection
between sparsity and the notion of robustness in statistics, in which a sparse subset of the data is allowed to take large values. 
This is is also suggested by work such as~\cite{wright2008robust,xu2010robust}.

\chapter{Quantitative Evaluation}
\label{sec:exp_intro}
To illustrate and compare the methods presented in this paper, we consider in this section three benchmarks.
These benchmarks are chosen to be representative of problems regularized with sparsity-inducing norms, involving different norms and different loss functions.
To make comparisons that are as fair as possible, each algorithm is implemented in \texttt{C/C++}, 
using efficient BLAS and LAPACK libraries for basic linear algebra operations. Most of these implementations have been made available in the open-source software SPAMS\footnote{\url{http://www.di.ens.fr/willow/SPAMS/}}.
All subsequent simulations are run on a single core of a 3.07Ghz CPU, with 8GB of memory.
In addition, 
we take into account several criteria which strongly
influence the convergence speed of the algorithms. 
In particular, we consider  
\begin{itemize}
\item[(a)] different problem scales, 
\item[(b)] different levels of correlations between input variables, 
\item[(c)] different strengths of regularization.
\end{itemize}
We also show the influence of the required precision by monitoring the time of
computation as a function of the objective function.

For the convenience of the reader, we list here the algorithms compared and the acronyms we use to refer to them throughout this section: 
the homotopy/LARS algorithm (LARS), 
coordinate-descent (CD), 
reweighted-$\ell_2$ schemes (Re-$\ell_2$), 
simple proximal method (ISTA) and its accelerated version (FISTA). Note that all methods except the working set methods are very simple to implement as each iteration is straightforward (for proximal methods such as FISTA or ISTA, as long as the proximal operator may be computed efficiently). On the contrary, as detailed in Section~\ref{sec:lars}, homotopy methods require some care in order to achieve the performance we report in this section.

We also include in the comparisons generic algorithms such as a
subgradient descent algorithm (SG), and a commercial software\footnote{Mosek, available at \texttt{http://www.mosek.com/}.}
for cone (CP), quadratic (QP) and second-order cone programming (SOCP) problems.

\section{Speed Benchmarks for Lasso}
\label{sec:exp_speed}
We first present a large benchmark evaluating the performance of
various optimization methods for solving the Lasso.

We perform small-scale $(n=200,p=200)$ and 
medium-scale $(n=2\,000,p=10\,000)$ experiments.
We generate design matrices as follows. For the scenario with low
correlations, all entries of $\mx$ are independently drawn from a Gaussian
distribution ${\mathcal N}(0,1/n)$, which is a setting often used to
evaluate optimization algorithms in the literature. For the scenario with
large correlations, we draw the rows of the matrix $\mx$ from a multivariate
Gaussian distribution for which the \emph{average absolute value} of the
correlation between two different columns is eight times the one of the scenario with
low correlations. 
Test data vectors $\vy=\mx\vw+\vn$
where $\vw$ are randomly generated, with two
levels of sparsity to be used with the two different levels of
regularization; $\vn$ is a noise vector whose entries are i.i.d. samples
from a Gaussian distribution ${\mathcal N}(0,0.01\|\mx\vw\|_2^2/n)$.
In the low regularization setting the sparsity of the vectors
$\vw$ is $s=0.5\min(n,p)$, and in the high regularization one $s=0.01\min(n,p)$,
corresponding to fairly sparse vectors.
For SG, 
we take the step size to be equal to
$a/(k+b)$, where $k$ is the iteration number, and $(a,b)$ are the best\footnote{``The best step size'' is understood here as being the step size leading to the smallest objective function after 500 iterations.}
parameters selected on a logarithmic grid 
$(a,b) \in \{10^{−3},\ldots,10\}\times \{10^2, 10^3, 10^4\}$; 
we proceeded this way not to disadvantage SG by an arbitrary choice of stepsize. 

To sum up, we make a comparison for $8$ different conditions ($2$
scales $\times$ $2$ levels of correlation $\times$ $2$ levels of regularization).
All results are reported on Figures~\ref{fig:curvesa},~\ref{fig:curvesb},
by averaging $5$ runs for each experiment.
Interestingly, we observe that the relative performance of the different methods
 change significantly with the scenario. 

\begin{figure}[h]
	\centering
   \subfloat[corr: low, reg: low]{\includegraphics[width=0.49\linewidth]{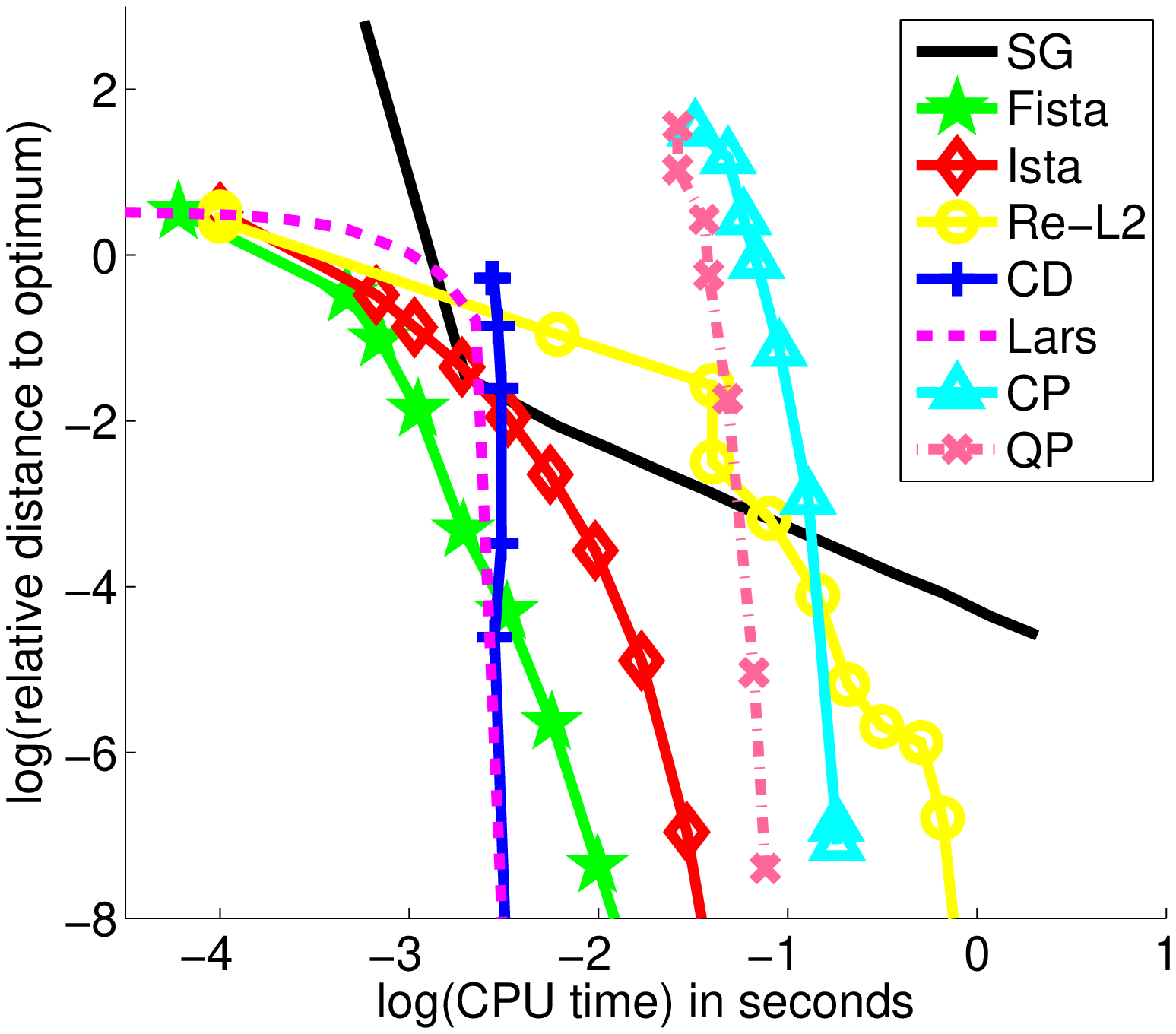}} \hfill 
   \subfloat[corr: low, reg: high]{\includegraphics[width=0.49\linewidth]{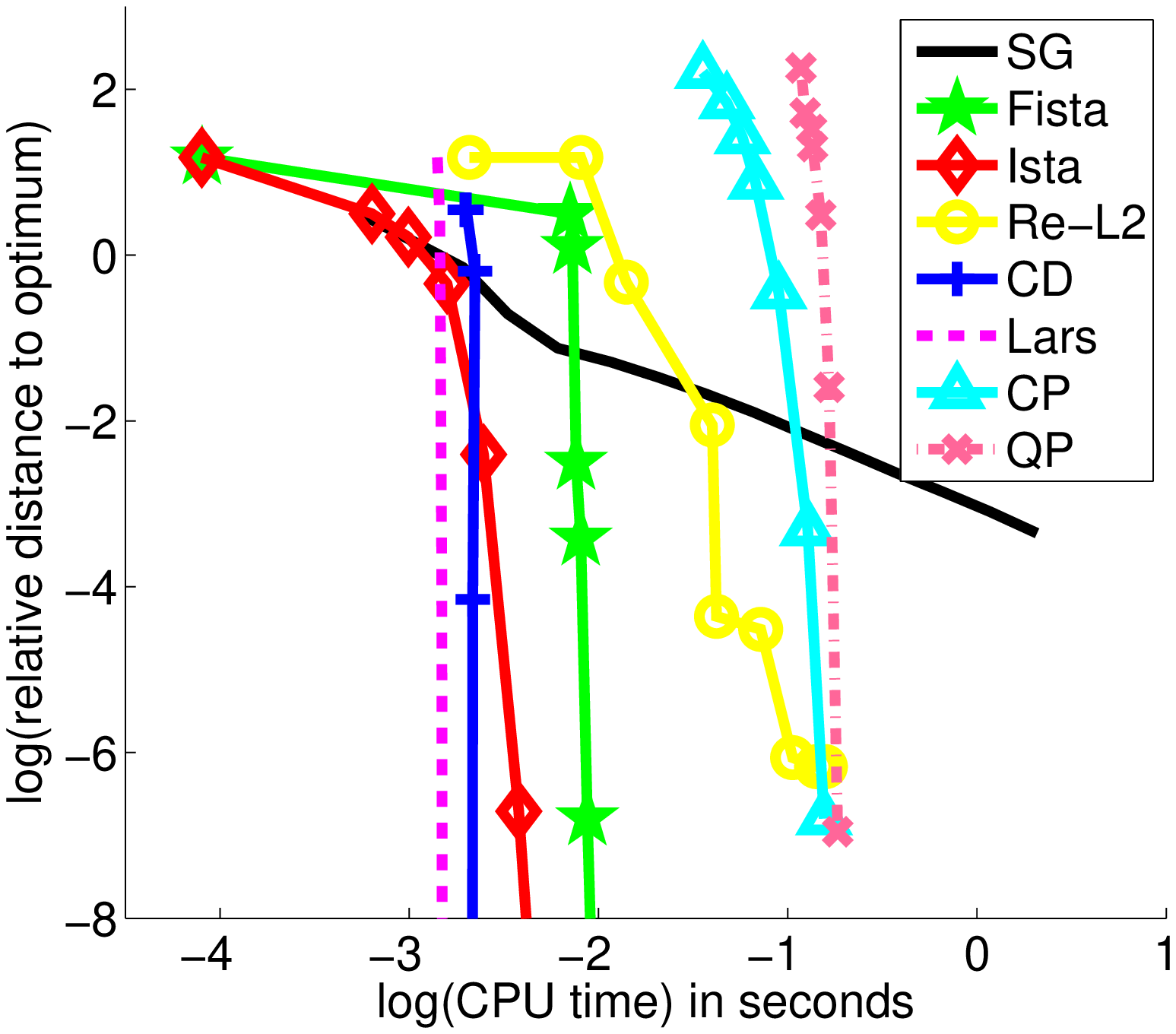}} \\
   \subfloat[corr: high, reg: low]{\includegraphics[width=0.49\linewidth]{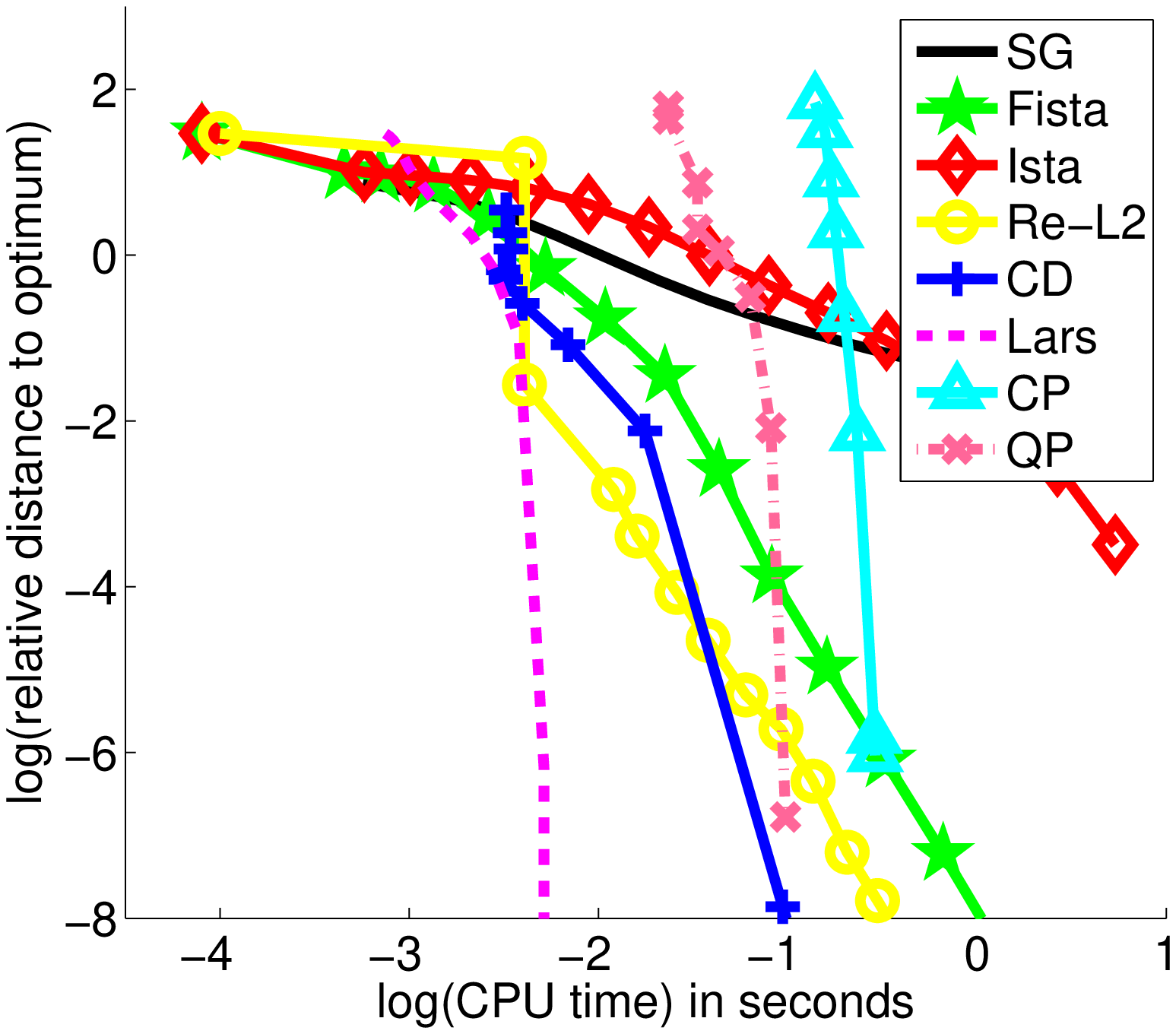}} \hfill 
   \subfloat[corr: high, reg: high]{\includegraphics[width=0.49\linewidth]{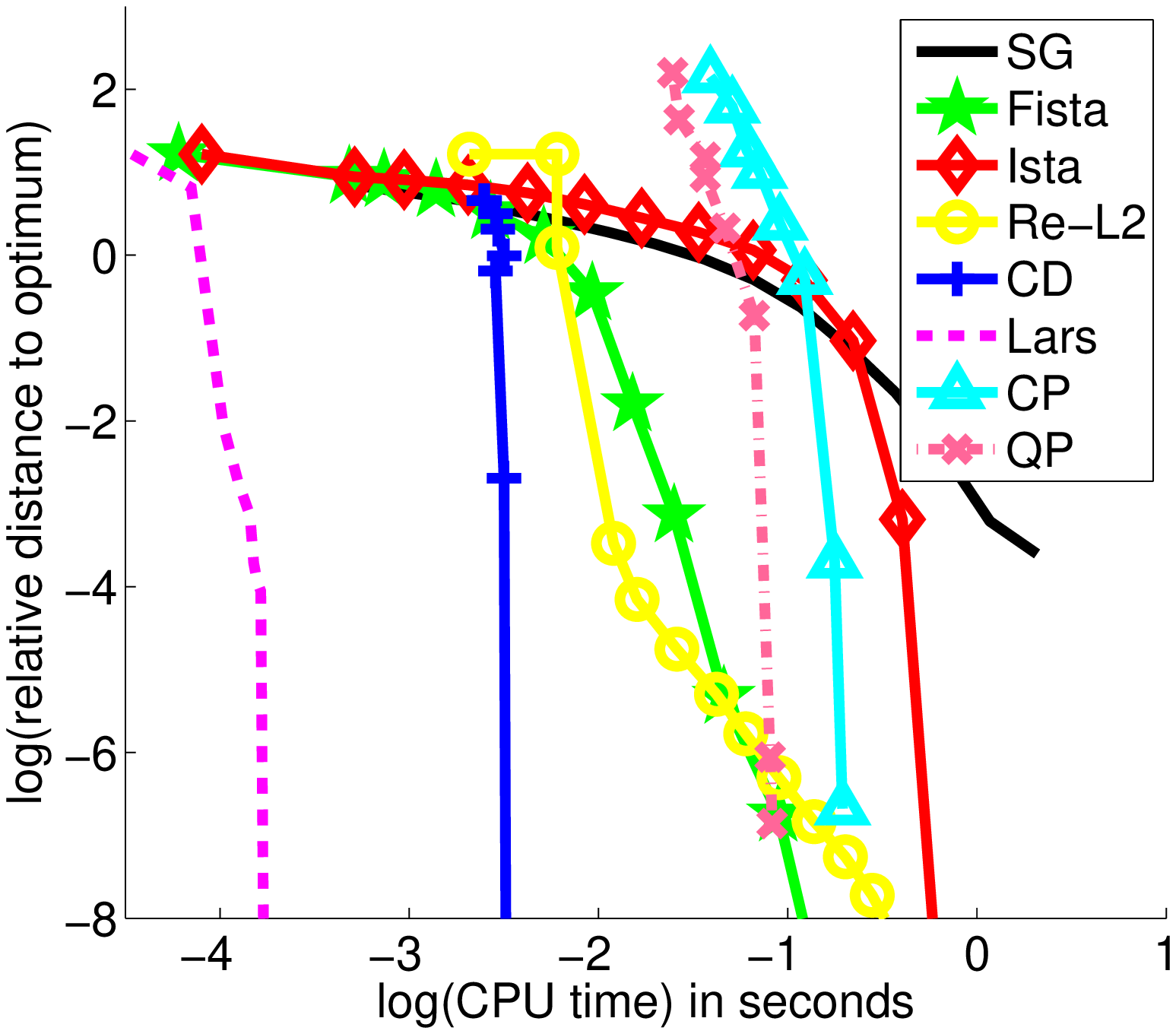}} 
   \caption{Benchmark for solving the Lasso for the small-scale experiment ($n=200,\ p=200$), for the two levels of correlation and two levels of regularization, and $8$ optimization methods (see main text for details). The curves represent the relative value of the objective function as a function of the computational time in second on a $\log_{10}/\log_{10}$ scale.}
   \label{fig:curvesa}
\end{figure}

\begin{figure}[h]
\centering
   \subfloat[corr: low, reg: low]{\includegraphics[width=0.49\linewidth]{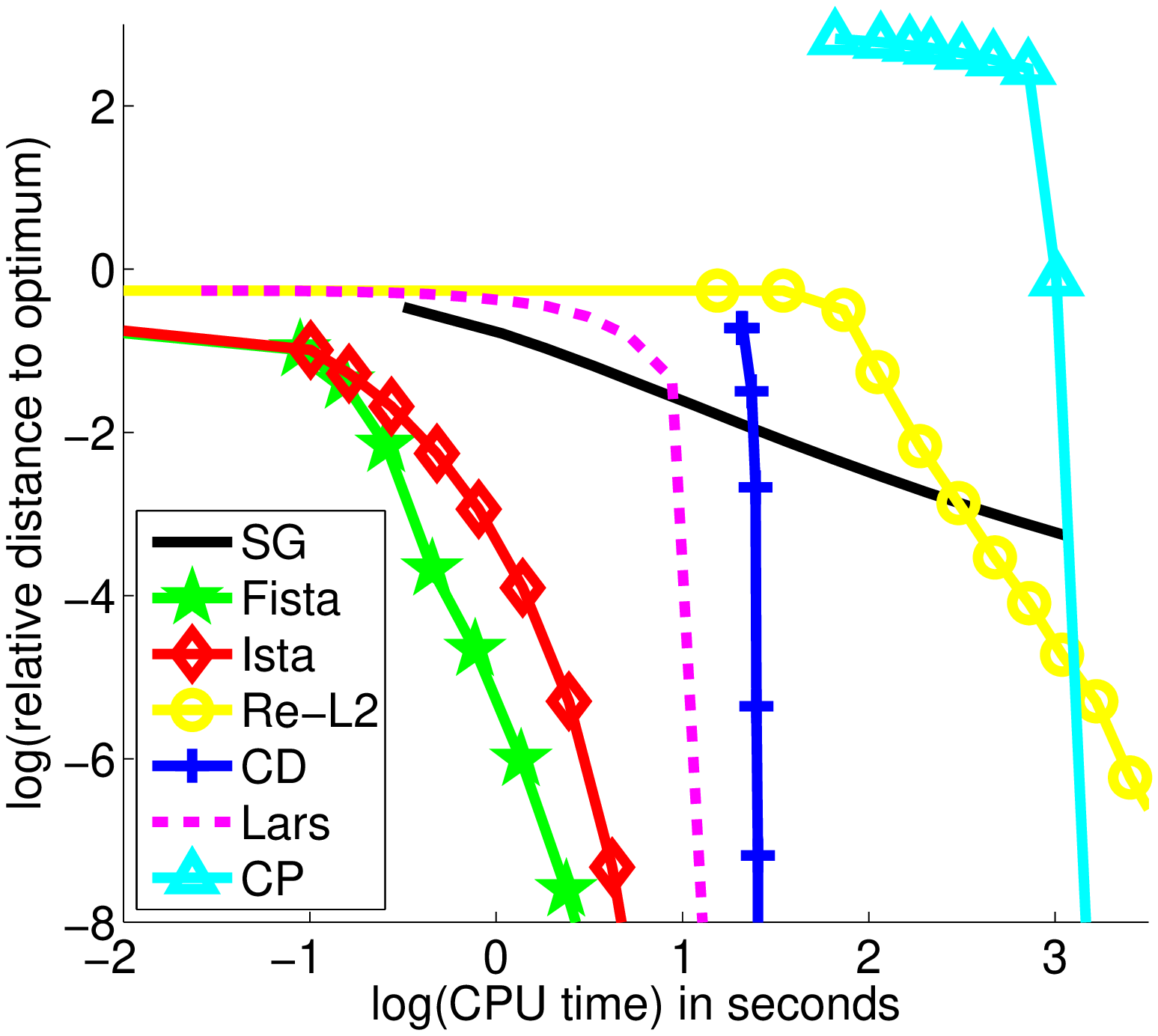}}  \hfill
   \subfloat[corr: low, reg: high]{\includegraphics[width=0.49\linewidth]{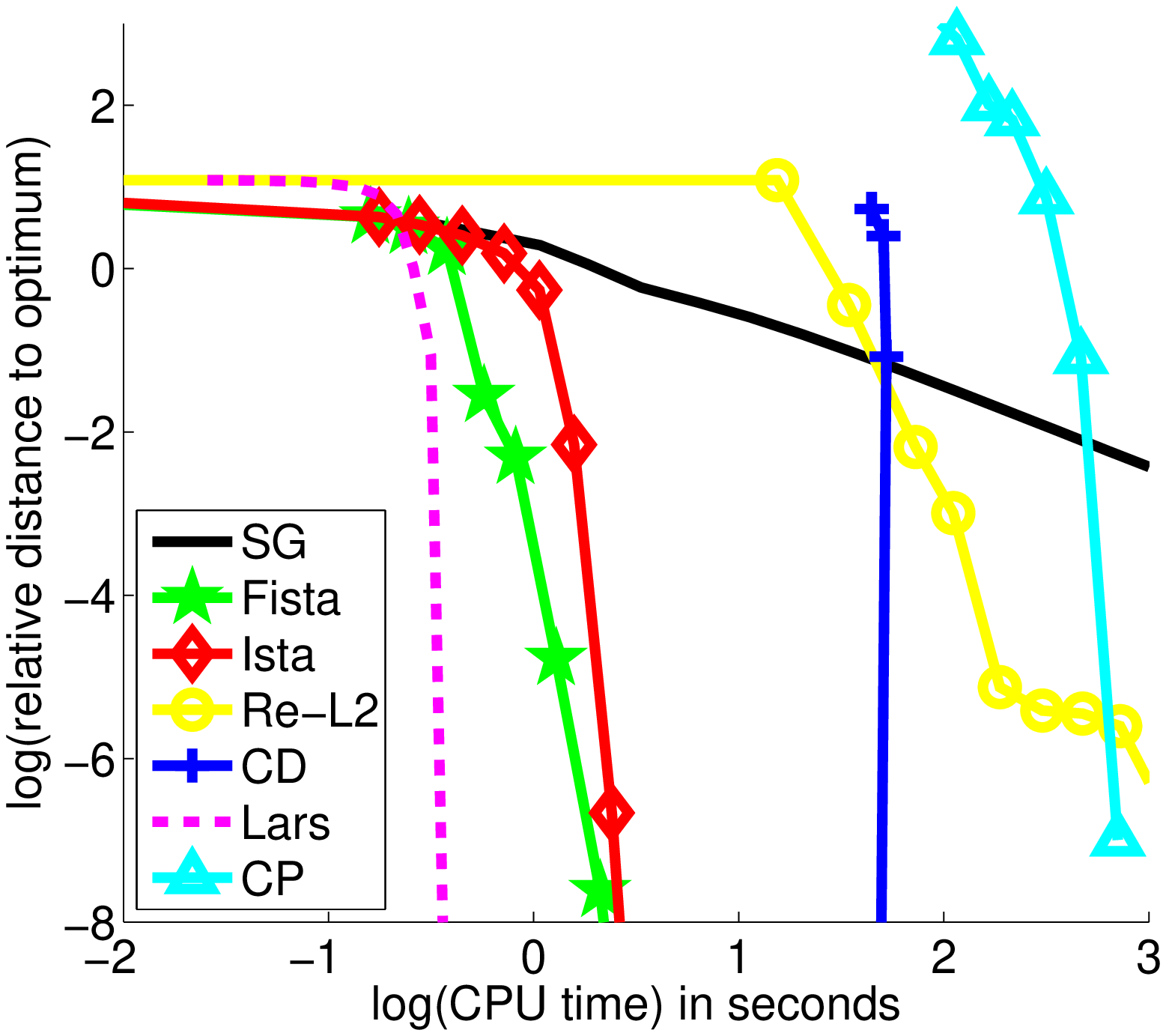}} \\
   \subfloat[corr: high, reg: low]{\includegraphics[width=0.49\linewidth]{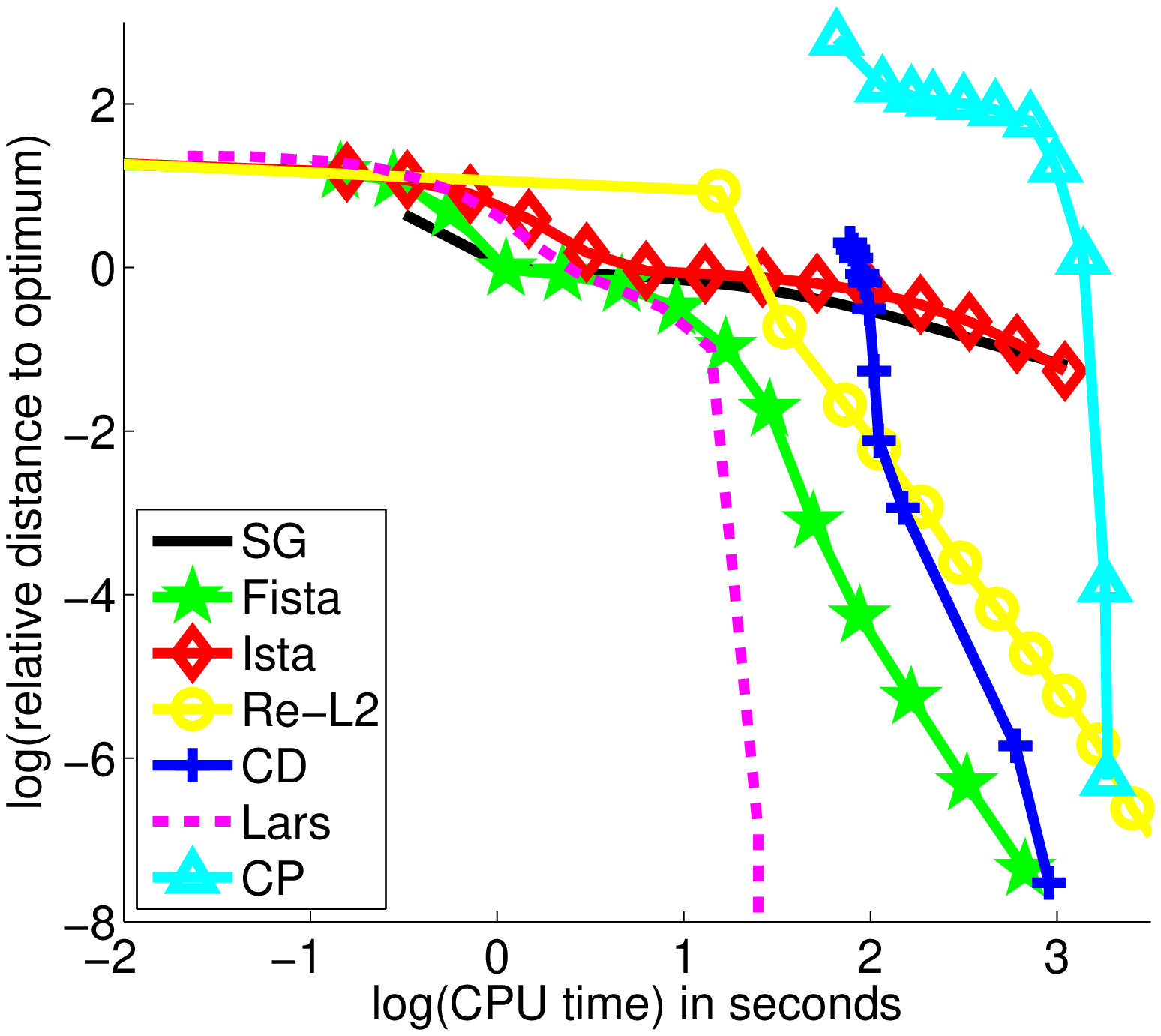}} \hfill
   \subfloat[corr: high, reg: high]{\includegraphics[width=0.49\linewidth]{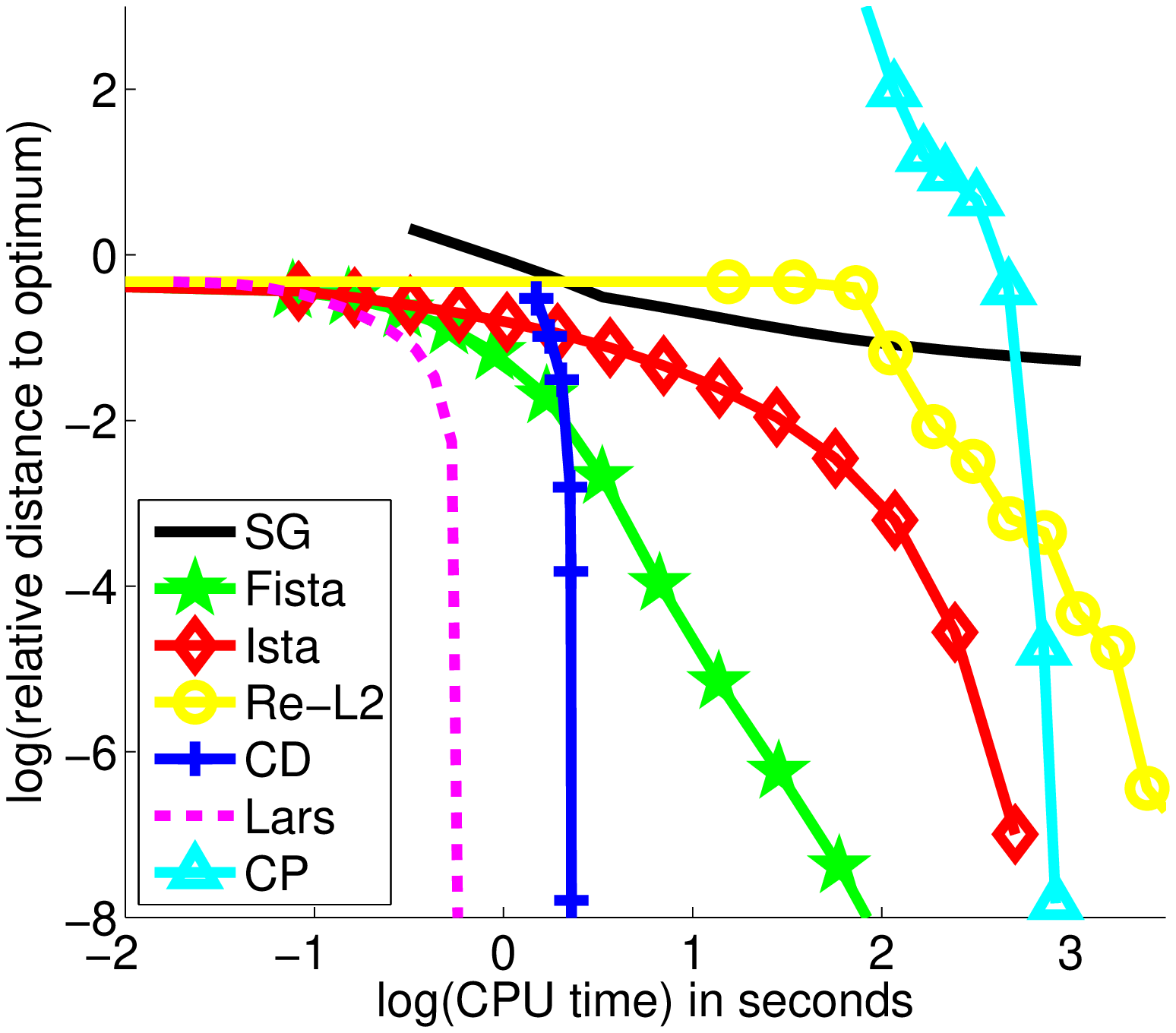}} 
   \caption{Benchmark for solving the Lasso for the medium-scale experiment $n=2\,000,\ p=10\,000$, for the two levels of correlation and two levels of regularization, and $8$ optimization methods (see main text for details). The curves represent the relative value of the objective function as a function of the computational time in second on a $\log_{10}/\log_{10}$ scale.}
   \label{fig:curvesb}
\end{figure}

Our conclusions for the different methods are as follows:
\begin{itemize}
   \item {\bfseries LARS/homotopy methods}: For the small-scale problem, LARS outperforms all
other methods for almost every scenario and precision regime.  It is therefore
\emph{definitely the right choice for the small-scale setting.} 
Unlike first-order methods, its performance
does not depend on the correlation of the design matrix $\mx$, but rather on the
sparsity $s$ of the solution.  In our larger scale setting, 
it has been competitive either when the solution is \emph{very sparse} (high
regularization), or when there is \emph{high correlation} in $\mx$ (in
that case, other methods do not perform as well). 
More importantly,
the homotopy algorithm gives an exact solution and computes the regularization path. 
   \item {\bfseries Proximal methods (ISTA, FISTA)}:
      FISTA outperforms ISTA in all scenarios but one. Both methods are close for high regularization or low
correlation, but FISTA is significantly better for high correlation or/and low
regularization.  These methods are almost always outperformed by LARS in the
small-scale setting, except for \emph{low precision and low correlation}.

Both methods \emph{suffer from correlated features}, which is consistent with the fact that 
their convergence rate depends on the correlation between input variables (convergence as a geometric sequence when the correlation matrix is invertible, and as the inverse of a degree-two polynomial otherwise).
 They are \emph{well
adapted to large-scale settings, with low or medium correlation}.
   \item {\bfseries Coordinate descent (CD)}: The theoretical analysis of these methods suggest that they behave in a similar way to proximal methods~\cite{nesterov2010efficiency,shalev2009stochastic}. However, empirically, we have observed
      that the behavior of CD often translates into a first ``warm-up'' stage followed by a fast convergence phase. 

      Its performance in the \emph{small-scale setting is competitive} (even
though always behind LARS), but \emph{less efficient in the large-scale one}.
For a reason we cannot explain, \emph{it suffers less than proximal methods
from correlated features.}
   \item {\bfseries Reweighted-$\ell_2$}: This method was outperformed in all our experiments by other dedicated methods.\footnote{Note that the reweighted-$\ell_2$ scheme 
requires solving iteratively large-scale linear system that are badly conditioned. Our implementation
uses LAPACK Cholesky decompositions, but a better performance might be obtained using a pre-conditioned conjugate gradient, especially in the very large scale setting.} Note that we considered only the smoothed alternating scheme of Section~\ref{sec:reweighted_l2} and not first order methods in $\veta$ such as that of \cite{simpleMKL}. A more exhaustive comparison should include these as well.
   \item {\bfseries Generic Methods (SG, QP, CP)}: 
      As expected, generic methods are not adapted for solving
the Lasso and are always outperformed by dedicated ones such as LARS.
\end{itemize}

Among the methods that we have presented, some require an initial overhead computation of the Gram matrix $\mx^\top\mx$: this is the case for coordinate descent and reweighted-$\ell_2$ methods. We took into account this overhead time in all figures, which explains the behavior of the corresponding convergence curves. Like homotopy methods, these methods could also benefit from an
offline pre-computation of $\bmt{\mx}\mx$ and would therefore be more competitive 
if the solutions corresponding to several values of the regularization parameter have to be computed.

We have considered in the above experiments the case of the square loss. 
Obviously, some of the conclusions drawn above would not be valid for other smooth losses. 
On the one hand, the LARS does no longer apply; on the other hand, 
proximal methods are clearly still available and coordinate descent schemes, 
which were dominated by the LARS in our experiments, 
would most likely turn out to be very good contenders in that setting.

\section{Group-Sparsity for Multi-task Learning}
\label{sec:exp_multitask}
For $\ell_1$-regularized least-squares regression, 
homotopy methods have appeared in the previous section as one of the best techniques, 
in almost all the experimental conditions. 

This second speed benchmark explores a setting where this homotopy approach cannot be applied anymore.
In particular, we consider a multi-class classification problem in the context of cancer diagnosis.
We address this problem from a multi-task viewpoint~\cite{obozinski-joint}.
To this end, we take the regularizer to be $\ell_1/\ell_2$- and $\ell_1/\ell_\infty$-norms, 
with (non-overlapping) groups of variables penalizing features across all classes~\cite{liu2009l1linf,obozinski-joint}.
As a data-fitting term, we now choose a simple ``1-vs-all'' logistic loss function.

We focus on two multi-class classification problems in the ``small $n$, large $p$'' setting,
based on two datasets\footnote{The two datasets we use are \textit{SRBCT} and \textit{14\_Tumors}, 
which are freely available at \texttt{http://www.gems-system.org/}.} 
of gene expressions.
The medium-scale dataset contains $n=83$ observations, $p=2\,308$ variables and $4$ classes, 
while the large-scale one contains $n=308$ samples, $p=15\,009$ variables and $26$ classes.
Both datasets exhibit highly-correlated features.

In addition to ISTA, FISTA, and SG, we consider here the block coordinate-descent (BCD) from~\cite{tseng2009coordinate}
presented in Section~\ref{sec:opt_methods_bcd}.
We also consider a working-set strategy on top of BCD, that optimizes over the full set of features
(including the non-active ones) only once every four iterations.
As further discussed in Section~\ref{sec:opt_methods_bcd}, 
it is worth mentioning that the multi-task setting is well suited for the method of~\cite{tseng2009coordinate} since
an appropriate approximation of the Hessian can be easily computed.
\begin{figure}[h]
   \centering
   \subfloat[scale: medium,\newline regul: low ]{\includegraphics[width=0.33\linewidth]{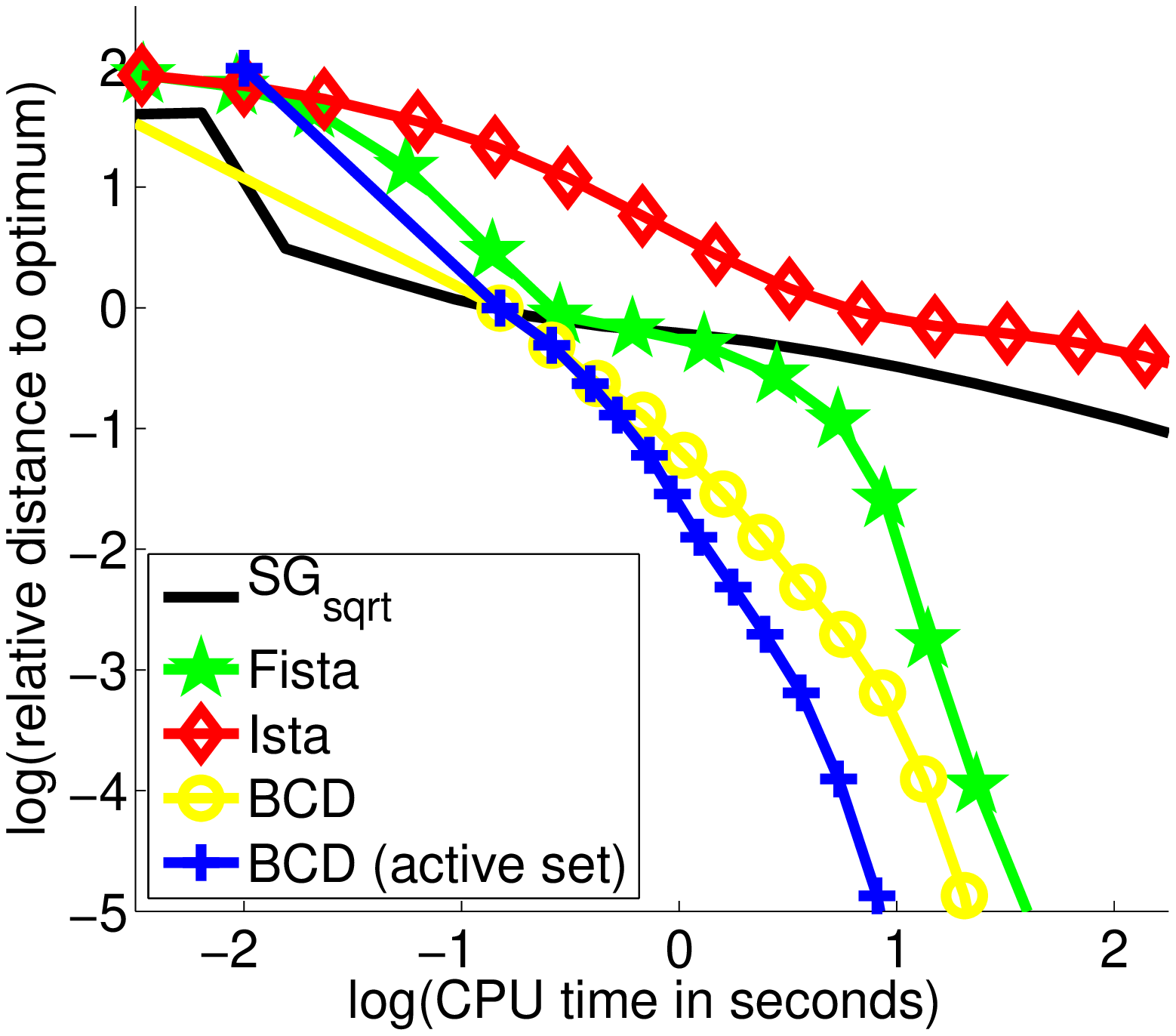}} \hfill
   \subfloat[scale: medium,\newline regul: medium]{\includegraphics[width=0.33\linewidth]{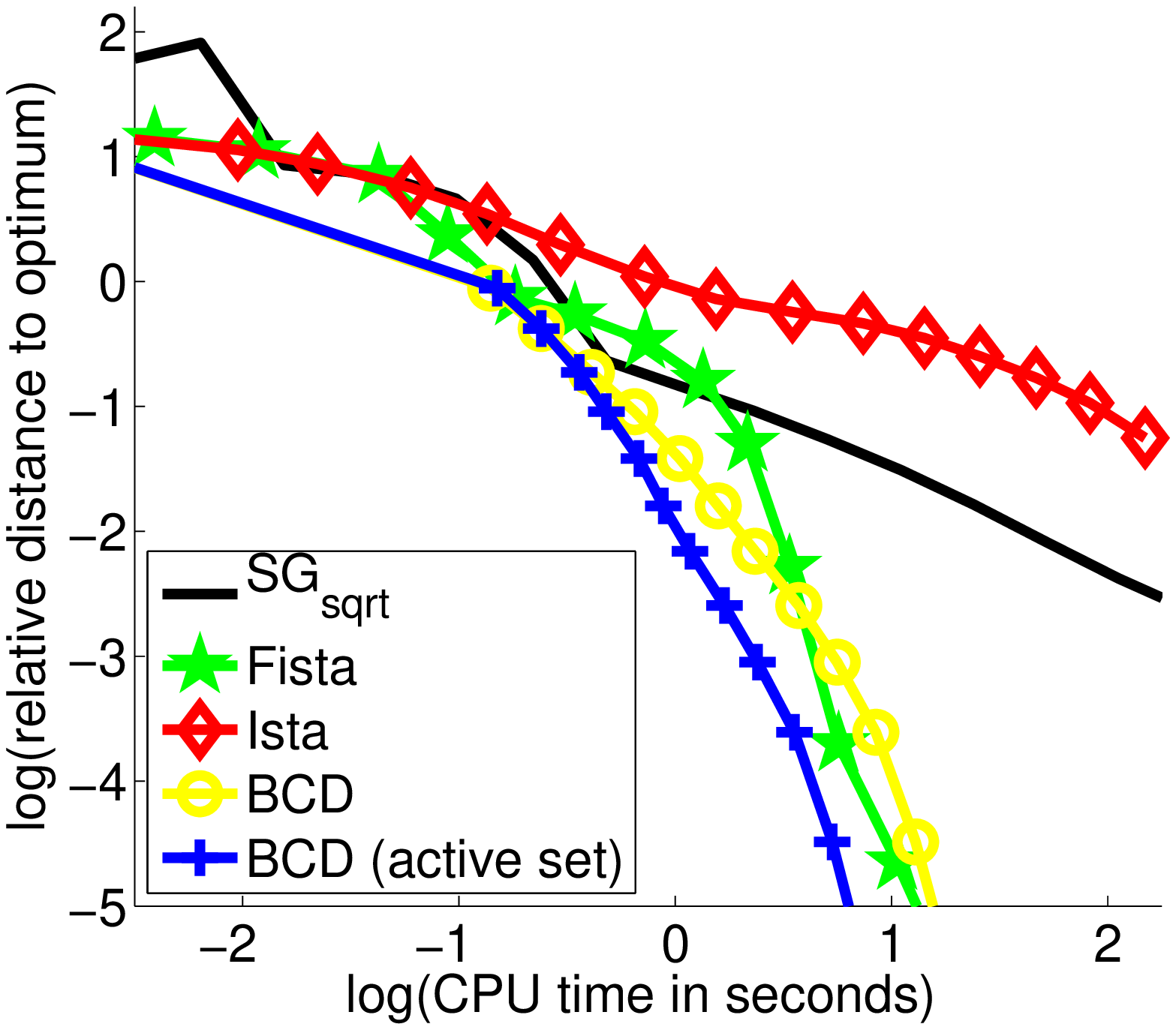}} \hfill
   \subfloat[scale: medium,\newline regul: high]{\includegraphics[width=0.33\linewidth]{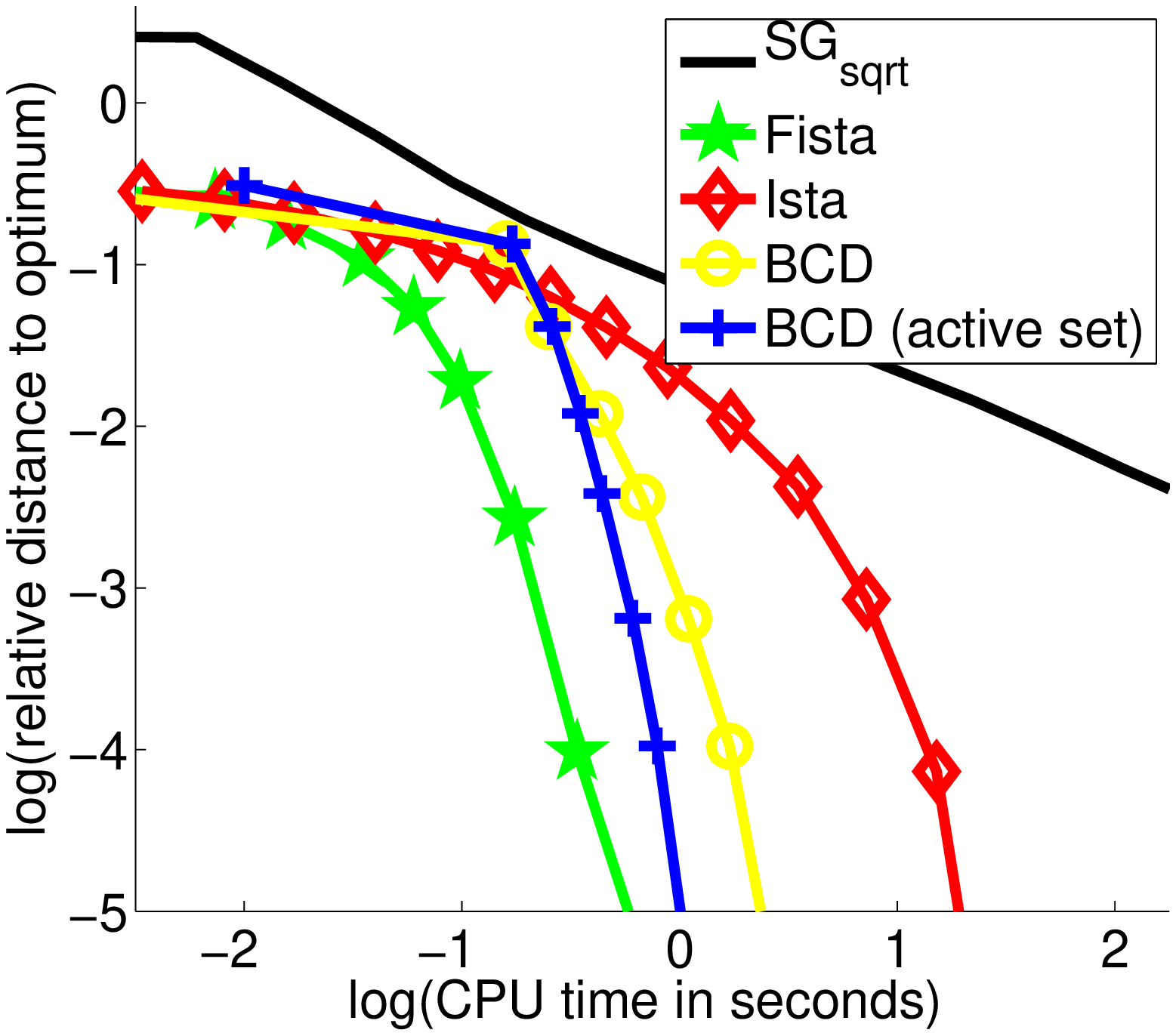}} \\
   \subfloat[scale: large,\newline regul: low]{\includegraphics[width=0.33\linewidth]{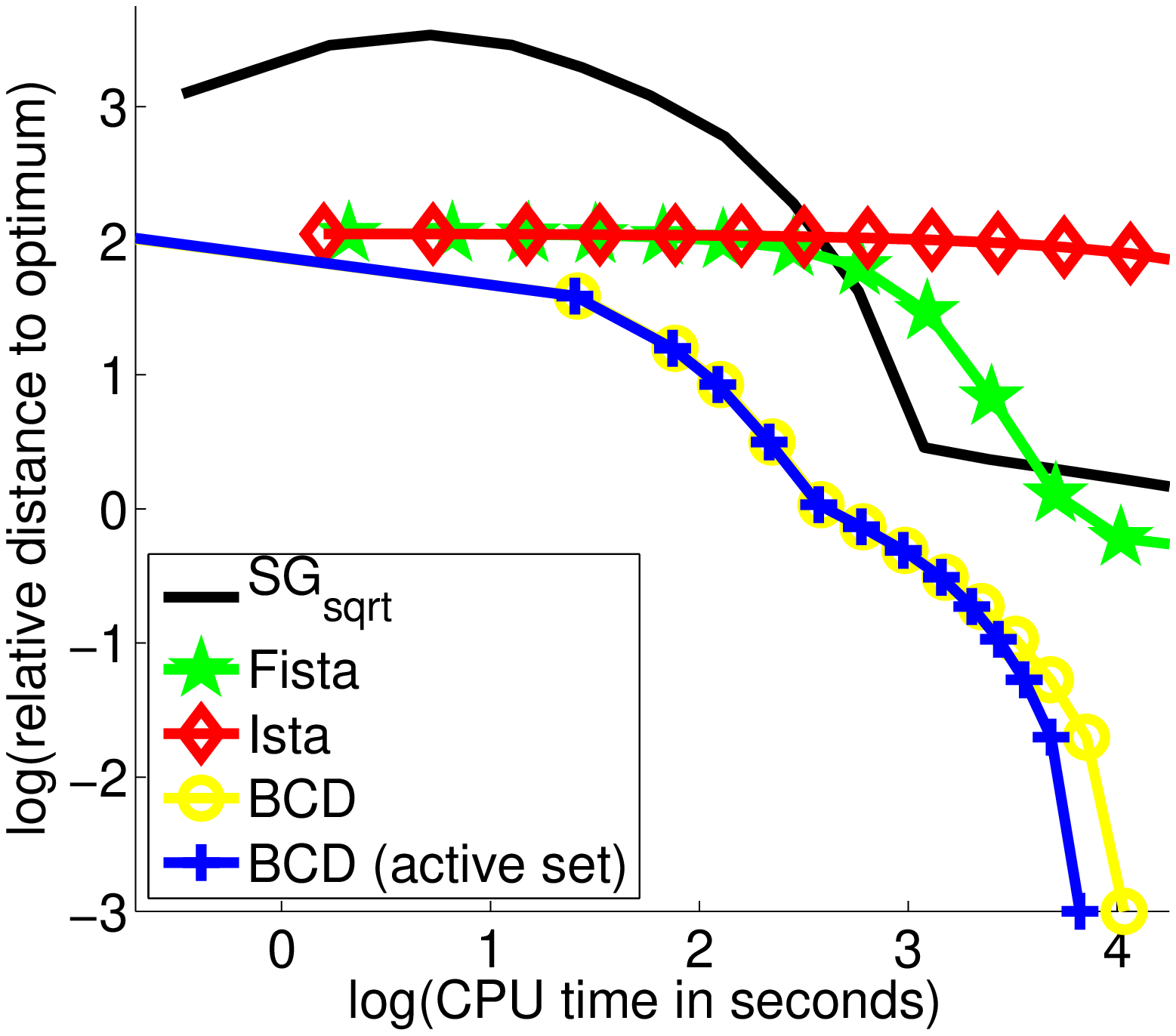}} \hfill
   \subfloat[scale: large,\newline regul: medium]{\includegraphics[width=0.33\linewidth]{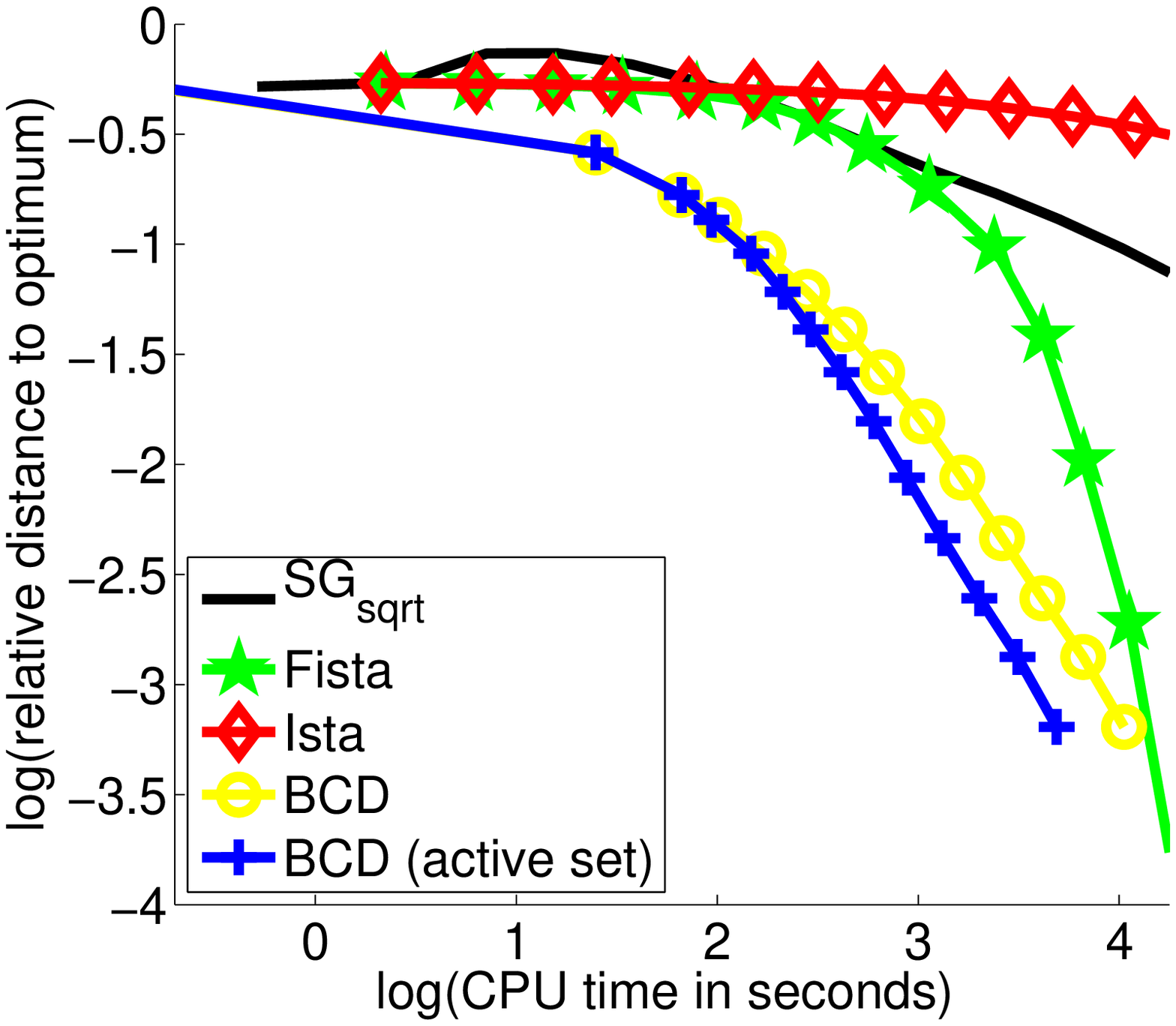}} \hfill
   \subfloat[scale: large,\newline regul: high]{\includegraphics[width=0.33\linewidth]{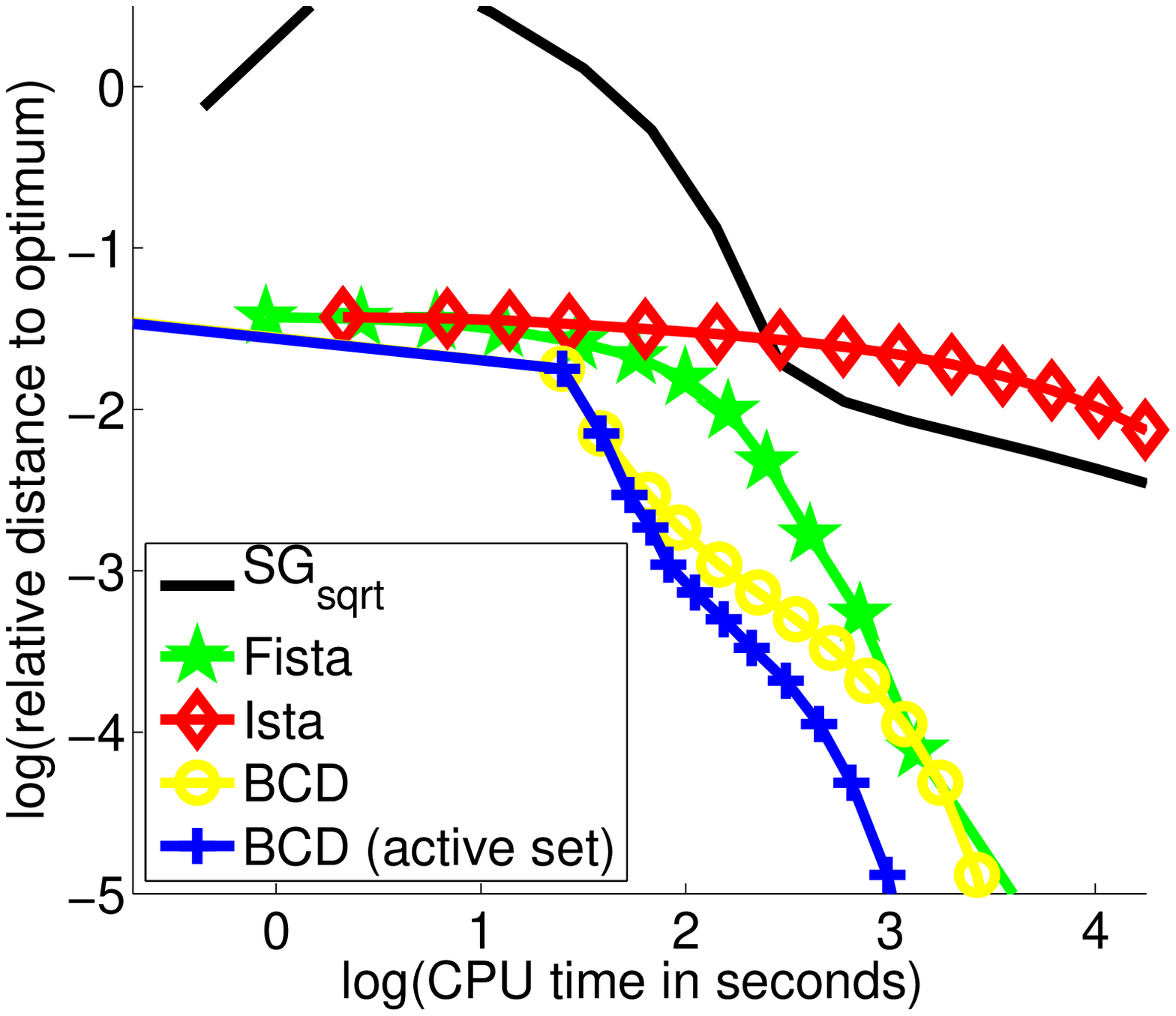}}
   \caption{Medium- and large-scale multi-class classification problems with an $\ell_1/\ell_2$-regularization,  
   for three optimization methods (see details about the datasets and the methods in the main text).
   Three levels of regularization are considered. The curves represent the relative value of the objective function as a function of the computation time in second on a $\log_{10}/\log_{10}$ scale.
   In the highly regularized setting, the tuning of the step-size for the subgradient turned out to be difficult, 
   which explains the behavior of SG in the first iterations.}
   \label{fig:mutlitask_bench_l1l2}
\end{figure}
\begin{figure}[h]
   \subfloat[scale: medium,\newline regul: low]{\includegraphics[width=0.33\linewidth]{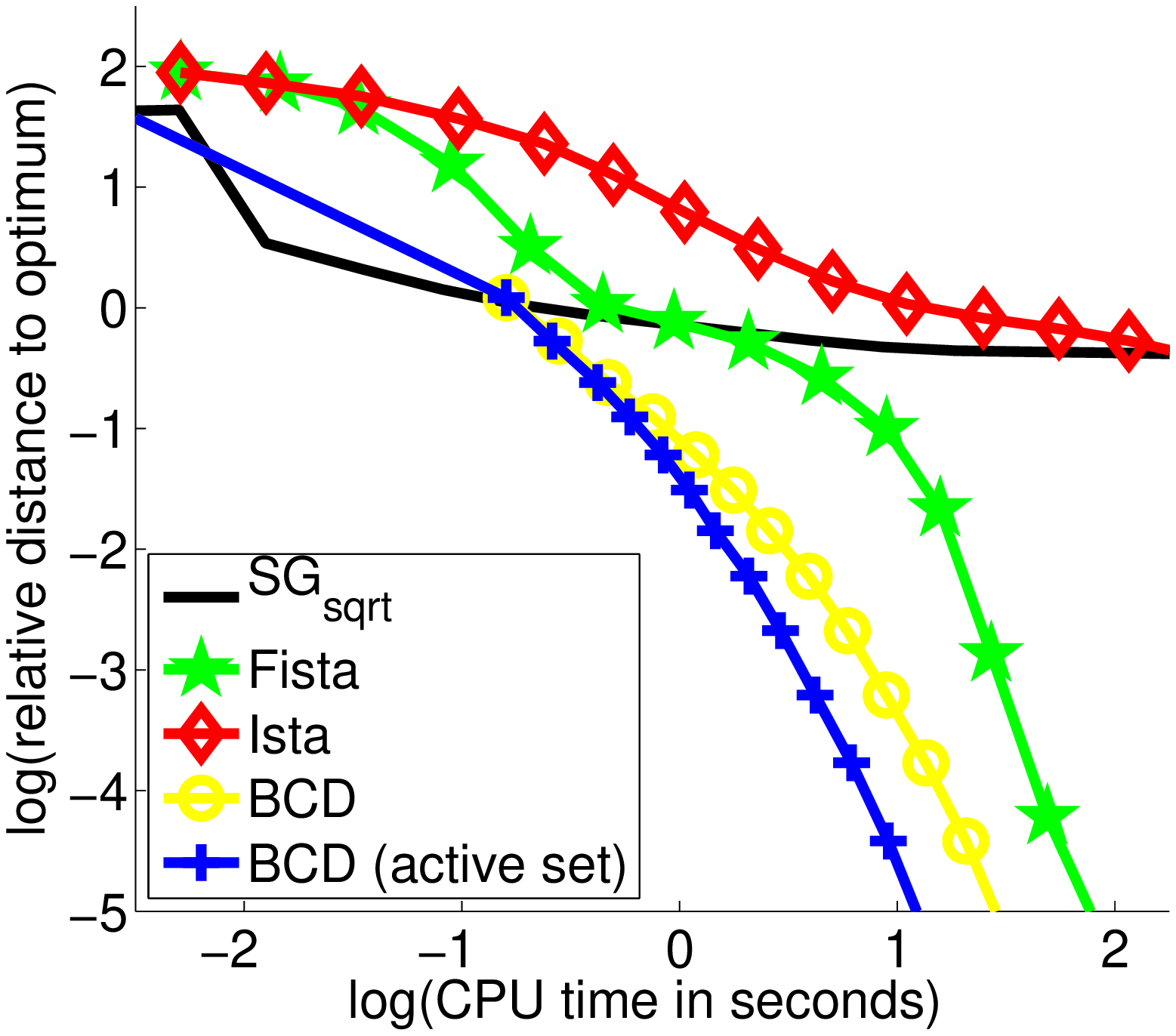}} \hfill
   \subfloat[scale: medium,\newline regul: medium]{\includegraphics[width=0.33\linewidth]{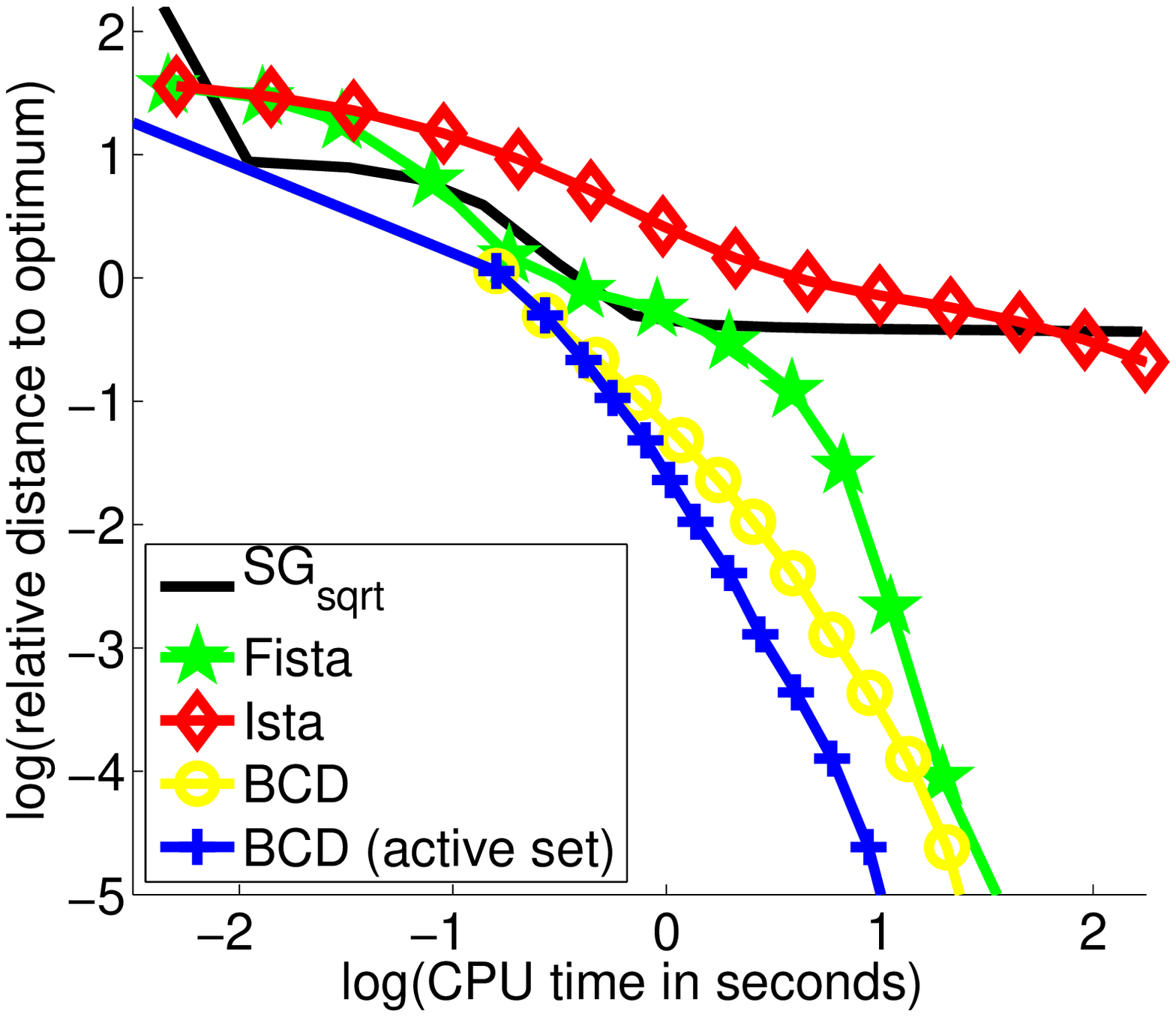}} \hfill
   \subfloat[scale: medium,\newline regul: high]{\includegraphics[width=0.33\linewidth]{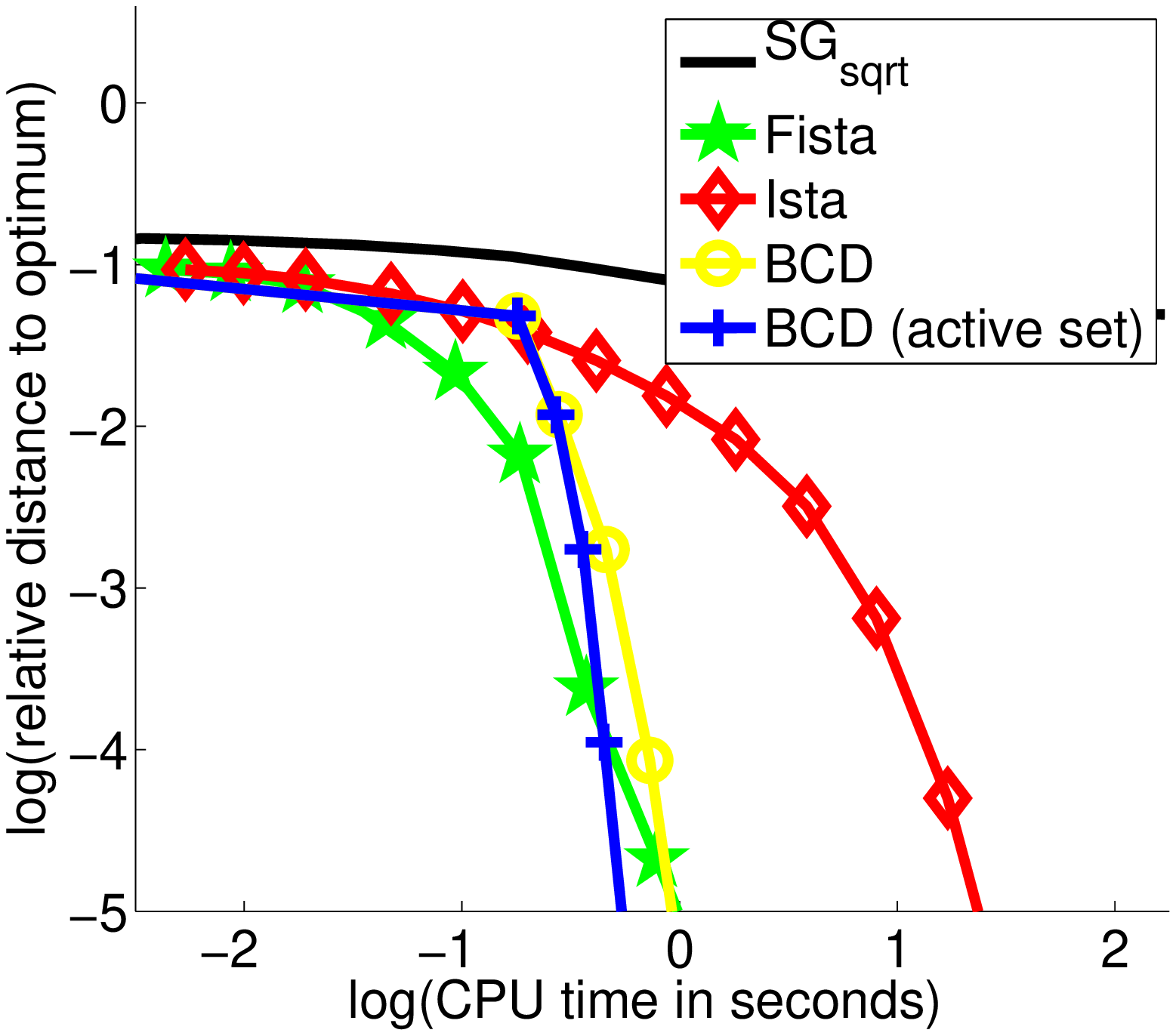}} \\
   \subfloat[scale: large,\newline regul: low]{\includegraphics[width=0.33\linewidth]{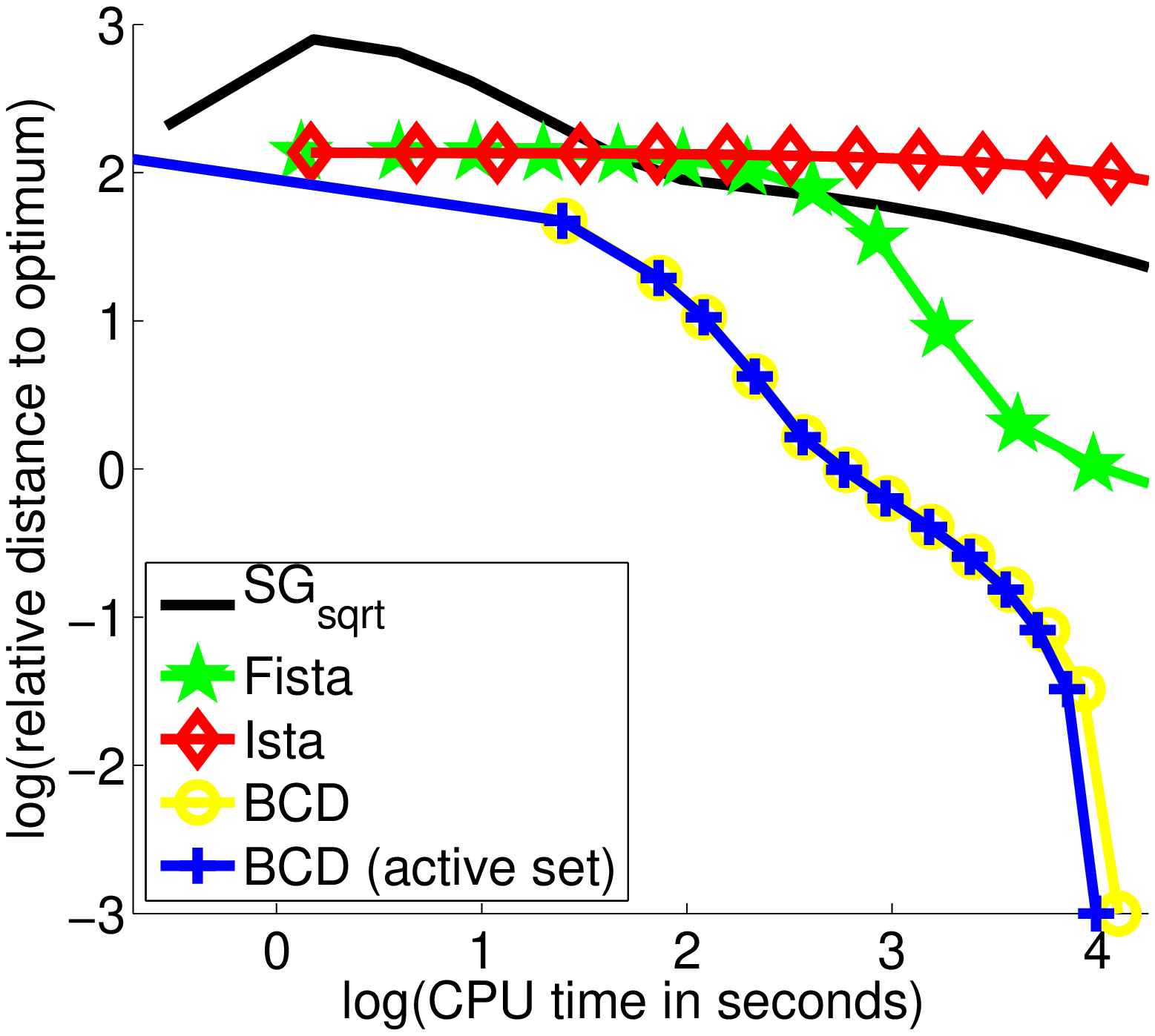}} \hfill
   \subfloat[scale: large,\newline regul: medium]{\includegraphics[width=0.33\linewidth]{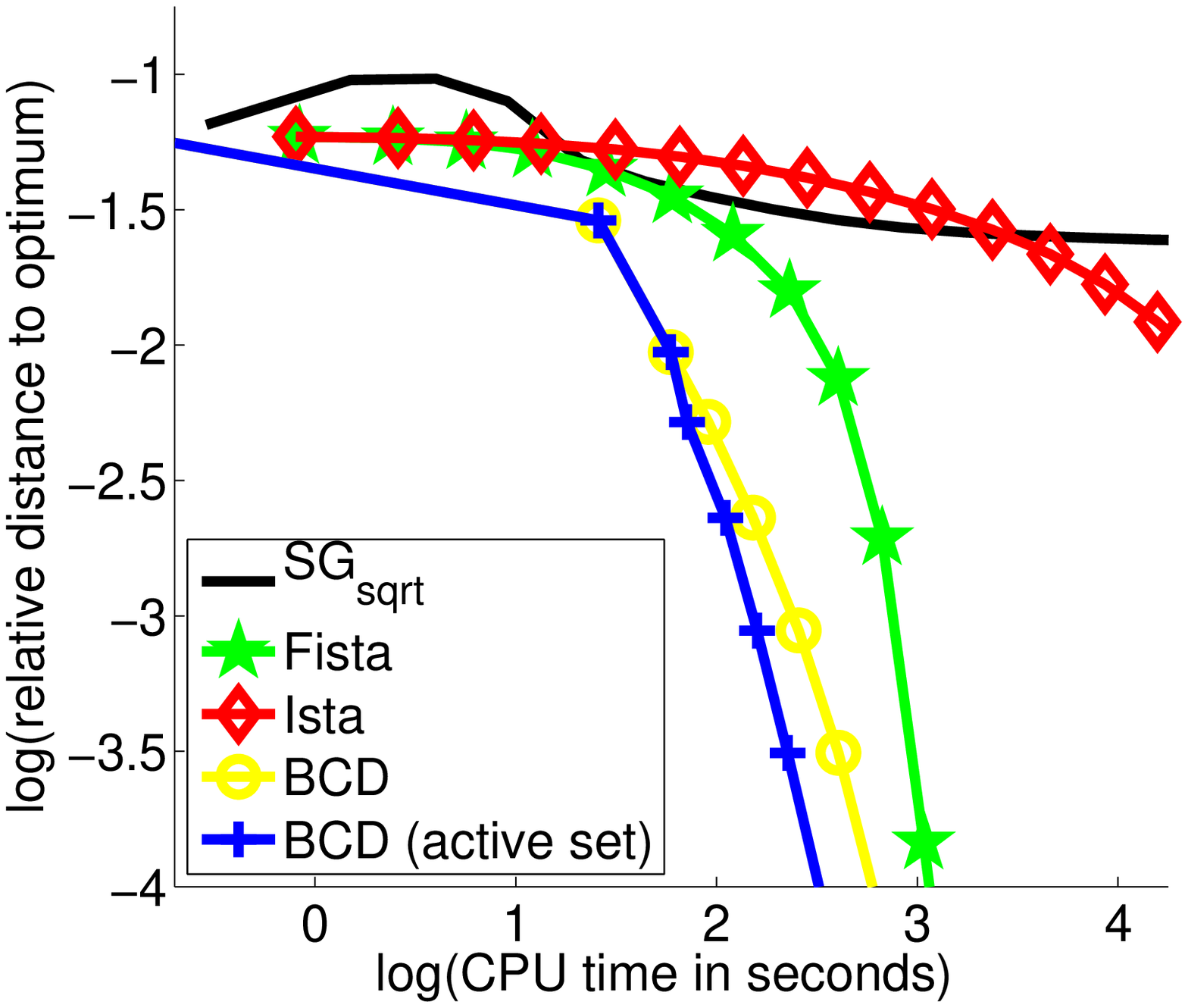}} \hfill
   \subfloat[scale: large,\newline regul: high]{\includegraphics[width=0.33\linewidth]{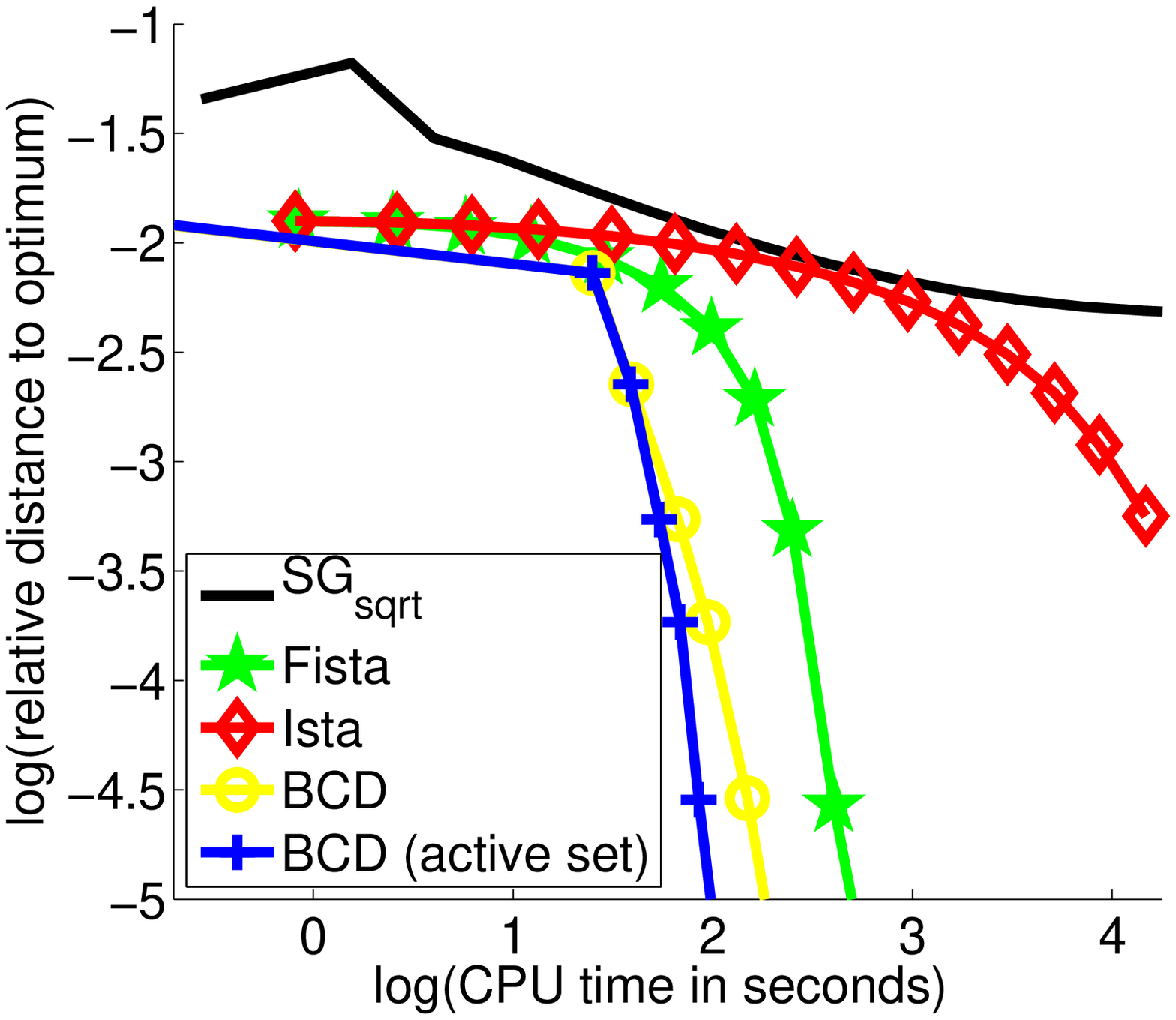}}
   \caption{Medium- and large-scale multi-class classification problems with an $\ell_1/\ell_\infty$-regularization 
   for three optimization methods (see details about the datasets and the methods in the main text).
   Three levels of regularization are considered. The curves represent the relative value of the objective function as a function of the computation time in second on a $\log_{10}/\log_{10}$ scale.
   In the highly regularized setting, the tuning of the step-size for the subgradient turned out to be difficult, 
   which explains the behavior of SG in the first iterations.}
   \label{fig:mutlitask_bench_l1linf}
\end{figure}

All the results are reported in Figures~\ref{fig:mutlitask_bench_l1l2} and \ref{fig:mutlitask_bench_l1linf}.
As expected in the light of the benchmark for the Lasso, BCD appears as the best option, 
regardless of the sparsity/scale conditions. 

\section{Structured Sparsity}
\label{sec:exp_struct}
In this second series of experiments, the optimization techniques of the previous sections
are further evaluated when applied to other types of loss and sparsity-inducing functions.
Instead of the $\ell_1$-norm previously studied, 
we focus on the particular \textit{hierarchical} $\ell_1/\ell_2$-norm $\Omega$ introduced in Section~\ref{sec:proximal_methods}.

From an optimization standpoint, although $\Omega$ shares some similarities with the $\ell_1$-norm (e.g., the convexity and the non-smoothness), it differs in that it cannot be decomposed into independent parts (because of the overlapping structure of $\Gc$).
Coordinate descent schemes hinge on this property and as a result, cannot be straightforwardly applied in this case.

\subsection{Denoising of Natural Image Patches}
In this first benchmark, we consider a least-squares regression problem regularized by $\Omega$ that arises in the context of the
denoising of natural image patches~\cite{Jenatton2010a}.
In particular, based on a hierarchical set of features that accounts for different types of edge orientations and frequencies in natural images, 
we seek to reconstruct noisy $16\!\times\!16$-patches.
Although the problem involves a small number of variables (namely $p=151$), it has to be solved repeatedly for thousands of patches, at moderate precision. It is therefore crucial to be able to solve this problem efficiently.

The algorithms that take part in the comparisons are 
ISTA, FISTA, Re-$\ell_2$, SG, and SOCP.
All results are reported in Figure~\ref{fig:struct_bench_patches}, by averaging $5$ runs.

We can draw several conclusions from the simulations.
First, we observe that across all levels of sparsity, 
the accelerated proximal scheme performs better, or similarly, than the other approaches.
In addition, as opposed to FISTA, ISTA seems to suffer in non-sparse scenarios.
In the least sparse setting, 
the reweighted-$\ell_2$ scheme matches the performance of FISTA.
However this scheme does not yield truly sparse solutions, and would therefore require a subsequent thresholding operation, which can be difficult to motivate in a principled way.
As expected, the generic techniques such as SG and SOCP do not compete with the dedicated algorithms. 
\begin{figure}
  \centering
   \subfloat[scale: small,\newline regul: low]{\includegraphics[width=0.33\linewidth]{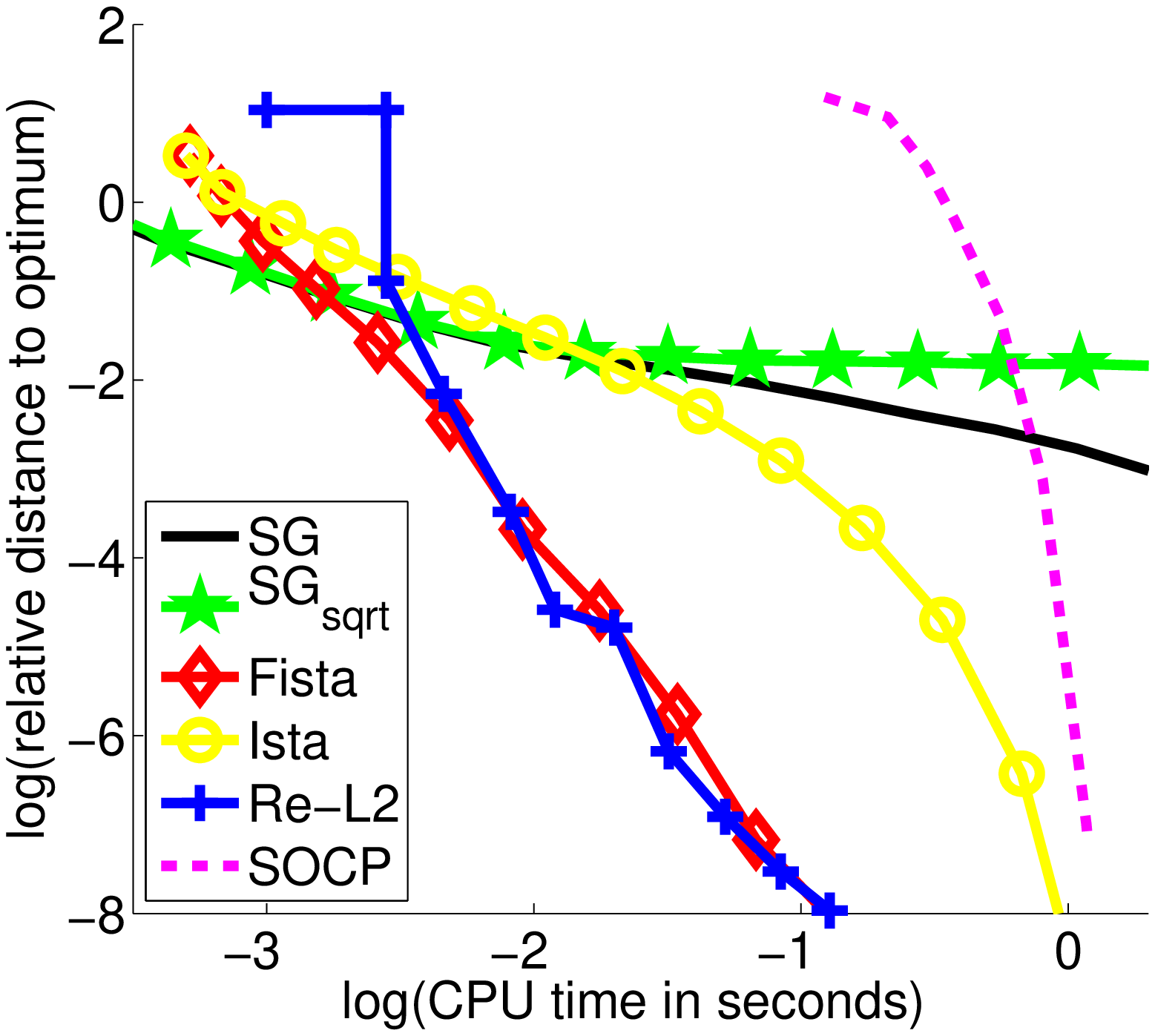}} \hfill
   \subfloat[scale: small,\newline regul: medium]{\includegraphics[width=0.33\linewidth]{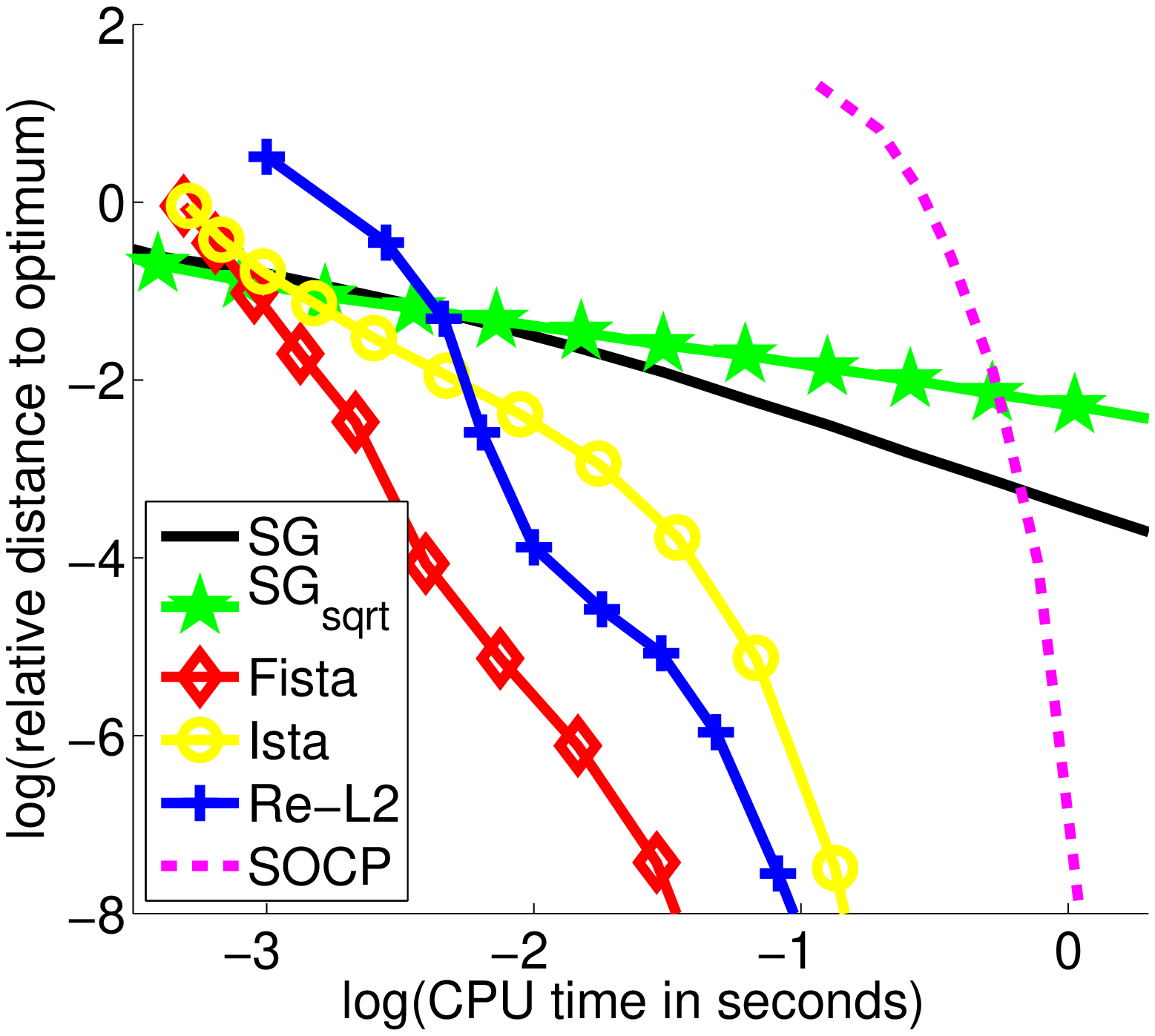}} \hfill
   \subfloat[scale: small,\newline regul: high]{\includegraphics[width=0.33\linewidth]{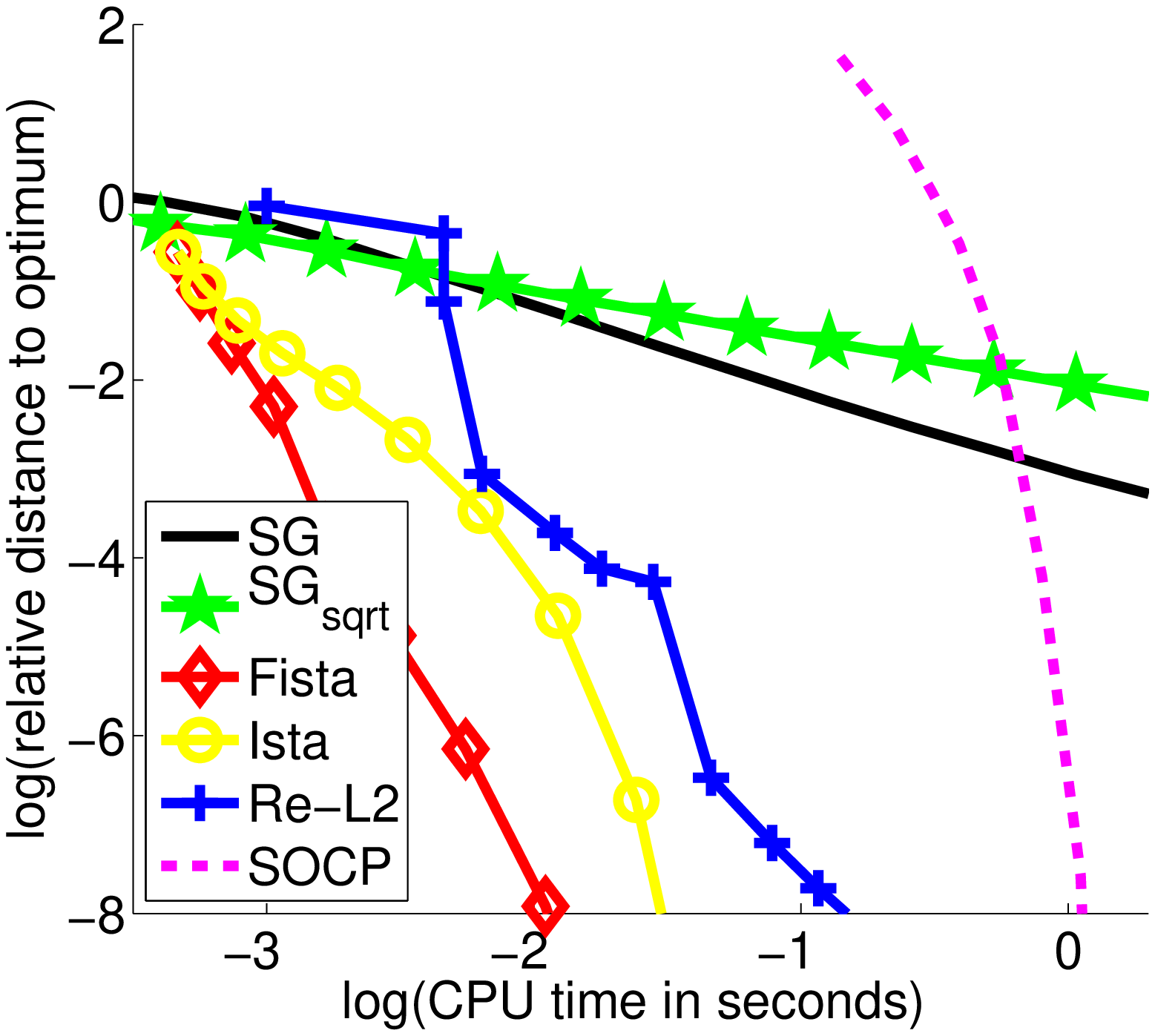}} \\
   \caption{Benchmark for solving a least-squares regression problem regularized by the hierarchical norm $\Omega$. 
   The experiment is small scale, $n=256, p=151$, and shows the performances of five optimization methods (see main text for details)
   for three levels of regularization. The curves represent the relative value of the objective function as a function of the computational time in second on a $\log_{10}/\log_{10}$ scale.}
   \label{fig:struct_bench_patches}
\end{figure}

\subsection{Multi-class Classification of Cancer Diagnosis}
This benchmark focuses on multi-class classification of cancer diagnosis and
reuses the two datasets from the multi-task problem of Section~\ref{sec:exp_multitask}.
Inspired by~\cite{kim3}, we build a tree-structured set of groups of features $\Gc$ by applying
Ward's hierarchical clustering~\cite{Johnson1967} on the gene expressions.
The norm $\Omega$ built that way aims at capturing the hierarchical structure of gene expression networks~\cite{kim3}.
For more details about this construction, see~\cite{Jenatton2011} in the context of neuroimaging.
The resulting datasets with tree-structured sets of features contain $p=4\,615$ and $p=30\,017$ variables, for 
respectively the medium- and large-scale datasets.

Instead of the square loss function, we consider the multinomial logistic loss function, which is better suited for multi-class classification problems. 
As a direct consequence, the algorithms whose applicability crucially depends on the choice of the loss function
are removed from the benchmark. 
This is for instance the case for reweighted-$\ell_2$ schemes that have closed-form updates available only with the square loss
(see Section~\ref{sec:reweighted_l2}).
Importantly, the choice of the multinomial logistic loss function requires one to optimize
over a matrix with dimensions $p$ times the number of classes 
(i.e., a total of $4\,615\times 4\approx18\,000$ and $30\,017\times 26\approx780\,000$ variables).
Also, for lack of scalability, generic interior point solvers could not be considered here.
To summarize, the following comparisons involve ISTA, FISTA, and SG.

\begin{figure}[h]
   \subfloat[scale: medium,\newline regul: low]{\includegraphics[width=0.33\linewidth]{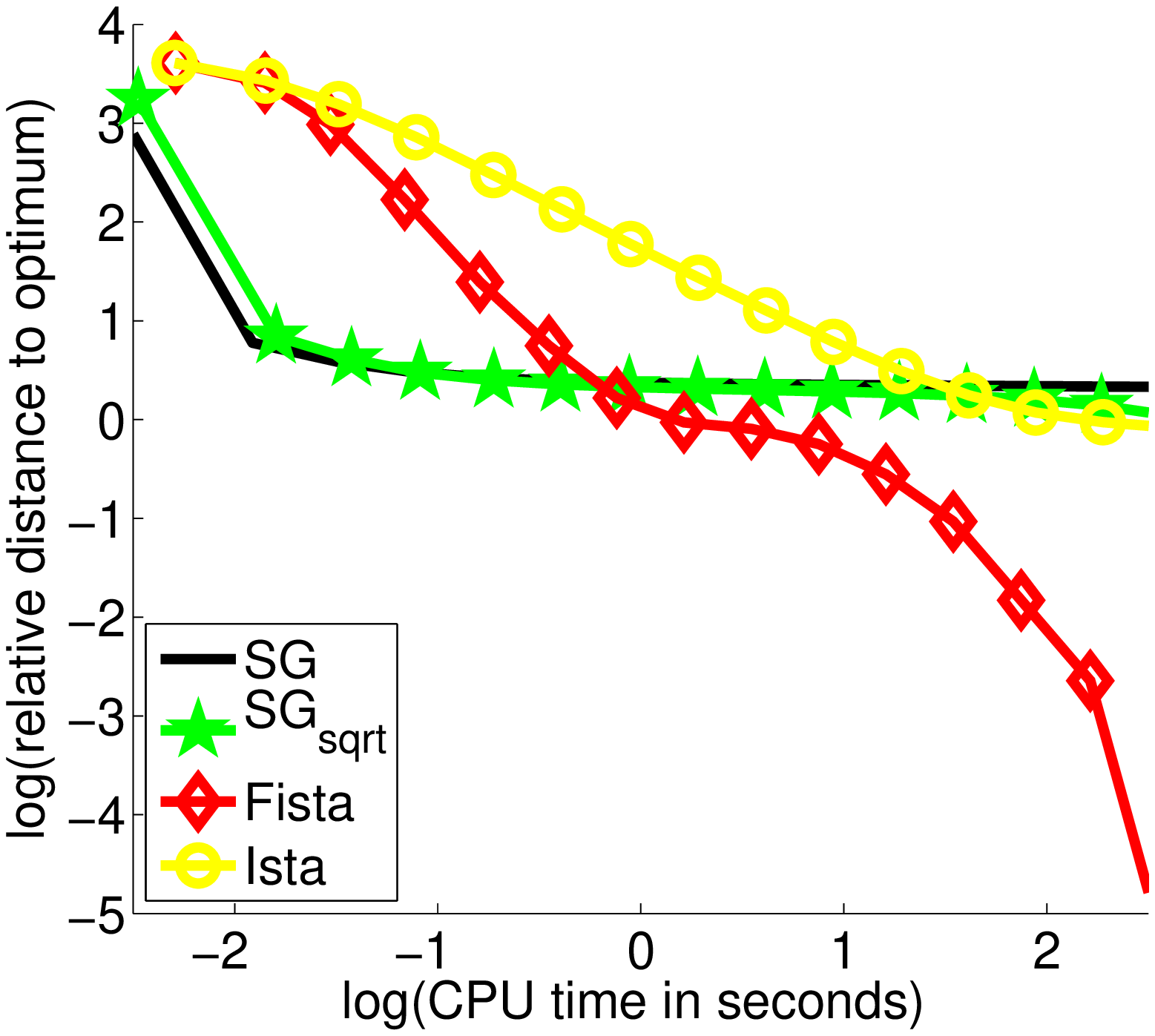}} \hfill
   \subfloat[scale: medium,\newline regul: medium]{\includegraphics[width=0.33\linewidth]{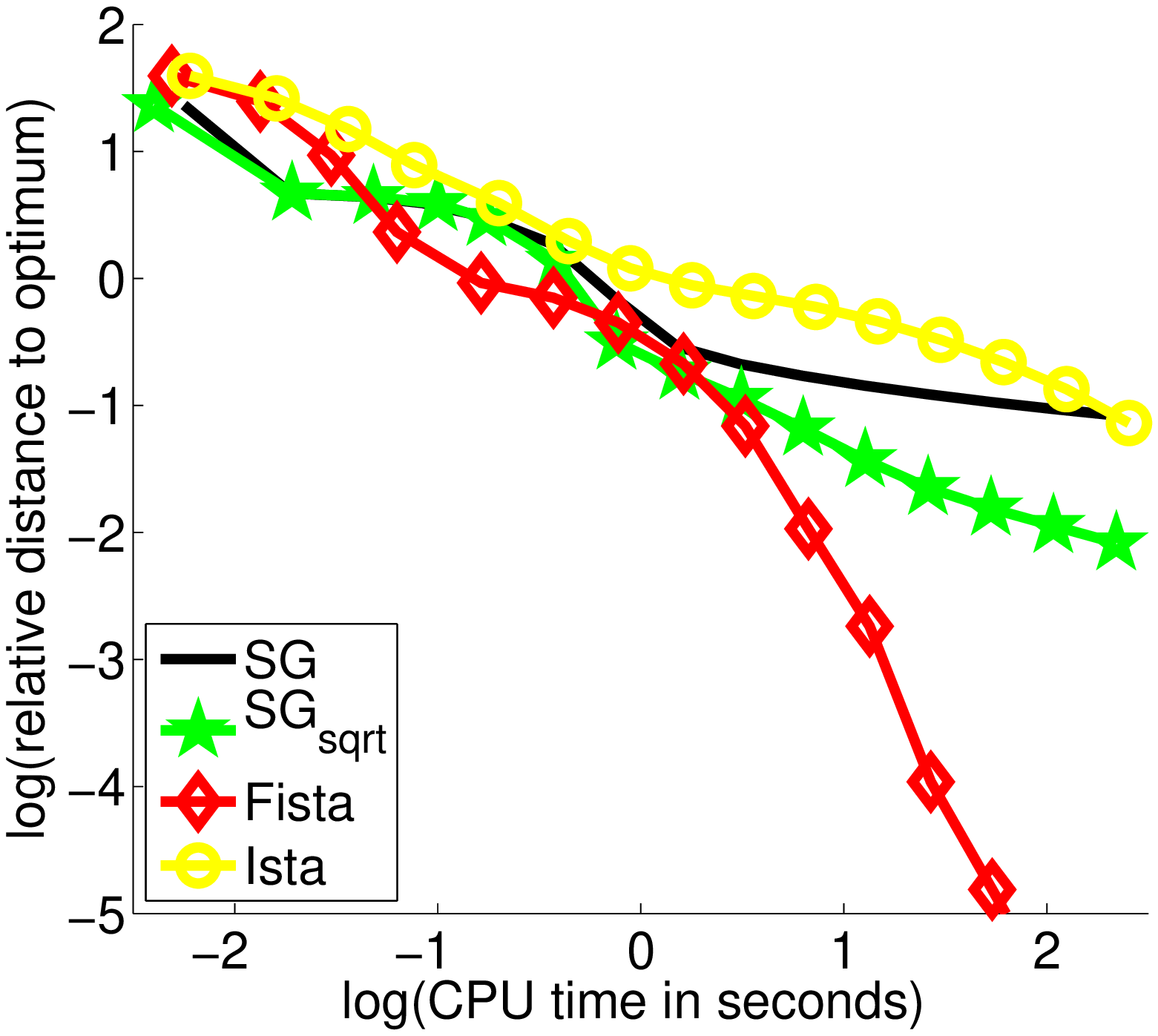}} \hfill
   \subfloat[scale: medium,\newline regul: high]{\includegraphics[width=0.33\linewidth]{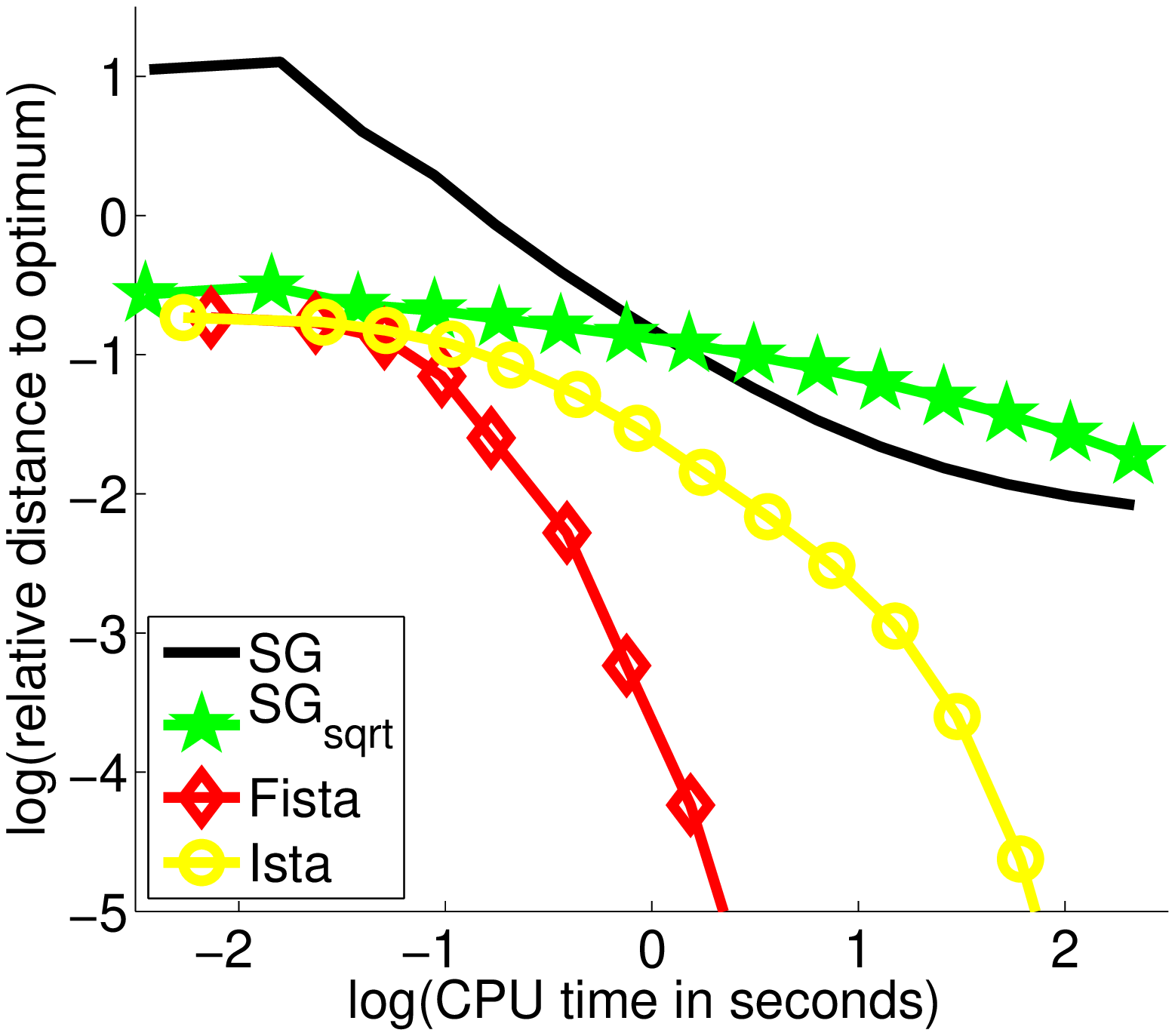}} \\
   \subfloat[scale: large,\newline regul: low]{\includegraphics[width=0.33\linewidth]{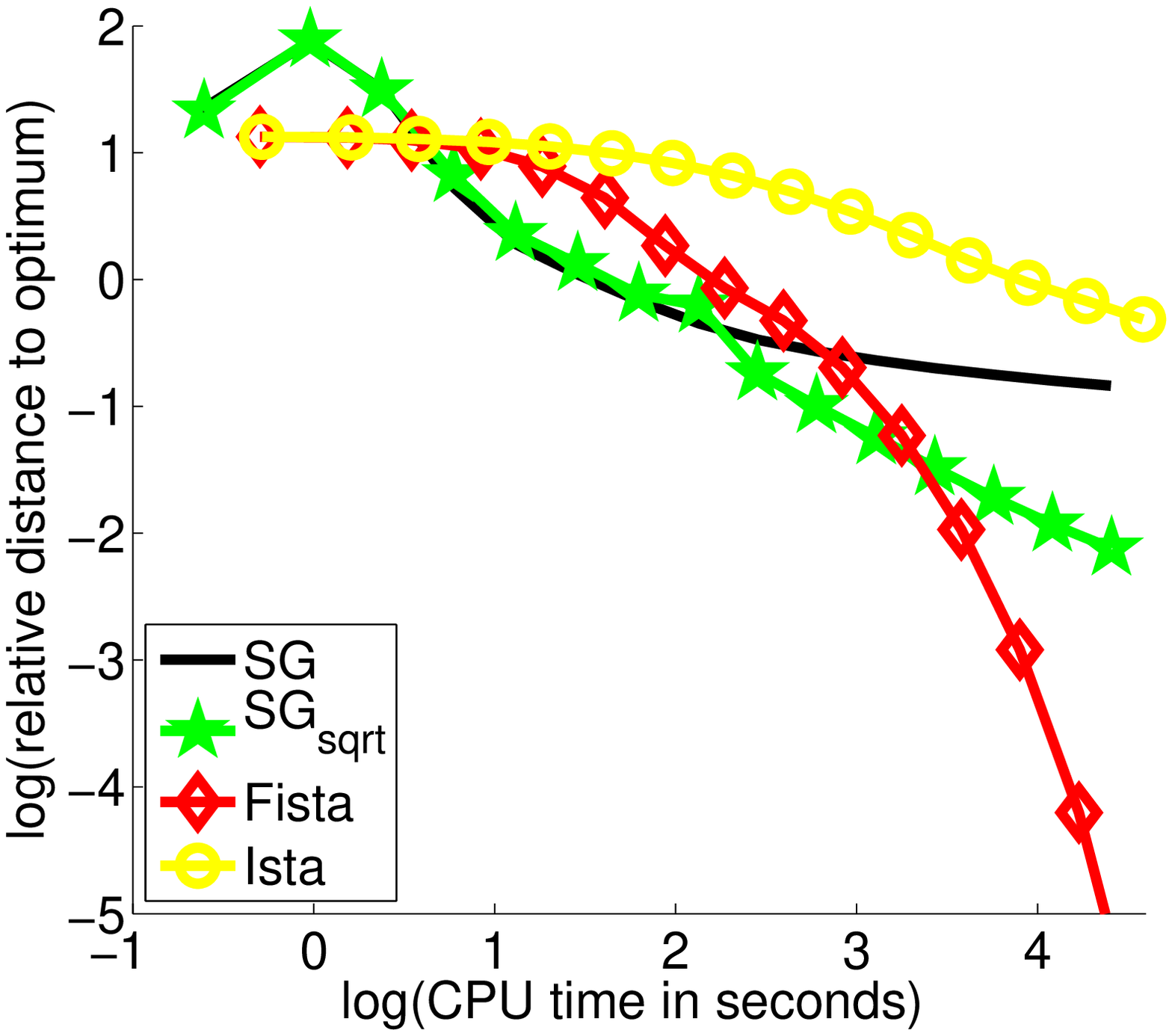}} \hfill
   \subfloat[scale: large,\newline regul: medium]{\includegraphics[width=0.33\linewidth]{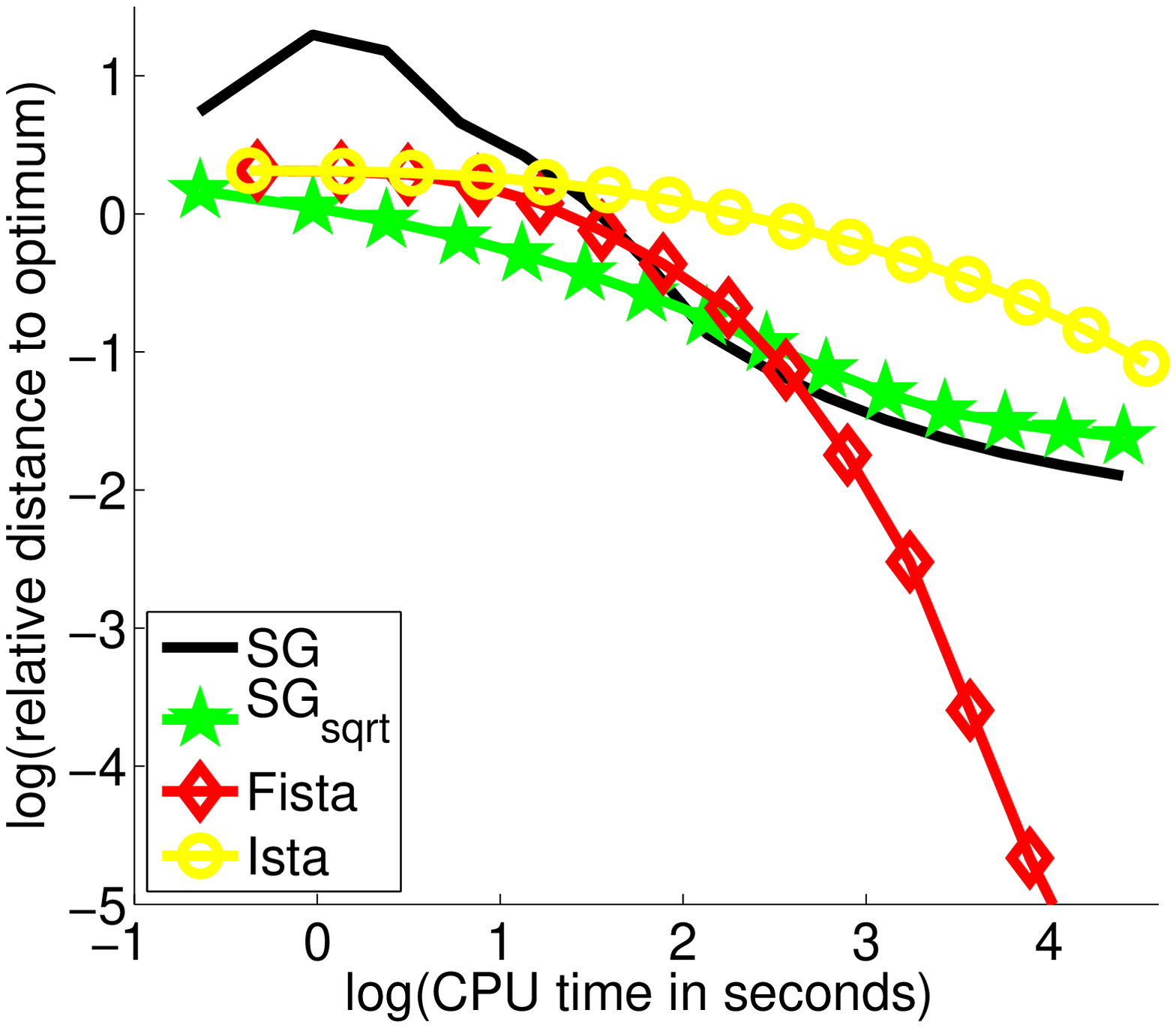}} \hfill
   \subfloat[scale: large,\newline regul: high]{\includegraphics[width=0.33\linewidth]{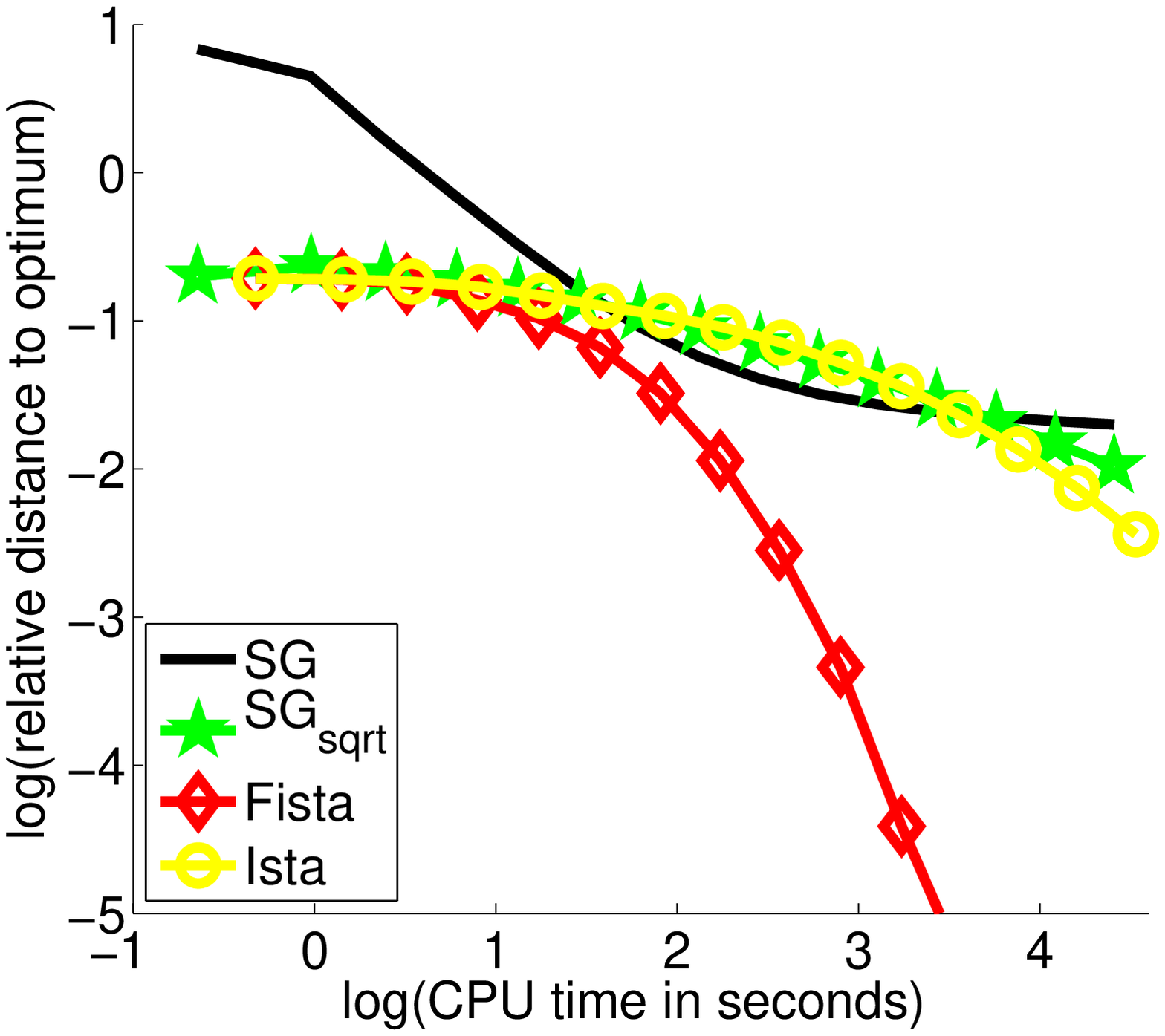}}
   \caption{Medium- and large-scale multi-class classification problems 
   for three optimization methods (see details about the datasets and the methods in the main text).
   Three levels of regularization are considered. The curves represent the relative value of the objective function as a function of the computation time in second on a $\log_{10}/\log_{10}$ scale.
   In the highly regularized setting, the tuning of the step-size for the subgradient turned out to be difficult, 
   which explains the behavior of SG in the first iterations.}
   \label{fig:struct_bench_bio}
\end{figure}

All the results are reported in Figure~\ref{fig:struct_bench_bio}.
The benchmark especially points out that the accelerated proximal scheme performs overall better that the two other methods. Again, it is important to note that both proximal algorithms yield sparse solutions, 
which is not the case for SG.
More generally, this experiment illustrates the flexibility of proximal algorithms with respect to the choice of the loss function.

\subsection{General Overlapping Groups of Variables}

We consider a structured sparse decomposition problem with overlapping groups
of $\ell_\infty$-norms, and compare the proximal gradient algorithm
FISTA~\cite{beck2009fast} consider the proximal operator presented in
Section~\ref{sec:proximal_operator} (referred to as ProxFlow~\cite{mairal10}). Since, the norm we use is a sum of several simple terms, we can bring to bear other optimization techniques which are dedicated to this situation, namely proximal splitting method known as alternating direction method of multipliers (ADMM)~(see, e.g.,~\cite{boydadmm,combettes2010proximal}). We consider two variants,
(ADMM) and (Lin-ADMM)---see more details in~\cite{Mairal2011}.

We consider a design matrix $\mx$ in $\R^{n
\times p}$ built from overcomplete dictionaries
of discrete cosine transforms (DCT), which are naturally organized on one-
or two-dimensional grids and display local correlations.  The following
families of groups $\Gc$ using this spatial information are thus
considered: (1) every contiguous sequence of length $3$ for the
one-dimensional case, and (2) every $3\!\times\!3$-square in the
two-dimensional setting.  We generate vectors~$\vy$ in $\R^{n}$ according to
the linear model $\vy = \mx\vw_0 + \vepsilon$, 
where $\vepsilon \sim \Nc(0,0.01\|\mx\vw_0\|_2^2)$.  The vector $\vw_0$ has
about $20\%$   nonzero components, randomly selected, while
respecting the structure of~$\Gc$, and uniformly generated in
$[-1,1]$.

In our experiments, the regularization parameter $\lambda$ is chosen to achieve
the same level of sparsity ($20\%$).  For SG, ADMM and Lin-ADMM, some parameters are
optimized to provide the lowest value of the objective function after $1\,000$
iterations of the respective algorithms.  For SG, we take the step size to be equal to $a/(k+b)$, where $k$
is the iteration number, and $(a,b)$ are the pair of parameters selected in
$\{10^{-3},\dots,10\}\!\times\!\{10^2,10^3,10^4\}$.
The parameter~$\gamma$ for ADMM is selected in $\{10^{-2},\ldots,10^{2}\}$.  The parameters
$(\gamma,\delta)$ for Lin-ADMM are selected in $\{10^{-2},\ldots,10^{2}\}
\times \{10^{-1},\ldots,10^8\}$.  For interior point methods, since
problem~(\ref{eq:formulation}) can be cast either as a quadratic (QP) or as a
conic program (CP), we show in Figure~\ref{fig:speed_cmp} the results for both
formulations.
On three problems of different sizes, with $(n,p)\in\{(100,10^3),(1024,10^4),(1024,10^5)\}$, our algorithms ProxFlow, ADMM and Lin-ADMM compare favorably
with the other methods, 
 (see
Figure~\ref{fig:speed_cmp}), except for ADMM in the large-scale setting which
yields an objective function value  similar to that of SG after~$10^4$ seconds.  Among
ProxFlow, ADMM and Lin-ADMM, ProxFlow is consistently better than Lin-ADMM,
which is itself better than ADMM. Note that for the small scale problem,
the performance of ProxFlow and Lin-ADMM is similar.
In addition, note that QP, CP, SG, ADMM and
Lin-ADMM do not obtain sparse solutions, whereas ProxFlow does.
\begin{figure}[hbtp]
    \centering
   \includegraphics[width=0.34\textwidth]{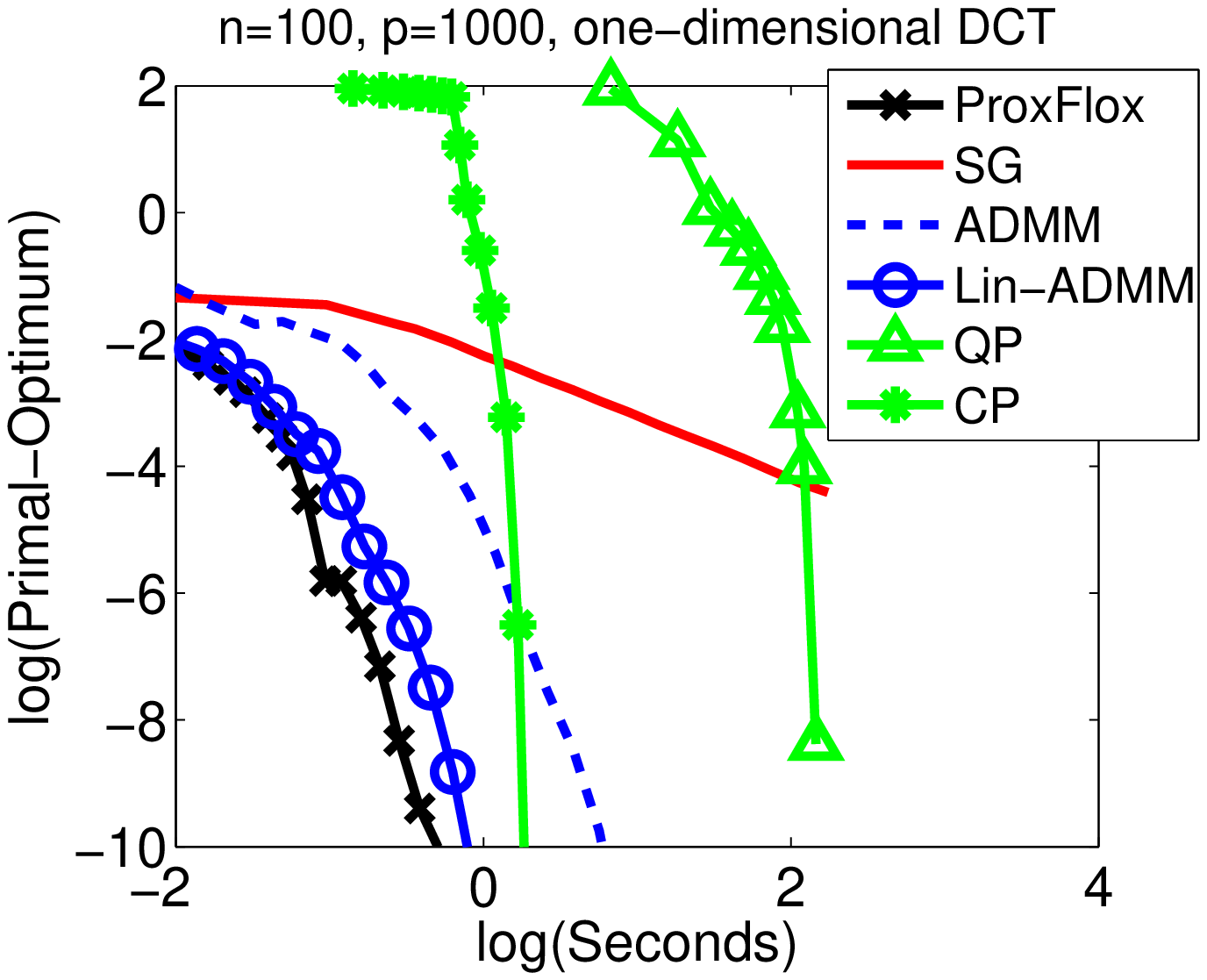} \hfill
   \includegraphics[width=0.32\textwidth]{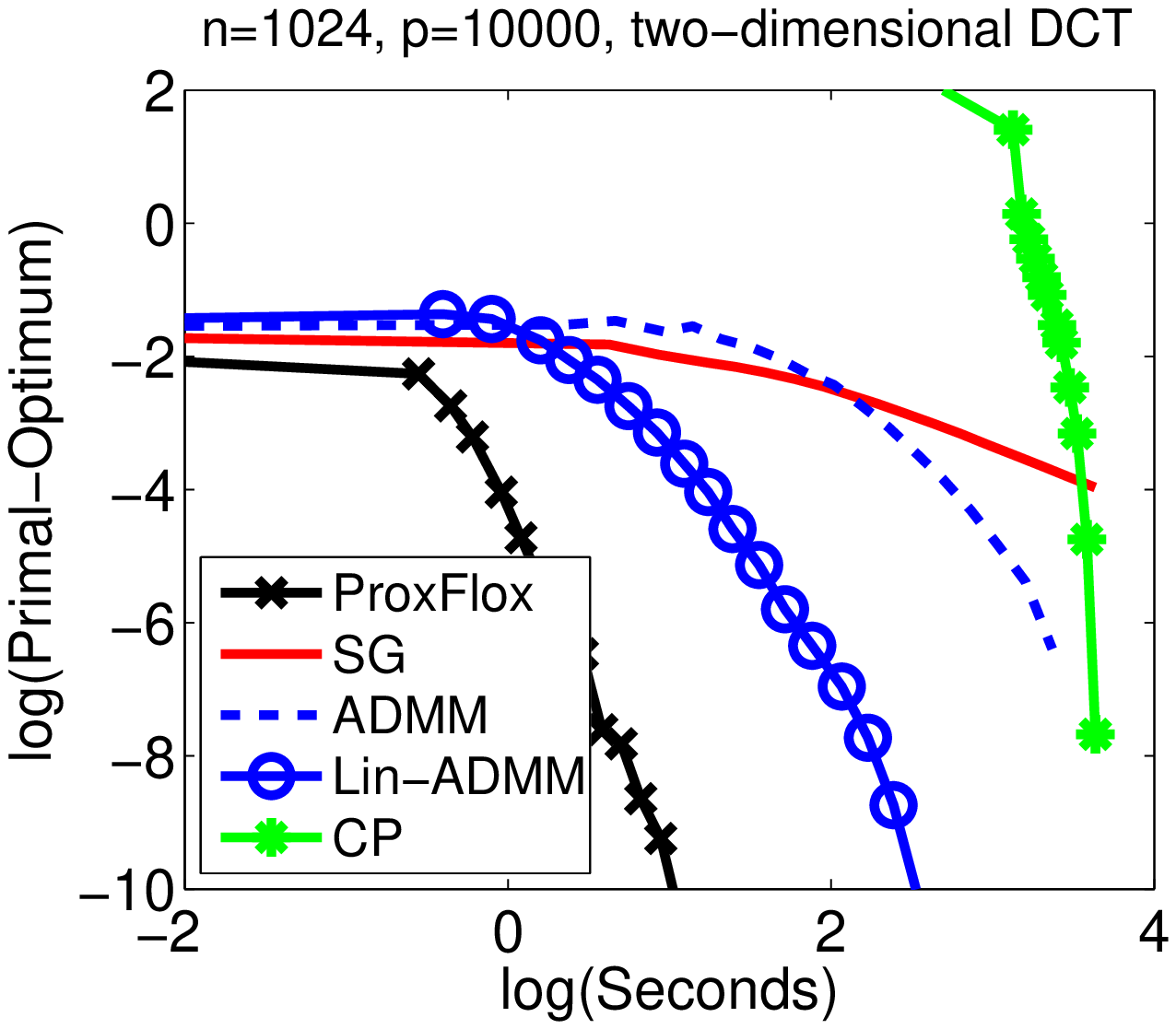} \hfill
   \includegraphics[width=0.32\textwidth]{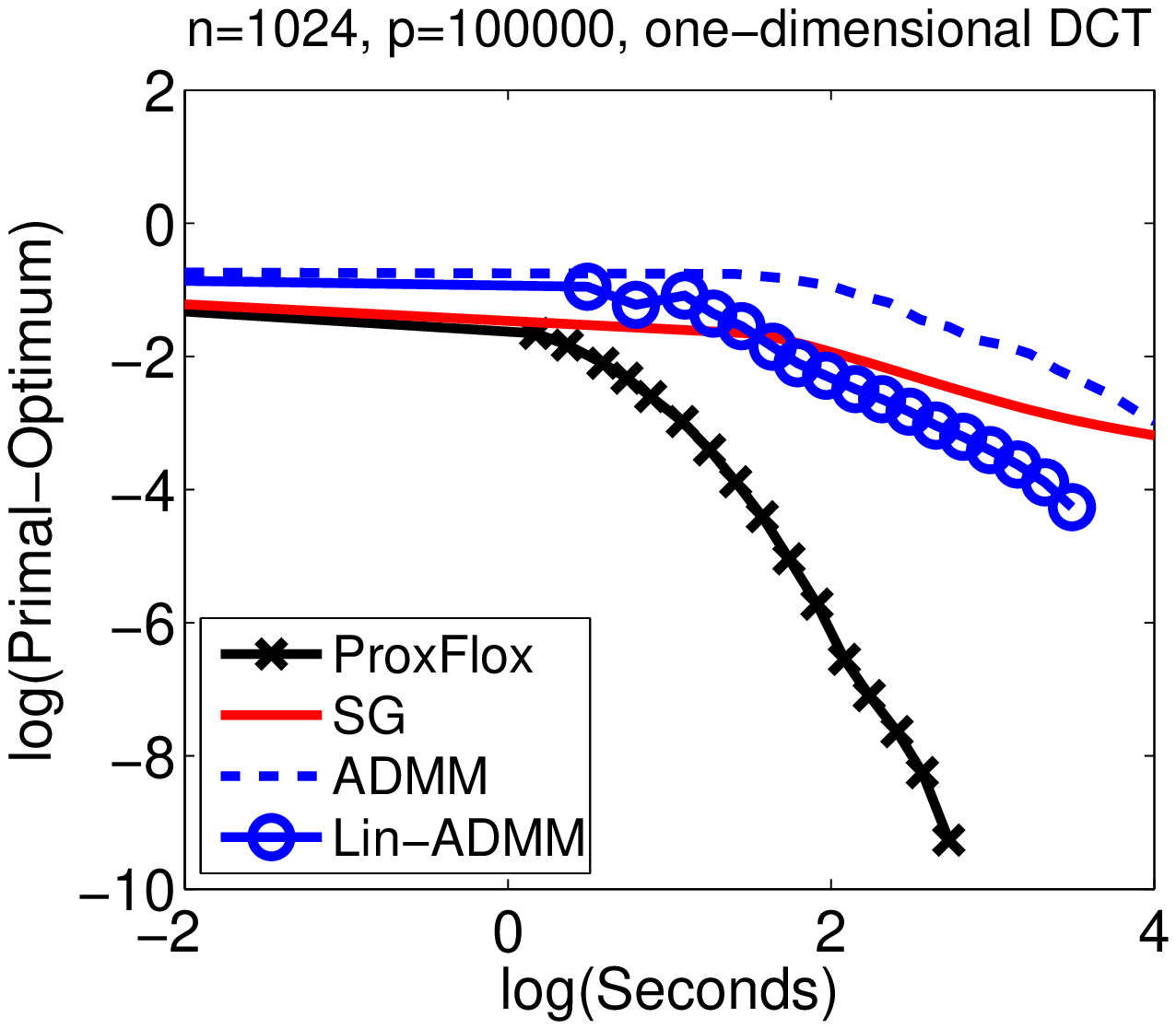} 
   \caption{Speed comparisons: distance to the optimal primal value versus CPU time (log-log scale). Due to the computational burden, QP and CP could not be run on every problem.} 
\label{fig:speed_cmp}
\end{figure}

\section{General Comments}
We conclude this section by a couple of general remarks on the experiments that we presented. 
First, the use of proximal methods is often advocated because of their optimal worst case complexities in $O(\frac{1}{t^2})$ (where $t$ is the number of iterations). In practice, in our experiments, these and several other methods exhibit empirically convergence rates that are at least linear, if not better, which suggests that the adaptivity of the method (e.g., its ability to take advantage of local curvature) might be more crucial to its practical success.
Second, our experiments concentrated on regimes that are of interest for sparse methods in machine learning where typically $p$ is larger than $n$ and where it is possible to find good sparse solutions. The setting where $n$ is much larger than $p$ was out of scope here, but would be worth a separate study, and should involve methods from stochastic optimization. Also, even though it might make sense from an optimization viewpoint, we did not consider problems with low levels of sparsity, that is with more dense solution vectors, since it would be a more difficult regime for many of the algorithms that we presented (namely LARS, CD or proximal methods).

\chapter{Extensions}
\label{sec:extensions}
We obviously could not cover exhaustively the literature on algorithms for sparse methods in this chapter.

Surveys and comparisons of algorithms for sparse methods have been proposed in \cite{schmidt2007fast} and \cite{yuan-comparison}.
These papers present quite a few algorithms, but focus essentially on $\ell_1$-regularization and unfortunately do not consider proximal methods. Also, it is not clear that the metrics used to compare the performance of various algorithms is the most relevant to machine learning; in particular, we present the full convergence curves that we believe are more informative than the ordering of algorithms at fixed precision.

Beyond the material presented here, there a few topics that we did not develop and that are worth mentioning.

In the section on proximal methods, we presented the proximal methods called forward-backward splitting methods.
We applied them to objectives which are the sum of two terms: a differentiable function with Lipschitz-continuous gradients and a norm. More generally these methods 
 apply to the sum of two semi-lower continuous (l.s.c.), proper, convex functions with
non-empty domain, and where one element is assumed differentiable with Lipschitz-continuous gradient~\cite{combettes2010proximal}.
The proximal operator itself dates back to~\cite{moreau1962fonctions} and
proximal methods themselves date back to~\cite{Lions1979,Martinet1970}. As of today, they 
have been extended to various settings~\cite{Chen1997,combettes2006signal,combettes2010proximal,Tseng1991}. 
In particular, instances of proximal methods are still applicable if the smoothness assumptions that we made on the loss are relaxed.
For example, the Douglas-Rachford splitting algorithm applies as soon as the objective function to minimize 
is only assumed l.s.c. proper convex, without any smoothness properties (although a l.s.c. convex function is continuous inside of its domain). 
The augmented Lagrangian techniques~(see \cite{boydadmm,combettes2010proximal,Glowinski1989} and numerous references therein) and more precisely their variants known as alternating-direction methods of multipliers are related to proximal methods via duality.
These methods are in particular applicable to cases where several regularizations and constraints are mixed \cite{Mairal2011,Tomioka2011Augmented}.

For certain combination of losses and regularizations, 
dedicated methods have been proposed. This is the case for linear regression with
the least absolute deviation (LAD) loss (also called $\ell_1$-loss) with an $\ell_1$-norm regularizer, which leads to a
linear program \cite{wu2008coordinate}. This is also the case for algorithms designed for classical multiple kernel learning when the regularizer is the squared norm \cite{simpleMKL,Sonnenburg2006Large,Suzuki2011SpicyMKL}; these methods are therefore not exactly comparable to the MKL algorithms presented in this monograph which apply to objective regularized by the unsquared norm (except for reweighted $\ell_2$-schemes, based on variational formulations for the squared norm).

In the context of proximal methods, the metric used to define the proximal operator can be modified by judicious rescaling operations, in order to fit better the geometry of the data~\cite{Duchi2010}. Moreover, they can be mixed with Newton and quasi-Newton methods, for further acceleration (see, e.g.,~\cite{bookchapter-mark}).

Finally, from a broader outlook, our---\textit{a priori} deterministic---optimization problem~(\ref{eq:formulation})
may also be tackled with stochastic optimization approaches, which has been the focus of much recent research~\cite{Bottou1998,Bottou2004,Duchi2010,Hu2009Accelerated,Shapiro2009,xiao2010dual}.

\chapter{Conclusions}
\label{sec:conclusions}
We have tried to provide in this monograph a unified view of sparsity and structured sparsity as it can emerge when convex analysis and convex optimization are used as the conceptual basis to formalize respectively problems and algorithms. In that regards, we did not aim at exhaustivity and other paradigms are likely to provide complementary views.

With convexity as a requirement however, using non-smooth norms as
regularizers is arguably the most natural way to encode sparsity constraints.
A main difficulty associated with these norms is that they are intrinsically non-differentiable; they are however fortunately also structured, so that a few concepts can be leveraged to manipulate and solve problems regularized with them. To summarize:

$-$ Fenchel-Legendre duality and the dual norm allow to compute subgradients, duality gaps and are also key to exploit sparsity algorithmically via working set methods. More trivially, duality also
provides an alternative formulation to the initial problem which is sometimes more tractable.

$-$ The proximal operator, which, when it can be computed efficiently (exactly or approximately), allows one to treat the optimization problem as if it were a smooth problem.

$-$ Quadratic variational formulations, which provide an alternative way to decouple the difficulties associated with the loss and the non-differentiability of the norm.

Leveraging these different tools lead us to present and compare four families of algorithms for sparse methods: proximal methods, block-coordinate descent algorithms, reweighted-$\ell_2$ schemes and the LARS/homotopy algorithm that are representative of the state of the art. 
The properties of these methods can be summarized as follows:

$-$ Proximal methods provide efficient and scalable algorithms that are applicable to a wide family of loss functions, that are simple to implement, compatible with many sparsity-inducing norms and often competitive with the other methods considered.

$-$ For the square loss, the homotopy method remains the fastest algorithm for (a) small and medium scale problems, since its complexity depends essentially on the size of the active sets, (b) cases with very correlated designs. It computes the whole path up to a certain sparsity level. Its main drawback is that it is difficult to implement efficiently, and it is subject to numerical instabilities. On the other hand, coordinate descent and proximal algorithms are trivial to implement.

$-$ For smooth losses, block-coordinate descent provides one of the fastest algorithms but it is limited to separable regularizers. 

$-$ For the square-loss and possibly sophisticated sparsity inducing regularizers, reweighted-$\ell_2$ schemes provide generic algorithms, that are still pretty competitive compared to subgradient and interior point methods.
For general losses, these methods currently require to solve iteratively $\ell_2$-regularized problems and it would be desirable to relax this constraint.

Of course, many learning problems are by essence non-convex and several approaches to inducing (sometimes more aggressively) sparsity are also non-convex. Beyond providing an overview of these methods to the reader as a complement to the convex formulations, we have tried to suggest that faced with non-convex non-differentiable and therefore potentially quite hard problems to solve, a good strategy is to try and reduce the problem to solving iteratively convex problems, since more stable algorithms are available and progress can be monitored with duality gaps.

Last but not least, duality suggests strongly that multiple kernel learning is in a sense the dual view to sparsity, and provides a natural way, via the ``kernel trick", to extend sparsity to reproducing kernel Hilbert spaces. We have therefore illustrated throughout the text that rather than being a vague connection, this duality can be exploited both conceptually, leading to the idea of structured MKL, and algorithmically to kernelize all of the algorithms we considered so as to apply them in the MKL and RKHS settings.


\begin{acknowledgements}
Francis Bach, Rodolphe Jenatton and Guillaume Obozinski are supported in
part by ANR under grant MGA ANR-07-BLAN-0311 and the European Research
Council (SIERRA Project).
Julien Mairal is supported by the NSF grant SES-0835531 and NSF award
CCF-0939370.
All the authors would like to thank the anonymous reviewers, whose comments have greatly contributed to improve the quality of this paper.
\end{acknowledgements}

\bibliographystyle{plain}
\bibliography{Bach-Jenatton-Mairal-Obozinski}

\begin{thebibliography}{100}

\bibitem{jake}
J.~Abernethy, F.~Bach, T.~Evgeniou, and J.-P. Vert.
\newblock A new approach to collaborative filtering: Operator estimation with
  spectral regularization.
\newblock {\em Journal of Machine Learning Research}, 10:803--826, 2009.

\bibitem{aflalo2011variable}
J.~Aflalo, A.~Ben-Tal, C.~Bhattacharyya, J.S. Nath, and S.~Raman.
\newblock Variable sparsity kernel learning.
\newblock {\em Journal of Machine Learning Research}, 12:565--592, 2011.

\bibitem{aharon2006}
M.~Aharon, M.~Elad, and A.~Bruckstein.
\newblock {K-SVD}: An algorithm for designing overcomplete dictionaries for
  sparse representation.
\newblock {\em IEEE Transactions on Signal Processing}, 54(11):4311--4322,
  2006.

\bibitem{srebro-mc}
Y.~Amit, M.~Fink, N.~Srebro, and S.~Ullman.
\newblock Uncovering shared structures in multiclass classification.
\newblock In {\em Proceedings of the International Conference on Machine
  Learning (ICML)}, 2007.

\bibitem{Archambeau:NIPS08b}
C.~Archambeau and F.~Bach.
\newblock Sparse probabilistic projections.
\newblock In {\em Advances in Neural Information Processing Systems (NIPS)},
  2008.

\bibitem{argyriou2008convex}
A.~Argyriou, T.~Evgeniou, and M.~Pontil.
\newblock Convex multi-task feature learning.
\newblock {\em Machine Learning}, 73(3):243--272, 2008.

\bibitem{Bach2008a}
F.~Bach.
\newblock Consistency of the group lasso and multiple kernel learning.
\newblock {\em Journal of Machine Learning Research}, 9:1179--1225, 2008.

\bibitem{tracenorm}
F.~Bach.
\newblock Consistency of trace norm minimization.
\newblock {\em Journal of Machine Learning Research}, 9:1019--1048, 2008.

\bibitem{hkl}
F.~Bach.
\newblock {Exploring large feature spaces with hierarchical multiple kernel
  learning}.
\newblock In {\em Advances in Neural Information Processing Systems (NIPS)},
  2008.

\bibitem{bach2010structured}
F.~Bach.
\newblock Structured sparsity-inducing norms through submodular functions.
\newblock In {\em Advances in Neural Information Processing Systems (NIPS)},
  2010.

\bibitem{shapinglevelsets}
{F.} {B}ach.
\newblock Shaping level sets with submodular functions.
\newblock In {\em Advances in Neural Information Processing Systems (NIPS)},
  2011.

\bibitem{bookchapter}
F.~Bach, R.~Jenatton, J.~Mairal, and G.~Obozinski.
\newblock Convex optimization with sparsity-inducing norms.
\newblock In S.~Sra, S.~Nowozin, and S.J. Wright, editors, {\em Optimization
  for Machine Learning}. MIT Press, 2011.

\bibitem{skm}
F.~Bach, G.~R.~G. Lanckriet, and M.~I. Jordan.
\newblock Multiple kernel learning, conic duality, and the {SMO} algorithm.
\newblock In {\em Proceedings of the International Conference on Machine
  Learning (ICML)}, 2004.

\bibitem{bach2008b}
F.~Bach, J.~Mairal, and J.~Ponce.
\newblock Convex sparse matrix factorizations.
\newblock {\em preprint arXiv:0812.1869v1}, 2008.

\bibitem{baraniuk}
R.G. Baraniuk, V.~Cevher, M.~Duarte, and C.~Hegde.
\newblock Model-based compressive sensing.
\newblock {\em IEEE Transactions on Information Theory}, 56(4):1982--2001,
  2010.

\bibitem{bauer1961absolute}
F.L. Bauer, J.~Stoer, and C.~Witzgall.
\newblock Absolute and monotonic norms.
\newblock {\em Numerische Mathematik}, 3(1):257--264, 1961.

\bibitem{beck2009fast}
A.~Beck and M.~Teboulle.
\newblock {A fast iterative shrinkage-thresholding algorithm for linear inverse
  problems}.
\newblock {\em SIAM Journal on Imaging Sciences}, 2(1):183--202, 2009.

\bibitem{bertsekas}
D.P. Bertsekas.
\newblock {\em Nonlinear programming}.
\newblock Athena Scientific, 1999.
\newblock 2nd edition.

\bibitem{bickel_lasso_dantzig}
P.~Bickel, Y.~Ritov, and A.~Tsybakov.
\newblock {Simultaneous analysis of Lasso and Dantzig selector}.
\newblock {\em Annals of Statistics}, 37(4):1705--1732, 2009.

\bibitem{borwein}
J.M. Borwein and A.S. Lewis.
\newblock {\em Convex analysis and nonlinear optimization: theory and
  examples}.
\newblock Springer-Verlag, 2006.

\bibitem{Bottou1998}
L.~Bottou.
\newblock {Online algorithms and stochastic approximations}.
\newblock {\em Online Learning and Neural Networks}, 5, 1998.

\bibitem{bottou2008tradeoffs}
L.~Bottou and O.~Bousquet.
\newblock {The tradeoffs of large scale learning}.
\newblock In {\em Advances in Neural Information Processing Systems (NIPS)},
  2008.

\bibitem{Bottou2004}
L.~Bottou and Y.~LeCun.
\newblock {Large scale online learning}.
\newblock In {\em Advances in Neural Information Processing Systems (NIPS)},
  2004.

\bibitem{boydadmm}
S.~Boyd, N.~Parikh, E.~Chu, B.~Peleato, and J.~Eckstein.
\newblock Distributed optimization and statistical learning via the alternating
  direction method of multipliers.
\newblock {\em Foundations and Trends in Machine Learning}, 3(1):1--124, 2011.

\bibitem{boyd.convex}
S.~Boyd and L.~Vandenberghe.
\newblock {\em {Convex Optimization}}.
\newblock Cambridge University Press, 2004.

\bibitem{bradley2009}
D.M. Bradley and J.A. Bagnell.
\newblock Convex coding.
\newblock In {\em Proceedings of the Twenty-Fifth Conference on Uncertainty in
  Artificial Intelligence (UAI)}, 2009.

\bibitem{brucker1984}
P.~Brucker.
\newblock {An $O(n)$ algorithm for quadratic knapsack problems}.
\newblock {\em Operations Research Letters}, 3(3):163--166, 1984.

\bibitem{burges2009dimension}
C.~Burges.
\newblock Dimension reduction: A guided tour.
\newblock {\em Machine Learning}, 2(4):275--365, 2009.

\bibitem{cai2010singular}
J.F. Cai, E.J. Cand{\`e}s, and Z.~Shen.
\newblock A singular value thresholding algorithm for matrix completion.
\newblock {\em SIAM Journal on Optimization}, 20:1956--1982, 2010.

\bibitem{candes4}
E.J. Cand\`es, M.~Wakin, and S.~Boyd.
\newblock Enhancing sparsity by reweighted {L}1 minimization.
\newblock {\em Journal of Fourier Analysis and Applications}, 14(5):877--905,
  2008.

\bibitem{Caron:ICML08}
F.~Caron and A.~Doucet.
\newblock Sparse {B}ayesian nonparametric regression.
\newblock In {\em Proceedings of the International Conference on Machine
  Learning (ICML)}, 2008.

\bibitem{cehver}
V.~Cehver, M.~Duarte, C.~Hedge, and R.G. Baraniuk.
\newblock Sparse signal recovery using markov random fields.
\newblock In {\em Advances in Neural Information Processing Systems (NIPS)},
  2008.

\bibitem{chambolle2005total}
A.~Chambolle.
\newblock Total variation minimization and a class of binary {MRF} models.
\newblock In {\em Proceedings of the 5th International Workshop on Energy
  Minimization Methods in Computer Vision and Pattern Recognition}, 2005.

\bibitem{chambolle2009}
A.~Chambolle and J.~Darbon.
\newblock On total variation minimization and surface evolution using
  parametric maximum flows.
\newblock {\em International Journal of Computer Vision}, 84(3):288--307, 2009.

\bibitem{Chandrasekaran2010Convex}
V.~Chandrasekaran, B.~Recht, P.A. Parrilo, and A.S. Willsky.
\newblock The convex geometry of linear inverse problems.
\newblock {\em preprint arXiv:1012.0621}, 2010.

\bibitem{Chen1997}
G.H.G. Chen and R.T. Rockafellar.
\newblock Convergence rates in forward-backward splitting.
\newblock {\em SIAM Journal on Optimization}, 7(2):421--444, 1997.

\bibitem{chen}
S.S. Chen, D.L. Donoho, and M.A. Saunders.
\newblock Atomic decomposition by basis pursuit.
\newblock {\em SIAM Journal on Scientific Computing}, 20:33--61, 1999.

\bibitem{combettes2010proximal}
P.L. Combettes and J.-C. Pesquet.
\newblock {\em Fixed-Point Algorithms for Inverse Problems in Science and
  Engineering}, chapter Proximal Splitting Methods in Signal Processing.
\newblock Springer-Verlag, 2011.

\bibitem{combettes2006signal}
P.L. Combettes and V.R. Wajs.
\newblock Signal recovery by proximal forward-backward splitting.
\newblock {\em SIAM Multiscale Modeling and Simulation}, 4(4):1168--1200, 2006.

\bibitem{cotter}
S.F. Cotter, J.~Adler, B.~Rao, and K.~Kreutz-Delgado.
\newblock Forward sequential algorithms for best basis selection.
\newblock In {\em IEEE Proceedings of Vision Image and Signal Processing},
  pages 235--244, 1999.

\bibitem{ingrid2010iteratively}
I.~Daubechies, R.~DeVore, M.~Fornasier, and C.~S. G\"unt\"urk.
\newblock Iteratively reweighted least squares minimization for sparse
  recovery.
\newblock {\em Communications on Pure and Applied Mathematics}, 63(1):1--38,
  2010.

\bibitem{schmidt2008}
E.~Van den Berg, M.~Schmidt, M.P. Friedlander, and K.~Murphy.
\newblock Group sparsity via linear-time projections.
\newblock Technical report, University of British Columbia, 2008.
\newblock Technical Report number TR-2008-09.

\bibitem{donoho}
D.L. Donoho and I.M. Johnstone.
\newblock Adapting to unknown smoothness via wavelet shrinkage.
\newblock {\em Journal of the American Statistical Association},
  90(432):1200--1224, 1995.

\bibitem{Duchi2010}
J.~Duchi, E.~Hazan, and Y.~Singer.
\newblock Adaptive subgradient methods for online learning and stochastic
  optimization.
\newblock {\em Journal of Machine Learning Research}, 12:2121--2159, 2011.

\bibitem{efron}
B.~Efron, T.~Hastie, I.~Johnstone, and R.~Tibshirani.
\newblock Least angle regression.
\newblock {\em Annals of Statistics}, 32(2):407--499, 2004.

\bibitem{elad2006}
M.~Elad and M.~Aharon.
\newblock Image denoising via sparse and redundant representations over learned
  dictionaries.
\newblock {\em IEEE Transactions on Image Processing}, 15(12):3736--3745, 2006.

\bibitem{engan1999}
K.~Engan, S.O. Aase, H.~Husoy, et~al.
\newblock Method of optimal directions for frame design.
\newblock In {\em Proceedings of the International Conference on Acoustics,
  Speech, and Signal Processing (ICASSP)}, 1999.

\bibitem{fan2001}
J.~Fan and R.~Li.
\newblock Variable selection via nonconcave penalized likelihood and its oracle
  properties.
\newblock {\em Journal of the American Statistical Association},
  96(456):1348--1360, 2001.

\bibitem{fazel}
M.~Fazel, H.~Hindi, and S.~Boyd.
\newblock A rank minimization heuristic with application to minimum order
  system approximation.
\newblock In {\em Proceedings of the American Control Conference}, volume~6,
  pages 4734--4739, 2001.

\bibitem{friedman2010note}
J.~Friedman, T.~Hastie, and R.~Tibshirani.
\newblock {A note on the group lasso and a sparse group lasso}.
\newblock {\em preprint arXiv:1001:0736v1}, 2010.

\bibitem{friedman2001}
J.~H. Friedman.
\newblock Greedy function approximation: a gradient boosting machine.
\newblock {\em Annals of Statistics}, 29(5):1189--1232, 2001.

\bibitem{fu_penalized_1998}
W.J. Fu.
\newblock Penalized regressions: the bridge versus the lasso.
\newblock {\em Journal of Computational and Graphical Statistics},
  7(3):397--416, 1998.

\bibitem{gasso2009}
G.~Gasso, A.~Rakotomamonjy, and S.~Canu.
\newblock Recovering sparse signals with non-convex penalties and {DC}
  programming.
\newblock {\em IEEE Transactions on Signal Processing}, 57(12):4686--4698,
  2009.

\bibitem{genkin_large-scale_2007}
A.~Genkin, D.D Lewis, and D.~Madigan.
\newblock Large-scale bayesian logistic regression for text categorization.
\newblock {\em Technometrics}, 49(3):291--304, 2007.

\bibitem{Glowinski1989}
R.~Glowinski and P.~Le~Tallec.
\newblock {\em {Augmented Lagrangian and operator-splitting methods in
  nonlinear mechanics}}.
\newblock Society for Industrial Mathematics, 1989.

\bibitem{grandvalet1999}
Y.~Grandvalet and S.~Canu.
\newblock Outcomes of the equivalence of adaptive ridge with least absolute
  shrinkage.
\newblock In {\em Advances in Neural Information Processing Systems (NIPS)},
  1999.

\bibitem{gribonval2011compressible}
R.~Gribonval, V.~Cevher, and M.E. Davies.
\newblock {Compressible Distributions for High-dimensional Statistics}.
\newblock {\em preprint arXiv:1102.1249v2}, 2011.

\bibitem{harchaoui2}
Z.~Harchaoui.
\newblock {\em M\'ethodes \`a Noyaux pour la D\'etection}.
\newblock PhD thesis, T\'el\'ecom ParisTech, 2008.

\bibitem{harchaoui}
Z.~Harchaoui and C.~L\'evy-Leduc.
\newblock Catching change-points with {L}asso.
\newblock In {\em Advances in Neural Information Processing Systems (NIPS)}.
  2008.

\bibitem{herrity}
K.K. Herrity, A.C. Gilbert, and J.A. Tropp.
\newblock Sparse approximation via iterative thresholding.
\newblock In {\em Proceedings of the International Conference on Acoustics,
  Speech and Signal Processing, (ICASSP)}, 2006.

\bibitem{Hu2009Accelerated}
C.~Hu, J.~Kwok, and W.~Pan.
\newblock Accelerated gradient methods for stochastic optimization and online
  learning.
\newblock In {\em Advances in Neural Information Processing Systems (NIPS)},
  2009.

\bibitem{huang2}
J.~Huang and T.~Zhang.
\newblock The benefit of group sparsity.
\newblock {\em Annals of Statistics}, 38(4):1978--2004, 2010.

\bibitem{huang}
J.~Huang, Z.~Zhang, and D.~Metaxas.
\newblock Learning with structured sparsity.
\newblock In {\em Proceedings of the International Conference on Machine
  Learning (ICML)}, 2009.

\bibitem{ishwaran2005spike}
H.~Ishwaran and J.S. Rao.
\newblock {Spike and slab variable selection: frequentist and Bayesian
  strategies}.
\newblock {\em Annals of Statistics}, 33(2):730--773, 2005.

\bibitem{Jacob2009}
L.~Jacob, G.~Obozinski, and J.-P. Vert.
\newblock Group {L}asso with overlaps and graph {L}asso.
\newblock In {\em Proceedings of the International Conference on Machine
  Learning (ICML)}, 2009.

\bibitem{jenatton_thesis}
R.~Jenatton.
\newblock {\em Structured Sparsity-Inducing Norms: Statistical and Algorithmic
  Properties with Applications to Neuroimaging}.
\newblock PhD thesis, Ecole Normale Sup\'erieure de Cachan, 2012.

\bibitem{jenatton}
R.~Jenatton, J-Y. Audibert, and F.~Bach.
\newblock Structured variable selection with sparsity-inducing norms.
\newblock {\em Journal of Machine Learning Research}, 12:2777--2824, 2011.

\bibitem{Jenatton2011}
R.~Jenatton, A.~Gramfort, V.~Michel, G.~Obozinski, F.~Bach, and B.~Thirion.
\newblock Multi-scale mining of f{MRI} data with hierarchical structured
  sparsity.
\newblock In {\em International Workshop on Pattern Recognition in Neuroimaging
  (PRNI)}, 2011.

\bibitem{Jenatton2010a}
R.~Jenatton, J.~Mairal, G.~Obozinski, and F.~Bach.
\newblock Proximal methods for sparse hierarchical dictionary learning.
\newblock In {\em Proceedings of the International Conference on Machine
  Learning (ICML)}, 2010.

\bibitem{Jenatton2010b}
R.~Jenatton, J.~Mairal, G.~Obozinski, and F.~Bach.
\newblock Proximal methods for hierarchical sparse coding.
\newblock {\em Journal of Machine Learning Research}, 12:2297--2334, 2011.

\bibitem{jenatton2010sspca}
R.~Jenatton, G.~Obozinski, and F.~Bach.
\newblock {Structured sparse principal component analysis}.
\newblock In {\em Proceedings of International Workshop on Artificial
  Intelligence and Statistics (AISTATS)}, 2010.

\bibitem{Johnson1967}
S.C. Johnson.
\newblock {Hierarchical clustering schemes}.
\newblock {\em Psychometrika}, 32(3):241--254, 1967.

\bibitem{Kavukcuoglu2009}
K.~Kavukcuoglu, M.A. Ranzato, R.~Fergus, and Y.~LeCun.
\newblock {Learning invariant features through topographic filter maps}.
\newblock In {\em Proc.\ IEEE Conf. Computer Vision and Pattern Recognition},
  2009.

\bibitem{kim3}
S.~Kim and E.P. Xing.
\newblock Tree-guided group lasso for multi-task regression with structured
  sparsity.
\newblock In {\em Proceedings of the International Conference on Machine
  Learning (ICML)}, 2010.

\bibitem{representer}
G.S. Kimeldorf and G.~Wahba.
\newblock Some results on {T}chebycheffian spline functions.
\newblock {\em Journal of Mathematical Analysis and Applications}, 33:82--95,
  1971.

\bibitem{koh2007interior}
K.~Koh, S.J. Kim, and S.~Boyd.
\newblock An interior-point method for large-scale l1-regularized logistic
  regression.
\newblock {\em Journal of Machine Learning Research}, 8:1555, 2007.

\bibitem{krishnapuram_sparse_2005}
B.~Krishnapuram, L.~Carin, M.A.T. Figueiredo, and A.J. Hartemink.
\newblock Sparse multinomial logistic regression: Fast algorithms and
  generalization bounds.
\newblock {\em IEEE Transactions on Pattern Analysis and Machine Intelligence},
  27(6):957--968, 2005.

\bibitem{genomic_fusion}
G.~R.~G. Lanckriet, T.~De Bie, N.~Cristianini, M.~I. Jordan, and W.~S. Noble.
\newblock A statistical framework for genomic data fusion.
\newblock {\em Bioinformatics}, 20:2626--2635, 2004.

\bibitem{gert}
G.R.G. Lanckriet, N.~Cristianini, L.~El Ghaoui, P.~Bartlett, and M.I. Jordan.
\newblock Learning the kernel matrix with semidefinite programming.
\newblock {\em Journal of Machine Learning Research}, 5:27--72, 2004.

\bibitem{ng-sparsecoding}
H.~Lee, A.~Battle, R.~Raina, and A.Y. Ng.
\newblock {Efficient sparse coding algorithms}.
\newblock In {\em Advances in Neural Information Processing Systems (NIPS)},
  2007.

\bibitem{augustin}
A.~Lef{\`e}vre, F.~Bach, and C.~F{\'e}votte.
\newblock Itakura-{S}aito nonnegative matrix factorization with group sparsity.
\newblock In {\em Proceedings of the International Conference on Acoustics,
  Speech, and Signal Processing (ICASSP)}, 2011.

\bibitem{Lions1979}
P.L. Lions and B.~Mercier.
\newblock Splitting algorithms for the sum of two nonlinear operators.
\newblock {\em SIAM Journal on Numerical Analysis}, 16(6):964--979, 1979.

\bibitem{liu2009l1linf}
H.~Liu, M.~Palatucci, and J.~Zhang.
\newblock {Blockwise coordinate descent procedures for the multi-task lasso,
  with applications to neural semantic basis discovery}.
\newblock In {\em Proceedings of the International Conference on Machine
  Learning (ICML)}, 2009.

\bibitem{Lounici2009}
K.~Lounici, M.~Pontil, A.~B. Tsybakov, and S.~van~de Geer.
\newblock Taking advantage of sparsity in multi-task learning.
\newblock {\em preprint arXiv:0903.1468}, 2009.

\bibitem{maculan1989linear}
N.~Maculan and G.~Galdino~de Paula.
\newblock {A linear-time median-finding algorithm for projecting a vector on
  the simplex of $\mathbb{R}^n$}.
\newblock {\em Operations research letters}, 8(4):219--222, 1989.

\bibitem{mairal_thesis}
J.~Mairal.
\newblock {\em Sparse coding for machine learning, image processing and
  computer vision}.
\newblock PhD thesis, Ecole Normale Sup\'erieure de Cachan, 2010.
\newblock \url{http://tel.archives-ouvertes.fr/tel-00595312}.

\bibitem{mairal2010}
J.~Mairal, F.~Bach, J.~Ponce, and G.~Sapiro.
\newblock Online learning for matrix factorization and sparse coding.
\newblock {\em Journal of Machine Learning Research}, 11:19--60, 2010.

\bibitem{mairal10}
J.~Mairal, R.~Jenatton, G.~Obozinski, and F.~Bach.
\newblock Network flow algorithms for structured sparsity.
\newblock In {\em Advances in Neural Information Processing Systems (NIPS)},
  2010.

\bibitem{Mairal2011}
J.~Mairal, R.~Jenatton, G.~Obozinski, and F.~Bach.
\newblock Convex and network flow optimization for structured sparsity.
\newblock {\em Journal of Machine Learning Research}, 12:2681--2720, 2011.

\bibitem{mallat}
S.~Mallat and Z.~Zhang.
\newblock Matching pursuit in a time-frequency dictionary.
\newblock {\em IEEE Transactions on Signal Processing}, 41(12):3397--3415,
  1993.

\bibitem{markowitz}
H.~Markowitz.
\newblock Portfolio selection.
\newblock {\em Journal of Finance}, 7(1):77--91, 1952.

\bibitem{Martinet1970}
B.~Martinet.
\newblock R{\'e}gularisation d'in{\'e}quations variationnelles par
  approximations successives.
\newblock {\em Revue française d'informatique et de recherche
  op{\'e}rationnelle, s{\'e}rie rouge}, 1970.

\bibitem{Martins2011}
A.F.T. Martins, N.A. Smith, P.M.Q. Aguiar, and M.A.T. Figueiredo.
\newblock Structured sparsity in structured prediction.
\newblock In {\em Proceedings of the Conference on Empirical Methods in Natural
  Language Processing (EMNLP)}, 2011.

\bibitem{micchelli2011regularizers}
C.A. Micchelli, J.M. Morales, and M.~Pontil.
\newblock Regularizers for structured sparsity.
\newblock {\em preprint arXiv:1010.0556v2}, 2011.

\bibitem{moreau1962fonctions}
J.J. Moreau.
\newblock {Fonctions convexes duales et points proximaux dans un espace
  hilbertien}.
\newblock {\em Comptes-Rendus de l'Acad{\'e}mie des Sciences de Paris,
  S{\'e}rie A, Math{\'e}matiques}, 255:2897--2899, 1962.

\bibitem{mosci2010solving}
S.~Mosci, L.~Rosasco, M.~Santoro, A.~Verri, and S.~Villa.
\newblock Solving structured sparsity regularization with proximal methods.
\newblock {\em Machine Learning and Knowledge Discovery in Databases}, pages
  418--433, 2010.

\bibitem{natarajan}
B.K. Natarajan.
\newblock Sparse approximate solutions to linear systems.
\newblock {\em SIAM Journal on Computing}, 24:227--234, 1995.

\bibitem{neal1996bayesian}
R.M. Neal.
\newblock {\em Bayesian learning for neural networks}, volume 118.
\newblock Springer Verlag, 1996.

\bibitem{needell}
D.~Needell and J.A. Tropp.
\newblock {CoSaMP: Iterative signal recovery from incomplete and inaccurate
  samples}.
\newblock {\em Applied and Computational Harmonic Analysis}, 26(3):301--321,
  2009.

\bibitem{negahban2009unified}
S.~Negahban, P.~Ravikumar, M.J. Wainwright, and B.~Yu.
\newblock {A unified framework for high-dimensional analysis of M-estimators
  with decomposable regularizers}.
\newblock In {\em Advances in Neural Information Processing Systems (NIPS)},
  2009.

\bibitem{nesterov2004introductory}
Y.~Nesterov.
\newblock {\em {Introductory lectures on convex optimization: a basic course}}.
\newblock Kluwer Academic Publishers, 2004.

\bibitem{nesterov2005smooth}
Y.~Nesterov.
\newblock Smooth minimization of non-smooth functions.
\newblock {\em Mathematical Programming}, 103(1):127--152, 2005.

\bibitem{nesterov2007gradient}
Y.~Nesterov.
\newblock {Gradient methods for minimizing composite objective function}.
\newblock Technical report, Center for Operations Research and Econometrics
  (CORE), Catholic University of Louvain, 2007.
\newblock Discussion paper.

\bibitem{nesterov2010efficiency}
Y.~Nesterov.
\newblock Efficiency of coordinate descent methods on huge-scale optimization
  problems.
\newblock Technical report, Center for Operations Research and Econometrics
  (CORE), Catholic University of Louvain, 2010.
\newblock Discussion paper.

\bibitem{nocedal}
J.~Nocedal and S.J. Wright.
\newblock {\em Numerical Optimization}.
\newblock Springer Verlag, 2006.
\newblock second edition.

\bibitem{Obozinski2011Group}
G.~Obozinski, L.~Jacob, and J.-P. Vert.
\newblock Group {L}asso with overlaps: the {L}atent {G}roup {L}asso approach.
\newblock {\em preprint HAL - inria-00628498}, 2011.

\bibitem{obozinski-joint}
G.~Obozinski, B.~Taskar, and M.I. Jordan.
\newblock {Joint covariate selection and joint subspace selection for multiple
  classification problems}.
\newblock {\em Statistics and Computing}, 20(2):231--252, 2009.

\bibitem{field1996}
B.A. Olshausen and D.J. Field.
\newblock Emergence of simple-cell receptive field properties by learning a
  sparse code for natural images.
\newblock {\em Nature}, 381:607--609, 1996.

\bibitem{osborne}
M.R. Osborne, B.~Presnell, and B.A. Turlach.
\newblock On the {L}asso and its dual.
\newblock {\em Journal of Computational and Graphical Statistics},
  9(2):319--37, 2000.

\bibitem{pontil}
M.~Pontil, A.~Argyriou, and T.~Evgeniou.
\newblock Multi-task feature learning.
\newblock In {\em Advances in Neural Information Processing Systems (NIPS)},
  2007.

\bibitem{simpleMKL}
A.~Rakotomamonjy, F.~Bach, S.~Canu, and Y.~Grandvalet.
\newblock Simple{MKL}.
\newblock {\em Journal of Machine Learning Research}, 9:2491--2521, 2008.

\bibitem{Rao2011}
N.S. Rao, R.D. Nowak, S.J. Wright, and N.G. Kingsbury.
\newblock {Convex approaches to model wavelet sparsity patterns}.
\newblock In {\em International Conference on Image Processing (ICIP)}, 2011.

\bibitem{Rapaport2008}
F.~Rapaport, E.~Barillot, and J.-P. Vert.
\newblock Classification of array{CGH} data using fused {SVM}.
\newblock {\em Bioinformatics}, 24(13):375--i382, 2008.

\bibitem{ravikumar2009sparse}
P.~Ravikumar, J.~Lafferty, H.~Liu, and L.~Wasserman.
\newblock Sparse additive models.
\newblock {\em Journal of the Royal Statistical Society. Series B, statistical
  methodology}, 71:1009--1030, 2009.

\bibitem{ritter}
K.~Ritter.
\newblock Ein verfahren zur l\"osung parameterabh\"angiger, nichtlinearer
  maximum-probleme.
\newblock {\em Mathematical Methods of Operations Research}, 6(4):149--166,
  1962.

\bibitem{rockafellar97}
R.T. Rockafellar.
\newblock {\em {Convex analysis}}.
\newblock Princeton University Press, 1997.

\bibitem{roth}
V.~Roth and B.~Fischer.
\newblock The {G}roup-{L}asso for generalized linear models: uniqueness of
  solutions and efficient algorithms.
\newblock In {\em Proceedings of the International Conference on Machine
  Learning (ICML)}, 2008.

\bibitem{rudin}
L.I. Rudin, S.~Osher, and E.~Fatemi.
\newblock Nonlinear total variation based noise removal algorithms.
\newblock {\em Physica D: Nonlinear Phenomena}, 60(1-4):259--268, 1992.

\bibitem{schmidt2007fast}
M.~Schmidt, G.~Fung, and R.~Rosales.
\newblock Fast optimization methods for {L1} regularization: A comparative
  study and two new approaches.
\newblock In {\em Proceedings of the European Conference on Machine Learning
  (ECML)}, 2007.

\bibitem{bookchapter-mark}
M.~Schmidt, D.~Kim, and S.~Sra.
\newblock Projected {N}ewton-type methods in machine learning.
\newblock In S.~Sra, S.~Nowozin, and S.J. Wright, editors, {\em Optimization
  for Machine Learning}. MIT Press, 2011.

\bibitem{schmidt2010convex}
M.~Schmidt and K.~Murphy.
\newblock Convex structure learning in log-linear models: Beyond pairwise
  potentials.
\newblock In {\em Proceedings of International Workshop on Artificial
  Intelligence and Statistics (AISTATS)}, 2010.

\bibitem{schmidt2011}
M.~Schmidt, N.~Le Roux, and F.~Bach.
\newblock Convergence rates of inexact proximal-gradient methods for convex
  optimization.
\newblock In {\em Advances in Neural Information Processing Systems (NIPS)},
  2011.

\bibitem{scholkopf-smola-book}
B.~Sch{\"o}lkopf and A.J. Smola.
\newblock {\em Learning with Kernels}.
\newblock MIT Press, 2001.

\bibitem{seeger2008bayesian}
M.W. Seeger.
\newblock {Bayesian inference and optimal design for the sparse linear model}.
\newblock {\em Journal of Machine Learning Research}, 9:759--813, 2008.

\bibitem{shalev2009stochastic}
S.~Shalev-Shwartz and A.~Tewari.
\newblock Stochastic methods for $\ell_1$-regularized loss minimization.
\newblock In {\em Proceedings of the International Conference on Machine
  Learning (ICML)}, 2009.

\bibitem{Shapiro2009}
A.~Shapiro, D.~Dentcheva, and A.P. Ruszczy{\'n}ski.
\newblock {\em {Lectures on stochastic programming: modeling and theory}}.
\newblock Society for Industrial Mathematics, 2009.

\bibitem{Shawe-Taylor2004}
J.~Shawe-Taylor and N.~Cristianini.
\newblock {\em Kernel Methods for Pattern Analysis}.
\newblock Cambridge University Press, 2004.

\bibitem{shevade2003simple}
S.K. Shevade and S.S. Keerthi.
\newblock A simple and efficient algorithm for gene selection using sparse
  logistic regression.
\newblock {\em Bioinformatics}, 19(17):2246, 2003.

\bibitem{Sonnenburg2006Large}
S.~Sonnenburg, G.~R\"atsch, C.~Sch\"afer, and B.~Sch\"olkopf.
\newblock Large scale multiple kernel learning.
\newblock {\em Journal of Machine Learning Research}, 7:1531--1565, 2006.

\bibitem{sprechmann2010collaborative}
P.~Sprechmann, I.~Ramirez, G.~Sapiro, and Y.C. Eldar.
\newblock {C-HiL}asso: A collaborative hierarchical sparse modeling framework.
\newblock {\em IEEE Transactions on Signal Processing}, 59(9):4183--4198, 2011.

\bibitem{Srebro2005Maximum}
N.~Srebro, J.D.M. Rennie, and T.S. Jaakkola.
\newblock Maximum-margin matrix factorization.
\newblock In {\em Advances in Neural Information Processing Systems (NIPS)},
  2005.

\bibitem{Suzuki2011SpicyMKL}
T.~Suzuki and R.~Tomioka.
\newblock Spicy{MKL}: a fast algorithm for multiple kernel learning with
  thousands of kernels.
\newblock {\em Machine Learning}, 85:77--108, 2011.

\bibitem{szafranski2007hierarchical}
M.~Szafranski, Y.~Grandvalet, and P.~Morizet-Mahoudeaux.
\newblock {Hierarchical penalization}.
\newblock In {\em Advances in Neural Information Processing Systems (NIPS)},
  2007.

\bibitem{tibshirani}
R.~Tibshirani.
\newblock Regression shrinkage and selection via the {L}asso.
\newblock {\em Journal of the Royal Statistical Society Series B},
  58(1):267--288, 1996.

\bibitem{tibshirani2005sparsity}
R.~Tibshirani, M.~Saunders, S.~Rosset, J.~Zhu, and K.~Knight.
\newblock Sparsity and smoothness via the fused {L}asso.
\newblock {\em Journal of the Royal Statistical Society Series B},
  67(1):91--108, 2005.

\bibitem{Tomioka2011Augmented}
R.~Tomioka, T.~Suzuki, and M.~Sugiyama.
\newblock Augmented {L}agrangian methods for learning, selecting and combining
  features.
\newblock In S.~Sra, S.~Nowozin, and S.J. Wright, editors, {\em Optimization
  for Machine Learning}. MIT Press, 2011.

\bibitem{tropp}
J.A. Tropp.
\newblock Greed is good: Algorithmic results for sparse approximation.
\newblock {\em IEEE Transactions on Signal Processing}, 50(10):2231--2242,
  2004.

\bibitem{tropp2}
J.A. Tropp, A.C. Gilbert, and M.J. Strauss.
\newblock Algorithms for simultaneous sparse approximation. part {I}: Greedy
  pursuit.
\newblock {\em Signal Processing, special issue "sparse approximations in
  signal and image processing"}, 86:572--588, 2006.

\bibitem{Tseng1991}
P.~Tseng.
\newblock Applications of a splitting algorithm to decomposition in convex
  programming and variational inequalities.
\newblock {\em SIAM Journal on Control and Optimization}, 29:119, 1991.

\bibitem{tseng2008accelerated}
P.~Tseng.
\newblock On accelerated proximal gradient methods for convex-concave
  optimization.
\newblock {\em submitted to SIAM Journal on Optimization}, 2008.

\bibitem{tseng2009coordinate}
P.~Tseng and S.~Yun.
\newblock {A coordinate gradient descent method for nonsmooth separable
  minimization}.
\newblock {\em Mathematical Programming}, 117(1):387--423, 2009.

\bibitem{turlach}
B.A. Turlach, W.N. Venables, and S.J. Wright.
\newblock Simultaneous variable selection.
\newblock {\em Technometrics}, 47(3):349--363, 2005.

\bibitem{Varoquaux2010a}
G.~Varoquaux, R.~Jenatton, A.~Gramfort, G.~Obozinski, B.~Thirion, and F.~Bach.
\newblock Sparse structured dictionary learning for brain resting-state
  activity modeling.
\newblock In {\em NIPS Workshop on Practical Applications of Sparse Modeling:
  Open Issues and New Directions}, 2010.

\bibitem{vert2010fast}
J.-P. Vert and K.~Bleakley.
\newblock Fast detection of multiple change-points shared by many signals using
  group {LARS}.
\newblock In {\em Advances in Neural Information Processing Systems (NIPS)},
  2010.

\bibitem{martin}
M.J. Wainwright.
\newblock Sharp thresholds for noisy and high-dimensional recovery of sparsity
  using $\ell_1$-constrained quadratic programming.
\newblock {\em IEEE Transactions on Information Theory}, 55(5):2183--2202,
  2009.

\bibitem{weisberg}
S.~Weisberg.
\newblock {\em Applied Linear Regression}.
\newblock Wiley, 1980.

\bibitem{wipf2008new}
D.P. Wipf and S.~Nagarajan.
\newblock A new view of automatic relevance determination.
\newblock {\em Advances in Neural Information Processing Systems (NIPS)}, 2008.

\bibitem{wipf2004sparse}
D.P. Wipf and B.D. Rao.
\newblock Sparse bayesian learning for basis selection.
\newblock 52(8):2153--2164, 2004.

\bibitem{wright2008robust}
J.~Wright, A.Y. Yang, A.~Ganesh, S.S. Sastry, and Y.~Ma.
\newblock Robust face recognition via sparse representation.
\newblock {\em IEEE Transactions on Pattern Analysis and Machine Intelligence},
  31(2):210--227, 2009.

\bibitem{wright2010accelerated}
S.J. Wright.
\newblock Accelerated block-coordinate relaxation for regularized optimization.
\newblock Technical report, Technical report, University of Wisconsin-Madison,
  2010.

\bibitem{wright2009sparse}
S.J. Wright, R.D. Nowak, and M.A.T. Figueiredo.
\newblock {Sparse reconstruction by separable approximation}.
\newblock {\em IEEE Transactions on Signal Processing}, 57(7):2479--2493, 2009.

\bibitem{wu2008coordinate}
T.T. Wu and K.~Lange.
\newblock Coordinate descent algorithms for lasso penalized regression.
\newblock {\em Annals of Applied Statistics}, 2(1):224--244, 2008.

\bibitem{xiao2010dual}
L.~Xiao.
\newblock {Dual averaging methods for regularized stochastic learning and
  online optimization}.
\newblock {\em Journal of Machine Learning Research}, 9:2543--2596, 2010.

\bibitem{xu2010robust}
H.~Xu, C.~Caramanis, and S.~Sanghavi.
\newblock Robust {PCA} via outlier pursuit.
\newblock {\em preprint arXiv:1010.4237}, 2010.

\bibitem{yuan-comparison}
G.X. Yuan, K.W. Chang, C.J. Hsieh, and C.J. Lin.
\newblock {A comparison of optimization methods for large-scale l1-regularized
  linear classification}.
\newblock Technical report, Department of Computer Science, National University
  of Taiwan, 2010.

\bibitem{yuan}
M.~Yuan and Y.~Lin.
\newblock Model selection and estimation in regression with grouped variables.
\newblock {\em Journal of the Royal Statistical Society Series B}, 68:49--67,
  2006.

\bibitem{zhao}
P.~Zhao, G.~Rocha, and B.~Yu.
\newblock The composite absolute penalties family for grouped and hierarchical
  variable selection.
\newblock {\em Annals of Statistics}, 37(6A):3468--3497, 2009.

\bibitem{Zhaoyu}
P.~Zhao and B.~Yu.
\newblock On model selection consistency of {L}asso.
\newblock {\em Journal of Machine Learning Research}, 7:2541--2563, 2006.

\bibitem{zou}
H.~Zou and T.~Hastie.
\newblock Regularization and variable selection via the elastic net.
\newblock {\em Journal of the Royal Statistical Society Series B},
  67(2):301--320, 2005.

\end{thebibliography}

\end{document}